\documentclass{article} 
\usepackage{iclr2026_conference,times}
\iclrfinalcopy

\usepackage[utf8]{inputenc} 
\usepackage[T1]{fontenc}    

\usepackage{hyperref}       
\usepackage{url}            
\usepackage{booktabs}       
\usepackage{amsfonts}       
\usepackage{nicefrac}       
\usepackage{microtype}      
\usepackage[dvipsnames]{xcolor}
\usepackage{mathtools}
\usepackage{amsmath,amssymb,amsthm}
\usepackage{cleveref}    
\usepackage{algorithm}
\usepackage{algorithmicx}
\usepackage{algpseudocode}
\usepackage{subcaption, comment}
\usepackage[normalem]{ulem}
\usepackage[most]{tcolorbox}
\usepackage{graphicx}
\usepackage{wrapfig}
\usepackage{adjustbox} 
\usepackage{tikz}
\usetikzlibrary{positioning}
\title{In-Context Learning for Pure Exploration}
\author{Alessio Russo\thanks{Equal contribution (alphabetical order).} \\
  Boston University \\
  \texttt{arusso2@bu.edu}
  \And
  Ryan Welch\footnotemark[1] \\
 Stanford University\\
  \texttt{rcwelch@stanford.edu}
  \And
  Aldo Pacchiano \\
  Boston University \\
  Broad Institute of MIT and Harvard \\
  \texttt{pacchian@bu.edu}
}

\newtheorem{theorem}{Theorem}[section]
\newtheorem{lemma}[theorem]{Lemma}
\newtheorem{proposition}[theorem]{Proposition}
\newtheorem{assumption}{Assumption}
\newtheorem{corollary}[theorem]{Corollary}

\theoremstyle{definition}
\newtheorem{example}{Example}[section]

\crefname{assumption}{assum.}{assums.}
\Crefname{theorem}{Assum.}{Assums.}
\crefname{theorem}{thm.}{thms.}
\Crefname{theorem}{Thm.}{Thms.}

\crefname{proposition}{prop.}{props.}
\Crefname{proposition}{Prop.}{Props.}

\crefname{lemma}{lem.}{lems.}
\Crefname{lemma}{Lem.}{Lems.}

\crefname{corollary}{cor.}{cors.}
\Crefname{corollary}{Cor.}{Cors.}

\crefname{definition}{def.}{defs.}
\Crefname{definition}{Def.}{Defs.}

\crefname{equation}{eq.}{eqs.}
\Crefname{equation}{Eq.}{Eqs.}

\crefname{figure}{fig.}{figs.}
\Crefname{figure}{Fig.}{Figs.}

\crefname{table}{tab.}{tabs.}
\Crefname{table}{Tab.}{Tabs.}

\crefname{appendix}{app.}{apps.}
\Crefname{appendix}{App.}{Apps.}

\DeclareMathOperator*{\argmax}{arg\,max}
\DeclareMathOperator*{\argmin}{arg\,min}
\DeclareMathOperator*{\argsup}{arg\,sup}
\DeclareMathOperator*{\arginf}{arg\,inf}

\renewcommand{\P}{\mathcal{P}}

\hypersetup{
    colorlinks = true,
    linkcolor = RedViolet,
    citecolor=NavyBlue,
    }

\newcommand{\refalgbyname}[2]{\hyperref[#1]{\texttt{\textbf{#2}}}}
\newcommand{\icpe}{\refalgbyname{algo:icpe_fixed_confidence}{ICPE}}

\begin{document}

\maketitle

\begin{abstract}

We study the  \emph{active sequential hypothesis testing} problem, also known as \emph{pure exploration}: given a new task, the learner \emph{adaptively collects data} from the environment to  efficiently determine an underlying correct hypothesis. A classical instance of this problem is the task of identifying the best arm in a multi-armed bandit problem (a.k.a. BAI, Best-Arm Identification), where actions index hypotheses. Another important case is generalized  search, a problem of determining the correct label through a sequence of strategically
selected queries that indirectly reveal information about the label.
In this work, we introduce \emph{In-Context Pure Explorer} (\icpe), which meta-trains Transformers to map \emph{observation histories} to \emph{query actions} and a \emph{predicted hypothesis}, yielding a model that transfers in-context. At inference time,  \icpe{} actively gathers evidence on new tasks and infers the true hypothesis without parameter updates.
Across deterministic, stochastic, and structured benchmarks, including BAI  and generalized search,  \icpe{} is competitive with  adaptive baselines while requiring no explicit modeling of information structure. Our results support Transformers as practical architectures for \emph{general sequential testing}. Code repository \url{https://github.com/rssalessio/icpe}.

\end{abstract}

\section{Introduction}\label{sec:introduction}

Sequential architectures have shown striking in-context learning (ICL) abilities: given a short sequence of examples, they can infer task structure and act without parameter updates \citep{lee2023supervised,schaul2010metalearning,bengio1990learning}. While this behavior is well documented for supervised input–output tasks, as well as regret minimization problems,  many real problems demand sequential experiment design: how do we allocate experiments to reliably infer an hypothesis? For instance, imagine a librarian trying to figure out which book you want by asking a series of questions. Similarly, in generalized search \citep{nowak2008generalized}, the learner adaptively chooses which tests to run, each partitioning the hypothesis class, to identify the true hypothesis as quickly as possible.
This raises a natural question: can we leverage ICL for adaptive \emph{data collection and hypothesis identification} across a family of  problems?

We study this question through the lens of Active Sequential Hypothesis Testing (ASHT) \citep{chernoff1992sequential,cohn1996active}, a.k.a. \emph{pure exloration} \citep{degenne2019pure}, where an agent adaptively performs measurements in an environment to identify a ground-truth hypothesis. In particular, we study a Bayesian formulation of ASHT, where each environment  is drawn from a family of possible problems ${\cal M}$.

Classically, ASHT has been studied  either (i) with a fixed confidence   $\delta$ (i.e., stop as soon as the predicted hypothesis is correct with error probability at most $\delta$) \citep{NEURIPS2024_1fb0a4de}
or (ii) a fixed sampling budget (use $N$ samples to predict the correct hypothesis) \citep{atsidakou2022bayesian}.  For example, in the fixed-confidence setting one can use ASHT  to minimize the number of  DNA-based tests performed to accurately detect cancer \citep{gan2021greedy}.  Another canonical instantiation is Best-Arm Identification (BAI) in stochastic multi-armed bandits \citep{audibert2010best}. In this problem the agent sequentially selects an action (the query) and observes a noisy reward:  the task is to identify the action with the highest mean reward\footnote{Note that, \emph{in this particular case}, the hypothesis space coincides with the query space of the agent.}. Other applications include medical diagnostics \citep{berry2010bayesian}, sensor management \citep{hero2011Sensor} and  recommender systems \citep{resnick1997recommender}.

Despite substantial progress \citep{ghosh1991brief,naghshvar2013active,naghshvar2012noisy,mukherjee2022chernoff}, solving ASHT problems remains difficult.
 Even in simple tabular environments, computing optimal sampling policies often requires strong modeling assumptions (known observation models that do not depend on the history, and/or known inference rules) and solving challenging (often nonconvex) programs~\citep{al2021navigating}. This leaves open whether one can \emph{learn}, in a simple way, to both gather informative data and infer the correct hypothesis without such assumptions.

To answer this question, we introduce In-Context Pure Explorer (\icpe{}), a Transformer-based architecture meta-trained on a family of tasks to jointly learn a data-collection policy and an inference rule, in both \emph{fixed-confidence}  and \emph{fixed-budget} regimes.
\icpe{} is a model that transfers in-context: at inference time, \icpe{} gathers evidence on new tasks and infers the true hypothesis without parameter updates \citep{schaul2010metalearning,bengio1990learning}.

The practical implementation of \icpe{} emerges naturally from the theory alone, showing how  a principled information-theoretic reward function can be used to train, using Reinforcement Learning (RL), an optimal data-collection policy.
Additionally,
\icpe{}  relaxes classical assumptions: the data-generation mechanism $P$ is unknown and may be history-dependent, and the mapping from data to hypotheses is also unknown (we do not assume a known likelihood or a hand-designed test). These facts, combined with the simple and practical implementation of \icpe{}, offer a new way to design efficient ASHT methods in more general environments.

\begin{figure}[t]
    \centering
    \vspace{-25pt}
    \begin{subfigure}[t]{0.6\textwidth}
      
    \includegraphics[width=\linewidth]{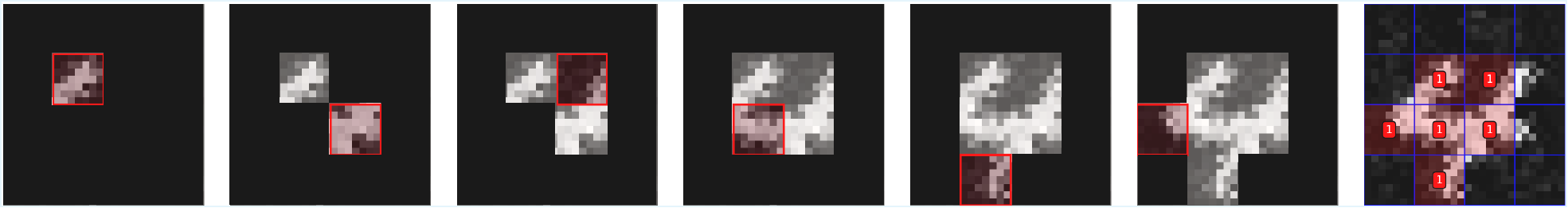}
    \caption{}
    \label{fig:mnist_ex}
    \end{subfigure}
    \hfill
    \begin{subfigure}[t]{0.38\textwidth}
      
   \includegraphics[width=\linewidth]{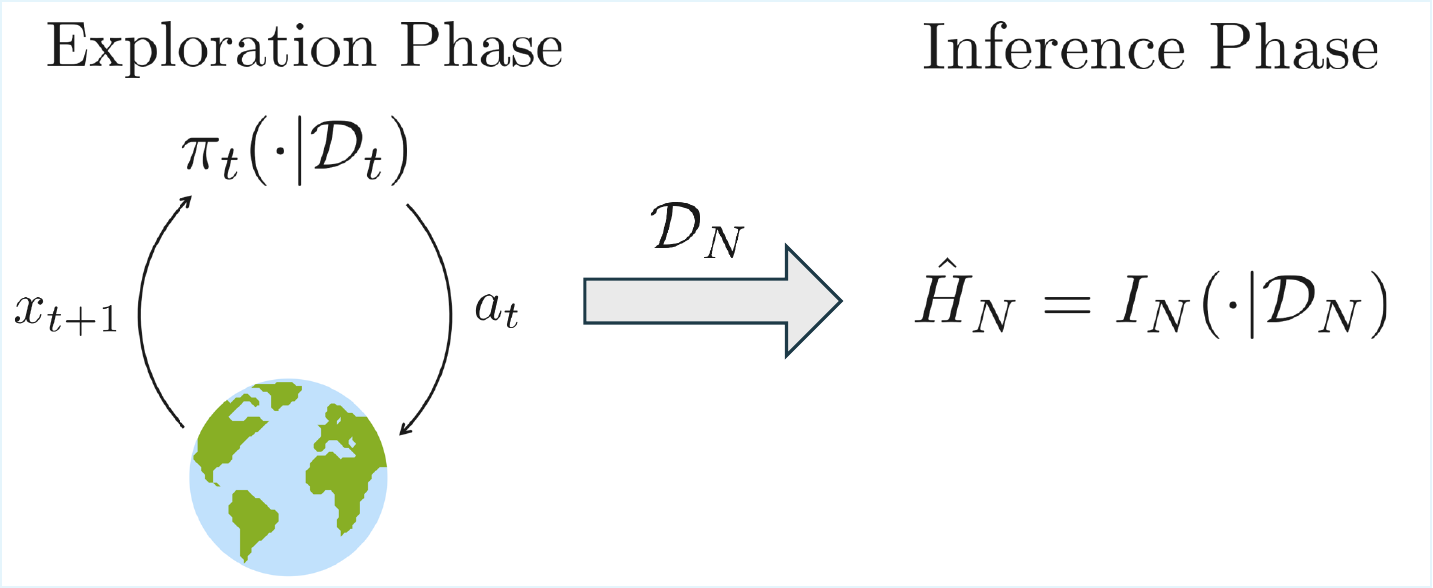}
    \caption{}
    \label{fig:mnist_ex}
    \end{subfigure}
    \caption{{\bf (a)} Generalized search example: \icpe{} starts from a masked image  (left), and sequentially reveals patches expected to  reduce the posterior entropy over labels. It stops once the inferred label is $\delta$-correct (right). {\bf (b)} After executing an action $a_t$, the agent  observes $x_{t+1}$. At inference time, the data collected is used to infer an hypothesis.}
    \vspace{-5pt}
\end{figure}

On BAI and generalized search tasks (deterministic, stochastic, structured), \icpe{} efficiently explores and achieves performance comparable to instance-dependent algorithms, while requiring only a forward pass at test time, and without requiring solving any complex optimization problem.

\section{Problem Setting}\label{sec:method}

 The problem we consider is as follows: on an environment instance $M\sim {\cal P}$, sampled from a prior ${\cal P}$ over an environment class ${\cal M}$, the learner chooses actions (queries\footnote{The reason why denote ``queries" as ``actions" stems from the fact that the problem can be modeled similarly to a Markov Decision Process (MDP)\citep{puterman2014markov}, and  queries correspond to actions in an MDP.}) $a_t$ in rounds $t=1,2,\dots$ and observes outcomes $x_{t+1}$. The aim is to gather a trajectory $\mathcal{D}_t=(x_1,a_1,\ldots,a_{t-1},x_t)$ that is informative enough to identify an environment-specific ground-truth hypothesis $H^\star$ with high probability.

Informally, we seek to answer the following question:
\begin{tcolorbox}
\begin{center}
\it Given an environment $M$ drawn from a prior ${\cal P}$, how can we learn (i) a sampling policy $\pi$ that collects data $\mathcal{D}$ from $M$ and (ii) an inference rule $I$ such that $I(\mathcal{D})$ reliably predicts $H^\star$?
\end{center}
\end{tcolorbox}

\noindent{\bf Environments, sampling policy and hypotheses.}
We consider environments  $(M=(\mathcal X,\mathcal A,\rho,P,H^\star)$ with observation space $\mathcal X$, action set $\mathcal A$, initial observation law $\rho\in\Delta(\mathcal X)$, and a (possibly history-dependent) generative mechanism $P=(P_t)_{t\ge1}$ such that $x_{t+1}\sim P_t(\cdot| \mathcal D_t,a_t)$.
All $M\in{\cal M}$ share the same $\mathcal X$ and $\mathcal A$.
The learner uses a (possibly randomized) policy $\pi=(\pi_t)_{t\ge1}$ with $a_t\sim \pi_t(\cdot| \mathcal D_t)$, and  a sequence of inference rules $I=(I_t)_{t\ge1}$ with $I_t:\mathcal D_t\to{\cal H}$ for a finite hypothesis set ${\cal H}$.
We assume throughout that $H^\star$ is induced by the environment via a measurable functional  $h^\star$, i.e., $H^\star \coloneqq h^\star(\rho,P)$, 
and is almost surely unique under ${\cal P}$.

\begin{example}[Best Arm Identification]\it In BAI an agent seeks to identify the best arm among $K$ arms. Upon selecting an action $a$ at time $t$, it observes a random reward $x_t$ distributed according to  a distribution $P(\cdot|a_t)$. The goal is to identify $a^\star =\argmax_a \mathbb{E}_{x\sim P(\cdot|a)}[x]$ (so $H^\star=a^\star$). Many algorithms exist for specific assumptions \citep{garivier2016optimal,jedra2020optimal}, but designs change drastically with the model, and extensions to richer settings can be difficult and often non-convex \citep{al2021adaptive}. 
\end{example}

\noindent{\bf Fixed confidence and fixed budget regimes.} Two regimes are usually considered in pure exploration:
\begin{itemize}
    \item {\bf Fixed confidence:} Given a target error level $\delta\in(0,1)$, the learner chooses: (i) a stopping time $\tau\in \mathbb{N} $ that denotes the total number of queries and \emph{marks}  the \emph{random} moment data collection stops, (ii) a data-collection policy $\pi$, (iii) and an inference rule $I$ that minimize the expected total number of queries $\tau$ while meeting a correctness guarantee:
\begin{equation}\label{eq:fixed_confidence_problem}
\inf_{\tau,\pi,I}\ \mathbb{E}^\pi\left[\tau\right]
\quad\text{s.t.}\quad
\mathbb{P}^\pi\left( I_\tau({\cal D}_\tau)=H^\star\right) \geq\ 1-\delta .
\end{equation}
where $\mathbb{P}^\pi(\cdot)$ denotes the probability of the underlying data collection process when $\pi$ gathers data from $M$, and $M$ is sampled from a prior $\P$.

\item {\bf Fixed budget:}
For a given horizon $N\in\mathbb{N}$, the learner chooses $\pi$ and $I$ to maximize the chance of predicting the correct hypothesis after exactly $N$ queries:
\begin{equation}\label{eq:fixed_budget_problem}
\sup_{\pi,I} \mathbb{P}^\pi\left( I_N({\cal D}_N)=H^\star\right).
\end{equation}
\end{itemize}
These two objectives capture the main operational modes of pure exploration: “stop when certain” and “maximize accuracy over a fixed horizon". 
Further note that the problem we propose to solve extends classical ASHT by allowing environment-specific, history-dependent observation kernels:  $x_{t+1}\sim P_t(\cdot|{\cal D}_t, a_t)$. Standard formulations assume memoryless dependence only on  $(H^\star, a_t)$\citep{naghshvar2013active,garivier2016optimal}. Moreover, whereas ASHT/BAI typically use known estimators (e.g., maximum likelihood), \emph{we learn the inference rule from data}. Consequently, both the sampling policy  $\pi$
 and the inference rule  $I$
 can depend on entire histories.

\section{ICPE: In-Context Pure Explorer}

In this section we describe \icpe{}, a meta-RL approach for solving \cref{eq:fixed_budget_problem,eq:fixed_confidence_problem}.  The implementation of \icpe{} is motivated from the theory. We first show  that learning an optimal inference rule $I$ amounts to computing a posterior distribution. Secondly, the policy $\pi$  can be learned using  RL with an appropriate reward function.

Importantly, the reward function used for training $\pi$ \emph{emerges} naturally from the problem formulation, and it \emph{is not} a user-chosen criterion, making it a principled information-theoretical  reward function.
We now describe  the theory,  and then describe the practical implementation of \icpe{}.

\vspace{-5pt}
\subsection{Theoretical Results}\label{subsec:theoretical_results}
\vspace{-5pt}
Our theoretical results highlight that the main quantity of interest, in both regimes in \cref{eq:fixed_budget_problem,eq:fixed_confidence_problem}, is the posterior distribution over the true hypothesis $\mathbb{P}(H^\star=H|{\cal D}_t)$. First, the \emph{optimal inference rule $I^\star$ is based on this posterior}. Secondly, \emph{this posterior naturally defines a reward function} that can characterize the optimality of a data-collection policy.

Throughout this section, we assume that $\mathcal X\subset\mathbb R$ is compact and $\mathcal A,{\cal H}$ are finite. We instantiate ${\cal M}$ via a parametrized family $\{(P_\omega,\rho_\omega):\omega\in\Omega\}$ with $\Omega$ compact and $\omega\mapsto(P_\omega,\rho_\omega)$ continuous, so a prior on $\Omega$ induces a prior on ${\cal M}$. For the sake of brevity, we provide informal statements here, and refer the reader to \cref{app:theoretical_results:icpe} for all the details.

We  have the following result about the optimality of the inference, proved in \cref{app:theoretical_results:icpe:posterior_distribution_and_inference_rule}.
\begin{proposition}[Inference Rule Optimality]\label{prop:main:optimal_inference_rule}
Let $t\geq 1$ and a  policy $\pi$.  The  optimal inference rule to $\sup_{I_t} \mathbb P^\pi(H^\star=I_t({\cal D}_t))$ is given by $I_t^\star(z)=\argmax_{H\in {\cal H}} {\mathbb P}(H^\star=H|{\cal D}_t=z)$.
\end{proposition}
Concretely, \cref{prop:main:optimal_inference_rule} identifies the optimal inference rule as the \textit{maximum a posteriori} estimator based on ${\mathbb P}(H^\star=H|{\cal D}_t)$, so that learning $I_\phi$ amounts to learning this posterior.
Based on this we can now differentiate between the two settings.

\noindent{\bf Fixed budget.} We begin with the simpler fixed budget case.  The key idea is to show that the optimal policy $\pi^\star$ maximizes an action-value function $Q$ \citep{sutton2018reinforcement}. First, define the following reward function: for $t< N$ let $r_t({\cal D}_t)\coloneqq 0$, and for $t=N$ set $r_N({\cal D}_N)=\max_{H} \mathbb{P}(H^\star =H|{\cal D}_N)$. In words, we assign a reward equal to the maximum value of the posterior distribution at the last time step and 0 otherwise.

Then, define  $V_N({\cal D}_N)=r_N({\cal D}_N)$ to be the optimal value at the last timestep $t=N$. From this definition, we can recursively define the $Q$-function as follows:
\[
Q_t({\cal D}_t,a) = \mathbb{E}_{x_{t+1}|({\cal D}_t,a)}[V_{t+1}(\underbrace{({\cal D}_t,a, x_{t+1})}_{={\cal D}_{t+1}})] \;\hbox{ and } \;V_t({\cal D}_t)=\max_{a\in {\cal A}}Q_t({\cal D}_t,a)\quad \forall t\leq N-1.
\]
where ``$x_{t+1}|({\cal D}_t,a)$'' denotes the posterior distribution of $x_{t+1}$ given $({\cal D}_t,a)$. 
Optimizing with respect to this reward function yields an optimal solution to (\ref{eq:fixed_budget_problem}), which we formalize in the following result proved in \cref{app:theoretical_results:icpe:fixed_budget:optimal_policy}.
\begin{theorem}[Policy Optimality for Fixed Budget]
\label{thm:main:3.2}
For all $t\geq 1$, define the policy $\pi_t^\star({\cal D}_t)=\argmax_{a\in {\cal A}} Q_t({\cal D},a)$. Then, $(\pi^\star, I_N^\star)$ (where $I_N^
\star$ is as in  \cref{prop:main:optimal_inference_rule}) are an optimal solution of \cref{eq:fixed_budget_problem}, and we have that 
\begin{equation}
\sup_{\pi,I} \mathbb{P}^\pi\left( I_N({\cal D}_N)=H^\star\right) = \mathbb{E}^{\pi^\star}[r_N({\cal D}_N)].
\end{equation}
\end{theorem} 
Simply speaking, \cref{thm:main:3.2} indicates that an optimal exploration policy in the fixed-budget setting is obtained by a greedy policy with respect to a $Q$-function whose terminal reward is the maximum posterior mass $r_N({\mathcal D}_N) = \max_H {\mathbb P}(H^\star = H \mid {\mathcal D}_N)$ (and zero reward for all other timesteps). A similar principle also holds for the fixed confidence setting.

\noindent{\bf Fixed confidence.} In the fixed confidence setting, we first simplify the problem by noting  that the stopping time $\tau$ can be simply embedded as a stopping action $a_{\rm stop}$ in the policy $\pi$ (see \cref{app:theoretical_results:icpe:fixed_confidence:dual}  for a formal justification). Hence, we extend the action set as ${\cal A}\gets {\cal A}\cup\{a_{\rm stop}\}$ and $\tau=\inf\{t\in \mathbb{N}: a_t=a_{\rm stop}\}$.
Then, as in classical ASHT literature \citep{naghshvar2013active}, we study the dual problem of \cref{eq:fixed_confidence_problem}, that is:
\begin{equation}
\inf_{\lambda\geq 0} \sup_{\pi,I} V_\lambda(\pi,I),\;\hbox{ where }\; V_\lambda(\pi,I)\coloneqq -\mathbb{E}^\pi[\tau] + \lambda\left[\mathbb{P}^\pi\left(I_\tau({\cal D}_\tau)=H^\star\right)-1+\delta\right].\label{eq:dual_problem_fixed_confidence}
\end{equation}
To show optimality of a policy, and satisfaction of the correctness constraint, there are 2 key observations to make: (1) one can show that the optimal inference rule $I^\star$ remains as in \cref{prop:optimal_inference_rule}; (2)  solving \cref{eq:dual_problem_fixed_confidence} amounts to solving an RL problem in $\pi$.

Indeed, similarly to the  the fixed budget setting, for $t\geq 1$ define the reward model as
\begin{equation}\label{eq:reward_model}
r_{t,\lambda}({\cal D}_t,a)=-\mathbf{1}_{\{a \neq a_{\rm stop}\}}+\lambda \mathbf{1}_{\{a=a_{\rm stop}\}}\max_H \mathbb{P}(H^\star=H|{\cal D}_t),\end{equation}
which simply penalizes the policy for each extra timestep,  accompanied by a reward proportional to the maximum posterior value at the stopping time. Accordingly, we define the $Q$-function as
\begin{equation}\label{eq:Q_function_theory}
Q_{t,\lambda}({\cal D}_t,a)= r_{t,\lambda}({\cal D}_t,a)+ {\bf 1}_{\{a \neq a_{\rm stop}\}}\mathbb{E}_{x_{t+1}|({\cal D}_t,a)}\left[\max_{a'} Q_{t+1,\lambda}\left(({\cal D}_t, a, x_{t+1}),a'\right)\right].
\end{equation}
Then, we have the following result (see \cref{app:theoretical_results:icpe:fixed_confidence:optimal_policy} and \cref{app:theoretical_results:icpe:fixed_confidence:identifiability_and_correctness} for a proof) indicating that optimizing with respect to this reward function yields an optimal solution to (\ref{eq:fixed_confidence_problem}).
\begin{theorem}[Policy Optimality for Fixed Confidence]
\label{thm:main:3.3}
Let $\pi^\star_{t,\lambda}({\cal D}_t)=\argmax_{a\in {\cal A}} Q_{t,\lambda}({\cal D}_t, a)$ and $\pi_\lambda^\star=(\pi_{t,\lambda})_t$. Then, for  $\lambda\geq 0$ the pair $(I^\star,\pi_\lambda^\star)$ , with $I^\star=(I_t^\star)_{t}$ defined as in \cref{prop:main:optimal_inference_rule}, is an optimal solution of $\sup_{\pi,I} V_\lambda(\pi,I)$. Furthermore, under suitable identifiability conditions (see \cref{assump:identifiability}), any maximizer $\lambda^\star $ of \cref{eq:dual_problem_fixed_confidence}  guarantees that $\pi_{\lambda^\star}^\star$ satisfies the $\delta$-correctness criterion.
\end{theorem}
Intuitively, for the fixed-confidence setting, we first recast the constrained problem in \cref{eq:fixed_confidence_problem} via a Lagrangian dual, then prove that \emph{any} admissible stopping rule $\tau$ can be represented as the selection time of an absorbing stopping action $a_{\rm stop}$. In \cref{thm:main:3.3}, we show that the resulting dual problem is solved by a greedy policy on the $Q$-function defined via the reward in \cref{eq:reward_model}, and that such policy achieves the desired level of correctness, $1-\delta$. This result establishes that both an optimal $\delta$-aware stopping rule and  exploration strategy can be learned on the extended action space ${\mathcal A} \cup \{ a_{\rm stop} \}$.
In the next section, we describe the practical implementation of \icpe{} based on these results using the Transformer architecture.
\vspace{-5pt}
\subsection{Practical Implementation: the ICPE algorithm}
\vspace{-5pt}
\begin{algorithm}[t]
 \footnotesize 
   \caption{\icpe{} (In-Context Pure Explorer)}
   \label{algo:icpe_fixed_confidence}
\begin{algorithmic}[1]
    
   \State {\bfseries Input:} Tasks distribution $\P$; confidence $\delta$; horizon $N$; initial  $\lambda$ and  hyper-parameter $T_\phi,T_\theta$.
   \Statex \texttt{\color{blue}//  Training phase}
     \State Initialize buffer ${\cal B}$, networks $Q_\theta,I_\phi$ and set $\bar\theta\gets\theta, \bar\phi\gets \phi$.
   \While{Training is not over}

   \State Sample environment $M\sim {\cal P}$ with hypothesis $H^\star$, observe $x_1\sim\rho$ and  set $t\gets 1$. 
   \Repeat
   \State Execute action $a_t=\argmax_a Q_\theta({\cal D}_t,a)$ in $M$ and observe $x_{t+1}$.
   
   \State Add partial trajectory $({\cal D}_t,a_t,x_{t+1},H^\star)$ to ${\cal B}$ and set $t\gets t+1$.
   \Until{$a_{t-1}=a_{\rm stop}$ (fixed confidence) or $t>N$ (fixed budget).}
   \State In the fixed confidence,   update  $\lambda$ according to \cref{eq:cost_update}.
    \State Sample batch $B\sim {\cal B}$ and update $\theta,\phi$ using ${\cal L}_{\rm inf}(B;\phi)$ (\cref{eq:loss_infernece}) and ${\cal L}_{\rm policy}(B;\theta)$ (\cref{eq:loss_fixed_budget} or \cref{eq:loss_fixed_confidence}).
\State Every $T_\phi$ steps set $\bar\phi\gets\phi $ (similarly, every $T_\theta$ steps set $\bar\theta\gets\theta$). 
   \EndWhile
    \Statex \hrulefill 
   \Statex \texttt{\color{blue}//  Inference phase}
   \State Sample unknown environment $M\sim \P$.
    \State Collect a trajectory ${\cal D}_N$ (or ${\cal D}_\tau$ in fixed confidence) according to a policy $\pi_t({\cal D}_t)=\arg\max_a Q_\theta({\cal D}_t,a)$, until $t=N$ (or $a_t=a_{\rm stop}$).
   \State {\bf Return} $\hat H_N=\argmax_H I_\phi(H|{\cal D}_N)$ (or $\hat H_\tau=\argmax_H I_\phi(H|{\cal D}_\tau)$ in the fixed confidence)
\end{algorithmic}
\end{algorithm}

We instantiate \icpe{} with two  learners: an \emph{inference network} $I_\phi(H| {\cal D}_t)$, parametrized by $\phi$, that  approximates the posterior $\mathbb{P}(H^\star=H|\mathcal D_t)$ (cf. \cref{prop:main:optimal_inference_rule}) and a \emph{Q-network} $Q_\theta({\cal D}_t,a)$, parametrized by $\theta$, whose greedy policy defines $\pi_\theta$ (and includes $a_{\rm stop}$ in the fixed confidence setting only). Both networks are implemented using Transformer architectures, and, for practical reasons, we impose a maximum trajectory length of $N$. This architecture handles both \emph{fixed budget} (\cref{eq:fixed_budget_problem}) and \emph{fixed confidence} (\cref{eq:fixed_confidence_problem}) settings. However, we find it important to explicitly note that while \cref{algo:icpe_fixed_confidence} abstracts the main ideas of \icpe{} in a unified way, in practice we train separate models for the fixed-budget and fixed-confidence regimes, each with their own reward and $Q$-function as derived in \cref{subsec:theoretical_results}.

\noindent{\bf Training phase.}
At training time \icpe{} interacts with an online environment: each episode draws an instance $M\sim\mathcal P$ and generates a trajectory.
We maintain a buffer $\mathcal B$ with tuples $({\cal D}_t,a_t,x_{t+1},H^\star_M)$,
where $H^\star$ is the true hypothesis for the sampled environment $M$ (from a single tuple we also obtain ${\cal D}_{t+1}=({\cal D}_t,a_t,x_{t+1})$). This buffer is used to sample mini-batches $B\subset\mathcal B$ to train $(\theta,\phi)$. Lastly, we treat each optimization in $(\phi,\theta)$ (and $\lambda$ too for the fixed confidence) separately, treating the other variables as fixed.

\noindent{\bf Training of $I_\phi$}.  We train $I_\phi$ to learn the posterior by SGD on the negative log-likelihood on a a batch $B\subset{\cal B}$ of partial trajectories sampled from the buffer:
\begin{equation}\label{eq:loss_infernece}
\mathcal L_{\rm inf}(\phi)= -\frac{1}{|B|}\sum_{({\cal D}_{t},a_t,x_{t+1},H^\star)\in  B} \log I_\phi(H^\star| {\cal D}_{t+1}).
\end{equation}
In expectation this is (up to an additive constant) equivalent to minimizing the KL-divergence between $\mathbb P(H^\star=H|{\cal D})$ and $I_\phi(H|{\cal D})$ (a similar loss is also used in \citep{lee2023supervised}). Lastly, we also set $\hat H_t=\argmax_H I_\phi(H|{\cal D}_t)$ to be predicted hypothesis with data ${\cal D}_t$.

\noindent{\bf Training in the Fixed Budget}. In the fixed budget we train $\theta$ using DQN~\citep{mnih2015human} and the rewards defined in the previous section.  We denote the target network $Q_{\bar \theta}$, which is parameterized by $\bar{\theta}$.
Since rewards are defined in terms of $I_\phi$, to improve training stability we introduce a separate target  inference network $I_{\bar\phi}$, parameterized by $\bar\phi$, which provides feedback for training $\theta$. These target networks are  periodically updated, setting $\bar\phi\gets\phi$ every $T_\phi$ steps (similarly,  $\bar\theta \gets \theta$ every $T_\theta$ steps).

 Hence, in the fixed budget, for a batch $B\sim {\cal B}$, we update $\theta$ by performing SGD on the following loss
\begin{equation}\label{eq:loss_fixed_budget}
    {\cal L}_{\rm policy}(B;\theta)= \frac{1}{|B|} \sum_{({\cal D}_t,a_t,x_{t+1})\in B}\left(\max_H I_{\bar \phi}(H|{\cal D}_t)\cdot \mathbf{1}_{\{t=N\}}  + \max_{a} Q_{\bar\theta}({\cal D}_{t+1},a)- Q_{\theta}({\cal D}_t,a_t)\right)^2.
\end{equation}

\noindent{\bf Training in the Fixed Confidence}. 
In this setting  we train $\theta$ similarly to the fixed budget setting. However, we also have a dedicated stop-action $a_{\rm stop}$ whose value depends solely on history. Thus, its $Q$-value can be updated at any time, allowing retrospective evaluation of stopping.
In other words, $Q_\theta({\cal D}_t,a_{\rm stop})$ can be updated for \emph{any} sampled transition $z\in {\cal B}$, even if the logged action $a_t\neq a_{\rm stop}$ (i.e., a “pretend to stop” update). This allows the model to retro-actively evaluate the quality of stopping earlier in a trajectory.

Then, based on \cref{eq:Q_function_theory}, we update $\theta$ by performing SGD on the following  $Q$-loss
\begin{align}
    {\cal L}_{\rm policy}(B;\theta) = \frac{1}{|B|} \sum_{({\cal D}_t,a_t,x_{t+1})\in B}&\Bigg[\mathbf{1}_{\{a_t\neq a_{\rm stop}\}}\cdot \left(-1+ \max_a Q_{\bar \theta}({\cal D}_{t+1},a)-Q_\theta({\cal D}_t,a_t)\right)^2\label{eq:loss_fixed_confidence}\\
    &\qquad+\left(\lambda \max_H I_{\bar \phi}(H|{\cal D}_t)-Q_\theta({\cal D}_t,a_{\rm stop})\right)^2\Bigg],
\end{align}

and note that the loss depends on $\lambda$. We learn $\lambda$ using a gradient descent update, which depends on the correctness of the predicted hypothesis. We sample $K$ trajectories $\{({\cal D}_\tau^{(i)},H_i^\star)\}_{i=1}^K$ with fixed $(\theta,\phi)$ and update $\lambda$ with a small learning rate $\beta$: 
\begin{equation}
      \lambda \gets \max\left[0,\lambda- \beta\left(\hat p -1+\delta\right)\right],\;\hbox{ where } \; \hat p=\frac{1}{K}\sum_{i=1}^K \mathbf{1}_{\{\argmax_H I_\phi(H|{\cal D}_\tau^{(i)})\} = H_i^\star\}}.\label{eq:cost_update}
   \end{equation}
The quantity $\hat p$ can be used to assess when to stop training by checking its empirical convergence.
In the fixed confidence, in practice we can stop whenever $\hat p\geq 1-\delta$ is stable and $\lambda$ is almost a constant.  However, to obtain rigorous guarantees care must be taken. In \cref{app:theoretical_results:icpe:fixed_confidence:identifiability_and_correctness} we discuss how to provide formal guarantees on the $\delta$-correctness of the resulting method, bsaed on a sequential testing procedure.

\noindent{\bf Inference phase.}  At inference time \icpe{} operates by simple forwards passes. An unknown task $M\sim \P$ is sampled, and actions are selected according to $a = \argmax_a Q_\theta({\cal D}_t,a)$. At the last timestep a hypothesis is predicted using $\hat H_N=\argmax_H I_\phi(H|{\cal D}_N)$ (or  $\hat H_\tau=\argmax_H I_\phi(H|{\cal D}_\tau)$ at the stopping time for the fixed confidence setting).

\noindent{\bf Theoretical guarantees and training correctness.} In \cref{prop:freeze_at_hit_monotone}, we describe how to decide when to stop training in order to guarantee that the resulting 
$(\pi_\theta,I_\phi)$ are 
$
\delta$-correct. Furthermore, we derive finite-sample guarantees for the fixed-budget \icpe{} meta-learning phase in a stylized setting in \cref{app:meta_training:finite_sample_analysis}. In \cref{thm:value_bound} we derive a bound on the sub-optimality of the policy $\pi^{(k)}$ at training epoch $k$ in terms of stage-wise Bellman residuals and concentrability coefficients. In \cref{thm:fixed_budget_finite_sample_analysis}, we additionally show how these residuals are controlled by an approximation term (capturing how well the function class can represent the Bellman update) and an estimation term that decays with the number and size of training batches. Together, these results yield an explicit finite-sample performance bound for \icpe{} in an ideal scenario.

\section{Empirical Evaluation}\label{sec:simulations}
We evaluate \icpe{} on a range of tasks: BAI on bandit problems, hypothesis testing in MDPs, and general search problems (pixel sampling and binary search). For bandits, we  consider different reward structures: deterministic, stochastic, with feedback graphs  and with hidden information.
Due to space limitations,  refer  to \cref{app:experiments} for the results on bandit problem with feedback graphs  and  MDPs. Also refer  to \cref{app:algorithms} for details on the algorithms.
In all experiments  \emph{we use a target accuracy value} of $\delta=0.1$, and shaded areas indicate $95\%$ confidence intervals computed via hierarchical bootstrapping (see \cref{app:experiments} for details).

\vspace{-5pt}
\subsection{Bandit Problems}
\label{subsec:bandits}
\vspace{-5pt}
We now apply \icpe{} to the classical BAI  problem within MAB tasks.  For the MAB setting we have a finite number of actions ${\cal A}=\{1,\dots,K\}$, corresponding to the actions in the MAB problem $M$. For each action $a$, we define a corresponding reward distribution $P(\cdot|a)$ from which rewards are sampled i.i.d. 
Then,  ${\cal P}$ is a prior distribution on the actions' rewards distributions. For the BAI problem, we let the true hypothesis be $H^\star =\argmax_a \mathbb{E}_{x\sim P(\cdot |a)}[x]$, so that the goal is to identify the best action (and thus ${\cal H}={\cal A}$).

\noindent\textbf{Stochastic Bandit Problems.}
We evaluate \icpe{} on  stochastic bandit environments for both the fixed confidence  and fixed budget setting (with $N=30$). Each action's reward distribution is normally distributed $\nu_a={\cal N}(\mu_a,0.5^2)$, with $(\mu_a)_{a\in {\cal A}}$ drawn from ${\cal P}$. In this case ${\cal P}$ is a uniform distribution over problems with minimum gap $\max_a \mu_a -\max_{b\neq a}\mu_a \geq \Delta_0$, with $\Delta_0=0.4$. Hence, an algorithm could exploit this property to infer $H^\star$ more quickly. For this case, we also derive some sample complexity bounds in \cref{app:theoretical_results}.
\begin{figure}[h!]
    \centering
    \begin{subfigure}[t]{0.325\textwidth}
        \centering
    \includegraphics[width=\linewidth]{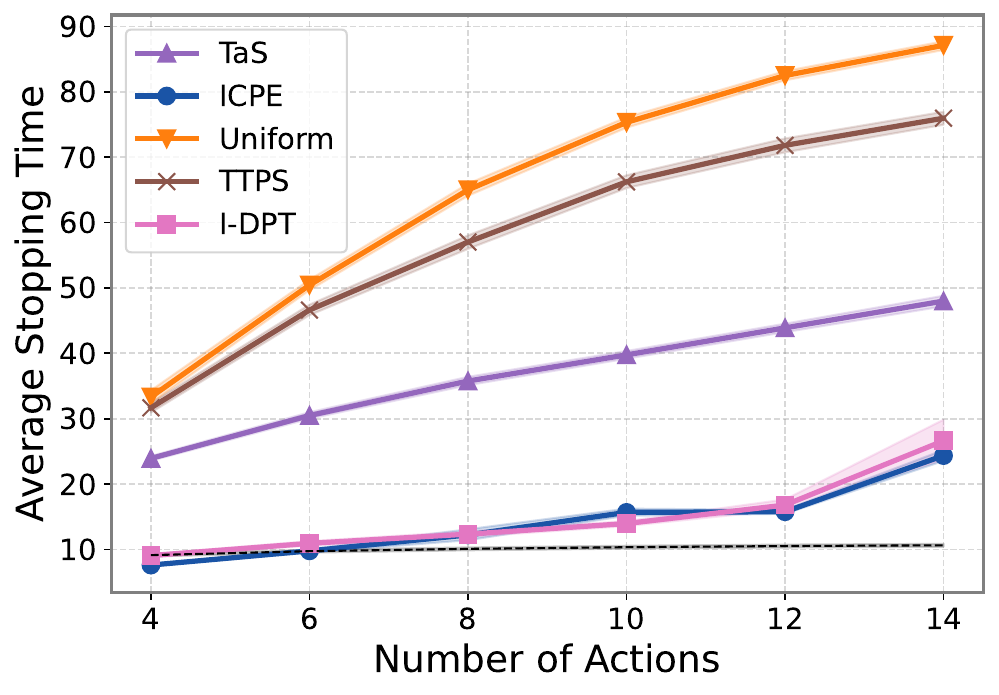}
    
    \caption{}
    \label{fig:stochastic_mab:fixed_confidence:avg_stopping_time}
    \end{subfigure}
    \hfill
    \begin{subfigure}[t]{0.325\textwidth}
        \centering\includegraphics[width=\linewidth]{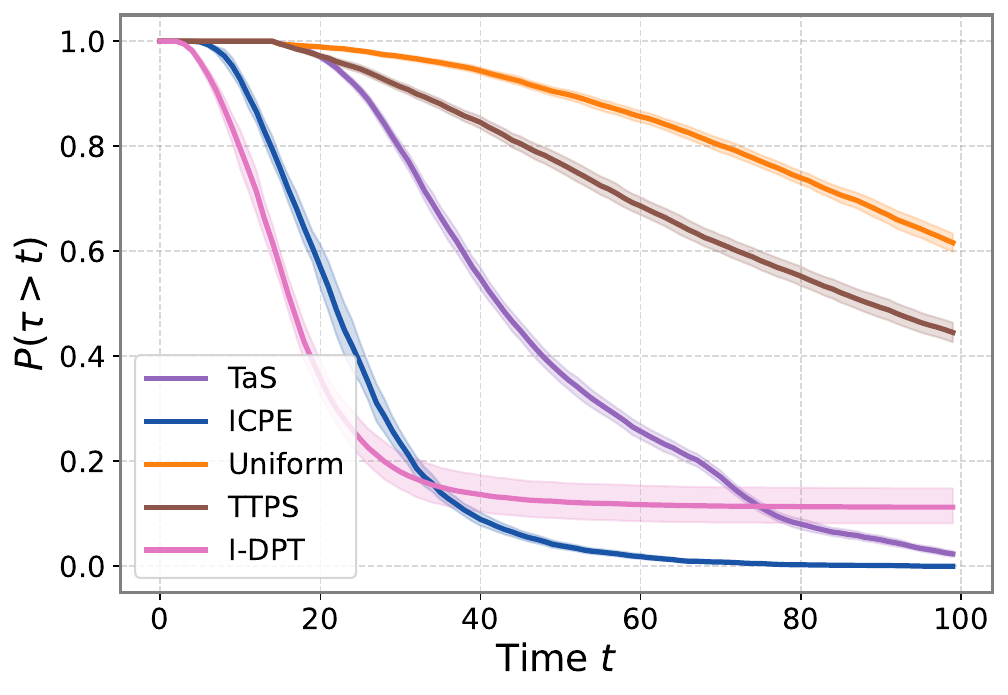}
        \caption{}
        \label{fig:stochastic_mab:fixed_confidence:survival_stopping_time}
    \end{subfigure}
    \hfill
    \begin{subfigure}[t]{0.325\textwidth}
        \centering
    \includegraphics[width=\linewidth]{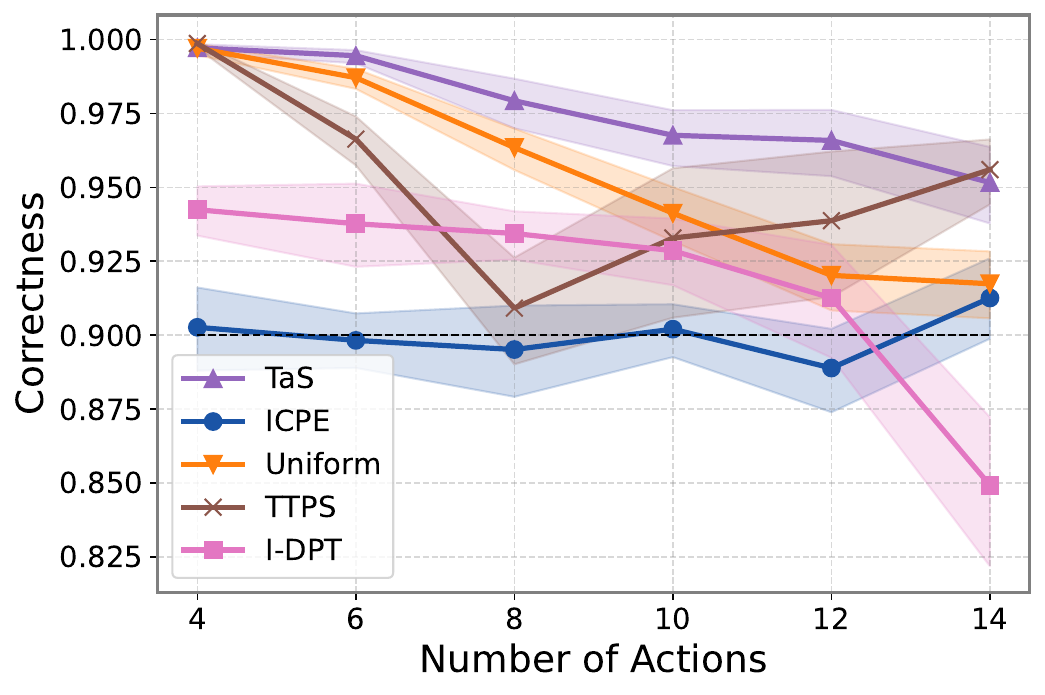}
    \caption{}
    \label{fig:stochastic_mab:fixed_confidence:correctness}
    \end{subfigure}
    \vspace{-10pt}
    \caption{Results for stochastic MABs with fixed confidence $\delta=0.1$ and $N=100$: {\bf(a)} average stopping time $\tau$; {\bf(b)} survival function of $\tau$; {\bf(c) }probability of correctness $\mathbb{P}^\pi(\hat H_\tau=H^\star)$.\protect\footnotemark}\label{fig:stochastic_mab:fixed_confidence}
    \vspace{-5pt}
\end{figure}

We compare against pure exploration baselines:  {\bf TaS} (Track-and-Stop)~\citep{garivier2016optimal} and {\bf TTPS} (Top-Two)~\citep{russo2018tutorial}, which are principled choices for hypothesis testing (asymptotically optimal or  close to optimal allocations that target the most confusable hypotheses).  We also include an ablation, ``{\bf $I$-DPT}'', which uses our learned inference $I_\phi(H| D_t)$ as in DPT~\citep{lee2023supervised} and acts greedily with respect to the posterior (and a simple confidence-threshold stop); this isolates the value of learning a query policy versus relying on posterior-driven greedy control. Details for $I$-DPT are in \cref{app:algorithms}.

In \cref{fig:stochastic_mab:fixed_confidence} are reported results for the fixed confidence.  In \cref{fig:stochastic_mab:fixed_confidence:avg_stopping_time} we see how \icpe{} is able to find an efficient strategy compared to other  techniques. 
Interestingly, also $I$-DPT seems to achieve relatively small sample complexities. However, the tail distribution of its $\tau$ is rather large compared to \icpe{} (\cref{fig:stochastic_mab:fixed_confidence:survival_stopping_time}) and the  correctness is smaller than $1-\delta$ for large values of $K$. Methods like TaS and TTPS achieve larger sample complexity, but also larger correctness values (\cref{fig:stochastic_mab:fixed_confidence:correctness}). This is a well known fact: theoretically-sound stopping rules, such as the ones used by TaS and TTPS, tend to be are overly conservative \citep{garivier2016optimal}.

Lastly, we verified the robustness of \icpe{} to distribution shifts. We trained \icpe{} in the stochastic fixed-confidence bandit setting  as described above,  and then evaluated the  trained model on bandit instances drawn from \emph{shifted} environment distributions.  We report the results in \cref{app:experimental_results:stochastic_bandit}. Across all experiments, we observed that both correctness and stopping time remain remarkably stable, with only minor fluctuations within the reported confidence intervals. This suggests that ICPE is not excessively sensitive to moderate shifts in the environment distribution around the training family.

Finally, for the sake of space, we refer the reader to  \cref{app:experimental_results:stochastic_bandit} for the results in the fixed budget setting.
\begin{figure}[t!]
        \centering    
        \includegraphics[width=0.65\linewidth]{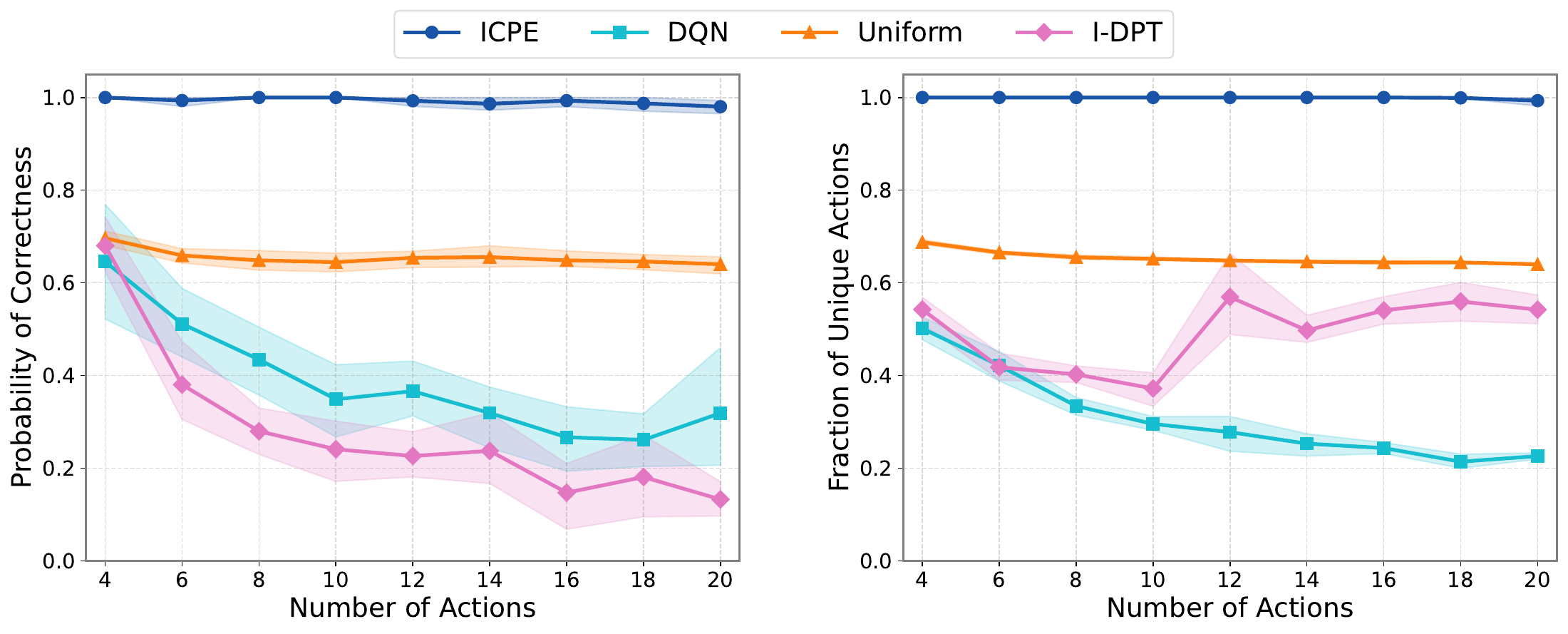}

    \caption{Deterministic bandits: (left) probability of correctly identifying the best action vs.  $K$; (right) average fraction of unique actions selected during exploration vs.   $K$.}
    \label{fig:deterministic}
    \vspace{-15pt}
\end{figure}

\noindent\textbf{Deterministic Bandits.}
We also evaluated \icpe{} in  deterministic bandit environments with a fixed budget $K$, equal to the number of actions. Thus, \icpe{} needs to learn to select each action only once to determine the optimal action. 
Since the rewards are deterministic, we cannot compare to classical BAI methods, which are tailored for stochastic environments. Instead, we compare to: (i) a uniform policy that uses a maximum likelihood estimator to estimate the best arm; (ii) {\bf DQN}~\citep{mnih2013playing}, which uses ${\cal D}_t$ as the state, and trains an $I$ network to infer the true hypothesis; (iii) and $I$-{\bf DPT}, acts greedily with respect to the posterior of $I_\phi$, as in DPT.
Figure~\ref{fig:deterministic} reports the results: \icpe{} consistently identifies optimal actions (correctness $\approx 1$) and learns optimal sampling strategies (fraction of unique actions $\approx 1$). Without being explicitly instructed to “choose each action exactly once”,  \icpe{} discovers on its own that sampling every action is exactly what yields enough information to identify the best. While the optimal exploration strategy in this setting is intuitive, baseline performance degrades sharply as the number of actions grows, illustrating that existing exploration methods can fail even in such simple environments.

\noindent\textbf{Bandit Problems with Hidden Information.}
To evaluate \icpe{} in structured settings, we introduce bandit environments with latent informational dependencies, termed \textit{magic actions}. In the single magic action case, the magic action $a_m$'s  reward is distributed according to ${\cal N}(\mu_{a_m}, \sigma_m^2)$, where $\sigma_m\in (0,1)$ and  $\mu_{a_m}\coloneqq\phi(\argmax_{a\neq a_m}\mu_a)$ encodes information about the optimal action’s identity through an invertible mapping $\phi$ that is unknown to the learner.  The index $a_m$ is fixed, and the mean rewards of the other actions $(\mu_a)_{a\neq a_m}$ are sampled from ${\cal P}$, a uniform distribution over models guaranteeing that $a_m$, as defined above, is not optimal (see \cref{app:subsec:sample_complexity_magic_action,app:subsec:bandit_hiden_info} for more details). Then, we define the reward distribution of the non-magic actions as ${\cal N}(\mu_a, (1-\sigma_m)^2)$.
\begin{figure}[h!]
    \centering
    \begin{subfigure}[t]{0.325\textwidth}
        \centering
        \includegraphics[width=\linewidth]{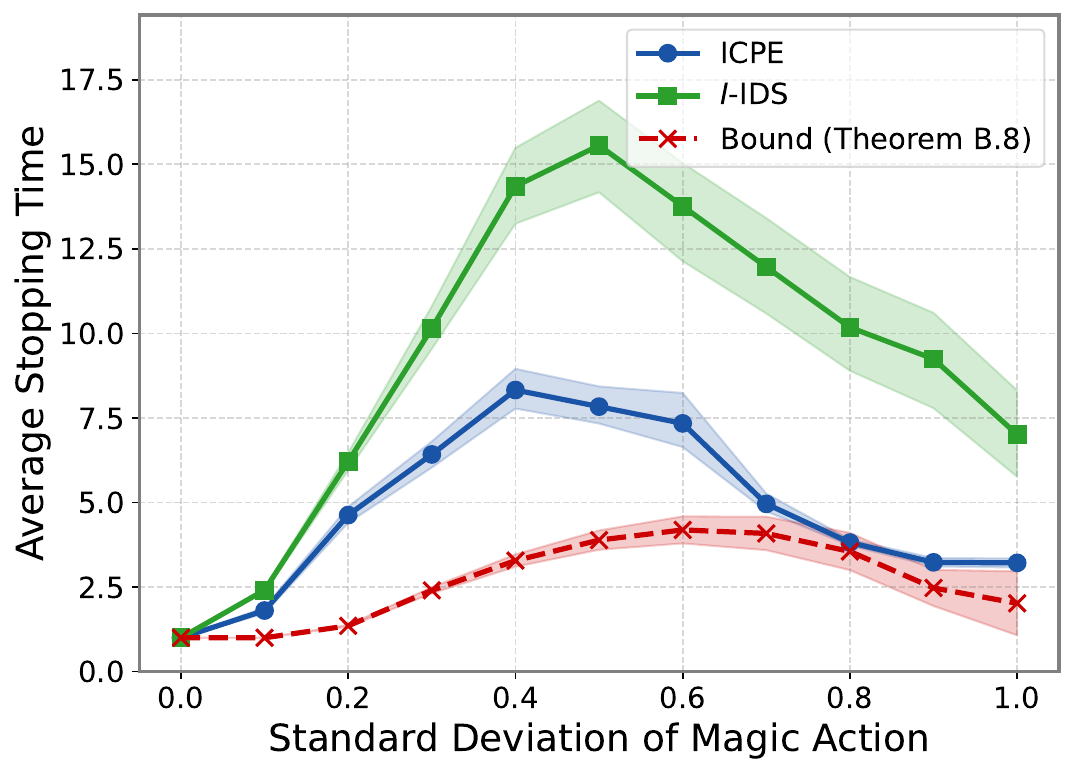}
        \caption{}
        \label{fig:noisy_magic}
    \end{subfigure}
    \hfill
    \begin{subfigure}[t]{0.325\textwidth}
        \centering
        \includegraphics[width=\linewidth]
        {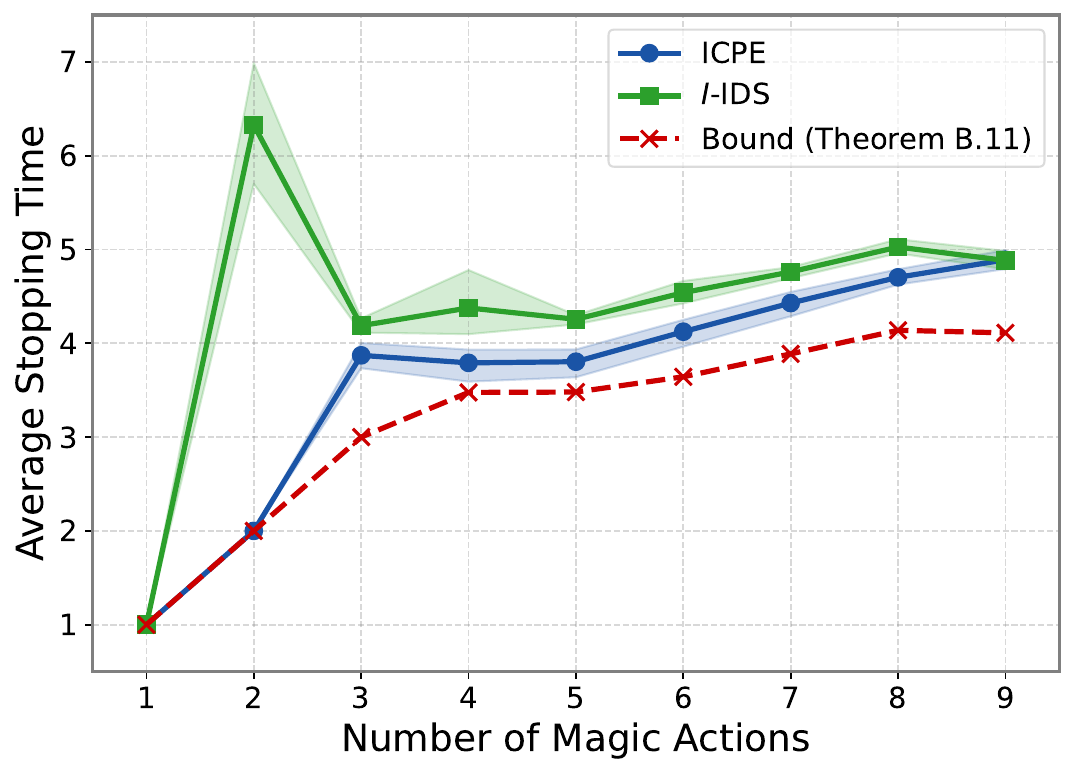}
        \caption{}
        \label{fig:mult_magic}
    \end{subfigure}
    \hfill
    \begin{subfigure}[t]{0.325\textwidth}
        \centering
        \includegraphics[width=\linewidth]{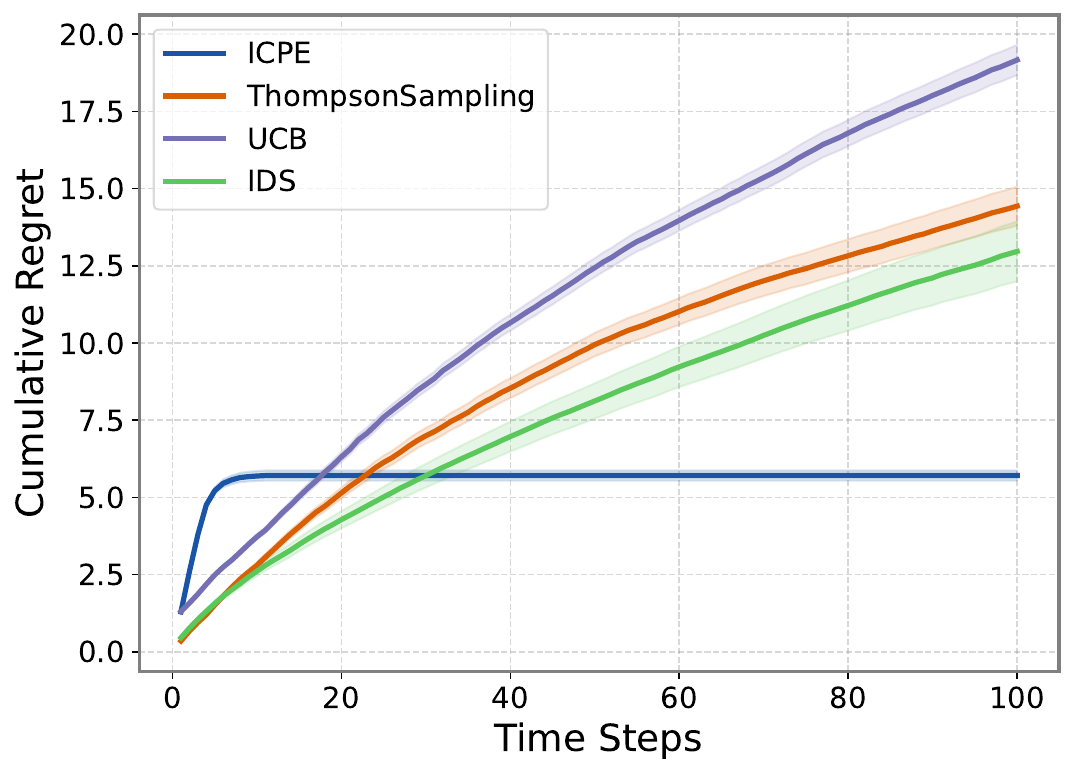}
        \caption{}
        \label{fig:regret}
    \end{subfigure}
    \vspace{-10pt}
    \caption{ {\bf (a)} Single magic action: average stopping time and the theoretical lower bound across varying $\sigma_m$. {\bf(b)} Magic chain: average stopping time  between \icpe{}, $I$-IDS vs. number of magic actions. {\bf(c)}  \icpe{} in a regret minimization task, with $\sigma_m=0.1$.}
    \label{fig:magic_results}
    \vspace{-10pt}
\end{figure}

In our first experiment, we  vary the standard deviation $\sigma_m$ in $[0,1]$.  This problem isolates whether \icpe{} can detect and exploit latent informational dependencies (via a single diagnostic action that encodes the optimal arm) and balance sampling across action based on varying uncertainty levels.

Regarding the baselines, applying classical  baselines (e.g., TaS) here is nontrivial: the magic action is coupled to the optimal arm via an \emph{unknown} map $\phi$, which would need to be encoded as inductive bias. Instead, we compare \icpe{} to ``$I$-{\bf IDS}'', which is standard pure exploration IDS \citep{russo2018learning} instantiated on top of ICPE’s trained inference $I_\phi$ for exploiting the magic action.

We evaluate in a fixed-confidence setting with error rate $\delta=0.1$. Figure~\ref{fig:noisy_magic} compares \icpe{}’s sample complexity against a theoretical lower bound (see \cref{app:theoretical_results}).
\icpe{} achieves sample complexities close to the theoretical  bound across all tested noise levels, consistently outperforming $I$-IDS.

Additionally, in \cref{fig:regret} we evaluate \icpe{} in a cumulative regret minimization setting, despite not being explicitly optimized for regret minimization. At the stopping $\tau$, \icpe{} commits to the identified best action (i.e., explore-then-commit strategy). As shown in the results, \icpe{} outperforms classic algorithms such as UCB, Thompson Sampling, and standard IDS initialized with Gaussian priors. 

To further challenge \icpe{}, we introduce a \emph{multi-layered ``magic chain'' bandit} environment, where there is a sequence of $n$ magic actions ${\cal A}_m\coloneqq\{a_{i_1},\dots,a_{i_n}\}\subset {\cal A}$ such that $\mu_{a_{i_j}}=\phi(\mu_{a_{i_{j+1}}})$, and $\mu_{a_{i_n}}=\phi(\argmax_{a\notin {\cal A}_m} \mu_a)$. The first index $i_1$ is known, and by following the chain, an agent can uncover the best action in $n$ steps. However, the optimal sample complexity depends on the ratio of magic actions to non-magic arms. Varying the number of magic actions from 1 to 9 in a 10-actions environment, Figure~\ref{fig:mult_magic} demonstrates \icpe{}'s empirical performance, outperforming $I$-IDS. 
\vspace{-5pt}
\subsection{General Search Problems: Pixel Sampling and Probabilistic Binary Search}
\vspace{-5pt}
We now evaluate the applicability of \icpe{} to general search problems, including structured real-world examples.

\begin{figure}
\vspace{-13pt}
  \centering
  
  \begin{subfigure}[t]{0.6\textwidth}
    \centering
    \includegraphics[width=\linewidth]{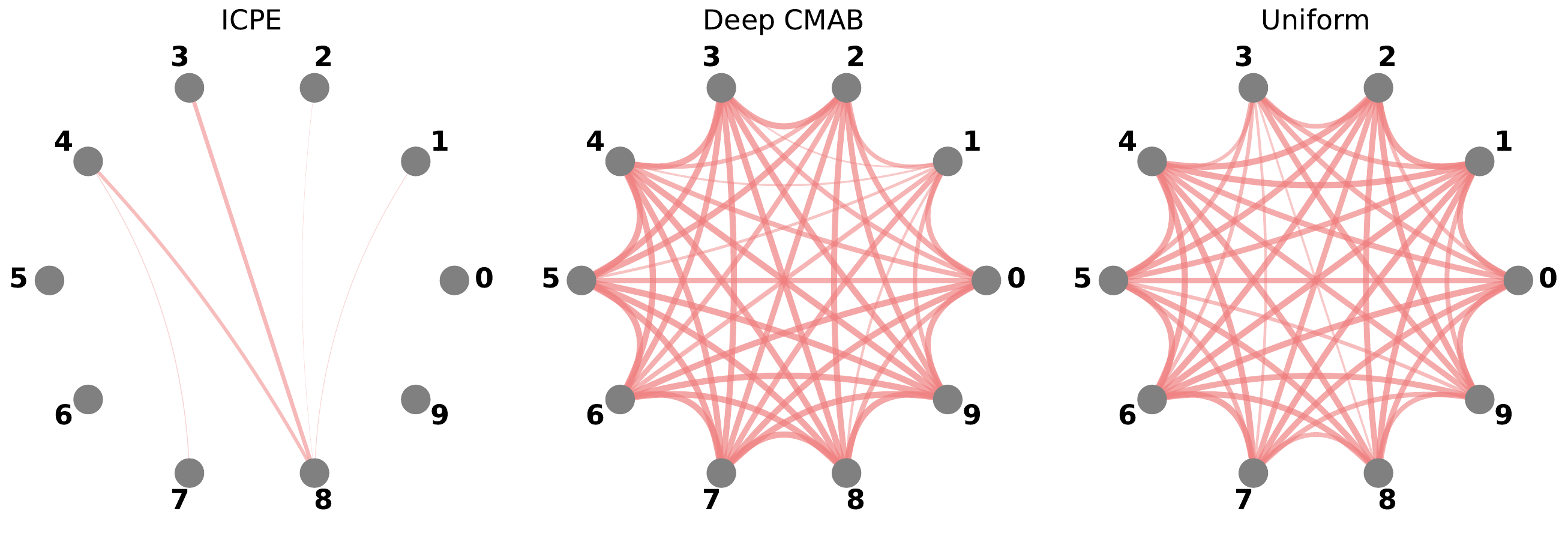}
    \caption{}
    \label{fig:chord}
  \end{subfigure}%
  \hfill
  \begin{subtable}[t]{0.38\textwidth}
  \vspace{-55pt}
    \centering
    \resizebox{\columnwidth}{!}{%
    \begin{tabular}{lcc}
      \toprule
      \textbf{Agent} & \textbf{Accuracy} & \textbf{Avg.\ Regions Used} \\
      \midrule
      \icpe{}      & $0.91 \pm 0.03$ & $10.09 \pm 0.11$ \\
      Deep CMAB    & $0.66 \pm 0.04$ & $7.90 \pm 0.09$  \\
      Uniform      & $0.25 \pm 0.04$ & $10.42 \pm 0.09$ \\
      \bottomrule
    \end{tabular}
    }
    \caption{}
    \label{tab:mnist_results}
  \end{subtable}
  \caption{MNIST pixel-sampling task: {\bf(a)} A chord between two digits indicates that their distributions were not significantly different ($p$-value $>0.05$, based on  a pairwise chi-squared test), with thicker chords representing higher $p$-values; {\bf(b)} accuracy and performance (mean $\pm$ 95\% CI)}
  \vspace{-10pt}
\end{figure}

\noindent {\bf Pixel sampling as generalized search.} We introduce a classification task inspired by active perception settings.  We consider the MNIST images~\citep{lecun1998gradient}, each partitioned into a set of 36 distinct pixel patches, corresponding to the query  space ${\cal A} = \{1,\dots,36\}$.  The agent starts from a blank (masked) image and, patch by patch, reveals pixels to quickly discover “what the image is about.” After choosing a query $a_t\in {\cal A}$ the agent observes $x_t$ (the revealed patch) and accumulates a partially observed image. After a budget \(N=12\), the agent outputs the predicted digit $\hat{H}_N\in\{0,\dots, 9\}$.

For this setting we consider a slight variation of \icpe{} that may be of interest: we consider an inference net $I$ that is a pre-trained classifier, trained on fully revealed images from ${\cal P}$. Using this network,  we benchmark \icpe{} against two baselines: standard uniform random sampling and Deep Contextual Multi-Armed Bandit ({\bf Deep CMAB})~\citep{collier2018deep}, which employs Bayesian neural networks to sample from a posterior distribution (Deep CMAB uses as rewards  the correctness probabilities computed by $I$). Importantly, we cannot compare to methods such as DPT since  ${\cal A}\neq {\cal H}$, the hypothesis space is different from the query space.

Table~\ref{tab:mnist_results} reports the classification accuracy and number of regions sampled. \icpe{} achieves substantially better performance than both baselines using fewer regions. However, to analyze whether \icpe{} learns a sampling strategy that adapts to the context of the task, we compare region selection distributions across digit classes using pairwise chi-squared tests. \icpe{} exhibits significantly more variation across classes than either baseline, as visualized in Figure~\ref{fig:chord}. This suggests \icpe{} adapts its exploration to class-conditional structure, rather than applying a generic sampling policy.

\begin{table}[h!]
\centering

\footnotesize
\setlength{\tabcolsep}{6pt}
\renewcommand{\arraystretch}{0.6}
\begin{tabular}{@{}rcccr@{}}
\toprule
\textbf{$K$ } & \textbf{Min Accuracy} & \textbf{Mean Stop Time} & \textbf{Max Stop Time} & $\log_2 K$ \\
\midrule
8   & 1.00 & $2.13 \pm 0.12$ & 3 & 3 \\
16  & 1.00 & $2.93 \pm 0.12$ & 4 & 4 \\
32  & 1.00 & $3.71 \pm 0.15$ & 5 & 5 \\
64  & 1.00 & $4.50 \pm 0.21$ & 6 & 6 \\
128 & 1.00 & $5.49 \pm 0.23$ & 7 & 7 \\
256 & 1.00 & $6.61 \pm 0.26$ & 8 & 8 \\
\bottomrule
\end{tabular}
\caption{\icpe{} performance on the binary search task as $K$ increases.}
\label{tab:main:binary_search}
\end{table}

\noindent \textbf{Probabilistic  binary search.}  We also evaluated \icpe{}'s capabilities to  autonomously meta-learn binary search.
We define an action space of ${\cal A} = \{1, \dotsc, K\}$, with $H^\star\in {\cal A}$. Pulling an arm above or below $H^\star$ yields a observation $x_t = -1$ or $x_t = +1$, respectively, providing directional feedback.  In \cref{tab:main:binary_search} we report results on $100$ held-out tasks per setting.  \icpe{} consistently achieves perfect accuracy with worst-case stopping times that match the optimal $\log_2(K)$ rate, demonstrating that it has successfully learned binary search. While simple, this task illustrates \icpe{}’s broader potential to learn efficient search strategies in domains where no hand-designed algorithm is available.

\section{Discussion and Conclusions}\label{sec:discussion}
Our results position \icpe{} within a broader line of work on \emph{active sequential hypothesis testing}~\citep{naghshvar2013active} and its close ties to exploration in RL~\citep{sutton2018reinforcement}. Regarding exploration, note that classical regret-minimization methods, including UCB variants~\citep{auer2002finite}, posterior sampling~\citep{osband2013more,russo2014learning}, and regret-focused IDS~\citep{russo2018tutorial}, optimize long-run reward, not hypothesis identification.  On the other hand, pure-exploration formulations in BAI \citep{audibert2010best}  yield sharp, instance-dependent procedures for hypothesis testing in fixed-confidence regimes (e.g., Track-and-Stop, \citealp{garivier2016optimal}).  However, these approaches assume to know the problem structure, which is not always possible if the user is not aware of such structure. Furthermore, computing an optimal data-collection policy remains a challenge in more general scenarios \citep{al2021navigating}, and we discuss some of these challenges in \cref{app:algorithm:tas}.

\icpe{} uses Transformers \citep{vaswani2017attention} to learn, in-context, a data collection policy and and inference rule. Transformers  have demonstrated remarkable in-context learning capabilities~\citep{brown2020language,garg2022can}. In-context learning \citep{moeini2025survey} is a form of meta-RL
\citep{beck2023survey}, where  agents can solve new tasks without updating any parameters by simply conditioning on histories. Building on this approach,  Transformers can mimic posterior sampling from offline data, as in DPT ~\citep{lee2023supervised}, or perform return-conditioning for regret minimization (e.g., {\sc ICEE}~\citealp{dai2024context}). 
 However, these approaches primarily target cumulative reward and typically lack a learned, $\delta$-aware stopping rule; applying them to hypothesis testing would require altering objectives, data-collection protocol, and add stopping semantics. Moreover, in generalized search where $\mathcal{A}\neq\mathcal{H}$, additional modeling is needed to map hypotheses to actions.

\icpe{} addresses these gaps by \emph{learning} to acquire information in-context. \icpe{} targets \emph{pure exploration for identification}: it splits inference and control, using a supervised inference network to provide task-relevant information signals, while an RL-trained Transformer learns acquisition policies that maximize information gain. This separation makes it possible to exploit rich, non-tabular structures  that are difficult to encode in hand-designed tests or confidence bounds.

Empirically, \icpe{} is competitive on unstructured bandits and extends naturally to structured and deterministic settings. The results on the MNIST dataset highlights a key strength: \icpe{} adapts sampling to the class-conditional structure. More broadly, \icpe{} suggests a path for \emph{data-driven generalized search}.

Limitations point to concrete avenues for future work. First, scaling to continuous or combinatorial hypothesis spaces to deal with more general scenarios is an important direction. However, such extensions require substantial further theoretical development, as rigorous formalisms for continuous hypothesis-testing frameworks remain an active area of research, even in classical pure-exploration settings (see, e.g., \citep{garivier2021nonasymptotic}). Second, extending \icpe{} to offline datasets is also a promising research direction. When offline data can be used to construct a reliable simulator, \icpe{} can already be applied directly. Moreover, even without such a simulator, \icpe{} could in principle be meta-trained purely from logged data using offline RL methods (e.g., IQL, CQL), and a systematic study of this offline regime is an important question for future work. Third, while the main focus of this work is to introduce and analyze \icpe{} as a general framework that can address a broad family of pure exploration problems, and we validate it on numerous BAI and active search tasks, we view real-world experiments as a natural next step. \icpe{} holds the promise to discover novel exploration and search algorithms in complex domains that do not offer a concrete way of finding an optimal solution a priori, such as determining efficient sequences of proteins to test in a lab~\citep{amin2024bayesian}, minimizing the number of tests required to detect cancer~\citep{gan2021greedy}, and expediting the design of materials with desired properties~\citep{talapatra2018autonomous}. In sum, we believe \icpe{} advances pure exploration by leveraging in-context learning to discover task-adaptive acquisition strategies, and it opens a route toward unifying classical sequential testing with learned, structure-aware search policies that scale to real problems.

\section*{Reproducibility Statement}
\addcontentsline{toc}{section}{Reproducibility Statement}
We have taken several measures to ensure the reproducibility of our results. All model architectures, optimization procedures, and hyperparameters are described in detail in the paper (see Sections 2–3 and Appendix C–D). Experiments were conducted using Python 3.10.12 and standard libraries including NumPy, SciPy, PyTorch, Pandas, Seaborn, Matplotlib, CVXPY, and Gurobi.

To facilitate replication, we provide our full source code at \url{https://github.com/rssalessio/icpe} under the MIT license. The code contains (i) implementations of ICPE and all baselines, (ii) configuration files specifying the hyperparameters for each experiment, and (iii) detailed instructions in the README.md file for installing dependencies and running all experiments. Running the provided scripts will reproduce the main results reported in the paper, including bandit, MDP, and generalized search benchmarks.
\section*{Acknowledgments}
\addcontentsline{toc}{section}{Acknowledgments}
The authors are pleased to acknowledge that the computational work reported on in this paper was performed on the Shared Computing Cluster administered by Boston University’s Research Computing Services and computing resources from the Laboratory for Information and Decision Systems at MIT. R.W. was supported by a Master of Engineering fellowship by the Eric and Wendy Schmidt Center at the Broad Institute.

\bibliographystyle{iclr2026_conference}

\bibliography{ref}

@inproceedings{garivier2016optimal,
	title        = {Optimal best arm identification with fixed confidence},
	author       = {Garivier, Aur{\'e}lien and Kaufmann, Emilie},
	year         = 2016,
	booktitle    = {Conference on Learning Theory},
	pages        = {998--1027},
	organization = {PMLR}
}

@article{hero2011Sensor,
  title = {Sensor Management: {{Past}}, Present, and Future},
  author = {Hero, Alfred O. and Cochran, Douglas},
  year = {2011},
  journal = {IEEE Sensors Journal},
  volume = {11},
  number = {12},
  pages = {3064--3075},
  doi = {10.1109/JSEN.2011.2167964}
}

@article{hantoute2008Characterizations,
  title = {Characterizations of the Subdifferential of the Supremum of Convex Functions},
  author = {Hantoute, A and L{\'o}pez, {\relax MA}},
  year = {2008},
  journal = {Journal of Convex Analysis},
  volume = {15},
  pages = {831--858}
}

@article{kaufmann2016complexity,
	title        = {On the complexity of best-arm identification in multi-armed bandit models},
	author       = {Kaufmann, Emilie and Capp{\'e}, Olivier and Garivier, Aur{\'e}lien},
	year         = 2016,
	journal      = {The Journal of Machine Learning Research},
	publisher    = {JMLR. org},
	volume       = 17,
	number       = 1,
	pages        = {1--42}
}

@inproceedings{zheng2005efficient,
	title        = {Efficient test selection in active diagnosis via entropy approximation},
	author       = {Zheng, Alice X and Rish, Irina and Beygelzimer, Alina},
	year         = 2005,
	booktitle    = {Conference on Uncertainty in Artificial Intelligence}
}

@inproceedings{poiani2024best,
  title = {Best-{{Arm Identification}} in {{Unimodal Bandits}}},
  booktitle = {Proceedings of {{The}} 28th {{International Conference}} on {{Artificial Intelligence}} and {{Statistics}}},
  author = {Poiani, Riccardo and Jourdan, Marc and Kaufmann, Emilie and Degenne, R{\'e}my},
  year = {2025},
  month = apr,
  pages = {2233--2241},
  publisher = {PMLR},
  issn = {2640-3498},
}

@article{ghosh1991brief,
	title        = {A brief history of sequential analysis},
	author       = {Ghosh, Bashkar K},
	year         = 1991,
	journal      = {Handbook of sequential analysis},
	publisher    = {Marcel Dekker New York},
	volume       = 1
}

@article{jedra2020optimal,
	title        = {Optimal best-arm identification in linear bandits},
	author       = {Jedra, Yassir and Proutiere, Alexandre},
	year         = 2020,
	journal      = {Advances in Neural Information Processing Systems},
	volume       = 33,
	pages        = {10007--10017}
}

@inproceedings{russo2023sample,
	title        = {On the sample complexity of representation learning in multi-task bandits with global and local structure},
	author       = {Russo, Alessio and Proutiere, Alexandre},
	year         = 2023,
	booktitle    = {Proceedings of the AAAI Conference on Artificial Intelligence},
	volume       = 37,
	pages        = {9658--9667}
}

@inproceedings{mukherjee2022chernoff,
	title        = {Chernoff sampling for active testing and extension to active regression},
	author       = {Mukherjee, Subhojyoti and Tripathy, Ardhendu S and Nowak, Robert},
	year         = 2022,
	booktitle    = {International Conference on Artificial Intelligence and Statistics},
	pages        = {7384--7432},
	organization = {PMLR}
}

@inproceedings{nguyen2025PriorDependent,
  title = {Prior-{{Dependent Allocations}} for {{Bayesian Fixed-Budget Best-Arm Identification}} in {{Structured Bandits}}},
  booktitle = {Proceedings of {{The}} 28th {{International Conference}} on {{Artificial Intelligence}} and {{Statistics}}},
  author = {Nguyen, Nicolas and Aouali, Imad and Gy{\"o}rgy, Andr{\'a}s and Vernade, Claire},
  year = {2025},
  month = apr,
  pages = {379--387},
  publisher = {PMLR},
  issn = {2640-3498},
  urldate = {2025-06-01},
  langid = {english}
}

@inproceedings{naghshvar2012noisy,
	title        = {Noisy bayesian active learning},
	author       = {Naghshvar, Mohammad and Javidi, Tara and Chaudhuri, Kamalika},
	year         = 2012,
	booktitle    = {2012 50th Annual Allerton Conference on Communication, Control, and Computing (Allerton)},
	pages        = {1626--1633},
	organization = {IEEE}
}

@article{cohn1996active,
	title        = {Active learning with statistical models},
	author       = {Cohn, David A and Ghahramani, Zoubin and Jordan, Michael I},
	year         = 1996,
	journal      = {Journal of artificial intelligence research},
	volume       = 4,
	pages        = {129--145}
}

@book{berry2010bayesian,
	title        = {Bayesian adaptive methods for clinical trials},
	author       = {Berry, Scott M and Carlin, Bradley P and Lee, J Jack and Muller, Peter},
	year         = 2010,
	publisher    = {CRC press}
}

@inproceedings{dasgupta2004analysis,
	title        = {Analysis of a greedy active learning strategy},
	author       = {Dasgupta, Sanjoy},
	year         = 2004,
	booktitle      = {Advances in Neural Information Processing Systems},
	volume       = 17
}

@article{russo2018tutorial,
	title        = {A tutorial on thompson sampling},
	author       = {Russo, Daniel J and Van Roy, Benjamin and Kazerouni, Abbas and Osband, Ian and Wen, Zheng and others},
	year         = 2018,
	journal      = {Foundations and Trends{\textregistered} in Machine Learning},
	publisher    = {Now Publishers, Inc.},
	volume       = 11,
	number       = 1,
	pages        = {1--96}
}

@article{chernoff1959sequential,
	title        = {Sequential Design of Experiments},
	author       = {Chernoff, Herman},
	year         = 1959,
	journal      = {The Annals of Mathematical Statistics},
	publisher    = {JSTOR},
	volume       = 30,
	number       = 3,
	pages        = {755--770}
}

@misc{gurobi,
	title        = {{Gurobi Optimizer Reference Manual}},
	author       = {{Gurobi Optimization, LLC}},
	year         = 2024,
	url          = {https://www.gurobi.com}
}

@misc{wandb,
	title        = {Experiment Tracking with Weights and Biases},
	author       = {Biewald, Lukas},
	year         = 2020,
	url          = {https://www.wandb.com/},
	note         = {Software available from wandb.com}
}

@article{diamond2016cvxpy,
	title        = {{CVXPY}: {A} {P}ython-embedded modeling language for convex optimization},
	author       = {Steven Diamond and Stephen Boyd},
	year         = 2016,
	journal      = {Journal of Machine Learning Research},
	volume       = 17,
	number       = 83,
	pages        = {1--5}
}

@article{kaufmann2021mixture,
	title        = {Mixture martingales revisited with applications to sequential tests and confidence intervals},
	author       = {Kaufmann, Emilie and Koolen, Wouter M},
	year         = 2021,
	journal      = {Journal of Machine Learning Research},
	volume       = 22,
	number       = 246,
	pages        = {1--44}
}

@article{golovin2010near,
	title        = {Near-optimal bayesian active learning with noisy observations},
	author       = {Golovin, Daniel and Krause, Andreas and Ray, Debajyoti},
	year         = 2010,
	journal      = {Advances in Neural Information Processing Systems},
	volume       = 23
}

@article{rainforth2024modern,
	title        = {Modern Bayesian experimental design},
	author       = {Rainforth, Tom and Foster, Adam and Ivanova, Desi R and Bickford Smith, Freddie},
	year         = 2024,
	journal      = {Statistical Science},
	publisher    = {Institute of Mathematical Statistics},
	volume       = 39,
	number       = 1,
	pages        = {100--114}
}

@article{arumugam2025toward,
	title        = {Toward Efficient Exploration by Large Language Model Agents},
	author       = {Arumugam, Dilip and Griffiths, Thomas L},
	year         = 2025,
	journal      = {arXiv preprint arXiv:2504.20997}
}

@inproceedings{brown2020language,
	title        = {Language models are few-shot learners},
	author       = {Brown, Tom and Mann, Benjamin and Ryder, Nick and Subbiah, Melanie and Kaplan, Jared D and Dhariwal, Prafulla and Neelakantan, Arvind and Shyam, Pranav and Sastry, Girish and Askell, Amanda and others},
	year         = 2020,
	booktitle      = {Advances in Neural Information Processing Systems},
	volume       = 33,
	pages        = {1877--1901}
}

@inproceedings{karnin2013almost,
	title        = {Almost optimal exploration in multi-armed bandits},
	author       = {Karnin, Zohar and Koren, Tomer and Somekh, Oren},
	year         = 2013,
	booktitle    = {International conference on machine learning},
	pages        = {1238--1246},
	organization = {PMLR}
}

@article{wang2023best,
	title        = {Best arm identification with fixed budget: A large deviation perspective},
	author       = {Wang, Po-An and Tzeng, Ruo-Chun and Proutiere, Alexandre},
	year         = 2023,
	journal      = {Advances in Neural Information Processing Systems},
	volume       = 36,
	pages        = {16804--16815}
}

@article{lindley1956measure,
	title        = {On a measure of the information provided by an experiment},
	author       = {Lindley, Dennis V},
	year         = 1956,
	journal      = {The Annals of Mathematical Statistics},
	publisher    = {Institute of Mathematical Statistics},
	volume       = 27,
	number       = 4,
	pages        = {986--1005}
}

@inproceedings{nowak2008generalized,
	title        = {Generalized binary search},
	author       = {Nowak, Robert},
	year         = 2008,
	booktitle    = {2008 46th annual Allerton conference on communication, control, and computing},
	pages        = {568--574},
	organization = {IEEE}
}

@article{bellala2010extensions,
	title        = {Extensions of generalized binary search to group identification and exponential costs},
	author       = {Bellala, Gowtham and Bhavnani, Suresh and Scott, Clayton},
	year         = 2010,
	journal      = {Advances in Neural Information Processing Systems},
	volume       = 23
}

@article{golovin2011adaptive,
	title        = {Adaptive submodularity: Theory and applications in active learning and stochastic optimization},
	author       = {Golovin, Daniel and Krause, Andreas},
	year         = 2011,
	journal      = {Journal of Artificial Intelligence Research},
	volume       = 42,
	pages        = {427--486}
}

@article{russo2023model,
	title        = {Model-free active exploration in reinforcement learning},
	author       = {Russo, Alessio and Proutiere, Alexandre},
	year         = 2023,
	journal      = {Advances in Neural Information Processing Systems},
	volume       = 36,
	pages        = {54740--54753}
}

@article{lisman2017viewpoints,
	title        = {Viewpoints: how the hippocampus contributes to memory, navigation and cognition},
	author       = {Lisman, John and Buzs{\'a}ki, Gy{\"o}rgy and Eichenbaum, Howard and Nadel, Lynn and Ranganath, Charan and Redish, A David},
	year         = 2017,
	journal      = {Nature neuroscience},
	publisher    = {Nature Publishing Group US New York},
	volume       = 20,
	number       = 11,
	pages        = {1434--1447}
}

@inproceedings{vaswani2017attention,
	title        = {Attention is all you need},
	author       = {Vaswani, Ashish and Shazeer, Noam and Parmar, Niki and Uszkoreit, Jakob and Jones, Llion and Gomez, Aidan N and Kaiser, {\L}ukasz and Polosukhin, Illia},
	year         = 2017,
	booktitle      = {Advances in Neural Information Processing Systems},
	volume       = 30
}

@article{rouyer2022near,
	title        = {A near-optimal best-of-both-worlds algorithm for online learning with feedback graphs},
	author       = {Rouyer, Chlo{\'e} and van der Hoeven, Dirk and Cesa-Bianchi, Nicol{\`o} and Seldin, Yevgeny},
	year         = 2022,
	journal      = {Advances in Neural Information Processing Systems},
	volume       = 35,
	pages        = {35035--35048}
}

@inproceedings{chen2021decision,
	title        = {Decision transformer: Reinforcement learning via sequence modeling},
	author       = {Chen, Lili and Lu, Kevin and Rajeswaran, Aravind and Lee, Kimin and Grover, Aditya and Laskin, Misha and Abbeel, Pieter and Srinivas, Aravind and Mordatch, Igor},
	year         = 2021,
	booktitle      = {Advances in Neural Information Processing Systems},
	volume       = 34,
	pages        = {15084--15097}
}

@article{garg2022can,
	title        = {What can transformers learn in-context? a case study of simple function classes},
	author       = {Garg, Shivam and Tsipras, Dimitris and Liang, Percy S and Valiant, Gregory},
	year         = 2022,
	journal      = {Advances in Neural Information Processing Systems},
	volume       = 35,
	pages        = {30583--30598}
}

@article{lee2023supervised,
	title        = {Supervised pretraining can learn in-context reinforcement learning},
	author       = {Lee, Jonathan and Xie, Annie and Pacchiano, Aldo and Chandak, Yash and Finn, Chelsea and Nachum, Ofir and Brunskill, Emma},
	year         = 2023,
	journal      = {Advances in Neural Information Processing Systems},
	volume       = 36,
	pages        = {43057--43083}
}

@article{russo2018learning,
	title        = {Learning to optimize via information-directed sampling},
	author       = {Russo, Daniel and Van Roy, Benjamin},
	year         = 2018,
	journal      = {Operations Research},
	publisher    = {INFORMS},
	volume       = 66,
	number       = 1,
	pages        = {230--252}
}

@article{cappe2013kullback,
	title        = {Kullback-Leibler upper confidence bounds for optimal sequential allocation},
	author       = {Capp{\'e}, Olivier and Garivier, Aur{\'e}lien and Maillard, Odalric-Ambrym and Munos, R{\'e}mi and Stoltz, Gilles},
	year         = 2013,
	journal      = {The Annals of Statistics},
	publisher    = {JSTOR},
	pages        = {1516--1541}
}

@article{auer2002finite,
	title        = {Finite-time analysis of the multiarmed bandit problem},
	author       = {Auer, Peter and Cesa-Bianchi, Nicolo and Fischer, Paul},
	year         = 2002,
	journal      = {Machine learning},
	publisher    = {Springer},
	volume       = 47,
	pages        = {235--256}
}

@inproceedings{audibert2010best,
	title        = {Best arm identification in multi-armed bandits},
	author       = {Audibert, Jean-Yves and Bubeck, S{\'e}bastien},
	year         = 2010,
	booktitle    = {COLT-23th Conference on learning theory-2010},
	pages        = {13--p}
}

@inproceedings{kaufmann2012thompson,
	title        = {Thompson sampling: An asymptotically optimal finite-time analysis},
	author       = {Kaufmann, Emilie and Korda, Nathaniel and Munos, R{\'e}mi},
	year         = 2012,
	booktitle    = {International conference on algorithmic learning theory},
	pages        = {199--213},
	organization = {Springer}
}

@article{russo2014learning,
	title        = {Learning to optimize via posterior sampling},
	author       = {Russo, Daniel and Van Roy, Benjamin},
	year         = 2014,
	journal      = {Mathematics of Operations Research},
	publisher    = {INFORMS},
	volume       = 39,
	number       = 4,
	pages        = {1221--1243}
}

@inproceedings{russo2025adaptive,
  title = {Adaptive {{Exploration}} for {{Multi-Reward Multi-Policy Evaluation}}},
  booktitle = {Proceedings of the 42nd {{International Conference}} on {{Machine Learning}}},
  author = {Russo, Alessio and Pacchiano, Aldo},
  year = 2025,
  month = oct,
  pages = {52382--52421},
  publisher = {PMLR},
  issn = {2640-3498},
  urldate = {2026-02-26},
}

@inproceedings{gopalan2014thompson,
	title        = {Thompson sampling for complex online problems},
	author       = {Gopalan, Aditya and Mannor, Shie and Mansour, Yishay},
	year         = 2014,
	booktitle    = {International conference on machine learning},
	pages        = {100--108},
	organization = {PMLR}
}

@inproceedings{russo2025pure,
	title        = {Pure Exploration with Feedback Graphs},
	author       = {Russo, Alessio and Song, Yichen and Pacchiano, Aldo},
	year         = 2025,
	booktitle    = {Proceedings of The 28th International Conference on Artificial Intelligence and Statistics},
	publisher    = {PMLR},
	series       = {Proceedings of Machine Learning Research}
}

@inproceedings{krishnamurthy2024can,
  title = {Can Large Language Models Explore In-Context?},
  booktitle = {Advances in Neural Information Processing Systems},
  author = {Krishnamurthy, Akshay and Harris, Keegan and Foster, Dylan J. and Zhang, Cyril and Slivkins, Aleksandrs},
  year = {2024},
  volume = {37},
}

@article{upsidedownrl,
	title        = {Training Agents using Upside-Down Reinforcement Learning},
	author       = {Rupesh Kumar Srivastava and Pranav Shyam and Filipe Mutz and Wojciech Jaskowski and Jürgen Schmidhuber},
	year         = 2019,
	journal      = {CoRR},
	volume       = {abs/1912.02877},
	url          = {http://arxiv.org/abs/1912.02877},
	publtype     = {informal},
	cdate        = 1546300800000
}

@article{kumar2019reward,
	title        = {Reward-conditioned policies},
	author       = {Kumar, Aviral and Peng, Xue Bin and Levine, Sergey},
	year         = 2019,
	journal      = {arXiv preprint arXiv:1912.13465}
}

@article{emmons2021rvs,
	title        = {Rvs: What is essential for offline rl via supervised learning?},
	author       = {Emmons, Scott and Eysenbach, Benjamin and Kostrikov, Ilya and Levine, Sergey},
	year         = 2021,
	journal      = {arXiv preprint arXiv:2112.10751}
}

@article{monea2024llms,
	title        = {LLMs Are In-Context Reinforcement Learners},
	author       = {Monea, Giovanni and Bosselut, Antoine and Brantley, Kiant{\'e} and Artzi, Yoav},
	year         = 2024,
	journal      = {arXiv preprint arXiv:2410.05362}
}

@article{dai2024context,
	title        = {In-context Exploration-Exploitation for Reinforcement Learning},
	author       = {Dai, Zhenwen and Tomasi, Federico and Ghiassian, Sina},
	year         = 2024,
	journal      = {arXiv preprint arXiv:2403.06826}
}

@article{oh2020discovering,
	title        = {Discovering reinforcement learning algorithms},
	author       = {Oh, Junhyuk and Hessel, Matteo and Czarnecki, Wojciech M and Xu, Zhongwen and van Hasselt, Hado P and Singh, Satinder and Silver, David},
	year         = 2020,
	journal      = {Advances in Neural Information Processing Systems},
	volume       = 33,
	pages        = {1060--1070}
}

@article{beck2023survey,
	title        = {A survey of meta-reinforcement learning},
	author       = {Beck, Jacob and Vuorio, Risto and Liu, Evan Zheran and Xiong, Zheng and Zintgraf, Luisa and Finn, Chelsea and Whiteson, Shimon},
	year         = 2023,
	journal      = {arXiv preprint arXiv:2301.08028}
}

@article{moeini2025survey,
	title        = {A Survey of In-Context Reinforcement Learning},
	author       = {Moeini, Amir and Wang, Jiuqi and Beck, Jacob and Blaser, Ethan and Whiteson, Shimon and Chandra, Rohan and Zhang, Shangtong},
	year         = 2025,
	journal      = {arXiv preprint arXiv:2502.07978}
}

@article{sun2025large,
	title        = {Large Language Model-Enhanced Multi-Armed Bandits},
	author       = {Sun, Jiahang and Wang, Zhiyong and Yang, Runhan and Xiao, Chenjun and Lui, John and Dai, Zhongxiang},
	year         = 2025,
	journal      = {arXiv preprint arXiv:2502.01118}
}

@article{nie2024evolve,
	title        = {EVOLvE: Evaluating and Optimizing LLMs For Exploration},
	author       = {Nie, Allen and Su, Yi and Chang, Bo and Lee, Jonathan N and Chi, Ed H and Le, Quoc V and Chen, Minmin},
	year         = 2024,
	journal      = {arXiv preprint arXiv:2410.06238}
}

@article{coda2023meta,
	title        = {Meta-in-context learning in large language models},
	author       = {Coda-Forno, Julian and Binz, Marcel and Akata, Zeynep and Botvinick, Matt and Wang, Jane and Schulz, Eric},
	year         = 2023,
	journal      = {Advances in Neural Information Processing Systems},
	volume       = 36,
	pages        = {65189--65201}
}

@article{achiam2023gpt,
	title        = {Gpt-4 technical report},
	author       = {Achiam, Josh and Adler, Steven and Agarwal, Sandhini and Ahmad, Lama and Akkaya, Ilge and Aleman, Florencia Leoni and Almeida, Diogo and Altenschmidt, Janko and Altman, Sam and Anadkat, Shyamal and others},
	year         = 2023,
	journal      = {arXiv preprint arXiv:2303.08774}
}

@article{harris2025should,
	title        = {Should You Use Your Large Language Model to Explore or Exploit?},
	author       = {Harris, Keegan and Slivkins, Aleksandrs},
	year         = 2025,
	journal      = {arXiv preprint arXiv:2502.00225}
}

@article{liu2024single,
	title        = {A Single Goal is All You Need: Skills and Exploration Emerge from Contrastive RL without Rewards, Demonstrations, or Subgoals},
	author       = {Liu, Grace and Tang, Michael and Eysenbach, Benjamin},
	year         = 2024,
	journal      = {arXiv preprint arXiv:2408.05804}
}

@article{even2006action,
	title        = {Action elimination and stopping conditions for the multi-armed bandit and reinforcement learning problems.},
	author       = {Even-Dar, Eyal and Mannor, Shie and Mansour, Yishay and Mahadevan, Sridhar},
	year         = 2006,
	journal      = {Journal of machine learning research},
	volume       = 7,
	number       = 6
}

@article{bubeck2011pure,
	title        = {Pure exploration in finitely-armed and continuous-armed bandits},
	author       = {Bubeck, S{\'e}bastien and Munos, R{\'e}mi and Stoltz, Gilles},
	year         = 2011,
	journal      = {Theoretical Computer Science},
	publisher    = {Elsevier},
	volume       = 412,
	number       = 19,
	pages        = {1832--1852}
}

@article{russo2025multi,
	title        = {Multi-reward best policy identification},
	author       = {Russo, Alessio and Vannella, Filippo},
	year         = 2025,
	journal      = {Advances in Neural Information Processing Systems},
	volume       = 37,
	pages        = {105583--105662}
}

@article{degenne2019non,
	title        = {Non-asymptotic pure exploration by solving games},
	author       = {Degenne, R{\'e}my and Koolen, Wouter M and M{\'e}nard, Pierre},
	year         = 2019,
	journal      = {Advances in Neural Information Processing Systems},
	volume       = 32
}

@article{morris2022motivates,
	title        = {On what motivates us: a detailed review of intrinsic v. extrinsic motivation},
	author       = {Morris, Laurel S and Grehl, Mora M and Rutter, Sarah B and Mehta, Marishka and Westwater, Margaret L},
	year         = 2022,
	journal      = {Psychological medicine},
	publisher    = {Cambridge University Press},
	volume       = 52,
	number       = 10,
	pages        = {1801--1816}
}

@article{nadel1991hippocampus,
	title        = {The hippocampus and space revisited},
	author       = {Nadel, Lynn},
	year         = 1991,
	journal      = {Hippocampus},
	publisher    = {Wiley Subscription Services, Inc., A Wiley Company Hoboken},
	volume       = 1,
	number       = 3,
	pages        = {221--229}
}

@article{o1979precis,
	title        = {Pr{\'e}cis of O'Keefe \& Nadel's The hippocampus as a cognitive map},
	author       = {O'keefe, John and Nadel, Lynn},
	year         = 1979,
	journal      = {Behavioral and Brain Sciences},
	publisher    = {Cambridge University Press},
	volume       = 2,
	number       = 4,
	pages        = {487--494}
}

@article{nadel2013hippocampus,
	title        = {The hippocampus: part of an interactive posterior representational system spanning perceptual and memorial systems.},
	author       = {Nadel, Lynn and Peterson, Mary A},
	year         = 2013,
	journal      = {Journal of Experimental Psychology: General},
	publisher    = {American Psychological Association},
	volume       = 142,
	number       = 4,
	pages        = 1242
}

@article{kagan1972motives,
	title        = {Motives and development.},
	author       = {Kagan, Jerome},
	year         = 1972,
	journal      = {Journal of personality and social psychology},
	publisher    = {American Psychological Association},
	volume       = 22,
	number       = 1,
	pages        = 51
}

@article{hunter2007matplotlib,
	title        = {Matplotlib: A 2D graphics environment},
	author       = {Hunter, John D},
	year         = 2007,
	journal      = {Computing in science \& engineering},
	publisher    = {IEEE},
	volume       = 9,
	number       = 3,
	pages        = {90--95}
}

@inproceedings{mckinney2010data,
	title        = {Data structures for statistical computing in python},
	author       = {McKinney, Wes and others},
	year         = 2010,
	booktitle    = {Proceedings of the 9th Python in Science Conference},
	volume       = 445,
	pages        = {51--56},
	organization = {Austin, TX}
}

@misc{michael_waskom_2017_883859,
	title        = {mwaskom/seaborn: v0.8.1 (September 2017)},
	author       = {Michael Waskom and Olga Botvinnik and Drew O'Kane and Paul Hobson and Saulius Lukauskas and David C Gemperline and Tom Augspurger and Yaroslav Halchenko and John B. Cole and Jordi Warmenhoven and Julian de Ruiter and Cameron Pye and Stephan Hoyer and Jake Vanderplas and Santi Villalba and Gero Kunter and Eric Quintero and Pete Bachant and Marcel Martin and Kyle Meyer and Alistair Miles and Yoav Ram and Tal Yarkoni and Mike Lee Williams and Constantine Evans and Clark Fitzgerald and Brian and Chris Fonnesbeck and Antony Lee and Adel Qalieh},
	year         = 2017,
	month        = sep,
	publisher    = {Zenodo},
	version      = {v0.8.1}
}

@article{NEURIPS2019_9015,
	title        = {Pytorch: An imperative style, high-performance deep learning library},
	author       = {Paszke, Adam and Gross, Sam and Massa, Francisco and Lerer, Adam and Bradbury, James and Chanan, Gregory and Killeen, Trevor and Lin, Zeming and Gimelshein, Natalia and Antiga, Luca and others},
	year         = 2019,
	journal      = {Advances in Neural Information Processing Systems (NeurIPS)},
	volume       = 32
}

@article{2020SciPy-NMeth,
	title        = {{{SciPy} 1.0: Fundamental Algorithms for Scientific Computing in Python}},
	author       = {Virtanen, Pauli and Gommers, Ralf and Oliphant, Travis E. and Haberland, Matt and Reddy, Tyler and Cournapeau, David and Burovski, Evgeni and Peterson, Pearu and Weckesser, Warren and Bright, Jonathan and {van der Walt}, St{\'e}fan J. and Brett, Matthew and Wilson, Joshua and Millman, K. Jarrod and Mayorov, Nikolay and Nelson, Andrew R. J. and Jones, Eric and Kern, Robert and Larson, Eric and Carey, C J and Polat, {\.I}lhan and Feng, Yu and Moore, Eric W. and {VanderPlas}, Jake and Laxalde, Denis and Perktold, Josef and Cimrman, Robert and Henriksen, Ian and Quintero, E. A. and Harris, Charles R. and Archibald, Anne M. and Ribeiro, Ant{\^o}nio H. and Pedregosa, Fabian and {van Mulbregt}, Paul and {SciPy 1.0 Contributors}},
	year         = 2020,
	journal      = {Nature Methods},
	volume       = 17,
	pages        = {261--272},
	doi          = {10.1038/s41592-019-0686-2},
	adsurl       = {https://rdcu.be/b08Wh}
}

@article{harris2020array,
	title        = {Array programming with {NumPy}},
	author       = {Charles R. Harris and K. Jarrod Millman and St{\'{e}}fan J. van der Walt and Ralf Gommers and Pauli Virtanen and David Cournapeau and Eric Wieser and Julian Taylor and Sebastian Berg and Nathaniel J. Smith and Robert Kern and Matti Picus and Stephan Hoyer and Marten H. van Kerkwijk and Matthew Brett and Allan Haldane and Jaime Fern{\'{a}}ndez del R{\'{i}}o and Mark Wiebe and Pearu Peterson and Pierre G{\'{e}}rard-Marchant and Kevin Sheppard and Tyler Reddy and Warren Weckesser and Hameer Abbasi and Christoph Gohlke and Travis E. Oliphant},
	year         = 2020,
	month        = sep,
	journal      = {Nature},
	publisher    = {Springer Science and Business Media {LLC}},
	volume       = 585,
	number       = 7825,
	pages        = {357--362},
	doi          = {10.1038/s41586-020-2649-2},
	url          = {https://doi.org/10.1038/s41586-020-2649-2}
}

@book{van1995python,
	title        = {Python reference manual},
	author       = {Van Rossum, Guido and Drake Jr, Fred L},
	year         = 1995,
	publisher    = {Centrum voor Wiskunde en Informatica Amsterdam}
}

@article{ariu2025policy,
  title={Policy Testing in Markov Decision Processes},
  author={Ariu, Kaito and Wang, Po-An and Proutiere, Alexandre and Abe, Kenshi},
  journal={arXiv preprint arXiv:2505.15342},
  year={2025}
}

@inproceedings{shukla2024preference,
  title = {Preference-Based Pure Exploration},
  booktitle = {Advances in {{Neural Information Processing Systems}}},
  author = {Shukla, Apurv and Basu, Debabrota},
  year = 2024,
  volume = {37},
  pages = {17313--17347}
}

@inproceedings{carlsson2024pure,
  title={Pure exploration in bandits with linear constraints},
  author={Carlsson, Emil and Basu, Debabrota and Johansson, Fredrik and Dubhashi, Devdatt},
  booktitle={International Conference on Artificial Intelligence and Statistics},
  pages={334--342},
  year={2024},
  organization={PMLR}
}

@article{karthik2024optimal,
  title={Optimal best arm identification with fixed confidence in restless bandits},
  author={Karthik, PN and Tan, Vincent YF and Mukherjee, Arpan and Tajer, Ali},
  journal={IEEE Transactions on Information Theory},
  volume={70},
  number={10},
  pages={7349--7384},
  year={2024},
  publisher={IEEE}
}

@inproceedings{NEURIPS2024_c6b2921f,
  title = {Finding Good Policies in Average-Reward {{Markov Decision Processes}} without Prior Knowledge},
  booktitle = {Advances in Neural Information Processing Systems},
  author = {Tuynman, Adrienne and Degenne, R{\'e}my and Kaufmann, Emilie},
  editor = {Globerson, A. and Mackey, L. and Belgrave, D. and Fan, A. and Paquet, U. and Tomczak, J. and Zhang, C.},
  year = 2024,
  volume = {37},
  pages = {109948--109979},
  publisher = {Curran Associates, Inc.},
  doi = {10.52202/079017-3489}
}

@article{auer2008near,
	title        = {Near-optimal regret bounds for reinforcement learning},
	author       = {Auer, Peter and Jaksch, Thomas and Ortner, Ronald},
	year         = 2008,
	journal      = {Advances in Neural Information Processing Systems (NeurIPS)},
	volume       = 21
}

@inproceedings{marjani2021navigating,
	title        = {Navigating to the Best Policy in Markov Decision Processes},
	author       = {Marjani, Aymen Al and Garivier, Aur{\'e}lien and Proutiere, Alexandre},
	year         = 2021,
	booktitle    = {Advances in Neural Information Processing Systems (NeurIPS)}
}

@book{sutton2018reinforcement,
	title        = {Reinforcement learning: An introduction},
	author       = {Sutton, Richard S and Barto, Andrew G},
	year         = 2018,
	publisher    = {MIT press}
}

@article{mnih2015human,
	title        = {Human-level control through deep reinforcement learning},
	author       = {Mnih, Volodymyr and Kavukcuoglu, Koray and Silver, David and Rusu, Andrei A and Veness, Joel and Bellemare, Marc G and Graves, Alex and Riedmiller, Martin and Fidjeland, Andreas K and Ostrovski, Georg and others},
	year         = 2015,
	journal      = {Nature},
	publisher    = {Nature Publishing Group},
	volume       = 518,
	number       = 7540,
	pages        = {529--533}
}

@inproceedings{lattimore2012pac,
	title        = {PAC bounds for discounted MDPs},
	author       = {Lattimore, Tor and Hutter, Marcus},
	year         = 2012,
	booktitle    = {Algorithmic Learning Theory: 23rd International Conference, ALT 2012, Lyon, France, October 29-31, 2012. Proceedings 23},
	pages        = {320--334},
	organization = {Springer}
}

@article{auer2002using,
	title        = {Using confidence bounds for exploitation-exploration trade-offs},
	author       = {Auer, Peter},
	year         = 2002,
	journal      = {Journal of Machine Learning Research},
	volume       = 3,
	number       = {Nov},
	pages        = {397--422}
}

@book{puterman2014markov,
	title        = {Markov decision processes: discrete stochastic dynamic programming},
	author       = {Puterman, Martin L},
	year         = 2014,
	publisher    = {John Wiley \& Sons}
}

@article{duan2016rl,
	title        = {RL$^2$: Fast reinforcement learning via slow reinforcement learning},
	author       = {Duan, Yan and Schulman, John and Chen, Xi and Bartlett, Peter L and Sutskever, Ilya and Abbeel, Pieter},
	year         = 2016,
	journal      = {arXiv preprint arXiv:1611.02779}
}

@article{al2021navigating,
	title        = {Navigating to the best policy in markov decision processes},
	author       = {Al Marjani, Aymen and Garivier, Aur{\'e}lien and Proutiere, Alexandre},
	year         = 2021,
	journal      = {Advances in Neural Information Processing Systems},
	volume       = 34,
	pages        = {25852--25864}
}

@article{resnick1997recommender,
	title        = {Recommender systems},
	author       = {Resnick, Paul and Varian, Hal R},
	year         = 1997,
	journal      = {Communications of the ACM},
	publisher    = {ACM New York, NY, USA},
	volume       = 40,
	number       = 3,
	pages        = {56--58}
}

@article{chen2023active,
	title        = {Active learning for contextual search with binary feedback},
	author       = {Chen, Xi and Liu, Quanquan and Wang, Yining},
	year         = 2023,
	journal      = {Management Science},
	publisher    = {INFORMS},
	volume       = 69,
	number       = 4,
	pages        = {2165--2181}
}

@article{gan2021greedy,
	title        = {Greedy approximation algorithms for active sequential hypothesis testing},
	author       = {Gan, Kyra and Jia, Su and Li, Andrew},
	year         = 2021,
	journal      = {Advances in Neural Information Processing Systems},
	volume       = 34,
	pages        = {5012--5024}
}

@article{degenne2019pure,
	title        = {Pure exploration with multiple correct answers},
	author       = {Degenne, R{\'e}my and Koolen, Wouter M},
	year         = 2019,
	journal      = {Advances in Neural Information Processing Systems},
	volume       = 32
}

@book{chernoff1992sequential,
	title        = {Sequential design of experiments},
	author       = {Chernoff, Herman},
	year         = 1992,
	publisher    = {Springer}
}

@article{naghshvar2013active,
	title        = {Active Sequential Hypothesis Testing},
	author       = {Naghshvar, Mohammad and Javidi, Tara},
	year         = 2013,
	journal      = {The Annals of Statistics},
	volume       = 41,
	number       = 6,
	pages        = {2703--2738}
}

@article{atsidakou2022bayesian,
	title        = {Bayesian fixed-budget best-arm identification},
	author       = {Atsidakou, Alexia and Katariya, Sumeet and Sanghavi, Sujay and Kveton, Branislav},
	year         = 2022,
	journal      = {arXiv preprint arXiv:2211.08572}
}

@book{bengio1990learning,
	title        = {Learning a synaptic learning rule},
	author       = {Bengio, Yoshua and Bengio, Samy and Cloutier, Jocelyn},
	year         = 1990,
	publisher    = {Citeseer}
}

@inproceedings{van2016deep,
	title        = {Deep reinforcement learning with double q-learning},
	author       = {Van Hasselt, Hado and Guez, Arthur and Silver, David},
	year         = 2016,
	booktitle    = {Proceedings of the AAAI conference on artificial intelligence},
	volume       = 30,
}

@article{jourdan2022top,
	title        = {Top two algorithms revisited},
	author       = {Jourdan, Marc and Degenne, R{\'e}my and Baudry, Dorian and de Heide, Rianne and Kaufmann, Emilie},
	year         = 2022,
	journal      = {Advances in Neural Information Processing Systems},
	volume       = 35,
	pages        = {26791--26803}
}

@article{mnih2013playing,
	title        = {Playing atari with deep reinforcement learning},
	author       = {Mnih, Volodymyr and Kavukcuoglu, Koray and Silver, David and Graves, Alex and Antonoglou, Ioannis and Wierstra, Daan and Riedmiller, Martin},
	year         = 2013,
	journal      = {arXiv preprint arXiv:1312.5602}
}

@article{ghosh2024fixed,
  title = {Fixed Budget Best Arm Identification in Unimodal Bandits},
  author = {Ghosh, Debamita and Hanawal, Manjesh Kumar and Zlatanov, Nikola},
  year = {2024},
  journal = {Transactions on Machine Learning Research},
  issn = {2835-8856}
}

@article{NEURIPS2024_1fb0a4de,
	title        = {Fixed Confidence Best Arm Identification in the Bayesian Setting},
	author       = {Jang, Kyoungseok and Komiyama, Junpei and Yamazaki, Kazutoshi},
	year         = 2024,
	journal      = {Advances in Neural Information Processing Systems},
	volume       = 37
}

@inproceedings{mannor2011bandits,
	title        = {From bandits to experts: On the value of side-observations},
	author       = {Mannor, Shie and Shamir, Ohad},
	year         = 2011,
	booktitle      = {Advances in Neural Information Processing Systems},
	volume       = 24
}

@inproceedings{al2021adaptive,
	title        = {Adaptive sampling for best policy identification in markov decision processes},
	author       = {Marjani, Aymen Al and Proutiere, Alexandre},
	year         = 2021,
	booktitle    = {International Conference on Machine Learning},
	pages        = {7459--7468},
	organization = {PMLR}
}

@article{schaul2010metalearning,
	title        = {Metalearning},
	author       = {Schaul, Tom and Schmidhuber, J{\"u}rgen},
	year         = 2010,
	journal      = {Scholarpedia},
	volume       = 5,
	number       = 6,
	pages        = 4650
}

@article{osband2013more,
	title        = {(More) efficient reinforcement learning via posterior sampling},
	author       = {Osband, Ian and Russo, Daniel and Van Roy, Benjamin},
	year         = 2013,
	journal      = {Advances in Neural Information Processing Systems (NeurIPS)},
	volume       = 26
}

@article{cecchi2017adaptive,
	title        = {Adaptive active hypothesis testing under limited information},
	author       = {Cecchi, Fabio and Hegde, Nidhi},
	year         = 2017,
	journal      = {Advances in Neural Information Processing Systems},
	volume       = 30
}

@article{lecun1998gradient,
	title        = {Gradient-based learning applied to document recognition},
	author       = {LeCun, Yann and Bottou, L{\'e}on and Bengio, Yoshua and Haffner, Patrick},
	year         = 1998,
	journal      = {Proceedings of the IEEE},
	publisher    = {Ieee},
	volume       = 86,
	number       = 11,
	pages        = {2278--2324}
}

@article{collier2018deep,
	title        = {Deep contextual multi-armed bandits},
	author       = {Collier, Mark and Llorens, Hector Urdiales},
	year         = 2018,
	journal      = {arXiv preprint arXiv:1807.09809}
}

@inproceedings{scherrer2012approximate,
    title = {Approximate modified policy iteration},
    booktitle = {Proceedings of the 29th international coference on international conference on machine learning},
    author = {Scherrer, Bruno and Ghavamzadeh, Mohammad and Gabillon, Victor and Geist, Matthieu},
    year = {2012},
    pages = {1889--1896},
}

@book{gyorfi_distribution-free_2002,
    address = {New York, NY},
    series = {Springer {Series} in {Statistics}},
    title = {A {Distribution}-{Free} {Theory} of {Nonparametric} {Regression}},
    copyright = {http://www.springer.com/tdm},
    isbn = {978-0-387-95441-7 978-0-387-22442-8},
    url = {http://link.springer.com/10.1007/b97848},
    urldate = {2025-11-19},
    publisher = {Springer},
    author = {Györfi, László and Kohler, Michael and Krzyżak, Adam and Walk, Harro},
    year = {2002},
    doi = {10.1007/b97848},
}

@article{JMLR:v13:lazaric12a,
    title = {Finite-sample analysis of least-squares policy iteration},
    volume = {13},
    url = {http://jmlr.org/papers/v13/lazaric12a.html},
    number = {98},
    journal = {Journal of Machine Learning Research},
    author = {Lazaric, Alessandro and Ghavamzadeh, Mohammad and Munos, Rémi},
    year = {2012},
    pages = {3041--3074},
}

@article{garivier2021nonasymptotic,
  title={Nonasymptotic sequential tests for overlapping hypotheses applied to near-optimal arm identification in bandit models},
  author={Garivier, Aur{\'e}lien and Kaufmann, Emilie},
  journal={Sequential Analysis},
  volume={40},
  number={1},
  pages={61--96},
  year={2021},
  publisher={Taylor \& Francis}
}

@article{amin2024bayesian,
  title={Bayesian optimization of antibodies informed by a generative model of evolving sequences},
  author={Amin, Alan Nawzad and Gruver, Nate and Kuang, Yilun and Li, Lily and Elliott, Hunter and McCarter, Calvin and Raghu, Aniruddh and Greenside, Peyton and Wilson, Andrew Gordon},
  journal={arXiv preprint arXiv:2412.07763},
  year={2024}
}

@article{talapatra2018autonomous,
  title={Autonomous efficient experiment design for materials discovery with Bayesian model averaging},
  author={Talapatra, Anjana and Boluki, Shahin and Duong, Thien and Qian, Xiaoning and Dougherty, Edward and Arr{\'o}yave, Raymundo},
  journal={Physical Review Materials},
  volume={2},
  number={11},
  pages={113803},
  year={2018},
  publisher={APS}
}

@inproceedings{russo2024Fair,
    title = {Fair best arm identification with fixed confidence},
    doi = {10.1109/CDC56724.2024.10886570},
    booktitle = {2024 {IEEE} 63rd conference on decision and control ({CDC})},
    author = {Russo, Alessio and Vannella, Filippo},
    year = {2024},
    pages = {1173--1180},
}

\newpage
\appendix
\part*{Appendix}
\tableofcontents
\newpage
\section*{Appendix}\label{app:appendix}
\section*{Limitations and Broader Impact}\label{app:sec:limitations}
   \addcontentsline{toc}{section}{Limitations and Broader Impact}
\paragraph{Finite vs continuous sets of hypotheses.} A limitation of this work is the assumption that ${\cal H}$ is finite. This is a common assumption in active sequential hypothesis testing, and the continuous case is also referred to as \emph{active regression} \citep{mukherjee2022chernoff}. We believe our framework can be extended to this case with a proper parametrization of the inference mapping $I$ that allows to sample from a continuous set.

\paragraph{On the prior set of problems ${\cal P}.$} One limitation of our approach is the assumption of access to a prior set of problems 
$\mathcal P$. Such set may lack a common structure, and need not be stationary. Nonetheless, we view this as a useful starting point for developing more sophisticated methods. A natural direction for future work is to extend our framework to an adversarial setting, in which problem instances can evolve or even be chosen to thwart the learner.

\paragraph{Online training.} Another limitation arises from assuming access to an online simulator from which we can sample $M\sim \P$ and training \icpe{}. Implicitely, this assumes access to $H^\star$ during training. Learning how to generalize to setting where $H^\star$ is not perfectly known at training time is an exciting research direction.  Furthermore, our main focus   is in the sequential process of starting from "no data", to being able to predict the right hypothesis as quickly as possible (see the MNIST example). We believe this framework to be  valuable when one can build verifiable simulations to train policies that transfer to real‐world problems.

\paragraph{Practical  limitations and transformers.} A limitation of \icpe{} is the current limit $N$ on the horizon of the trajectory. This is due to the computation cost of training and using transformer architectures. Future work could investigate how to extend this limit, or completely remove it.

 Another technical limitation of \icpe{} is the hardness to scaling to larger problems. This is closely related to the above limitation, and it is mainly an issue of investigating how to improve the current architecture of \icpe{} and/or distribute training.

 Lastly, we believe that \icpe{} does not use the full capabilities of transformer architectures. For example, during training and evaluation, we always use the last hidden state of the transformer to make prediction, while the other hidden states are left untouched.

\paragraph{Bayesian BAI.} Some of our work falls within the Bayesian Best Arm Identification theoretical framework. However, the Bayesian setting is less known compared to the frequentist one, and only recently some work \citep{NEURIPS2024_1fb0a4de} studied the unstructured Gaussian case. Future work should compare \icpe{} more thoroughly with Bayesian techniques once the Bayesian setting is more developed.

\paragraph{Broader impact.} This paper primarily focuses on foundational research in  pure exploration problems. Although we do not directly address societal impacts, we
recognize their importance. The methods proposed here improve the sample efficiency of active sequential hypothesis testing procedures,  and could be applied in various contexts with societal implications. For instance, our technique could be used in decision-making systems in healthcare, finance, and autonomous vehicles, where biases or
errors could have significant consequences. Therefore, while the immediate societal impact of our work may not be evident, we urge future researchers and practitioners to carefully consider the ethical implications and potential negative impacts in their specific applications
\newpage
\section{Extended Related Work}\label{app:sec:extended_related_work}

\paragraph{Exploration for Regret Minimization.}
The problem of exploration is particular relevant in RL \citep{sutton2018reinforcement},  and many strategies have been introduced, often with the goal of minimizing regret. Notably, approaches based on Posterior Sampling \citep{kaufmann2012thompson,osband2013more,russo2014learning,gopalan2014thompson} and Upper Confidence Bounds \citep{auer2002finite,auer2008near,cappe2013kullback,lattimore2012pac,auer2002using} have received significant attention. However, the problem of minimizing regret  is a relevant objective only when one cares about the rewards accumulated so far, and does not answer  the problem  of how to efficiently gather data to reach some desired goal. In this context, \emph{Information-Directed Sampling} (IDS) \citep{russo2014learning,russo2018tutorial} has been proposed to strike a balance between minimizing regret and maximizing information gain, where the latter is quantified as the mutual information between the true optimal action and the subsequent observation. However, when the information structure is unknown, it effectively becomes a significant challenge to exploit it.  Importantly, if the state does not encode the structure of the problem, RL techniques may not be able to exploit hidden information. 

\paragraph{In-Context Learning, LLMs and Return Conditioned Learning.}
Recently, Transformers \citep{vaswani2017attention,chen2021decision} have demonstrated remarkable in-context learning capabilities \citep{brown2020language,garg2022can}. In-context learning \citep{moeini2025survey} is a form of meta-RL
\citep{beck2023survey}, where  agents can solve new tasks without updating any parameters by simply conditioning on additional context such as their action-observation histories.  When provided with a few supervised input-output examples, a pretrained model can predict the most likely next token \citep{lee2023supervised}. Building on this ability, \citet{lee2023supervised} recently showed that Transformers can be trained in a supervised manner using offline data to mimic posterior sampling in reinforcement learning.
In \citep{krishnamurthy2024can} the authors investigate the extent to which LLMs \citep{achiam2023gpt} can perform in-context exploration in multi-armed bandit problems. Similarly, other works \citep{coda2023meta,monea2024llms,nie2024evolve,harris2025should,sun2025large} evaluate the in-context learning capabilities of LLMs in sequential decision making problems, with \citep{harris2025should} showing that LLMs can help at exploring large action spaces with inherent semantics. On a different note, in \citep{arumugam2025toward} investigate how to use LLMs to implement PSRL, leveraging the  full expressivity and fluidity of natural language to express the prior and current knowledge about the  problem.

In \citep{dai2024context} the authors presente {\sc ICEE} (In-Context Exploration Exploitation), a method closely related to \icpe{}. {\sc ICEE} uses Transformer architectures to  perform in-context exploration-exploration for RL. ICEE tackles this challenge by expanding the framework of return conditioned RL with in-context learning \citep{chen2021decision,emmons2021rvs}. Return conditioned learning is a type of technique where the agent learns the return-conditional distribution of actions in each state. Actions are then sampled  from the distribution of actions that receive high return. This methodoloy was first proposed for the online RL setting by work on Upside Down RL \citep{upsidedownrl} and Reward Conditioned Policies  \citep{kumar2019reward}.
Lastly, we note the important contribution of RL$^2$~\citep{duan2016rl}, which proposes to represent an RL policy as the hidden state of an RNN, whose weights are learned via RL. \icpe{}  employs a similar idea,  but focuses on a different objective (identification), and splits the process into a supervised inference network that provides rewards to an RL‐trained transformer network that selects actions to maximize information gain.

\paragraph{Active Pure Exploration in Bandit and RL Problems.} Other strategies consider the \emph{pure exploration problem} \citep{even2006action,audibert2010best,bubeck2011pure,kaufmann2016complexity}, or Best Arm Identification (BAI), in which the
samples collected by the agent  are no longer perceived as rewards, and the agent must actively optimize its exploration strategy to identify the optimal action. In this pure exploration framework, the task is typically formulated as a hypothesis testing problem: given a desired goal, the agent must reject the hypothesis that the observed data could have been generated by any environment whose behavior is fundamentally inconsistent with the true environment \citep{garivier2016optimal}. This approach leads to instance-dependent exploration strategies that adapt to the difficulty of the environment and has been extensively studied in the context of bandit problems under the fixed confidence setting \citep{even2006action,garivier2016optimal,degenne2019non,NEURIPS2024_1fb0a4de,russo2023sample,shukla2024preference,carlsson2024pure,karthik2024optimal,russo2024Fair,poiani2024best, russo2025pure}, where the objective is to identify the optimal policy using the fewest number of samples while maintaining a specified level of confidence. Similar ideas have been applied to Markov Decision Processes for identifying the best policy \citep{al2021adaptive,marjani2021navigating,russo2023model,NEURIPS2024_c6b2921f,russo2025multi}, policy testing \citep{ariu2025policy}, or rapidly estimating the value of a given policy \citep{russo2025adaptive}.
Another setting is that of of identifying the best arm in  MAB problems with a fixed horizon. In this case characterizing the complexity of the problem is challenging, and this is an area of work that is less developed compared to the fixed confidence one \citep{wang2023best,karnin2013almost,audibert2010best,atsidakou2022bayesian,nguyen2025PriorDependent,ghosh2024fixed}. Because of this reason, we believe \icpe{} can help better understand the nuances of this specific setting.

However, while BAI strategy are powerful, they may be suboptimal when the underlying information structure is not adequately captured within the hypothesis testing framework.
Hence, the issue of leveraging hidden environmental information, or problem with complex information structure remains a difficult problem. Although IDS and BAI techniques offer frameworks to account for such structure, extending these approaches to Deep Learning is difficult, particularly when the  information structure is unknown to the learner.

A closely related work is that of \citep{liu2024single}.
In \citep{liu2024single} the authors  present empirical evidence of skills and directed exploration
emerging from using RL with a sparse reward and a contrastive loss.   They define a goal state, and encode a sparse reward using that goal state. Their objective, which maximizes the probability of reaching the goal state, is similar to ours, where in our framework the goal state would be a hypothesis. Note, however, that they do not learn an inference network, and  we do not assume the observations to possess the Markov property.

\paragraph{Active Learning and Active Sequential Hypothesis Testing}

In the problem of active sequential hypothesis testing \citep{chernoff1992sequential,ghosh1991brief,lindley1956measure,naghshvar2013active,naghshvar2012noisy,mukherjee2022chernoff,gan2021greedy}, a learner is tasked with adaptively performing a sequence of actions to identify an unknown property of the environment. Each action yields noisy feedback about the true hypothesis, and the goal is to minimize the number of samples required to make a confident and correct decision. Similarly, active learning \citep{cohn1996active,chen2023active} studies the problem of data selection, and, closely related,  Bayesian Active Learning \citep{golovin2011adaptive}, or Bayesian experimental design \citep{rainforth2024modern}, studies
how to adaptively select from a number of expensive tests in order to identify an unknown hypothesis sampled from a known prior distribution.

Active sequential hypothesis testing generalizes the pure exploration setting in bandits and RL by allowing for the identification of arbitrary hypotheses, rather than just the optimal action. However, most existing approaches assume full knowledge of the observation model \citep{naghshvar2013active}, which is the distribution of responses for each action under each hypothesis. While some work has attempted to relax this assumption to partial knowledge~\citep{cecchi2017adaptive}, it remains highly restrictive in practice. As in bandit settings, real-world exploration and hypothesis testing often proceed without access to the true observation model, requiring strategies that can learn both the structure and the hypothesis from interaction alone.

\paragraph{Algorithm Discovery.} Our method is also closely related to the problem of discovering algorithms \citep{oh2020discovering}. In fact, one can argue that \icpe{} is effectively discovering active sampling techniques. This is particularly important for BAI and Best Policy Identification (BPI) problems, where often one needs to solve a computationally expensive optimization technique numerous times. For BPI the problem is even more exacerbated, since the optimization problem is usually non-convex \citep{al2021adaptive,russo2025adaptive}.

\paragraph{Cognitive Theories of Exploration.}
Our approach draws inspiration from cognitive theories of exploration. Indeed, in animals, exploration arises naturally from detecting mismatches between sensory experiences and internal cognitive maps—mental representations encoding episodes and regularities within environments \citep{o1979precis,nadel2013hippocampus}. Detection of novelty prompts updates of these cognitive maps, a function strongly associated with the hippocampus \citep{nadel1991hippocampus,lisman2017viewpoints}. Conversely, exploration can also be explicitly goal-directed: psychological theories posit that an internal representation of goals, combined with cognitive maps formed through experience, guides adaptive action selection \citep{kagan1972motives,morris2022motivates}. \icpe{} embodies these cognitive principles computationally: the exploration ($\pi$) network learns an internal model (analogous to a cognitive map), while the inference ($I$) network encodes goal-directed evaluation. This interplay enables \icpe{} to effectively manage exploration as an adaptive, structure-sensitive behavior.
\newpage
\section{Theoretical Results}\label{app:theoretical_results}
In this section we provide different theoretical results: first, we describe the theoretical results for \icpe{}. Then, we discuss some sample complexity results for different MAB problems with structure.

\subsection{ICPE: Theoretical Results}\label{app:theoretical_results:icpe}
In this subsection we present the theoretical results of \icpe{}. We begin by describing the problem setup. After that, we present results for the fixed budget and fixed confidence regimes respectively.

\subsubsection{Problem setup}\label{app:theoretical_results:icpe:problem_setup}

We now provide a formal definition of the underlying probability measures of the problem we consider. To that aim, it is important to formally define what a model $M$ is, as well as the definition of policy $\pi$ and inference rule $I$ (infernece rules are also known as  recommendation rules).

\paragraph{Spaces and $\sigma$-fields.}
We let ${\cal X}\subset\mathbb R$ be nonempty, compact, and endowed with its Borel $\sigma$-field ${\cal B}({\cal X})$.
Let ${\cal A}=\{1,\dots,K\}$ be a finite action set with the discrete $\sigma$-field $2^{\cal A}$, and let ${\cal H}$ be a finite hypothesis set with the discrete $\sigma$-field.
Write $\Delta({\cal X})$ for the set of Borel probability measures on $({\cal X},{\cal B}({\cal X}))$, equipped with the topology of weak convergence and its Borel $\sigma$-field ${\cal B}(\Delta({\cal X}))$.

For $t\in\mathbb N$, define the trajectory space
\[
{\cal Z}_t\coloneqq ({\cal X}\times{\cal A})^{t-1}\times{\cal X},
\qquad
z_t=(x_1,a_1,\ldots,a_{t-1},x_t),
\]
with the product topology and Borel $\sigma$-field ${\cal B}({\cal Z}_t)$.
Since ${\cal X}$ is compact metric and ${\cal A}$ is finite, each $({\cal Z}_t,{\cal B}({\cal Z}_t))$ is standard Borel (in fact compact Polish).
Set ${\cal Z}_\infty\coloneqq {\cal X}\times({\cal A}\times{\cal X})^{\mathbb N}$ with the product $\sigma$-field ${\cal B}({\cal Z}_\infty)$.

\paragraph{Observation dynamics and parameterization.} To define the dynamics $(\rho,P)$ we introduce a parametrization in $\omega \in \Omega$.
We take $(\Omega,d)$ to be a compact metric with Borel $\sigma$-field ${\cal B}(\Omega)$ and metric $d$.
For each $\omega\in\Omega$, we assume that $\rho$ and $P$ are functionals of $\omega$:
\begin{itemize}
\item $\rho_\omega\in\Delta({\cal X})$ is the initial observation law (a Borel probability measure on ${\cal X}$).
\item $P_s^\omega(\,\cdot\,|z_s,a_s)\in\Delta({\cal X})$ is a Borel probability \emph{kernel} for each round $s\ge1$: for every $(z_s,a_s)$, $P_s^\omega(\cdot|z_s,a_s)$ is a probability measure on $({\cal X},{\cal B}({\cal X}))$, and for every $C\in{\cal B}({\cal X})$ the map $(z_s,a_s)\mapsto P_s^\omega(C|z_s,a_s)$ is measurable.
\end{itemize}
We assume the following \emph{weak continuity} in $\omega$: for every bounded continuous $f:{\cal X}\to\mathbb R$,
\[
\omega\mapsto \int f\,d\rho_\omega \quad\text{and}\quad
(\omega,z_s,a_s)\mapsto \int f(x')\,P_s^\omega(dx'|z_s,a_s)
\ \ \text{are continuous.}
\]
Equivalently, $\omega\mapsto\rho_\omega$ and $(\omega,z_s,a_s)\mapsto P_s^\omega(\cdot|z_s,a_s)$ are continuous maps into $\Delta({\cal X})$ with the weak topology.

\paragraph{Set of models (environments).} To define the set of models, consider a mapping
 $\phi:\Omega\to \Delta({\cal X})\times\prod_{s\ge1}\big(\Delta({\cal X})^{{\cal Z}_s\times{\cal A}}\big)$ so that 
$\phi(\omega)=(\rho_\omega,P^\omega)$, where $P^\omega\coloneqq (P_s^\omega)_{s\ge1}$. Set
 ${\cal M}^\sharp\coloneqq \phi(\Omega)$ with the product of weak topologies.  We indicate by $(\rho,P)\in {\cal M}^\sharp$ a model in this set (hence $P=P^\omega$ for some $\omega$).
By continuity of $\phi$ and compactness of $\Omega$, ${\cal M}^\sharp$ is compact.

We also let $h^\star: \Delta({\cal X})\times\prod_{s\ge1}\big(\Delta({\cal X})^{{\cal Z}_s\times{\cal A}}\big)\to {\cal H}$,  with ${\cal H}$ finite (with the discrete topology), to be a Borel measurable mapping defining the ground truth hypothesis $H^\star=h^\star(\rho,P)$ for a pair $(\rho,P)\in {\cal M}^\sharp$. Then, we define ${\cal M}$ as
\[
{\cal M}\ \coloneqq\ \left\{\,(\rho,P,h^\star(\rho,P)): (\rho,P)\in {\cal M}^\sharp\right\}
\]
to be the push-forward set of environments\footnote{We omit ${\cal X},{\cal A}$ from the definition since these sets are the same for all models in ${\cal M}^\sharp$.}.
Therefore, a prior $Q$ on $\Omega$ induces a prior on ${\cal M}$ (and ${\cal M}^\sharp$) by pushforward: $\P \coloneqq Q\circ(\omega\mapsto(\phi(\omega), h^\star(\phi(\omega))))^{-1}$ and $\P^\sharp \coloneqq Q\circ \phi^{-1}$. In the following, we mainly work with ${\cal M}^\sharp$ and use  $\P$ and $\P^\sharp$ interchangeably whenever clear from the context.

\paragraph{Policies and inference (recommendation) rules.}
A (possibly randomized) policy $\pi=(\pi_s)_{s\geq 1}$ is a sequence of Borel probability kernels $\pi_s(\,\cdot\,|z_s)\in\Delta({\cal A})$, $s\ge1$, with $\pi_s:\ ({\cal Z}_s,{\cal B}({\cal Z}_s))\to(\Delta({\cal A}),{\cal B}(\Delta({\cal A})))$ measurable.
Deterministic policies are the special case $\pi_s(\,\cdot\,|z_s)=\delta_{\alpha_s(z_s)}$ for some measurable $\alpha_s:{\cal Z}_s\to{\cal A}$.
An inference rule at timestep $t$ is defined as any Borel map $I_t:{\cal Z}_t\to{\cal H}$. We also define an inference rule as the collection $I\coloneqq (I_s)_{s\geq 1}$.

\paragraph{Path laws and probability measures.}
Fix $M\in M^\sharp$, with $M=(\rho,(P_s)_s)$, and a policy $\pi=(\pi_s)_{s\ge1}$.
By the Ionescu--Tulcea theorem, there exists a unique probability measure $\mathbb P_{M,t}^\pi$ on $({\cal Z}_t,{\cal B}({\cal Z}_t))$ such that for all cylinder sets
$C=C_1\times B_1\times\cdots\times B_{t-1}\times C_t$, with $C_i\in{\cal B}({\cal X})$ and $B_i\subset{\cal A}$,
\begin{align*}
\mathbb P_{M,t}^\pi(C)
&=\int_{C_1}\rho({\rm d}x_1)
\prod_{s=1}^{t-1}\Bigg[
\int_{B_s}\pi_s({\rm d}a_s|z_s)\,
\int_{C_{s+1}}P_s({\rm d}x_{s+1}\,|\,z_s,a_s)\Bigg],
\end{align*}
equivalently,
\[
\mathbb P_{M,t}^\pi(dz_t)=\rho(dx_1)\prod_{s=1}^{t-1}\big[\pi_s(da_s|z_s)\,P_s({\rm d}x_{s+1}|z_s,a_s)\big].
\]
Analogously, one obtains a unique path measure $\mathbb P_{M}^\pi$ on $({\cal Z}_\infty,{\cal B}({\cal Z}_\infty))$.

Now, define the joint law on ${\cal M}^\sharp\times{\cal Z}_t$ by
\[
\mathbf P_t^\pi({\rm d} M,{\rm d}z_t)\coloneqq {\cal P}({\rm d} M)\,\mathbb P_{M,t}^\pi({\rm d}z_t),
\]
and the  trajectory marginal
\[
\mathbb P_{t}^\pi(A)\coloneqq \int_{{\cal M}^\sharp} \mathbb P_{M,t}^\pi(A)\,\P({\rm d} M),\qquad A\in{\cal B}({\cal Z}_t).
\]
Similarly, we also define $\mathbf P^\pi({\rm d} M, {\rm d} z)$ on $M^\sharp\times {\cal Z}_\infty$ and $\mathbb P^\pi(A)$ for $A\in {\cal B}({\cal Z}_\infty)$.

Lastly, since $\Omega$ and ${\cal Z}_t$ are standard Borel, regular conditional probabilities exist on $\Omega\times{\cal Z}_t$; by pushforward through $\phi$ they induce regular conditional probabilities on ${\cal M}^\sharp\times{\cal Z}_t$, for all $t\geq 1$.

\subsubsection{Posterior distribution over the true hypothesis and inference rule optimality}\label{app:theoretical_results:icpe:posterior_distribution_and_inference_rule}
We first record a domination assumption that allows us to express likelihoods w.r.t. fixed reference measures and obtain continuity.

\begin{assumption}[Domination and joint continuity]\label{assump:domination_param}
There exist probability measures $\lambda_0,\lambda$ on $({\cal X},{\cal B}({\cal X}))$ such that, for all $(\rho,P)\in{\cal M}^\sharp$, $s\in\mathbb N$, and $(z,a)\in{\cal Z}_s\times{\cal A}$,
\[
\rho(\cdot)\ll \lambda_0(\cdot)
\quad\text{and}\quad
P_s(\cdot\,|\,z,a)\ll \lambda(\cdot).
\]
We also let $p_0^\omega(x)\coloneqq \frac{d\rho_\omega}{d\lambda_0}(x)$ and $p_s^\omega(x\,|\,z,a)\coloneqq \frac{dP_s^\omega(\cdot\,|\,z,a)}{d\lambda}(x)$ be the corresponding densities (versions chosen jointly measurable).
\end{assumption}

\noindent\emph{Remark.} For compact ${\cal X}\subset\mathbb R$, such a dominating pair always exists (e.g., Lebesgue on ${\cal X}$).

Under \cref{assump:domination_param}, define the (policy-independent) likelihood for $(\omega,z)\in\Omega\times{\cal Z}_t$:
\[
\ell_t(M,z)\coloneqq p_0(x_1)\prod_{s=1}^{t-1}p_s(x_{s+1}\,|\,z_s,a_s).
\]

We now give a posterior kernel representation that is independent of $\pi$.

\begin{lemma}[Posterior kernel]\label{lemma:posterior_distribution_param}
For each $t\in\mathbb N$ there exists a probability kernel $R_t:{\cal Z}_t\times{\cal B}(M^\sharp)\to[0,1]$, independent of $\pi$, such that for every policy $\pi$ and all $A\in{\cal B}(M^\sharp)$, $Z\in{\cal B}({\cal Z}_t)$,
\[
\mathbf P_t^\pi(M\in A,\,{\cal D}_t\in Z)=\int_Z R_t(A |z)\,\mathbb P_{M\sim \P,t}^\pi(dz).
\]
Consequently, for $B\subset{\cal H}$,
\[
\mathbb P_t(H^\star\in B|{\cal D}_t=z)\coloneqq R_t\big(\{(\rho,P)\in {\cal M}^\sharp: h^\star(\rho,P)\in B\}|z\big)
\quad\text{for }\ \mathbb P_{{t}}^\pi\text{-a.e. }z.
\]
\end{lemma}

\begin{proof}
Fix $\pi$ and $t$.
Define the reference measure on ${\cal Z}_t$ (depending on $\pi$)
\[
\nu_t^\pi(dz)\coloneqq \lambda_0({\rm d}x_1)\prod_{s=1}^{t-1}\big[\pi_s({\rm d}a_s|z_s)\,\lambda({\rm d}x_{s+1})\big].
\]
By construction and \cref{assump:domination_param}, $\mathbb P_{M,t}^\pi\ll \nu_t^\pi$ for every $\omega$, and the Radon–Nikodym density is
\[
\frac{d\mathbb P_{M,t}^\pi}{d\nu_t^\pi}(z)=\ell_t(M,z),
\]
which does not depend on $\pi$.
Therefore,
\[
\mathbf P_t^\pi(M\in A,{\cal D}_t\in Z)=\int_A\int_Z \ell_t(M,z)\,\nu_t^\pi({\rm d}z)\,\P({\rm d} M),
\]
and
\[
\mathbb P_{{t}}^\pi(Z)=\int_Z\int_{{\cal M}^\sharp} \ell_t(M,z)\,\P({\rm d} M)\,\nu_t^\pi({\rm d}z).
\]
Absolute continuity $\mathbf P_t^\pi(A,\cdot)\ll \mathbb P_{{t}}^\pi(\cdot)$ follows immediately, and the Radon–Nikodym derivative is the displayed ratio, which we denote $R_t(A|z)$. Standard arguments show $R_t(\cdot|z)$ is a probability measure and $z\mapsto R_t(A|z)$ is measurable; independence of $\pi$ is evident from the formula. Mapping through $h^\star$ yields the posterior on ${\cal H}$.
\end{proof}

Define then \begin{equation}\mathbb P^\pi_t(H^\star=H) \coloneqq \int_{{\cal Z}_t} \mathbb P_t(H^\star=H|{\cal D}_t=z) \mathbb{P}^\pi_{t}({\rm d }z).\end{equation}
We now provide a proof of the optimality of an inference rule. In the following, we use the following quantity 
\begin{equation}
    r_t(z)\coloneq \max_{H\in {\cal H}} {\mathbb P}_t(H^\star=H|{\cal D}_t=z),
\end{equation}
which is the  maximum value of the posterior distribution at time $t$ for some dataset $z$.

\begin{proposition}\label{prop:optimal_inference_rule} Consider a fixed policy $\pi$. Let $t\in \mathbb{N}$ and $z\sim \mathbb{P}_{{t}}^\pi$. For any $t$ the  optimal inference rule to $\sup_{I_t} \mathbb P^\pi_t(H^\star=I_t({\cal D}_t))$ is given by $I_t^\star(z)=\argmax_{H\in {\cal H}} {\mathbb P}_t(H^\star=H|{\cal D}_t=z)$ (break ties according to some fixed ordering).
\end{proposition}
\begin{proof}
Fix a policy $\pi$ and an inference rule $ I_t$ at timestep $t$.
 By definition, we have
\begin{align*}
{\mathbb P}_t^\pi(H^
\star= I_t({\cal D}_t))&=\int_{{\cal Z}_t} \sum_{H\in {\cal H}}\mathbf{1}_{\{I_t(z)=H\}} {\mathbb P}_t(H^
\star=H|{\cal D}_t=z) \mathbb{P}_{t}^\pi({\rm d}z) \tag{Posterior independent of $\pi$},\\
&\leq  \int_{{\cal Z}_t} \max_{H\in {\cal H}}{\mathbb P}_t(H^
\star=H|{\cal D}_t=z) \mathbb{P}_{t}^\pi({\rm d}z),\\
&= \int_{{\cal Z}_t} r_t(z)  \mathbb{P}_{t}^\pi({\rm d}z).
\end{align*}

However, for any ${\cal D}_t=z$ choosing $I_t(z)=\argmax_{H\in {\cal H}}{\mathbb P}_t^\pi(H^
\star=H|{\cal D}_t=z) $ (break ties according to some fix ordering $\Rightarrow$ hence $ I_t(z)$ is Borel measurable) yields that 
\[
{\mathbb P}_t^\pi(H^
\star= I_t({\cal D}_t))  = \int_{{\cal Z}_t} r_t(z)  \mathbb{P}_{t}^\pi({\rm d}z),
\]
which concludes the proof.
\end{proof}

\subsubsection{Fixed budget setting: optimal policy}\label{app:theoretical_results:icpe:fixed_budget:optimal_policy}
We now turn our attention to the fixed budget setting. In particular, we prove that an optimal policy $\pi_t^\star$ in ${\cal D}_t$ attains the optimal value defined as  $V_t({\cal D}_t)=\max_a \mathbb{E}_{x_{t+1}|({\cal D}_t,a)}[V_{t+1}(({\cal D}_{t}, a, x_{t+1})|{\cal D}_t, a]$ with $V_N({\cal D}_N)=\max_H {\mathbb P}_t(H^\star=H|{\cal D}_N)$ (see a rigorous definition of $x_{t+1}|({\cal D}_t,a)$ below).

First, note that from \cref{prop:optimal_inference_rule} we can deduce that the optimal objective in the fixed budget satisfies, for all $t\geq 1$,
\[
\sup_{\pi,  I_t} {\mathbb P}_t^\pi(H^
\star= I_t({\cal D}_t)) =  \sup_{\pi} \mathbb{E}_t^\pi[r_t({\cal D}_t)]
\]
where $ r_t(z)\coloneq \max_{H\in {\cal H}} {\mathbb P}_t(H^\star=H|{\cal D}_t=z)$.

We now show that there exists an optimal deterministic policy $\pi^\star$ that optimally solves the  fixed budget regime. First, define the  following posterior mixture for any $t\in \mathbb{N}, X\in {\cal B}({\cal X}),(z,a)\in {\cal Z}_t\times {\cal A}$:
\begin{equation}\label{eq:posterior_mixture}
\bar P_t(x'\in X|z,a) \coloneqq  \int_{{\cal M}^\sharp} P_t(X|z,a) \, R_t({\rm d} M|z),
\end{equation}
where $P_t(\cdot|z,a)$ is the transition function at step $t$ in $(M,z,a)$. This is simply the posterior $x'|(z,a)$, that in the main text of the manuscript is denoted by $x_{t+1}|({\cal D}_t,a)$.

\paragraph{Optimal value.}
Define the value at $N\in \mathbb{N}$ as $V_N(z_N)= r_N(z_N)$ for any $z_N\in {\cal Z}_N$, and define the value function for $z_t\in {\cal Z}_t, a\in {\cal A}$ as
\begin{equation}
    V_t(z_t)=\max_{a\in {\cal A}} Q_t(z_t,a),\quad Q_t(z_t,a)=\int_{\cal X} V_{t+1}(\underbrace{z_t,a,x'}_{= z_{t+1}})\,\bar P_t({\rm d}x'|z_t,a),\quad t=1,\dots,N-1,
\end{equation}

For some ordering on ${\cal A}$, for $z\in {\cal Z}_t$ define $\pi_t^\star(z)\in \argmax_{a\in{\cal A}} Q_t(z,a)$ (break ties according to the ordering).  We have the following result.

\begin{proposition}
    For any $t\in \{1,\dots,N-1\}, z\in {\cal Z}_t$, the policy $\pi_t^\star(z)\in \argmax_{a\in{\cal A}} Q_t(z,a)$ (break ties according to a fixed ordering) is an optimal policy, that is
    \begin{equation}
          \sup_{\pi,I_N} {\mathbb P}_N^\pi(H_M^\star=I_N({\cal D}_N)) =   \mathbb{E}_N^{\pi^\star}[r_N({\cal D}_N)].
    \end{equation}
\end{proposition}
\begin{proof}
    To prove the result, we use \cref{lemma:upper_bound_value_fixed_budget}, which shows that  $\mathbb{E}_N^\pi[V_N({\cal D}_N)| {\cal D}_t] \leq V_t({\cal D}_t)$ holds $\mathbb{P}_t^\pi\hbox{-almost surely}$ for any policy $\pi$ and $t\in [T]$, with equality if $\pi=\pi^\star$. Then, using this inequality we can show that for $t=1$, with $z_1\in {\cal X}$, we have   \[
    \mathbb{E}_N^\pi[V_N({\cal D}_N)|{\cal D}_1=z_1] \leq V_1(z_1) \Rightarrow \mathbb{E}_N^\pi[r_N({\cal D}_N)] \leq \mathbb{E}_1[V_1({\cal D}_1)],
    \]
    with equality if $\pi=\pi^\star$, implying $\pi^\star$ is optimal since we can attain the r.h.s. (note that it does not depend on $\pi$). Hence
    \[
\sup_{\pi,I_N} {\mathbb P}_N^\pi(H_M^
\star=I_N({\cal D}_N)) =   \sup_\pi\mathbb{E}_N^{\pi}[r_N({\cal D}_N)] =   \mathbb{E}_N^{\pi^\star}[r_N({\cal D}_N)].
    \]
    where the first equality follows from \cref{prop:optimal_inference_rule}.
\end{proof}

We  now prove the result used in the proof.
\begin{lemma}\label{lemma:upper_bound_value_fixed_budget}
Consider the fixed budget setting with horizon $N$. For any policy $\pi=(\pi_s)_{s\geq 1}$, $t\in \{1,\dots,N\}, z\in {\cal Z}_t$, we have
\begin{equation}
    \mathbb{E}_N^\pi[V_N({\cal D}_N)| {\cal D}_t=z] \leq V_t(z) \quad \mathbb{P}_{M\sim \P,t}^\pi\hbox{-almost surely},
\end{equation}
with equality when $\pi=\pi^\star$.
\end{lemma}
\begin{proof}
    We prove it by backward induction. For $t=N$ the equality holds by definition. Assume it holds for $t+1$, then  at time $t$ for any policy $\pi$ we have
    \begin{align*}
        \mathbb{E}_N^\pi[V_N({\cal D}_N)| {\cal D}_t=z] &= \mathbb{E}_N^\pi[\mathbb{E}_N^\pi[V_N({\cal D}_N)| {\cal D}_{t+1}]| {\cal D}_t=z],\\
        &\leq \mathbb{E}_{t+1}^\pi[ V_{t+1}({\cal D}_{t+1})| {\cal D}_t=z],\\
        &= \mathbb{E}_{t}^\pi[ Q_{t}(z,\pi_t(z))| {\cal D}_t=z],\\
        &\leq V_t(z),
    \end{align*}
    and if $\pi=\pi^\star$ then both inequalities hold since they hold at $t+1$. Hence the result holds also for $t$, which concludes the induction argument.
\end{proof}

\subsubsection{Fixed confidence setting: dual problem formulation}\label{app:theoretical_results:icpe:fixed_confidence:dual}

In the fixed confidence setting we are interested in solving the following problem.

\begin{equation}\label{eq:problem_fixed_confidence}
        \inf_{\tau,\pi, I} \mathbb{E}^\pi[\tau] \quad \hbox{ subject to }\quad \mathbb{P}^\pi(H^\star=I_\tau({\cal D}_\tau))\geq 1-\delta, \;\mathbb{E}^\pi[\tau]<\infty.
    \end{equation} 
where $\tau$ is a stopping time adapted to $(\sigma({\cal D}_t))_t$ (recall that it counts the total number of observations; thus, if $\tau=t \Rightarrow $ we have observations $x_1,\dots, x_t$), $\pi=(\pi_s)_{s\geq 1}$, is a collection of policies, and $I=(I_s)_{s\geq 1}$, is a sequence of recommendation rules. Furthermore, we have that $\mathbb P^\pi(\tau<\infty)=1$ (this follows from $\mathbb E^\pi[\tau]<\infty$).

\paragraph{Dual problem and optimal recommendation rule}  In the following we focus on the dual problem of \cref{eq:problem_fixed_confidence}. First, we show what is the dual problem, and what is the optimal recomendation rule.
\begin{proposition}
    The Lagrangian dual of the problem in \cref{eq:problem_fixed_confidence} is given by
    \begin{equation}
        \sup_{\lambda\geq 0} \inf_{\pi,\tau,I} V_\lambda(\pi,\tau,I)=\sup_{\lambda\geq 0} \inf_{\pi,\tau,I}\lambda(1-\delta) +\mathbb{E}^\pi\left[ \tau-\lambda\mathbb{P}_\tau(H^\star=I_\tau({\cal D}_\tau)|{\cal D}_\tau)\right],
    \end{equation}
    where $\lambda\geq 0$ is the Lagrangian variable.
\end{proposition}
\begin{proof}
Since  any feasible solution  stops almost surely, we can also write
\begin{align*}
\mathbb{P}^\pi(H_M^\star=I_\tau({\cal D}_\tau)) &= \sum_{t\geq 1} \mathbb{P}_t^\pi(I_t({\cal D}_t)=H^\star,\tau=t), \tag{law of total probability}\\
&= \sum_{t\geq 1} \mathbb{E}_t^\pi\left[\mathbf{1}_{\{\tau=t\}}\mathbf{1}_{\{H^\star=I_t({\cal D}_t)\}}\right],\\
&= \sum_{t\geq 1} \mathbb{E}_t^\pi\left[\mathbb{E}_t^\pi\left[\mathbf{1}_{\{\tau=t\}}\mathbf{1}_{\{H^\star=I_t({\cal D}_t)\}}\Big|{\cal D}_t\right]\right], \tag{tower rule}\\
&= \sum_{t\geq 1} \mathbb{E}_t^\pi\left[\mathbf{1}_{\{\tau=t\}}\mathbb{E}_t^\pi\left[\mathbf{1}_{\{H^\star=I_t({\cal D}_t)\}}\Big|{\cal D}_t\right]\right], \tag{$\{\tau=t\}\in \sigma({\cal D}_t)$}\\
&= \sum_{t\geq 1} \mathbb{E}_t^\pi\left[\mathbf{1}_{\{\tau=t\}}\mathbb{P}_t(H^\star=I_t({\cal D}_t)|{\cal D}_t)\right]\tag{Outer expectation integrates over ${\cal D}_t$}
\end{align*}
where we used the fact that the posterior distribution does not depend on $\pi$ (\cref{lemma:posterior_distribution_param}).
Using this decomposition, for $\lambda\geq 0$ (the dual variable) we can write the Lagrangian dual of the problem as
\begin{align*}
    V_\lambda(\pi,\tau,I)&\coloneqq\mathbb{E}^\pi\left[\tau+ \lambda\left(1-\delta- \sum_{t\geq 1} \mathbf{1}_{\{\tau=t\}}\mathbb{P}_t(H^\star=I_t({\cal D}_t)|{\cal D}_t)\right)\right],\\
    &=\lambda-\lambda\delta +\mathbb{E}^\pi\left[ \sum_{t\geq 1} \mathbf{1}_{\{t\leq\tau\}}-\lambda\mathbf{1}_{\{\tau=t\}}\mathbb{P}_t(H^\star=I_t({\cal D}_t)|{\cal D}_t)\right],
\end{align*}
where we used that $\mathbb{E}^\pi[\tau]=\mathbb{E}^\pi\left[\sum_{t=1}^\tau1\right] =\mathbb{E}^\pi\left[\sum_{t=1}^\infty \mathbf{1}_{\{t\leq\tau\}}\right]$.
\end{proof}

For the dual problem, we now show we can embed the stopping rule as a stopping action. Define the extended action space $\bar {\cal A}\coloneqq {\cal A}\cup \{a_{\rm stop}\}$, where $a_{\rm stop}$ is absorbing. We show that for every $(\tau,\pi,I)$ there exists a policy $\bar\pi=(\bar \pi_s)_{s\geq 1}$, with $\bar\pi_s:{\cal Z}_s\to \Delta(\bar{\cal A})$, such that $V_\lambda(\pi,\tau,I)=V_\lambda(\bar \pi,I)$, where
\[
V_\lambda(\bar \pi,I)\coloneqq  \lambda(1-\delta) +\mathbb{E}^{\bar \pi}\left[ \bar \tau-\lambda\mathbb{P}_{\bar \tau}(H^\star=I_{\bar \tau}({\cal D}_{\bar \tau})|{\cal D}_{\bar \tau})\right].
\]
and
\[
{\bar \tau} =\inf\{t: a_t=a_{\rm stop}\}.
\]
At the beginning of round  $t$, given  ${\cal D}_t$, the learner may stop by choosing  $a_{\rm stop}$
 (termination at $t$, no new observation) or continue by choosing  $a_t\neq a_{\rm stop}$ and then observing  $x_{t+1}$. If the learner decides to stop after seeing $x_{t+1}$
, this is equivalent to choosing $a_{\rm stop}$ at round $t+1$ , leading to $\bar \tau=t+1=\tau$.
\begin{lemma}
   For every $(\pi,\tau,I)$ with $\tau\ge 1$ a.s., there exists a policy $\bar\pi$ on $\bar{\cal A}={\cal A}\cup\{a_{\rm stop}\}$ with stopping rule $\bar \tau=\inf\{t:a_t=a_{\rm stop}\}$ such that $V_\lambda(\pi,\tau,I)=V_\lambda(\bar\pi,I)$.
\end{lemma}
\begin{proof}
Since $\{\tau=t\}\in \sigma({\cal D}_t)$ note that there exists $A_t\in {\cal B}({\cal Z}_t)$ such that $\{\tau=t\}=\{{\cal D}_t\in A_t\}$.
Define, for all $t\ge 1$ and $z\in{\cal Z}_t$,
\[
\bar\pi_t(a_{\rm stop}| z)=\mathbf 1_{\{z\in A_t\}},\qquad
\bar\pi_t(a| z)=\pi_t(a| z)\ \ (a\in{\cal A}).
\]
Clearly we have $\{\bar \tau=t\}=\{a_t=a_{\rm stop}\}=\{\tau=t\}$,
 then
    \begin{align*}
    V_\lambda(\bar \pi,I)&=\lambda(1-\delta) +\mathbb{E}^{\bar \pi}\left[ \bar \tau-\lambda\mathbb{P}_{\bar \tau}(H^\star=I_{\bar \tau}({\cal D}_{\bar \tau})|{\cal D}_{\bar \tau})\right],\\
    &=\lambda(1-\delta) +\mathbb{E}^{\bar \pi}\left[ \sum_{t\geq 1} \mathbf{1}_{\{t\leq\bar\tau\}}-\lambda\mathbf{1}_{\{\bar \tau=t\}}\mathbb{P}_t(H^\star=I_t({\cal D}_{t})|{\cal D}_{t})\right],\\
    &=\lambda(1-\delta) +\mathbb{E}^{\bar \pi}\left[ \sum_{t\geq 1}^{\bar\tau} 1-\lambda\mathbf{1}_{\{a_t=a_{\rm stop}\}}\mathbb{P}_t(H^\star=I_t({\cal D}_{t})|{\cal D}_{t})\right],\\
    &=\lambda(1-\delta) +\mathbb{E}^{\bar \pi}\left[ \sum_{t\geq 1}^{ \tau} 1-\lambda\mathbf{1}_{\{\tau=t\}}\mathbb{P}_t(H^\star=I_t({\cal D}_{t})|{\cal D}_{t})\right],\\
    &=\lambda(1-\delta) +\mathbb{E}^{ \pi}\left[ \sum_{t\geq 1}{\bf 1}_{\{t\leq \tau\}}-\lambda\mathbf{1}_{\{\tau=t\}}\mathbb{P}_t(H^\star=I_t({\cal D}_{t})|{\cal D}_{t})\right],\\
    &=\lambda(1-\delta) +\mathbb{E}^{ \pi}\left[ \tau-\lambda\mathbb{P}_\tau(H^\star=I_\tau({\cal D}_{\tau})|{\cal D}_{\tau})\right].
    \end{align*}
\end{proof}

Then, in the following, we assume to work with the extended space $\bar {\cal A}$, and indicate by $\tau=\inf\{t:a_t=a_{\rm stop}\}$ the stopping time. We avoid the bar notation for simplicity.

We now show what is the optimal inference rule.
\begin{proposition}
    For any $t\in \mathbb{N}$ the optimal inference rule satisfies $I_t({\cal D}_t)\in\argmax_{H\in {\cal H}} \mathbb{P}_t(H^\star=H|{\cal D}_t=z_t)$ (break ties according to some fixed ordering), where  $z_t\in {\cal Z}_t$. Moreover, we also have
    \begin{equation}
        \sup_{\lambda\geq 0} \inf_{\pi,I} V_\lambda(\pi,I)=\sup_{\lambda\geq 0} \inf_{\pi} \lambda(1-\delta) +\mathbb{E}^\pi\left[\tau-\lambda r_\tau({\cal D}_\tau)\right],
    \end{equation}
    where $\tau=\inf\{t: a_t=a_{\rm stop}\}$ and   $r_t(z_t)\coloneqq\max_H \mathbb{P}_t(H^\star=H|{\cal D}_t=z_t)$ . 
\end{proposition}
\begin{proof}
First, we optimize over recommendation rules. For any $t\in \mathbb{N}, z_t\in {\cal Z}_t$ define $r_t(z_t)= \max_{H\in {\cal H}} \mathbb{P}_t(H^\star=H|{\cal D}_t=z_t)$ as before. Then, for fixed $(\pi,\tau)$ we have that
\begin{align*}
\inf_{I} V_\lambda(\pi,I)&=\lambda-\lambda\delta +\mathbb{E}^\pi\left[ \inf_{I}\sum_{t\geq 1} \mathbf{1}_{\{t\leq\tau\}}-\lambda\mathbf{1}_{\{\tau=t\}}\mathbb{P}_t(H^\star=I_t({\cal D}_t)|{\cal D}_t)\right],\\
&=\lambda-\lambda\delta +\mathbb{E}^\pi\left[ \sum_{t\geq 1} \mathbf{1}_{\{t\leq\tau\}}-\lambda\mathbf{1}_{\{\tau=t\}}\sup_{I_t}\mathbb{P}_t(I_t({\cal D}_t)=H^\star|{\cal D}_t)\right],\\
&=\lambda-\lambda\delta +\mathbb{E}^\pi\left[ \sum_{t\geq 1} \mathbf{1}_{\{t\leq\tau\}}-\lambda\mathbf{1}_{\{\tau=t\}}r_t({\cal D}_t)\right],
\end{align*}
where the last step follows from \cref{prop:optimal_inference_rule}. Therefore, we can also conclude that
\[
\sup_{\lambda\geq 0} \inf_{\pi,I}V_\lambda(\pi,I)=\sup_{\lambda\geq 0} \inf_{\pi} \lambda-\lambda\delta +\mathbb{E}^\pi\left[\tau-\lambda r_\tau({\cal D}_\tau)\right].
\]
\end{proof}

\subsubsection{Fixed confidence setting: optimal policy}\label{app:theoretical_results:icpe:fixed_confidence:optimal_policy}
 We now optimize over policies. Recall that ${\cal A}$ includes the stopping action, and $\tau=\inf\{t:a_t=a_{\rm stop}\}$. For $t\in \mathbb{N}, z\in {\cal Z}_t$ define the optimal value to go
\[
V_t(z;\lambda) \coloneqq \inf_{\pi:\tau\geq t} \lambda-\lambda\delta +\mathbb{E}^\pi\left[\tau-t-\lambda r_\tau({\cal D_\tau}) |{\cal D}_t=z\right].
\]

Also define the following optimal $Q$-function for $z\in {\cal Z}_t, a \neq a_{\rm stop}$
\[
Q_t(z,a;\lambda)\coloneqq 1+\int_{\cal X} V_{t+1}(\underbrace{z,a,x'}_{=z'};\lambda)\,\bar P_t({\rm d}x'|z,a)
\]
where $\bar P_t$ is the posterior mixture defined in \cref{eq:posterior_mixture}. We also set
\[
Q_t(z,a_{\rm stop};\lambda)\coloneqq \lambda(1-\delta-r_t(z)).
\]
Consider then the  policy $\pi_\lambda^\star=(\pi_t^\star)_t$, where $\pi_t^\star(z;\lambda)\in\argmin_{a\in {\cal A}} Q_t(z,a;\lambda)$, where we break ties according to some fixed ordering over ${\cal A}$.  We have then the following result.
\begin{proposition}
    $\pi_\lambda^\star$ is  a $\lambda$-optimal policy. Furthermore, the optimal value for $z\in {\cal Z}_t$ satisfies 
    \begin{equation}
V_t(z;\lambda)=\min_a  Q_t(z,a;\lambda) .
\end{equation}
\end{proposition}
\begin{proof}
    Fix ${\cal D}_t=z, z\in {\cal Z}_t$. Assume the optimal stopping action stops at $\tau=t$ for such $z$. Then $V_t(z;\lambda)=\lambda-\lambda\delta -\lambda r_t(z)=Q_t(z,a_{\rm stop};\lambda)$.

    Otherwise, assume the optimal stopping rule stops for $\tau >t$. Then
    \begin{align*}
    V_t(z;\lambda) &= \inf_{\pi:\tau> t} \lambda-\lambda\delta +\mathbb{E}^\pi\left[\tau-t-\lambda r_\tau({\cal D}_\tau) |{\cal D}_t=z\right],\\
    &= \inf_{\pi:\tau>t}1+\mathbb{E}_a^\pi\left[\int_{\cal X}V_{t+1}(z, a,x';\lambda)\,\bar P_t({\rm d}x'|z,a)\right],\\
    &= \min_{a\neq a_{\rm stop}} Q_t(z,a;\lambda).
    \end{align*}

    We then clearly obtain the lower bound
    \[
    V_t(z;\lambda)\geq \min\left\{\lambda(1-\delta-r_t(z)), \min_{a\in {\cal A}} Q_t(z,a;\lambda)\right\}.
    \]
    We now show that $\pi^\star$ achieves this value, and thus is optimal. 

    \begin{enumerate}
        \item If $\tau=t$, then $\lambda(1-\delta-r_t(z)\leq \min_{a\neq a_{\rm stop}} Q_t(z,a;\lambda)$ and the value to go for $\pi^\star$ is exactly $\lambda(1-\delta -r_t(z))$. 
        \item if $\tau \neq t$, then $\lambda(1-\delta-r_t(z)\geq\min_{a\neq a_{\rm stop}} Q_t(z,a;\lambda)$, and the value to go is $Q_t(z,\pi_t(z);\lambda)=\min_a Q_t(z,a;\lambda)=V_t(z;\lambda)$.
    \end{enumerate}
    Therefore, the value to go for $\pi^\star$ at time $t$ in ${\cal D}_t=z$ attains the lower bound $\min_{a\in {\cal A}} Q_t(z,a;\lambda)$, 
and is thus optimal. Applying the result from the previous proposition leads to the desired result.
\end{proof}

\subsubsection{Fixed confidence setting: identifiability and correctness}\label{app:theoretical_results:icpe:fixed_confidence:identifiability_and_correctness}
 Lastly, to verify the correctness, we need to make an explicit identifiability assumption.
\begin{assumption}\label{assump:identifiability}
    For every $\delta>0$ there exists a policy $\pi$ with $\mathbb{E}^\pi[\tau]<\infty$, such that $\mathbb{E}^\pi[r_\tau({\cal D}_\tau)]\geq 1-\delta$.
\end{assumption}

 Now we show that the optimization problem solved by ICPE can lead to a $\delta$-correct policy and stopping rule. To that aim, define  
\[{\cal S}(\lambda)\coloneqq\argmin_{\pi} V_\lambda(\pi),\,\hbox{ where }\, V_\lambda(\pi)\coloneqq \lambda(1-\delta)+\mathbb{E}^\pi[\tau-\lambda r_\tau({\cal D}_\tau)].\]
Observe that the  set ${\cal S}(\lambda)$ is not empty since we know that $\pi_\lambda^\star$ belongs to it. We have the following result.
\begin{lemma}\label{lemma:monotonicity_phi}
    Define the set  $\Phi(\lambda)=\{\mathbb{E}^\pi[r_\tau({\cal D}_\tau)]: \pi\in {\cal S}(\lambda)\}$. Then, any $\phi(\lambda)\in \Phi(\lambda)$ is non-decreasing and under \cref{assump:identifiability}  any $\phi(\lambda)\in \Phi(\lambda)$ satisfies $\lim_{\lambda\to\infty} \phi(\lambda)=1$.
\end{lemma}
\begin{proof}
We first prove the limit, and then prove the monotonicity.
\\

\noindent\underline {\it  Step 1: $\lim_{\lambda\to\infty} \phi(\lambda)=1$}.
    For $\epsilon>0$ consider a policy $\pi_\epsilon$ such that  $\mathbb{E}^{\pi_\epsilon}[r_{\tau}({\cal D}_{\tau})]\geq 1-\epsilon$.

    Define $g(\lambda)\coloneqq \inf_{\pi,I}V_\lambda(\pi,I)=\inf_{\pi} V_\lambda(\pi)$.
    
    Now, assume that some feasible minimizer $\pi\in {\cal S}(\lambda)$ satisfies $\mathbb{E}^\pi[r_\tau({\cal D}_\tau)]\leq 1-2\epsilon$. We proceed by contradiction and show that this is not possible. First, note that
    \[
    V_\lambda(\pi_\epsilon)-g(\lambda)=\mathbb{E}^{\pi_\epsilon}[\tau]-\mathbb{E}^{\pi}[\tau] +\lambda(\mathbb{E}^{\pi}[r_\tau({\cal D}_{\tau})]-\mathbb{E}^{\pi_\epsilon}[r_\tau({\cal D}_{\tau})]).
    \]
     Observe then that  $\mathbb{E}^{\pi}[r_\tau({\cal D}_{\tau})]-\mathbb{E}^{\pi_\epsilon}[r_\tau({\cal D}_{\tau})]\leq-\epsilon<0$.
    Therefore, we obtain that whenever $\lambda > \frac{\mathbb{E}^{\pi_\epsilon}[\tau]}{\epsilon}$ we have that
    \[
       V_\lambda(\pi_\epsilon)-g(\lambda)\leq \mathbb{E}^{\pi_\epsilon}[\tau]-\lambda\epsilon<0,
    \]
    which is however a contradiction to $g(\lambda)$ being a minimum. Hence, any feasible solution $\pi\in {\cal S}(\lambda)$ must satisfy 
    $\mathbb{E}^\pi[r_\tau({\cal D}_\tau)]> 1-2\epsilon$ for $\lambda> \mathbb{E}^{\pi_\epsilon}[\tau]/\epsilon$. Since any $\pi\in {\cal S}(\lambda)$ satisfies $\mathbb{E}^\pi[\tau]<\infty$,  we have that for any fixed $\epsilon>0$ we get $\lim_{\lambda\to\infty}\inf_{\pi\in {\cal S}(\lambda)} \mathbb{E}^{\pi}[r_\tau({\cal D}_{\tau})]> 1-2\epsilon$. Since the statement holds for any $\epsilon>0$, letting $\epsilon\to 0$ yields the desired result.
\\

\noindent\underline {\it  Step 2: Monotonicity}.
Consider two feasible optimal solutions $\pi_1\in {\cal S}(\lambda_1)$ and $\pi_2\in {\cal S}(\lambda_2)$, with  $\lambda_2> \lambda_1$. We have that
    \[
    g(\lambda_2)-V_{\lambda_1}(\pi_1)\leq V_{\lambda_2}(\pi_1)-V_{\lambda_1}(\pi_1)=(\lambda_2-\lambda_1)(1-\delta-\mathbb{E}^{\pi_1}[r_{\tau_1}({\cal D}_{\tau_1})])
    \]
    
    and 
    \[
    g(\lambda_1)-V_{\lambda_2}(\pi_2)\leq V_{\lambda_1}(\pi_2)-V_{\lambda_2}(\pi_2)=(\lambda_1-\lambda_2)(1-\delta-\mathbb{E}^{\pi_2}[r_{\tau_2}({\cal D}_{\tau_2})]).
    \]
    Summing up, and using that $g(\lambda_i)=V_{\lambda_i}(\pi_i)$ we have that
    \[
    0 \leq (\lambda_2-\lambda_1)(\mathbb{E}^{\pi_2}[r_{\tau_2}({\cal D}_{\tau_2})]-\mathbb{E}^{\pi_1}[r_{\tau_1}({\cal D}_{\tau_1})]).
    \]
    Since $\lambda_2>\lambda_1$, we must have that $\mathbb{E}^{\pi_2}[r_{\tau_2}({\cal D}_{\tau_2})]-\mathbb{E}^{\pi_1}[r_{\tau_1}({\cal D}_{\tau_1})]\geq 0$. Since we chose the elements in ${\cal S}$ arbitrarily, it implies that any $\phi(\lambda)\in \Phi(\lambda)$ is non-decreasing.
\end{proof}

Lastly, to verify the correctness, we use the fact that the sub-gradient of the optimal value of the dual problem is non-decreasing.
To show this result, we employ the following proposition  from \citep{hantoute2008Characterizations} (see Prop. 3.1 therein), which characterizes the subdifferential of the supremum of a family of affine functions.

\begin{proposition}[Subdifferential of the supremum of affine functions \citep{hantoute2008Characterizations}]\label{prop:subdifferential_affine_functions}
    Given a non-empty set $\{(a_t,b_t):t\in {\cal T}\} \subset \mathbb{R}^2$, and the supremum function $f(x):\mathbb{R}\to\mathbb{R}\cup\{\infty\}$
    \[
    f(x)=\sup \{ a_tx- b_t: t\in {\cal T}\},
    \]
    for every $x\in {\rm dom}f$ we have
    \[
    \partial f(x)= \cap_{\epsilon>0} {\rm cl}\left({\rm conv}\{a_t: t\in {\cal T}_\epsilon(x)\} + B(x)\right)
    \]
    with \[ {\cal T}_\epsilon(x)\coloneqq \{t\in {\cal T}: a_tx -b_t \geq f(x)-\epsilon\},\]
    and
    \[
    B(x)\coloneqq  \left\{y\in \mathbb{R}: (y,yx) \in \left(\overline{\rm conv}\{(a_t,b_t):t\in {\cal T}\}\right)_\infty \right\},
    \]
    where $C_\infty$ is the recession cone of  a set $C$ and
      $\overline{{\rm conv}(\cdot)}$ denotes the closed convex hull of a set.
    In particular, if $x\in {\rm int}({\rm dom} f)$ we have
    \[
    \partial f(x)=\cap_{\epsilon>0} \overline{{\rm conv}}\left\{a_t: t\in {\cal T}_\epsilon(x)\right\}.
    \]
\end{proposition}

This last proposition permits us to define the subdifferential of the supremum of affine functions, and, as we see later, we can also  find a  lower bound on any subdifferential $d\in \partial f(x)$. 
\\

We are now ready to state the identifiability result.

\begin{proposition}Consider  \cref{assump:identifiability}, and, for simplicity, assume the set of optimal policies ${\cal S}(\lambda)$ is a singleton for each $\lambda$. Then, an optimal solution $(\lambda^\star,\pi_{\lambda^\star}^\star)$ satisfies
    \begin{equation}
        \boxed{
        \mathbf{P}^{\pi_{\lambda^\star}^\star}(H_M^\star=\hat H_{\tau_{\lambda^\star}^\star})\geq 1-\delta,
        }
    \end{equation}
    for any critical point $\lambda^\star\in \argmax_{\lambda\geq 0} \inf_{\pi,I} V_\lambda(\pi,I)$. 
\end{proposition}
\begin{proof}
Define $g(\lambda)\coloneqq \inf_{\pi,I}V_\lambda(\pi,I)=\inf_{\pi} V_\lambda(\pi)$. Clearly $V_\lambda$ is differentiable with respect to $\lambda$ for all $\pi$, and we have $\partial V_\lambda(\pi)/\partial\lambda=1-\delta-\mathbb{E}^\pi[r_\tau({\cal D}_\tau)]$.
\\

\noindent\underline{\it Part 1: application of \cref{prop:subdifferential_affine_functions}.} We now derive the subdifferential of $g(\lambda)$ for $\lambda>0$.
For $t=\pi\in {\cal T}$, let $a_t=-(1-\delta-\mathbb{E}^\pi[r_\tau({\cal D}_\tau)])$ and $b_t= \mathbb{E}^\pi[\tau]$. Then
\[
-g(\lambda)=\sup\{a_t\lambda -b_t:t\in {\cal T}\}.
\]
By \cref{prop:subdifferential_affine_functions} it follows that for $\lambda\in \mathbb{R}$
\[
\partial (-g(\lambda))= \cap_{\epsilon>0} \overline{{\rm conv}}\left\{a_t: t\in {\cal T}_\epsilon(\lambda)\right\},\quad {\cal T}_\epsilon(\lambda )\coloneqq \{t\in {\cal T}: a_t\lambda -b_t \geq -g(\lambda)-\epsilon\},
\]
where $\partial (-g(\lambda))$ is the  subdifferential of $-g$.

Using that ${\cal S}(\lambda)\subseteq {\cal T}_\epsilon(\lambda)$ for all $\epsilon\geq 0$, we conclude  that for any $d\in \partial(-g(\lambda))$ we have 
\[
-d \geq \inf_{\pi\in{\cal S}(\lambda)} 1-\delta-\mathbb{E}^{\pi}[r_\tau({\cal D}_\tau)]= 1-\delta-\mathbb{E}^{\pi_\lambda^\star}[r_\tau({\cal D}_\tau)].
\]
Defining $\phi(\lambda)=\mathbb{E}^{\pi_\lambda^\star}[r_\tau({\cal D}_\tau)]$, we note that $\phi(\lambda)\in \Phi(\lambda)$.
\\

Next, consider the case $\lambda=0$. From \cref{prop:subdifferential_affine_functions}, we have
\[
 B(0)\coloneqq  \left\{y\in \mathbb{R}: (y,0) \in \left(\overline{\rm conv}\{(a_t,b_t):t\in {\cal T}\}\right)_\infty \right\}.
\]
Let $C=\overline{\rm conv}\{(a_t,b_t):t\in {\cal T}\}$.
For  a nonempty closed convex set $C\subset\mathbb{R}^2$  the recession cone is defined as $C_\infty=\{y\in \mathbb{R}^2|\forall x\in C,\forall t\geq 0: x+yt \in C\}$. By contradiction, assume $(y,0) \in C_\infty$, then for any $(a,b)\in C ,t\geq 0$ we have $(a+yt,b)\in C$. However, $a_t\in [-1+\delta,\delta]$ for all $t\in {\cal T}$, bounded. Hence, there exists $t>0$ such that $a+yt\notin [-1+\delta,\delta]$, which is a contradiction. Since $y$ is arbitrary, only the $0$ element satisfies the condition, and thus  $B(0)=\{0\}$. Therefore the set of subdifferentials in $0$ is simply given by  $\partial (-g(0))=\cap_{\epsilon>0} \overline{{\rm conv}}\left\{a_t: t\in {\cal T}_\epsilon(0)\right\}$.
\\

\noindent\underline{\it Part 2: critical points.}  Define the following value:
\[
    \bar \lambda \coloneqq \inf\{\lambda\geq 0: \phi(\lambda)\geq 1-\delta\}.
    \]

By \cref{lemma:monotonicity_phi}, since $\phi(\lambda)\in \Phi(\lambda)$  we know that $\bar\lambda<\infty$. Then,
  for any $0\leq \lambda<\bar\lambda$ we have that any $d\in \partial (-g(\lambda))$ satisfies 
  \[
  -d\geq 1-\delta-\phi^+(\lambda)>0
  \]
  hence $-d>0$   for $0<\lambda<\bar\lambda$. Since $-d$ is a superdifferential (we are maximizing $g$!), any critical solution   $\lambda^\star\in \argmax g(\lambda)$ satisfies $\lambda^\star \in [\bar \lambda,\infty)$. Furthermore, such critical point exists: as $\lambda\to\infty$, the differential $-d$ becomes negative (since $\phi^+(\lambda)\to1$ by \cref{lemma:monotonicity_phi}), implying that  $g(\lambda)$ decreases. Hence, the maximum is attained in $\lambda^\star\in [0,\infty)$.

  Then, since $\lambda^\star \in [\bar \lambda,\infty)$, we have that \[ 1-\delta\leq \phi(\lambda^\star)= \mathbb{E}^{\pi_{\lambda^\star}^\star}[r_{\tau_{\lambda^\star}^\star}({\cal D}_{\tau_{\lambda^\star}^\star})]=\mathbb{E}^{\pi_{\lambda^\star}^\star}[\mathbf{P}_{\tau_{\lambda^\star}^\star}(H_M^\star=\hat H_{\tau_{\lambda^\star}^\star}|{\cal D}_{\tau_{\lambda^\star}^\star})]=\mathbf{P}^{\pi_{\lambda^\star}^\star}(H_M^\star=\hat H_{\tau_{\lambda^\star}^\star}).\]

\end{proof}

\paragraph{Training-time certification and stopping.} To obtain formal guarantees on correctness, note that sequentially testing the accuracy $\hat p$ during training, where
\[
\hat p=\frac{1}{K}\sum_{i=1}^K \mathbf{1}_{\{H_i^\star=\argmax_H I_\phi(H|{\cal D}_\tau^{(i)})\}},
\]
may not imply  $\delta$-correctness, unless we adopt the correct sequential test.
Alternatively, one can simply avoid to sequentially test the accuracy of the model, and simply stop training at a fixed number of epochs $T_E$, where $T_E$ is fixed a priori. Then, the user can test the model $(\theta_{T_E},\phi_{T_E})$ on a number of i.i.d. trajectories to evaluate  a lower bound on the accuracy of the model (e.g., through a simple Hoeffding bound).

On the other hand, if we want to stop training as soon as the model is $\delta$-correct, then we should employ a sequential testing procedure to decide when to stop.
To that aim, we need to introduce an additional confidence $\delta'\in (0,1/2)$. This value becomes the desired correctness of the method, while $\delta$ is chosen to satisfy $\delta<\delta'$, with $\delta'-\delta$ sufficiently large. The reason is simple: by forcing the model to be more accurate, it becomes easier (for the test that we use) to detect that the accuracy crossed the threshold $1-\delta'$.

We  employ the following procedure.
\begin{itemize}
    \item At epoch $t=1,2,\dots$ we evaluate  $(\theta_t,\phi_t,\lambda_t)$ on $K$ i.i.d. rollouts (independent of the training updates at epoch $t$, and sampled on $K$ different environments $M_i\sim \P$).
    \begin{itemize}
        \item  For each $n\in\{1,\dots,K\}$ let $Z_{t,n}\in\{0,1\}$ indicate whether the returned hypothesis on that rollout equals $H_i^\star$ on the $i$-th environment, and set $X_t=\frac{1}{K}\sum_{n=1}^K Z_{t,n}$ with conditional mean $p_t\coloneqq\mathbb E[Z_{t,1}\mid{\cal F}_{t-1}]$.
    \end{itemize}
    \item We adopt the rule: fix $\eta\in(0,1)$ and stop at the first epoch
\[
T\ \coloneqq\ \inf\Big\{t\ge1:\ \frac1t\sum_{s=1}^t X_s\ \ge\ (1-\delta')\ +\ \frac1t\sqrt{\,2\!\left(1+\tfrac{1-\delta'}{B}\,t\right)\ln\!\Big(\tfrac{\sqrt{1+\frac{1-\delta'}{B}t}}{\eta}\Big)}\ \Big\},
\]
then \emph{freeze} the parameters and return $(\theta_T,\phi_T,\lambda_T)$. The proposition below (an anytime bound via a mixture-martingale) guarantees
\[
\mathbb P\Big(\exists t:\ \tfrac1t\sum_{s=1}^t X_s\ \text{crosses the boundary}\ \Big|\ \sup_{t\ge1}p_t\le 1-\delta'\Big)\ \le\ \eta,
\]
so, with probability at least $1-\eta$, we only stop when the global null “$p_t\le 1-\delta'$ for all epochs” is false, i.e., there exists some $s\le T$ with $p_s>1-\delta'$. If, in addition, the epochwise performance is nondecreasing ($p_1\le p_2\le\cdots$, a property that typically arises when the method converges), then $p_T\ge 1-\delta'$, and the returned model is $\delta'$-correct with confidence $1-\eta$.
\end{itemize}
 
\begin{proposition}[Training correctness]\label{prop:freeze_at_hit_monotone}
Let $({\cal F}_t)_{t\ge1}$ be the training filtration, with ${\cal F}_t=\sigma(x_1,a_1,x_2,\dots,a_{t-1},x_t)$. For each epoch $t$ let $Z_{t,1},\dots,Z_{t,K}$ be conditionally i.i.d. ${\rm Ber}(p_t)$ given ${\cal F}_{t-1}$, with $X_t\coloneqq K^{-1}\sum_{n=1}^K Z_{t,n}$ and $p_t\coloneqq\mathbb E[Z_{t,1}|{\cal F}_{t-1}]$. For $\eta\in (0,1)$, define the stopping time
\[
T\ \coloneqq\ \inf\!\left\{t\ge1:\ \frac1t\sum_{s=1}^t X_s\ \ge\ (1-\delta')\ +\ \frac1t\sqrt{\,2\!\left(1+\tfrac{1-\delta'}{K}\,t\right)\ln\!\Big(\tfrac{\sqrt{1+\frac{1-\delta'}{K}t}}{\eta}\Big)}\ \right\}.
\]
Assume further that with probability at least $1-\xi$ there exists a (finite, ${\cal F}_t$-stopping) time $t_0$ such that 
\begin{equation}\label{eq:eventual_monotone}
p_{t+1}\ge p_t\quad\forall\,t\geq t_0, \;\hbox{ and }\; \sup_{t\geq t_0} p_t\geq 1-\delta.
\end{equation}
Then
\[
\mathbb P\big(T<\infty,\,  p_T\ge 1-\delta'\big) \geq \mathbb P(T<\infty)- (\eta+\xi).
\]
\end{proposition}
\begin{proof}
Let $\mathcal E$ be the event in \eqref{eq:eventual_monotone}. The idea is to construct an event ${\cal G}$ such that $\{T<\infty\}\cap {\cal E}\cap {\cal G} \subseteq \{T<\infty,\, p_T\geq 1-\delta'\}$.
On ${\cal E}$, define the stopping time $S=\inf\{t\geq t_0: p_t\geq 1-\delta'\}$.

We let ${\cal G}=\{T\geq S\}$. Clearly, by \cref{prop:stopping_rule_training} we have that on ${\cal E}$ the event ${\cal G}$ happens with probability at-least $1-\eta$. Therefore, on $\{T<\infty\}\cap {\cal E}\cap {\cal G} $ we have that $T<\infty$ and $p_T\geq p_S\geq 1-\delta'$, hence $\{T<\infty\}\cap {\cal E}\cap {\cal G} \subseteq \{T<\infty,\,p_T\geq 1-\delta'\}$.

Let $A=\{T<\infty\}$. Using the following decomposition of $A$ in disjoint regions
\[
A=(A\cap{\cal E}\cap{\cal G})\cup (A\cap{\cal E}\cap{\cal G}^c)\cup (A\cap{\cal E}^c),
\]
we obtain 
\begin{align*}
\mathbb P(T<\infty,\,p_T\geq 1-\delta')&\geq \Pr(\{T<\infty\}\cap {\cal E}\cap {\cal G}),\\
&=\mathbb P(\{T<\infty\})-\mathbb P(\{T<\infty\}\cap {\cal E}\cap{\cal G}^c)-\mathbb P(\{T<\infty\}\cap{\cal E}^c),\\
&\geq\mathbb P(\{T<\infty\})-\mathbb P( {\cal E}\cap{\cal G}^c)-\mathbb P({\cal E}^c),\\
&=\mathbb P(\{T<\infty\})-\eta-\xi.
\end{align*}
\end{proof}

\noindent\emph{Remark.}
Condition \eqref{eq:eventual_monotone} is one natural way to formalize “monotone convergence from some epoch $t_0$ with high probability.”   Under \eqref{eq:eventual_monotone} the first epoch $S$ with $p_S\ge 1-\delta'$ exists a.s., and the anytime validity ensures we do not stop before $S$ except with probability at most $\eta$. Hence, upon stopping, the returned snapshot is $\delta'$-correct with probability at least $1-\eta-\xi$.
We also note that the event ${\cal E}$ is  a consequence of \cref{lemma:monotonicity_phi}, from the monotonicity of $\mathbb{E}^\pi[r_\tau({\cal D}_\tau)]$ in $\lambda$.

Lastly, note that the test that we use considers the average over epochs of $X_n$. If $\delta'=\delta$,  this average may take a long time to converge to $1-\delta$, and even to cross the threshold. Hence, we practically run the algorithm with confidence $\delta$, with $\delta<\delta'$ (where $\delta'$ is the desired accuracy), so that $(1/t) \sum_n X_n$ converges to $1-\delta>1-\delta'$ (and this fact can help the test trigger earlier).

Lastly, we prove an anytime bound via a
mixture-martingale on the repeated tests on $p_t$.

\begin{proposition}\label{prop:stopping_rule_training}
For all $t\geq 1, B\in \mathbb{N}$, let $X_t=\frac{1}{B}\sum_{n=1}^B Z_{t,n}$, where, for each $t$, $(Z_{t,n})_{n=1}^B$ are conditionally i.i.d. Bernoulli random variables with mean $p_t$ given ${\cal F}_{t-1}$, where ${\cal F}_t=\sigma(x_1,a_1,x_2,\dots,a_{t-1},x_t)$.  Assume that $\sup_{t\geq 1}p_t\leq 1-\delta'$. Then, for all $\eta\in (0,1)$ we have 
\[
    \mathbb{P}\left(\exists t\geq 1: \frac{1}{t}\sum_{n=1}^tX_n\geq  (1-\delta')+\frac{1}{t}\sqrt{2\left(1+\frac{1-\delta'}{B}t\right)\ln\left(\frac{\sqrt{1+\frac{1-\delta'}{B}t}}{\eta}\right)} \right) \leq \eta .
    \]
\end{proposition}
\begin{proof}
Let $S_t=\sum_{i=1}^t X_i-p_i$.

    For any $\lambda\geq 0, \alpha>0$, let $\phi_t(\lambda) = \frac{\alpha B}{p_t(1-p_t)}\ln\mathbb{E}[e^{\frac{\lambda}{B}(Z-p_t)}|{\cal F}_{t-1}]$ be the (normalized) CGF of $Z \sim {\rm Ber}(p_t)$. Define $V_t=\frac{p_t(1-p_t)}{\alpha} $ be a measure of variance. Then, for $M_t(\lambda)=\exp\left(\lambda S_t - \sum_{i=1}^t \phi_i(\lambda) V_i\right)$ we get that 
    \begin{align*}
        \mathbb{E}\left[M_t(\lambda) |{\cal F}_{t-1}\right]&= \mathbb{E}\left[\exp\left(\sum_{i=1}^t \lambda (X_i-p_i) - \phi_i(\lambda) V_i\right) |{\cal F}_{t-1}\right],\\
&=M_{t-1}(\lambda)\mathbb{E}\left[\exp\left(\lambda (X_t -p_t) - \phi_t(\lambda) V_t\right) |{\cal F}_{t-1}\right],\\
&=M_{t-1}(\lambda)\mathbb{E}\left[\exp\left(\frac{\lambda(\sum_{n=1}^B Z_{t,n}-p_t)}{B}-B\ln\mathbb{E}[e^{\frac{\lambda}{B}(Z-p_t)}|{\cal F}_{t-1}]\right) |{\cal F}_{t-1}\right],\\
&=M_{t-1}(\lambda)\frac{\mathbb{E}\left[\exp\left(\frac{\lambda(\sum_{n=1}^B Z_{t,n}-p_t)}{B} \right)|{\cal F}_{t-1}\right]}{\mathbb{E}[\exp\left(\frac{\lambda}{B}(Z-p_t)\right)|{\cal F}_{t-1}]^B} ,\\
&=M_{t-1}(\lambda)\frac{\mathbb{E}\left[\exp\left(\frac{\lambda}{B} (Z -p_t)\right)|{\cal F}_{t-1}\right]^B}{\mathbb{E}[\exp\left(\frac{\lambda}{B}(Z-p_t)\right)|{\cal F}_{t-1}]^B}=M_{t-1}(\lambda).
    \end{align*}

    Since $M_t\geq 0, \lambda \geq 0$, we have that $M_t(\lambda)$ is a non-negative martingale (hence, also a super-martingale).

We use the method of mixtures to integrate $M_t(\lambda)$ over a prior over $\lambda$.  To do so, we need to find an appropriate lower  bound on $M_t$.
Consider then $\phi_t(\lambda)$: we can use the fat that  $\phi_t(\lambda)\leq \alpha \lambda^2/(2B)$ from the sub-gaussianity of $Z$. 
Then,  choose a prior $\pi({\rm d}\lambda)= \sqrt{2/\pi} e^{-\lambda^2/2} {\rm d}\lambda$ (a half normal). We obtain
\begin{align*}
\int_0^{\infty} M_t(\lambda) \,\pi({\rm d}\lambda)&=\int_0^{\infty} \exp\left(\lambda S_t  -\sum_{i=1}^t \phi_i(\lambda) V_i\right)\,\pi({\rm d}\lambda),\\
&\geq \int_0^{\infty} \exp\left(\lambda S_t  -\sum_{i=1}^t \frac{\alpha \lambda ^2}{2B} \frac{1-\delta'}{\alpha}\right)\,\pi({\rm d}\lambda), \tag{$V_i\leq p_t/\alpha \leq (1-\delta')/\alpha$}\\
&= \sqrt{2/\pi}\int_0^{\infty } \exp\left(\lambda S_t  - \frac{\lambda^2}{2}\left[1+\frac{(1-\delta')}{B} t\right]\right){\rm d}\lambda.
\end{align*}
Since the Gaussian integral satisfies
\[
\int_0^\infty e^{-a\lambda^2 + b\lambda}\,{\rm d}\lambda=e^{\frac{b^2}{4a}}\int_0^\infty e^{-a\left(\lambda-\frac{b}{2a}\right)^2}\,{\rm d}\lambda\geq e^{\frac{b^2}{4a}}\int_{0}^\infty e^{-ax^2}\,{\rm d}x = e^{\frac{b^2}{4a}}\frac{1}{2} \sqrt{\frac{\pi}{a}},
\]
 for $v_t=(1-\delta')t /B$ we can lower bound the integral over $M_t(\lambda)$ as
\[
\int_0^{\infty} M_t(\lambda) \,\pi({\rm d}\lambda)\geq  \frac{1}{\sqrt{1+v_t}}e^{\frac{S_t^2}{2(1+v_t)}}.
\]

Therefore, by Ville's inequality we obtain
    \begin{align*}
    \mathbb{P}\left(\exists t\geq 1: \int_0^\infty M_t(\lambda) \,\pi({\rm d}\lambda)\geq \frac{1}{\eta}\right) 
    \leq \eta.
    \end{align*}

    Therefore, with probability $1-\eta$ for all $t\geq 1$ we have
    \[
    \frac{1}{\sqrt{1+v_t}}e^{\frac{S_t^2}{2(1+v_t)}} < \frac{1}{\eta} \Rightarrow S_t < \sqrt{2(1+v_t)\ln\left(\frac{\sqrt{1+v_t}}{\eta}\right)}
    \]
    Since $S_t\geq \sum_{i=1}^t X_i -(1-\delta')$, we  obtain the desired result.
    \end{proof}

\newpage
\subsection{Meta-training: Finite Sample Analysis}\label{app:meta_training:finite_sample_analysis}

\begin{algorithm}[H]
\caption{Finite-budget idealized  ICPE training}
\textbf{Inputs:} Value function space ${\cal F}$, posterior rewards $ r_N$, reference measure $\mu$.\\
\textbf{Init:} choose $Q^{(0)}\in\mathcal F$; for $k\ge0$ set $\pi^{(k+1)}=\mathcal G(Q^{(k)})$ with $\pi_t^{(k+1)}(z)\in\arg\max_{a}Q_t^{(k)}(z,a)$.
\begin{algorithmic}[1]
\For{$k=0,\dots$}
   \State \textbf{Sampling:} draw a batch $B_k=\{(z^{(i)},a^{(i)},t^{(i)})\}_{i=1}^B$ i.i.d.\ from $\mu$.
  \For{each $(z,a,t)\in B_k$}
      \State sample $x'\sim \bar P_t(\cdot\mid z,a)$ and let $z'=(z,a,x')$.
      \State set targets:
      \[
      \hat Q_{t}^{(k+1)}(z,a)\gets
      \begin{cases}
      Q_{t+1}^{(k)}(z',\,\pi_{t+1}^{(k+1)}(z')\big), & t<N,\\
       r_N(z'), & t=N.
      \end{cases}
      \]
   \EndFor
  
    \State \textbf{Regression:} for each $t=1,\dots,N$, fit
   \[
   Q_t^{(k+1)}\in\arg\min_{Q\in{\cal F}_t} \widehat{\mathcal L}_{t}(Q,\hat Q^{(k+1)}; B_{k,t}).
   \]
\EndFor
\end{algorithmic}
\end{algorithm}

We work in the Bayes/history MDP induced by the prior over environments. Let $\{\mathcal Z_t\}_{t=1}^N$ be the history spaces (as in ICPE), with $z\in\mathcal Z_t$ encoding the full trajectory prefix up to stage $t$. The terminal (posterior) reward is
\[
r_N(z)=\max_{H\in\mathcal H}\, \mathbb P_N\big(H^\star=H|{\cal D}_N=z),
\]
with $z\in {\cal Z}_N$.

\paragraph{Reference sampling law.}
Let $\mu$ be a probability distribution on triples $(z,a,t)\in\bigcup_{t=1}^N(\mathcal Z_t\times\mathcal A\times\{t\})$, with stage marginals $\mu_t$ on $\mathcal Z_t\times\mathcal A$.
During training, all regression samples are drawn i.i.d. from $\mu$: this measure represents sampling from idealized replay buffer. 
\\

The next sample is then sampled according to $x'\sim\bar P_t(\cdot\mid z,a)$, where 
\begin{equation}
\bar P_t(x'\in X|z_t,a) \coloneqq  \int_{{\cal M}^\sharp} P_t(X|z_t,a) \, R_t({\rm d} M|z_t),
\end{equation}
so that the next history is $z_{t+1}=(z_t,a,x')$.

\paragraph{Function class and Stage-wise Bellman operators.}
We let $\mathcal F\subset\prod_{t=1}^N\{\mathcal Z_t\times\mathcal A\to[0,1]\}$ be the $Q$-function class.
For $Q=(Q_t)_{t=1}^N\in\mathcal F$, define the stagewise greedy policy
\[
\pi_{t}= {\cal G}(Q_t)\in\arg\max_{a\in\mathcal A} Q_t(\cdot,a)\qquad(t=1,\dots,N).
\]

For a nonstationary policy $\pi=(\pi_1,\dots,\pi_N)$ and a $Q$-array $Q=(Q_1,\dots,Q_N)$, define for $t=1,\dots,N, z\in {\cal Z}_t$ the operator
\[
\left[ \Gamma_t^\pi Q\right](z,a)\coloneqq \mathbb E_{x'\sim\bar P_t(\cdot| z,a)}\left[Q_{t+1}\left(z', \pi_{t+1}(z')\right)\right],\qquad 
z'=(z,a,x').
\]
At the last stage, with terminal posterior reward $ r_N$ 
\[
\left[ \Gamma_N^\pi Q\right](z,a)\coloneqq   r_N\left(z\right), \quad\forall a.
\]
We also define the optimal operator 
\[ \Gamma _t^\star Q(z,a)\coloneqq \mathbb E\left[\mathbf 1_{\{t<N\}}\max_{a'} Q_{t+1}(z',a')+\mathbf 1_{\{t=N\}} r_N(z)\right].
\]
In the following we write $ \Gamma ^\pi Q=( \Gamma _1^\pi Q,\dots, \Gamma _N^\pi Q)$ and similarly for $ \Gamma ^\star,{\cal G}$.

Given a policy $\pi$ we also indicate by $Q_t^\pi(z,a)$ the true  $Q$-value of $\pi$ at $(z,a,t)$. Similarly,  we define the value as $V_t^\pi(z)=Q_t^\pi(z,\pi_t(z))$. We similarly define the optimal value $V_t^\star$.

\paragraph{Concentrability (w.r.t.\ $\mu$).}
Let $\nu_t^{\pi}$ be the occupancy measure on $(z_t)$ at stage $t$ under policy $\pi$, when the initial history is sampled from the prior-induced initial distribution $\rho$ (where $\rho$ is the initial observation distribution in $M$), that is
\[
\nu_t^\pi(\cdot)\coloneqq  \mathbb E_{M\sim\P}^\pi\left[ \rho P_{1}\cdots P_{t-1}(\cdot) \right].
\]
with $\nu_1(\cdot)=\mathbb{E}_{M\sim \P}[\rho(\cdot)]$.
\begin{assumption}
    For all $t=1,\dots, N$ we assume that $\nu_t^\pi \ll \mu_t^Z$, where $\mu_t^Z$ is the marginal of $\mu$ on ${\cal Z}_t$.
\end{assumption}

Define then the concentrability coefficients
\[
c_\infty(t)\coloneqq \sup_{\pi}\left\|\frac{\mathrm d\nu_t^{\pi}}{\mathrm d\mu_t^Z}\right\|_\infty.
\]
Recall also assumption \ref{assump:domination_param}, which states  that there exist probability measures $\lambda_0,\lambda$ on $({\cal X},{\cal B}({\cal X}))$ such that, for all $(\rho,P)\in{\cal M}^\sharp$, $s\in\mathbb N$, and $(z,a)\in{\cal Z}_s\times{\cal A}$,
\[
\rho(\cdot)\ll \lambda_0(\cdot)
\quad\text{and}\quad
P_s(\cdot|z,a)\ll \lambda(\cdot).
\]
We make the following additional assumption.
\begin{assumption}
For all $(\rho,P)\in M^\sharp, s\in \mathbb{N}$ and $(z,a)\in {\cal Z}_s\times A$ we assume that ${\rm d}\rho_M(\cdot)/{\rm d}\lambda_0 (\cdot)$ and ${\rm d}P_s(\cdot|z,a)/{\rm d}\lambda(\cdot)$ are upper semicontinuous.
\end{assumption}
Hence, by compactness and upper semicontinuity there exist $L_0,L_1$ such that 
\[
\sup_{\rho\in M^\sharp}\sup_{x\in {\cal X}} \frac{\mathrm d\rho}{\mathrm d\lambda_0}(x) \leq L_0,\quad \max_{t=1,\dots,N}\sup_{P\in M^\sharp}\sup_{x\in {\cal X}, z\in {\cal Z}_t, a\in {\cal A}} \frac{\mathrm dP_t(\cdot|z,a)}{\mathrm d\lambda}(x) \leq L_1.
 \]
Consequently, one can bound $c_\infty$ as follows
\[
c_\infty(t) \leq L_0L_1^{t}.
\]

\paragraph{Function class and losses.}
Let $\mathcal F_t$ be a hypothesis class for $Q_t$. We indicate by $B_k\subset(\mathcal Z\times\mathcal A\times[N])$ a batch of samples, and by $B_{k,t}=\{(z,a,s)\in B_k: s=t\}$. Hence, for a batch $B_{k}$ with targets $\hat Q^{(k+1)}$, define the empirical squared loss
\[
\widehat{\mathcal L}_t(Q, \hat Q^{(k+1)}; B_{k,t})\coloneqq \frac{1}{| B_{k,t}|}\sum_{(z_t,a)\in B_{k,t}}\left(Q_t(z_t,a)- \hat Q_t^{(k+1)}(z_t,a)\right)^2,
\]
and the Monte Carlo targets are
\[
\hat Q_{t}^{(k+1)}(z,a)=
\begin{cases}
Q_{t+1}^{(k)}(z',\pi_{t+1}^{(k+1)}(z')), & t<N,\\
 r_N(z'), & t=N,
\end{cases}
\quad z'=(z,a,x'),\;\; x'\sim\bar P_t(\cdot| z,a).
\]
with $\pi^{(k+1)}=  {\cal G}(Q^{(k)})$.
We also define  the true loss 
\[
 {\mathcal L}(Q^{(k)}, Q^{(k-1)})\coloneqq  \mathbb E_{(z,a,t)\sim\mu}\left[ \left( \Gamma_t^{\pi^{(k)}}Q^{(k-1)}(z,a)-Q_t^{(k)}(z,a)\right)^2\right],
\]
and $ {\mathcal L}_{t}(Q^{(k)}, Q^{(k-1)})\coloneqq  \mathbb E_{(z,a)\sim\mu_t}\left[ \left( \Gamma_t^{\pi^{(k)}}Q^{(k-1)}(z,a)-Q_t^{(k)}(z,a)\right)^2\right]$. In the following, for simplicity, we also write ${\cal L}_{k,t}\coloneqq {\mathcal L}_{t}(Q^{(k)}, Q^{(k-1)}).$
\\

In each epoch $k$ a regression problem is solved, where the training set $\{(z^{(i)}, a^{(i)}, t^{(i)}, \hat Q^{(k+1)})\}$ and $\hat Q_{t^{(i)}}^{(k+1)}(z^{(i)}, a^{(i)})$ is an unbiased estimate of the target defined by $ \Gamma_t Q$.

\subsubsection{Main results} The main results are the following ones.

\paragraph{Error propagation.}
We first obtain a result on the error propagation that bounds the sub-optimality of the policy at training epoch $k$. This result holds for a general function space ${\cal F}=({\cal F}_t)_{t=1}^N$. In the following, we denote the overlal value of a policy $\pi$ by $J(\pi)=\mathbb{E}_{\P}^\pi[r_N(z_N)]$ and define $\pi^\star \in \argsup_\pi J(\pi)$. 

\begin{theorem}[Sub-optimality of policy $\pi^{(k)}$]\label{thm:value_bound}
Let $J(\pi)=\mathbb{E}_{\P}^\pi[r_N(z_N)]$ and $\pi^\star \in \argsup_\pi J(\pi)$.
For $k\geq N+1$,
    we have that
\[
|J(\pi^\star)-J(\pi^{(k)})|\leq \|w\|_2 \left[\sqrt{S_{k-1}^{(1,N)}}+  2\sqrt{(N+1) \sum_{u=k-N}^k S_u^{(2,N)}}+\sqrt{D_k^{(1,N)}}\right]
\]
where $w=(w_u)_{u=1}^N$ is the vector of concentrability coefficients, with  $w_u\coloneqq c_\infty(u)\kappa_u$; $S_m^{(a,b)}= \sum_{u=a}^b {{\cal L}_{m,u}}$ is the sum of losses for epoch $m$ along the timesteps $(a,a+1,\dots,b)$; $D_m^{(a,b)}=\sum_{u=a}^b {\cal L}_{m-u,u}$ is  the diagonal sum of losses.
\end{theorem}

\paragraph{Finite-sample performance bound.}
We now show how the losses that appear in the previous result can be bounded to derive a  finite-sample performance bounds.

To approximate the target, for each $t=1,\dots, N$ we consider a linear function space ${\cal F}_t$ of dimension $d_t$ with bounded basis function $\{\varphi_{t,i}\}_{i=1}^{d_t}$ $\|\varphi_{t,i}\|_\infty \leq C_b$. For each $t$ we consider a linear family with parameter $\alpha_t\in \mathbb{R}^{d_t}$  and features $\phi_t:{\cal Z}_t\times {\cal A}\to\mathbb{R}^{d_t}$, thus ${\cal F}_t=\{(z,a)\mapsto \phi_t(z,a)^\top \alpha_t: \alpha_t\in \mathbb{R}^{d_t}\}$.

At epoch  $k$  regression returns a linear predictor  $\tilde Q_t^{(k)}$. We then define the  $Q$-function used by the algorithm as the truncation $Q_t^{(k)}=\mathbb{T}(\tilde Q_t^{(k)})$.
In the analysis,  $Q_t^{(k)}$
 always denotes this truncated version.
\begin{theorem}[Fixed-budget finite-sample training error]\label{thm:fixed_budget_finite_sample_analysis}
Fix $\delta\in(0,1)$ and choose $\delta'=\delta/(4kN)$. Suppose (i) the features are bounded, $\sup_{z,a}\|\phi_t(z,a)\|_2\le C_b$; (ii) concentrability holds with coefficients $c_\infty(t)$ and $\kappa_t$; and (iii) the batch size satisfies
\[
B\geq \frac{2}{p_{\min}\eta^2}\log\frac{4kN}{\delta}.
\]
for some $\eta\in (0,1)$.
Then, for $k\geq N+1$, with probability at least $1-\delta$,
\[
|J(\pi^\star)-J(\pi^{(k)})|\leq O\left(  NC _0 \left[\sqrt{\sum_{t=1}^N\beta_t^2}+ \sqrt{\sum_{t=1}^N\frac{ d_t}{(1-\eta)p_{\rm min}B} \log\frac{4kN}{\delta}}\right]\right),
\]
where $\beta_t=\sup_{Q\in\mathcal F,\pi}\inf_{f\in\mathcal F_t}\|\Gamma_t^\pi Q-f\|_{\mu_t}$ and $C_0=\big(\sum_{t=1}^N c_\infty(t)^2\kappa_t^2\big)^{1/2}$, and $p_{\rm min}=\min_t\mu(t)$, where $\mu(t)$ is the marginal over timesteps of the buffer distribution.
\end{theorem}
\paragraph{Intuition.} \cref{thm:value_bound} shows that the performance gap $J(\pi^\star)-J(\pi^{(k)})$ is controlled by how well each  step approximates the Bellman update: the terms $S_{k-1}^{(1,N)}$, $\sum_{u=k-N}^k S_u^{(2,N)}$, and $D_k^{(1,N)}$ aggregate the single–step squared Bellman residuals ${\cal L}_{k,t}$ across time and across a window of epochs, and the concentrability vector $w=(c_\infty(t)\cdot\kappa_t)_{t=1}^N$ measures how much these local errors can be amplified when propagated along the trajectory distribution. The finite-sample bound in \cref{thm:fixed_budget_finite_sample_analysis} then replaces these abstract residuals with explicit statistical quantities: each ${\cal L}_{k,t}$ is bounded by an \emph{approximation} term $\beta_t$ (how well the function class can represent an exact  update) plus an \emph{estimation} term that decays as $\sqrt{d_t/((1-\eta)p_{\min}B)}$. In other words, the final rate cleanly separates an approximation error, captured by $\sqrt{\sum_t \beta_t^2}$, from a sample error, captured by $\sqrt{\sum_t d_t/((1-\eta)p_{\min}B)}$, and both are scaled by the horizon $N$ and the concentrability constant $C_0$, which quantify how errors accumulate along the history MDP.
\subsubsection{Convergence Analysis: Proof of \cref{thm:value_bound}}
To prove \cref{thm:value_bound}, we follow an analysis similar to the one in  \citep{scherrer2012approximate}. However, note that their setting is quite different from ours:  we do not have the classical discounted Bellman operator, and as a consequence the proofs are different.

We begin by defining the following key quantities :
\begin{enumerate}
    \item At iteration $k$ we indicate by $\pi_t^{(k)}={\cal G}(Q_t^{(k-1)})$ the greedy policy.
    \item The one-step evaluation $Q_t^{(k)}=\Gamma_t^{\pi^{(k)}} Q^{(k-1)}+\epsilon_t^{(k)}$,  and $\epsilon_t^{(k)}$ is the regression error and $Q_t^{(k)}$ is computed according to $ Q_t^{(k)}\in\arg\min_{Q\in{\cal F}_t} \widehat{\mathcal L}_{t}(Q,\hat Q^{(k)})$ for all $t=1,\dots, N$. We also write 
    $ V_t^{(k)}(z)=[\Gamma_t^{\pi^{(k)}} Q^{(k-1)}](z,\pi_t^{(k)}(z))$.
    \item We define $ V_t^{\pi^{(k)}}$ to be the value of $\pi^{(k)}$ under $ \Gamma $, that is, the true value of $\pi^{(k)}$ with rewards $ r$. Similarly, we define $ Q_{t}^{\pi^{(k)}}$ to be the $Q$-value.
    \item The Bellman residual w.r.t. the next greedy policy: $b_t^{(k)}=Q_t^{(k)}- \Gamma_t^{\pi^{(k+1)}} Q^{(k)}$
    \item The performance gap $\ell_t^{(k)}=V_t^\star-  V_t^{\pi^{(k)}}\geq 0$. 
    \item Distance before approximation: $d_t^{(k)}=V_t^\star-V_t^{(k)}$.
    \item The shift: $s_t^{(k)}=V_t^{(k)}-  V_t^{\pi^{(k)}}$.
\end{enumerate}
Therefore $\ell_t^{(k)}= s_t^{(k)}+d_t^{(k)}$: this is the quantity we wish to bound for $t=1$. The proof of \cref{thm:value_bound} is based on bounding $s_t$ and $d_t$ separately. We begin by proving a lemma that we use repeatedly in all of the proofs.

\begin{lemma}\label{lemma:bound_epsilon}
    Let $\kappa_t\coloneqq \sqrt{\operatorname{ess sup}_z \max_a \frac{1}{\mu_t(a|z)}}$. Let $\mu_t^Z$ be the marginal of $\mu_t$ on ${\cal Z}_t$. Then, for any $t$, measurable function $f_t:{\cal Z}_t\to {\cal A}$,  we have that
    \[
     \mathbb{E}_{z\sim\mu_t^Z}[|\epsilon_{t}^{(k)}(z,f_t(z))|] \leq \kappa_t \sqrt{{\mathcal L}_{t}(Q^{(k)}, Q^{(k-1)})}.
    \]
\end{lemma}
\begin{proof}
    Consider $|\epsilon_{t}^{(k)}(z,f_t(z))|$, then
    \begin{align*}
    \mathbb{E}_{z\sim\mu_t^Z}[|\epsilon_{t}^{(k)}(z,f_t(z))|]&=\mathbb{E}_{\mu_t^Z}\left[\sum_a \mathbf{1}_{\{f_t(z)=a\}}|\epsilon_{t}^{(k)}(z,a)|\right],\\
    &=\mathbb{E}_{z\sim\mu_t^Z}\left[\sum_a \sqrt{\frac{\mu_t(a|z)}{\mu_t(a|z)}}\mathbf{1}_{\{f_t(z)=a\}}|\epsilon_{t}^{(k)}(z,a)|\right],\\
    &\leq 
    \sqrt{\mathbb{E}_{z\sim\mu_t^Z}\left[\sum_a \frac{\mathbf{1}_{\{f_t(z)=a\}}}{\mu_t(a|z)}\right]\mathbb{E}_{z\sim\mu_t^Z}\left[\sum_a|\epsilon_{t}^{(k)}(z,a)|^2\mu_t(a|z)\right]} \tag{Cauchy-Schwartz},\\
    &\leq 
    \sqrt{\mathbb{E}_{z\sim\mu_t^Z}\left[ \frac{1}{\mu_t(f_t(z)|z)}\right]\mathbb{E}_{(z,a)\sim \mu_t}\left[|\epsilon_{t}^{(k)}(z,a)|^2\right]},\\
    &\leq 
    \kappa_t \sqrt{ {\mathcal L}_{t}(Q^{(k)}, Q^{(k-1)})} \tag{by definition }.
    \end{align*}
\end{proof}

We now have the bound on $d_t^{(k+1)}$.
\begin{lemma}\label{lemma:bound_d}
    For $t=1,\dots, N$, and all $k\geq 1$ we have that
    \begin{align*}
    \mathbb{E}_{z\sim \nu_t^{\pi^{(k+1)}}}\left[d_t^{(k+1)}(z)\right]&\leq \mathbb{E}_{z\sim \nu_t^{\pi^{(k+1)}}}\left[b_t^{(k)}(z,\pi_t^{(k+1)}(z))\right]+\sum_{j=0}^{N-t}c_\infty(t+j)\kappa_{t+j}\sqrt{{\cal L}_{k-j,t+j}}.
    \end{align*}
   
\end{lemma}
\begin{proof}
    Consider $ d_N^{(k)}(z)= V_N^\star(z)- V_N^{(k)}(z)=r_N(z)- r_N(z)=0$. Then $d_N^{(k)}(z)=0$ for all $z\in {\cal Z}_N$.

    For $t<N$ we have
    \begin{align*}
         d_t^{(k+1)}(z)&= V_t^\star(z)- [\Gamma_t^{\pi^{(k+1)}} Q^{(k)}](z,\pi_t^{(k+1)}(z)),\\
         &= \max_a Q_t^\star(z,a)- [\Gamma_t^{\pi^{(k+1)}} Q^{(k)}](z,\pi_t^{(k+1)}(z)),\\
          &= \max_a Q_t^\star(z,a)- [\Gamma_t^{\pi^{(k+1)}} Q^{(k)}](z,\pi_t^{(k+1)}(z))\pm Q_t^{(k)}(z,\pi_t^{(k+1)}(z)),\\
          &= \max_a Q_t^\star(z,a)-Q_t^{(k)}(z,\pi_t^{(k+1)}(z))+ Q_t^{(k)}(z,\pi_t^{(k+1)}(z))- [\Gamma_t^{\pi^{(k+1)}} Q^{(k)}](z,\pi_t^{(k+1)}(z)),\\
          &= \max_a Q_t^\star(z,a)-Q_t^{(k)}(z,\pi_t^{(k+1)}(z))+ b_t^{(k)}(z,\pi_t^{(k+1)}(z)),\\
          &\leq  \max_a [Q_t^\star(z,a)-Q_t^{(k)}(z,a)]+ b_t^{(k)}(z,\pi_t^{(k+1)}(z)),
    \end{align*}
    where the last step follows from the greediness of $\pi_t^{(k+1)}$ w.r.t. $Q_t^{(k)}$.
    Define $\Delta_t^{(k)}(z,a)=Q_t^\star(z,a)-Q_t^{(k)}(z,a)$ and $\Delta_t^{(k)}(z)=\max_a \Delta_t^{(k)}(z,a)$. Then
    \[
    d_t^{(k+1)}(z) \leq  \Delta_t^{(k)}(z) +  b_t^{(k)}(z,\pi_t^{(k+1)}(z)).
    \]
    We are now tasked with bounding $\Delta_t^{(k)}$. To that aim, observe that $Q_t^\star= \Gamma^\star_t Q^\star$, thus
    \begin{align*}
        \Delta_t^{(k)}(z,a)&=[\Gamma^\star_t Q^\star](z,a)- [\Gamma_t^{\pi^{(k)}}Q^{(k-1)}](z,a)-\epsilon_t^{(k)}(z,a),\\
        &=\mathbb{E}_{z'\sim \bar P(\cdot|z,a)}[V_{t+1}^\star(z')-Q_{t+1}^{(k-1)}(z',\pi_{t+1}^{(k)}(z'))]- \epsilon_t^{(k)}(z,a),\\
         &=\mathbb{E}_{z'\sim \bar P(\cdot|z,a)}[\max_{a'} Q_{t+1}^\star(z',a')-Q_{t+1}^{(k-1)}(z',\pi_{t+1}^{(k)}(z'))]- \epsilon_t^{(k)}(z,a),\\
          &\leq \mathbb{E}_{z'\sim \bar P(\cdot|z,a)}\left[\Delta_{t+1}^{(k-1)}(z')\right]- \epsilon_t^{(k)}(z,a).\tag{similarly to above}
    \end{align*}
Therefore, we have that
\[\Delta_t^{(k)}(z) \leq \max_a \mathbb{E}_{z'\sim \bar P(\cdot|z,a)}\left[\Delta_{t+1}^{(k-1)}(z')\right]+ \max_a|\epsilon_t^{(k)}(z,a)|,
\]
from which we can recursively show that 
\[
\mathbb{E}_{z\sim \nu_t^{\pi^{(k+1)}}}\left[\Delta_t^{(k)}(z)\right]\leq \sum_{j=0}^{N-t}c_\infty(t+j)\kappa_{t+j}\sqrt{{\cal L}_{k-j,t+j}}
\]
using \cref{lemma:bound_epsilon}.

\end{proof}

We now have the bound on $s_t^{(k)}$.
\begin{lemma}\label{lemma:bound_s}
    For all $t=1,\dots, N$ and $k$
 we have that
 \[
     \mathbb{E}_{z \sim \nu_t^{\pi^{(k)}}}[s_t^{(k)}(z)]  = \sum_{j=1}^{N-t} \mathbb{E}_{z' \sim \nu_{t+j}^{\pi^{(k)}}}\left[b_{t+j}^{(k-1)}(z',\pi_{t+j}^{(k)}(z'))\right ].
 \]
 \end{lemma}
\begin{proof}
First, note that $ s_N^{(k)}(z)=0$ . Then, for $t<N$ we have
    \begin{align*}
        s_t^{(k)}(z)&= V_t^{(k)}(z)- V_t^{\pi^{(k)}}(z),\\
        &=[ \Gamma_t^{\pi^{(k)}} Q^{(k-1)}](z,\pi_t^{(k)}(z))-  Q_{t}^{\pi^{(k)}}(z,\pi_t^{(k)}(z)),\\
        &= \mathbb{E}_{x'|z,\pi_t^{(k)}(z)}\left[Q_{t+1}^{(k-1)}(z',\pi_{t+1}^{(k)}(z')) -  V_{t+1}^{\pi^{(k)}}(z')\, \Big|\, z'=(z,\pi_t^{(k)}(z),x')\right ],\\
        &= \mathbb{E}_{x'|z,\pi_t^{(k)}(z)}\left[b_{t+1}^{(k-1)}(z',\pi_{t+1}^{(k)}(z'))+[ \Gamma_{t+1}^{\pi^{(k)}}Q^{(k-1)}](z',\pi_{t+1}^{(k)}(z')) -  V_{t+1}^{\pi^{(k)}}(z')\, \Big|\, z'=(z,\pi_t^{(k)}(z),x')\right ],\\
        &= \mathbb{E}_{x'|z,\pi_t^{(k)}(z)}\left[b_{t+1}^{(k-1)}(z',\pi_{t+1}^{(k)}(z'))+V_{t+1}^{(k)}(z') -  V_{t+1}^{\pi^{(k)}}(z')\, \Big|\, z'=(z,\pi_t^{(k)}(z),x')\right ],\\
        &= \mathbb{E}_{x'|z,\pi_t^{(k)}(z)}\left[b_{t+1}^{(k-1)}(z',\pi_{t+1}^{(k)}(z'))+s_{t+1}^{(k)}(z') \, \Big|\, z'=(z,\pi_t^{(k)}(z),x')\right ],\\
        &= \sum_{j=1}^{N-t} \mathbb{E}\left[b_{t+j}^{(k-1)}(z',\pi_{t+j}^{(k)}(z'))\, \Big|\, z_t=z,\hbox{ then follow } \pi^{(k)}\right ].
    \end{align*}
    Therefore
    \[
    \mathbb{E}_{z \sim \nu_t^{\pi^{(k)}}}[s_t^{(k)}(z)]  = \sum_{j=1}^{N-t} \mathbb{E}_{z' \sim \nu_{t+j}^{\pi^{(k)}}}\left[b_{t+j}^{(k-1)}(z',\pi_{t+j}^{(k)}(z'))\right ].
    \]
\end{proof}

\begin{lemma}\label{lemma:bound_b}
    For all $t=1,\dots, N,\forall a\in{\cal A}$ and epochs $k\geq N$ we have that
    \begin{align*}
        \mathbb{E}_{z\sim \nu_t^{\pi^{(k)}}}&\left[b_t^{(k)}(z,a)\,\Big |\,  \pi^{(k)},\dots, \pi^{(k-(N-t)+1)}\right ]\\
        &\qquad\qquad\leq c_\infty(t)\kappa_t \sqrt{{\cal L}_{k,t}}+c_\infty(N)\kappa_N\left[\sqrt{{\cal L}_{k-(N-t), N}}+\sqrt{{\cal L}_{k-(N-t-1), N}}\right]\\
        &\qquad\qquad\quad+\sum_{j=1}^{N-t-1} c_\infty(t+j) \kappa_{t+j} \left[\sqrt{
        {\cal L}_{k-j, t+j}}+\sqrt{
        {\cal L}_{k-j+1, t+j}}\right].
    \end{align*}
\end{lemma}
\begin{proof}
    \noindent\emph{(One-step recursion).} For $t<N$ write
    \begin{align*}
     b_t^{(k)}&=Q_t^{(k)}- \Gamma_t^{\pi^{(k+1)}} Q^{(k)},\\
     &=  \Gamma_t^{\pi^{(k)}} Q^{(k-1)} - \Gamma_t^{\pi^{(k+1)}} Q^{(k)}+\epsilon_t^{(k)}.
    \end{align*}
    Use the definition of $ \Gamma_t^\pi$, we have that at time $t=N$ we get $b_N^{(k)}=\epsilon_{N}^{(k)}$. For $t<N$ we get
    \begin{align*}
    b_{t}^{(k)}&(z,a)= \epsilon_t^{(k)}(z,a)+\mathbb E_{x'\sim\bar P_t(\cdot| z,a)}\left[Q_{t+1}^{(k-1)}\left(z', \pi_{t+1}^{(k)}(z')\right)-Q_{t+1}^{(k)}\left(z', \pi_{t+1}^{(k+1)}(z')\right) \,\Big |\, z'=(z,a,x')\right],\\
    =&\epsilon_t^{(k)}(z,a)+ \mathbb E_{x'\sim\bar P_t(\cdot| z,a)}\left[Q_{t+1}^{(k-1)}\left(z', \pi_{t+1}^{(k)}(z')\right)-Q_{t+1}^{(k)}\left(z', \pi_{t+1}^{(k+1)}(z')\right) \pm Q_{t+1}^{(k)}(z',\pi_{t+1}^{(k)}(z'))\,\Big |\, z'=(z,a,x')\right],\\
    =&\epsilon_t^{(k)}(z,a) +\mathbb E_{x'\sim\bar P_t(\cdot| z,a)}\Big[Q_{t+1}^{(k-1)}\left(z', \pi_{t+1}^{(k)}(z')\right)-Q_{t+1}^{(k)}(z',\pi_{t+1}^{(k)}(z'))\\
    &\qquad\qquad\qquad\quad +Q_{t+1}^{(k)}(z',\pi_{t+1}^{(k)}(z'))-Q_{t+1}^{(k)}\left(z', \pi_{t+1}^{(k+1)}(z')\right)\,\Big |\, z'=(z,a,x')\Big],\\
    &\leq \epsilon_t^{(k)}(z,a) +\mathbb E_{x'\sim\bar P_t(\cdot| z,a)}\left[Q_{t+1}^{(k-1)}\left(z', \pi_{t+1}^{(k)}(z')\right)-Q_{t+1}^{(k)}(z',\pi_{t+1}^{(k)}(z'))\,\Big |\, z'=(z,a,x')\right],
    \end{align*}
    where in the last inequality, we used that  $Q_{t+1}^{(k)}(z',\pi_{t+1}^{(k+1)}(z'))\geq Q_{t+1}^{(k)}(z',\pi_{t+1}^{(k)}(z'))$ (since $\pi^{(k+1)}={\cal G}(Q^{(k)})$. 
    Now, using the definition $ b_{t+1}^{(k)}=Q_{t+1}^{(k)}-\Gamma_{t+1}^{\pi^{(k+1)}} Q^{(k)}$, we continue with $Q_{t+1}^{(k)}= \Gamma_{t+1}^{\pi^{(k)}} Q^{(k-1)}+\epsilon_{t+1}^{(k)}$
    \begin{align*}
   &=  \epsilon_t^{(k)}(z,a) +\mathbb E_{x'\sim\bar P_t(\cdot| z,a)}\left[Q_{t+1}^{(k-1)}\left(z', \pi_{t+1}^{(k)}(z')\right)- [ \Gamma_{t+1}^{\pi^{(k)}} Q^{(k-1)} +\epsilon_{t+1}^{(k)}](z',\pi_{t+1}^{(k)}(z'))\,\Big |\, z'=(z,a,x')\right],\\
   &=\epsilon_t^{(k)}(z,a) +\mathbb E_{x'\sim\bar P_t(\cdot| z,a)}\left[b_{t+1}^{(k-1)}\left(z', \pi_{t+1}^{(k)}(z')\right)- \epsilon_{t+1}^{(k)}(z',\pi_{t+1}^{(k)}(z'))\,\Big |\, z'=(z,a,x')\right].
    \end{align*}
    Therefore
    \begin{align*}
    b_t^{(k)}(z,a) &\leq \epsilon_t^{(k)}(z,a) +\mathbb E_{x'\sim\bar P_t(\cdot| z,a)}\left[b_{t+1}^{(k-1)}\left(z', \pi_{t+1}^{(k)}(z')\right)- \epsilon_{t+1}^{(k)}(z',\pi_{t+1}^{(k)}(z'))\,\Big |\, z'=(z,a,x')\right].
    \end{align*}
    Thus
    \[
    b_t^{(k)}(z,a) \leq |\epsilon_t^{(k)}(z,a)| +\mathbb E_{x'\sim\bar P_t(\cdot| z,a)}\left[b_{t+1}^{(k-1)}\left(z', \pi_{t+1}^{(k)}(z')\right)+ |\epsilon_{t+1}^{(k)}(z',\pi_{t+1}^{(k)}(z'))|\,\Big |\, z'=(z,a,x')\right].
    \]
    
\noindent\emph{(Unrolling).} Let $z_{t+1}$ be the state observed after taking action $a$ in $z$ in round $t$. Then denote the successive states by $z_{t+j}$, sampled by following $\pi^{(k)}$. Then,
    unrolling the last upper bound yields
    \begin{align*}
    b_t^{(k)}(z,a) &\leq |\epsilon_t^{(k)}(z,a)|+ \mathbb E\Big[ |\epsilon_{t+1}^{(k-1)}(z_{t+1}, \pi_{t+1}^{(k)}(z_{t+1}))| + b_{t+2}^{(k-2)}\left(z_{t+2}, \pi_{t+2}^{(k-1)}(z_{t+2})\right)\\
    &\qquad\qquad\quad+ |\epsilon_{t+2}^{(k-1)}(z_{t+2},\pi_{t+2}^{(k-1)}(z_{t+2}))|+ |\epsilon_{t+1}^{(k)}(z_{t+1},\pi_{t+1}^{(k)}(z_{t+1}))| \,\Big|\, z_t=z,a_t=a \Big],\\
    &\leq \mathbb{E}\Big[|\epsilon_t^{(k)}(z_t,a_t)|+|\epsilon_{N}^{(k-(N-t))}(z_N,\pi_{N}^{(k-(N-t)+1)}(z_N))| + |\epsilon_{N}^{(k-(N-t)+1)}(z_N, \pi_N^{(k-(N-t)+1)}(z_N))|   \\
    &\quad\quad+ \sum_{j=1}^{N-t-1}|\epsilon_{t+j}^{(k-j)}(z_{t+j},\pi_{t+j}^{(k-j+1)}(z_{t+j}))|+|\epsilon_{t+j}^{(k-j+1)}(z_{t+j}, \pi_{t+j}^{(k-j+1)}(z_{t+j}))|\,\Big|\, z_t=z,a_t=a \Big].
    \end{align*}
    Therefore, using \cref{lemma:bound_epsilon}
    \begin{align*}
        \mathbb{E}_{z\sim \nu_t^{\pi^{(k)}}}&\left[b_t^{(k)}(z,a)\,\Big |\,\pi^{(k)},\dots, \pi^{(k-(N-t)+1)}\right ]\\
        &\leq c_\infty(t)\kappa_t \sqrt{{\cal L}_{k,t}}+c_\infty(N)\kappa_N\left[\sqrt{{\cal L}_{k-(N-t), N}}+\sqrt{{\cal L}_{k-(N-t-1), N}}\right]\\
        &\,+\sum_{j=1}^{N-t-1} c_\infty(t+j) \kappa_{t+j} \left[\sqrt{
        {\cal L}_{k-j, t+j}}+\sqrt{
        {\cal L}_{k-j+1, t+j}}\right].
    \end{align*}
\end{proof}

We now prove the bound on $J(\pi^\star)-J(\pi^{(k)})$ in \cref{thm:value_bound}, where $J(\pi)=\mathbb{E}_{\P}^\pi[r_N(z_N)]$. 

\begin{proof}[Proof of \cref{thm:value_bound}]
    Note that $0\leq J(\pi^\star)- J(\pi^{(k)})= \mathbb{E}_{z\sim \nu_1}[\ell_1^{(k)}(z)]$. Using the decomposition $\ell_1^{(k)}=s_1^{(k)}+d_1^{(k)}$ and \cref{lemma:bound_d,lemma:bound_s}, we obtain that
\begin{align*}
     \mathbb{E}_{z\sim \nu_1}[&\ell_1^{(k)}(z)]= \mathbb{E}_{z\sim \nu_1}\left[s_1^{(k)}(z)\right]+  \mathbb{E}_{z\sim \nu_1}\left[d_1^{(k)}(z)\right],\\
     \leq& \sum_{u=2}^{N} \mathbb{E}_{z' \sim \nu_{u}^{\pi^{(k)}}}\left[b_{u}^{(k-1)}(z',\pi_{u}^{(k)}(z'))\right ]+\mathbb{E}_{z\sim \nu_1^{\pi^{(k)}}}\left[b_1^{(k-1)}(z,\pi_1^{(k)}(z))\right]+\sum_{u=1}^{N}c_\infty(u)\kappa_{u}\sqrt{{\cal L}_{k-u,u}},\\
     =& \sum_{u=1}^{N} \mathbb{E}_{z' \sim \nu_{u}^{\pi^{(k)}}}\left[b_{u}^{(k-1)}(z',\pi_{u}^{(k)}(z'))\right ]+\sum_{u=1}^{N}c_\infty(u)\kappa_{u}\sqrt{{\cal L}_{k-u,u}}.
\end{align*}

From \cref{lemma:bound_b} we know that
 \begin{align*}
        \mathbb{E}_{z\sim \nu_t^{\pi^{(k)}}}\left[b_t^{(k)}(z,a)\,\Big |\,  \pi^{(k)},\dots, \pi^{(k-(N-t)+1)}\right ]&\leq c_\infty(t)\kappa_t \sqrt{{\cal L}_{k,t}}+c_\infty(N)\kappa_N\left[\sqrt{{\cal L}_{k-(N-t), N}}+\sqrt{{\cal L}_{k-(N-t-1), N}}\right]\\
        &\quad+\sum_{j=1}^{N-t-1} c_\infty(t+j) \kappa_{t+j} \left[\sqrt{
        {\cal L}_{k-j, t+j}}+\sqrt{
        {\cal L}_{k-j+1, t+j}}\right].
    \end{align*}
    hence
 \begin{align*}
        \mathbb{E}_{z\sim \nu_t^{\pi^{(k)}}}&\left[b_t^{(k-1)}(z,\pi_t^{(k)}(z))\right ]\\
        &\leq c_\infty(t)\kappa_t \sqrt{{\cal L}_{k-1,t}}+c_\infty(N)\kappa_N\left[\sqrt{{\cal L}_{k-(N-t)-1, N}}+\sqrt{{\cal L}_{k-(N-t), N}}\right]\\
        &\,+\sum_{j=1}^{N-t-1} c_\infty(t+j) \kappa_{t+j} \left[\sqrt{
        {\cal L}_{k-j-1, t+j}}+\sqrt{
        {\cal L}_{k-j, t+j}}\right].
    \end{align*}
Using the last inequality we obtain
    \begin{align*}
         \sum_{u=1}^{N} \mathbb{E}_{z' \sim \nu_{u}^{\pi^{(k)}}}\left[b_{u}^{(k-1)}(z',\pi_{u}^{(k)}(z'))\right ]  &\leq  \sum_{u=1}^{N}  c_\infty(u)\kappa_u \sqrt{{\cal L}_{k-1,u}}+c_\infty(N)\kappa_N\left[\sqrt{{\cal L}_{k-(N-u)-1, N}}+\sqrt{{\cal L}_{k-(N-u), N}}\right]\\
          &\qquad\quad + \sum_{u=1}^N\sum_{j=1}^{N-u-1} c_\infty(u+j) \kappa_{u+j} \left[\sqrt{
        {\cal L}_{k-j-1, u+j}}+\sqrt{
        {\cal L}_{k-j, u+j}}\right]= (\star).
    \end{align*}
    Re-indexing the last term by $s=u+j$, we have
        \begin{align*}
         (\star)  &=  \sum_{u=1}^{N}  c_\infty(u)\kappa_u \sqrt{{\cal L}_{k-1,u}}+c_\infty(N)\kappa_N\left[\sqrt{{\cal L}_{k-(N-u)-1, N}}+\sqrt{{\cal L}_{k-(N-u), N}}\right]\\
          &\qquad\quad + \sum_{s=2}^{N-1}c_\infty(s) \kappa_{s} \sum_{j=1}^{s-1} \left[\sqrt{
        {\cal L}_{k-j-1, s}}+\sqrt{
        {\cal L}_{k-j, s}}\right].
    \end{align*}
    At this point, define $w_u=c_\infty(u)\kappa_u, w^{(a,b)}=(w_u)_{u=a}^b$, $ Z_{m}^{(a,b)}=\sum_{u=a}^b \sqrt{{\cal L}_{u,m}}$ and $S_m^{(a,b)}= \sum_{u=a}^b {{\cal L}_{m,u}}$. Then,
 \begin{align*}
        (\star) &\leq   \|w^{(1,N)}\|_2 \sqrt{S_{k-1}^{(1,N)}}+ w_N\left[ Z_N^{(k-N,k-1)}+Z_N^{(k-N+1,k)}\right] \tag{Applied Cauchy-Schwartz}\\&\qquad+ \sum_{s=2}^{N-1}w_s \left[Z_s^{(k-s,k-2)}+Z_s^{(k-s+1,k-1)}\right],\\
          &\leq  \|w^{(1,N)}\|_2 \sqrt{S_{k-1}^{(1,N)}}+ 2\sum_{s=2}^{N}w_s Z_s^{(k-s,k)}, \tag{Increased the sum range of $Z$}\\
          &\leq  \|w^{(1,N)}\|_2 \left[\sqrt{S_{k-1}^{(1,N)}}+ 2 \sqrt{\sum_{s=2}^{N}  \left(Z_s^{(k-s,k)}\right)^2}\right] \tag{By $\|w^{(2,N)}\|_2\leq \|w^{(1,N)}\|_2$}.
    \end{align*}
    Now, observe that
    \[
    \left(Z_s^{(k-s,k)}\right)^2= \left(\sum_{u=k-s}^k \sqrt{{\cal L}_{u,s}}\right)^2 \leq (s+1) \left(\sum_{u=k-s}^k {\cal L}_{u,s}\right) \leq (N+1) \sum_{u=k-N}^k {\cal L}_{u,s},
    \]
    therefore
    \[
     \sum_{s=2}^N  (N+1) \sum_{u=k-N}^k {\cal L}_{u,s}=  (N+1) \sum_{u=k-N}^k\sum_{s=2}^N {\cal L}_{u,s} =(N+1) \sum_{u=k-N}^k S_u^{(2,N)}.
    \]

    Thus
    \[
     (\star)\leq \|w^{(1,N)}\|_2 \left[\sqrt{S_{k-1}^{(1,N)}}+ 2 \sqrt{(N+1) \sum_{u=k-N}^k S_u^{(2,N)}}\right].
    \]
    We can  plug this back into the original bound on $\ell_1^{(k)}$. Define $D_s^{(a,b)}=\sum_{u=a}^b {\cal L}_{s-u,u}$ to be the diagonal sum of losses, then
      \begin{align*}
        \mathbb{E}_{z\sim \nu_1}[|\ell_1^{(k)}(z)|]&\leq  \|w^{(1,N)}\|_2 \left[\sqrt{S_{k-1}^{(1,N)}}+  2\sqrt{(N+1) \sum_{u=k-N}^k S_u^{(2,N)}}\right]+\sum_{u=1}^{N}c_\infty(u)\kappa_{u}\sqrt{{\cal L}_{k-u,u}},\\
        &\leq  \|w^{(1,N)}\|_2 \left[\sqrt{S_{k-1}^{(1,N)}}+  2\sqrt{(N+1) \sum_{u=k-N}^k S_u^{(2,N)}}\right]+ \|w^{(1,N)}\|_2 \sqrt{D_k^{(1,N)}},\\
        &\leq \|w^{(1,N)}\|_2 \left[\sqrt{S_{k-1}^{(1,N)}}+  2\sqrt{(N+1) \sum_{u=k-N}^k S_u^{(2,N)}}+\sqrt{D_k^{(1,N)}}\right].
    \end{align*}
\end{proof}

\subsubsection{Finite Sample Analysis: Proof of \cref{thm:fixed_budget_finite_sample_analysis}}

\begin{proof}[Proof of \cref{thm:fixed_budget_finite_sample_analysis}]
\noindent{\bf Preliminaries.}
We now consider the error due to the evaluation step. In each epoch $k$ a regression problem is solved, where the training set $\{(z^{(i)}, a^{(i)}, t^{(i)}, \hat Q^{(k+1)})\}$ and $\hat Q_{t^{(i)}}^{(k+1)}(z^{(i)}, a^{(i)})$ is an unbiased estimate of the target defined by $ \Gamma_t Q$.

To approximate the target, for each $t=1,\dots, N$ we consider a linear function space ${\cal F}_t$ of dimension $d_t$ with bounded basis function $\{\varphi_{t,i}\}_{i=1}^{d_t}$ $\|\varphi_{t,i}\|_\infty \leq C_b$. For each $t$ we consider a linear family with parameter $\alpha_t\in \mathbb{R}^{d_t}$  and features $\phi_t:{\cal Z}_t\times {\cal A}\to\mathbb{R}^{d_t}$, thus ${\cal F}_t=\{(z,a)\mapsto \phi_t(z,a)^\top \alpha_t: \alpha_t\in \mathbb{R}^{d_t}\}$.

Recall the losses

\[
 {\mathcal L}_{k,t}\coloneqq  \mathbb E_{(z,a)\sim\mu_t}\left[ \left(Y_t^{(k)}(z,a)-Q_t^{(k)}(z,a)\right)^2\right],
\]
where
\[
Y_{t}^{(k)}(z,a)=[ \Gamma_t^{\pi^{(k)}}Q^{(k-1)}](z,a)
\]
and we also define the error $\epsilon_t^{(k)}=Q_t^{(k)}-Y_{t}^{(k)}$.

For a batch $B_k$ we denote by $B_{k,t}=\{(z,a,s)\in  B_k: s=t\}$ the elements in that batch of size $t$, and let $n_{k,t}=|B_{k,t}|$.

We then let  $Y_{k,t}=(Y_{t}^{(k)}(z,a))_{(z,a)\in B_{k,t}}$ (\emph{true targets}) and $\hat Q_{k,t}=(\hat Q_t^{(k)}(z,a))_{(z,a)\in B_{k,t}}$ (\emph{noisy targets}), and define ${\cal F}_{k,t}=\{\Phi_{k,t}\alpha_t :\alpha_t\in \mathbb{R}^{d_t}\}$, where 
$
\Phi_{k,t}=(
    \phi_t(z,a)^\top)_{(z,a)\in B_{k,t}}$
is a matrix where each row corresponds to the features of some $(z,a)\in B_{k,t}$. We then denote by $\Pi_{k,t}$ the $L_2( \hat\mu_{k,t})$-projection on ${\cal F}_{k,t}$, where $\hat\mu_{k,t}(z,a)=\sum_{(z',a')\in B_{k,t}}\delta_{(z',a')}(z,a)$ is the empirical norm at epoch $k$ for timestep $t$. We also define  $\Pi_t$ to be the $L_2(\mu_t)$ projection on ${\cal F}_t$, where $\mu_t$ is the marginal over trajectories of $\mu$ at timesteps $t$.

We set $\tilde Q_{k,t}\coloneqq\Pi_{k,t}\hat Q_{k,t}=(\tilde Q_t^{(k)}(z,a))_{(z,a)\in B_{k,t}}$, where $\tilde Q_t^{(k)}$  is the result of linear regression and
its truncation (by $1$) is $Q_t^{(k)}$ ($Q_t^{(k)}=\mathbb{T} (\tilde Q_t^{(k)})$). Define also 
 $\hat Y_{k,t} \coloneqq\Pi_{k,t} Y_{k,t}$ and the errors $\xi_{k,t}\coloneqq Y_{k,t}-\hat Q_{k,t}$ and $\hat \xi_{k,t} \coloneqq \Pi_{k,t}\xi_{k,t}$. We note that $\xi_{k,t}$ has mean 0, and $|(\xi_{k,t})_i|\leq 1$.

 In the following, we denote by $\|f\|_{\mu_t} = \sqrt{\int f(z,a)^2 {\rm d}\mu_t(z,a)}$ the $L_2(\mu_t)$-norm of $f$, and similarly we also indicate the $L_2( \hat\mu_{k,t})$-norm (empirical) by  $\|f\|_{\hat\mu_{k,t}} = \sqrt{\frac{1}{n_{k,t}}\sum_{(z,a)\in B_{k,t}} f(z,a)^2}$.

\noindent{\bf Bounding the error.}
Our goal is to bound
\[
\|e_t^{(k)}\|_{\mu_t}=\|Y_t^{(k)}-Q_t^{(k)}\|_{\mu_t}=\|Y_t^{(k)}-\mathbb{T}(\tilde Q_t^{(k)})\|_{\mu_t}.
\]

By a variation of theorem 11.2 in \citep{gyorfi_distribution-free_2002} (see \citep{JMLR:v13:lazaric12a} corollary 12), we also know that
\begin{align*}
    \|Y_t^{(k)}-\mathbb{T}(\tilde Q_t^{(k)})\|_{\mu_t}- 2\|Y_{k,t}-\tilde Q_{k,t}\|_{\hat\mu_{k,t}}\leq 24\sqrt{\frac{2}{n_{k,t}} \Lambda(n_{k,t},d_t,\delta')}.
\end{align*}
with probability at least $1-\delta'$, where $\Lambda(n_{k,t},d_t,\delta')=2(d_t+1)\log(n_{k,t}) + \log(\frac{e}{\delta'}) + \log\left(9(12e)^{2(d_t+1)}\right)$.
Therefore
\[
\|Y_t^{(k)}-\mathbb{T}(\tilde Q_t^{(k)})\|_{\mu_t}\leq 2\|Y_{k,t}-\tilde Q_{k,t}\|_{ \hat\mu_{k,t}}+24\sqrt{\frac{2}{n_{k,t}} \Lambda(n_{k,t},d_t,\delta')}.
\]

So, for each $t$ the error is
\[
\|Y_{k,t}- \tilde  Q_{k,t}\|_{\hat\mu_{k,t}} \leq \|\tilde  Q_{k,t}- \hat Y_{k,t}\|_{ \hat\mu_{k,t}}+\|Y_{k,t}- \hat Y_{k,t}\|_{ \hat\mu_{k,t}} = \|\hat \xi_{k,t}\|_{\hat\mu_{k,t}}+\|Y_{k,t}- \hat Y_{k,t}\|_{\hat\mu_{k,t}}.
\]
Furthermore $\|\hat \xi_{k,t}\|_{\hat\mu_{k,t}}^2=\langle \hat\xi_{k,t}, \hat \xi_{k,t}\rangle =\langle \xi_{k,t},\hat \xi_{k,t}\rangle$  by the orthogonal projection, and, by an application of a variation of Pollard’s inequality \citep{gyorfi_distribution-free_2002} we have that
\[
\langle \xi_{k,t},\hat \xi_{k,t}\rangle\leq 4 \|\hat \xi_{k,t}\|_{\hat\mu_{k,t}}\sqrt{\frac{2}{n_{k,t}} \log\left(\frac{3(9e^2n_{k,t})^{d_t+1}}{\delta'}\right)}
\]
holds with probability at least $1-\delta'$. Therefore, we are left with bounding $\|Y_{k,t}- \hat Y_{k,t}\|_{\hat\mu_{k,t}}$.

Define $\hat \alpha_t^\star$ as the parameter satisfying $f_{\hat\alpha_t^\star}\in {\cal F}_t$ such that $f_{\hat \alpha_t^\star}(z,a)=[ \Pi_{k,t} Y_t^{(k)}](z,a)$ for all $(z,a)\in B_{k,t}$. Also define $\alpha_t^\star$ to be the optimal projection (w.r.t. $\mu_t$) of $Y_{t}^{(k)}$ in ${\cal F}_t$, i.e.,  $f_{\alpha_t^\star}=\Pi_t Y_t^{(k)}$. Then, again by a variation of Theorem 11.2 \cite{gyorfi_distribution-free_2002} (see also \citep{JMLR:v13:lazaric12a} corollary 13), we have the following sequence of inequality
\begin{align*}
\|Y_{k,t}- \hat Y_{k,t}\|_{ \hat\mu_{k,t}} &=\|Y_{k,t} - f_{\hat\alpha_t^\star}\|_{\hat\mu_{k,t}},\\ &\leq \|Y_{k,t} - f_{\alpha_t^\star}\|_{\hat \mu_{t}},\\
&\leq 2\|Y_{t}^{(k)} - f_{\alpha_t^\star}\|_{ \mu_{t}} + 12\left(1+\|\alpha_t^\star\|_2\sup_{(z,a)\in {\cal Z}_t\times {\cal A}}\|\phi_t(z,a)\|_2\right)\sqrt{\frac{2}{n_{k,t}}\log\left(\frac{3}{\delta'}\right)},
\end{align*}
that hold with probability at least $1-\delta'$. In conclusion, we have shown that
\begin{align*}
    \|e_t^{(k)}\|_{\mu_t}\leq& 2\Bigg[  2\|Y_{t}^{(k)} - f_{\alpha_t^\star}\|_{ \mu_{t}}  + 12\left(1+\|\alpha_t^\star\|_2\sup_{(z,a)\in {\cal Z}_t\times {\cal A}}\|\phi_t(z,a)\|_2\right)\sqrt{\frac{2}{n_{k,t}}\log\left(\frac{3}{\delta'}\right)}\\
    &\qquad +   4 \sqrt{\frac{2}{n_{k,t}} \log\left(\frac{3(9e^2n_{k,t})^{d_t+1}}{\delta'}\right)}\Bigg]+24\sqrt{\frac{2}{n_{k,t}} \Lambda(n_{k,t},d_t,\delta')}.
\end{align*}

\noindent {\bf Union bound for the random batch.}
At this point, let $\mu(t)$ be the marginal of $\mu$ over the timesteps $t=1,\dots,N$. Let $p_{\min}=\min_t \mu(t)$. Then $n_{k,t}\coloneqq|B_{k,t}|\sim {\rm Binom}(B, \mu(t))$ and 
\begin{align*}
    \mathbb{P}\left( n_{k,t} \leq (1-\eta)\mu(t)B\right) &\leq \exp\left(- \frac{\mu(t)B\eta^2}{2}\right) ,\\
    &\leq \exp\left(- \frac{p_{\rm min}B\eta^2}{2}\right).
\end{align*}
Therefore, for a fixed $t$  for $B=\frac{2}{\eta^2p_{\min}}\log\frac{1}{\delta'}$ we obtain that 
\[
 \mathbb{P}\left( n_{k,t} \geq (1-\eta)\mu(t)B\right) \geq 1-\delta'.
\]

Therefore, by setting $\delta=4Nk\delta'$, through a union bound, we can conclude that
\[
\|e_{t}^{(k)}\|_{\mu_t}\leq 4\inf_{f\in {\cal F}_t}\|Y_{k,t}-f\|_{\mu_t} +\eta_t((1-\eta)p_{\rm min}B,d_t,\delta) + \eta_t'((1-\eta)p_{\rm min}B,d_t,\delta),
\]
holds with probability $1-\delta$ for all $i=1,\dots,k$, $t=1,\dots, N$, where
\begin{align*}
    \eta_t(n,d_t,\delta)&=32 \sqrt{\frac2n \log\left(\frac{4\cdot 27Nk(12e^2n)^{2(d_t+1)}}{3\delta}\right)},\\
    \eta_t'(n,d_t,\delta)&=24\left(1+\|\alpha_t^\star\|_2\sup_{(z,a)\in {\cal Z}_t\times {\cal A}}\|\phi_t(z,a)\|_2\right)\sqrt{\frac{2}{n}\log\left(\frac{12Nk}{\delta}\right)}.
\end{align*}

\noindent{\bf Bounding $S_k^{(a,b)}$ in terms of the error.}
Let \[
\beta_{t}=\sup_{Q\in {\cal F}_t,\pi }\inf_{f\in {\cal F}_t} \| \Gamma_t^{\pi}Q-f\|_{\mu_t}.\]
Since $S_k^{(a,b)}=\sum_{u=a}^b {\cal L}_{k,u}$ and $\sqrt{{\cal L}_{k,t}}=\|e_t^{(k)}\|_{\mu_t}$, we have that
\[
\sqrt{S_k^{(a,b)}}= \left\| \left(\sqrt{{\cal L}_{k,t}}\right)_{t=a}^b\right\|_2 \leq 4\|( \beta_t)_{t=a}^b\|_2 + \|(\eta_t)_{t=a}^b\|_2 +\|(\eta_t')_{t=a}^b\|_2 ,
\]
with probability $1-\delta$.

Similarly, we have that
\begin{align*}
\sqrt{(N+1) \sum_{u=k-N}^k S_u^{(2,N)}} &\leq \sqrt{N+1} \sqrt{\sum_{u=k-N}^k 16\|( \beta_t)_{t=2}^N\|_2^2 + \|(\eta_t)_{t=2}^N\|_2^2 +\|(\eta_t')_{t=2}^N\|_2^2 },\\
&\leq  (N+1) \sqrt{16\|( \beta_t)_{t=2}^N\|_2^2 + \|(\eta_t)_{t=2}^N\|_2^2 +\|(\eta_t')_{t=2}^N\|_2^2 },\\
&\leq  (N+1) \left[4\|( \beta_t)_{t=2}^N\|_2 + \|(\eta_t)_{t=2}^N\|_2 +\|(\eta_t')_{t=2}^N\|_2\right].
\end{align*}
which holds with probability $1-\delta$.

Lastly, we consider $\sqrt{D_k^{(1,N)}}$ where $D_k^{(1,N)}=\sum_{u=1}^N {\cal L}_{k-u,u}$. We have with probability $1-\delta$
\[
\sqrt{D_k^{(1,N)}}=\sqrt{\sum_{u=1}^N {\cal L}_{k-u,u}}\leq4\|( \beta_t)_{t=1}^N\|_2 + \|(\eta_t)_{t=1}^N\|_2 +\|(\eta_t')_{t=1}^N\|_2.
\]
Therefore, in conclusion
\begin{align*}
    \sqrt{S_{k-1}^{(1,N)}}+  &\sqrt{(N+1) \sum_{u=k-N}^k S_u^{(2,N)}}+\sqrt{D_k^{(1,N)}}\\& \leq 4\|(\beta_t)_{t=2}^N\|_2 + \|(\eta_t)_{t=2}^N\|_2 +\|(\eta_t')_{t=2}^N\|_2+(N+1)\left[4\|( \beta_t)_{t=2}^N\|_2 + \|(\eta_t)_{t=2}^N\|_2 +\|(\eta_t')_{t=2}^N\|_2\right],\\
    & + 4\|(\beta_t)_{t=1}^{N}\|_2 + \|(\eta_t)_{t=1}^{N}\|_2 +\|(\eta_t')_{t=1}^{N}\|_2,\\
    &\leq (N+3)\left[4\|(\beta_t)_{t=1}^{N}\|_2 + \|(\eta_t)_{t=1}^{N}\|_2 +\|(\eta_t')_{t=1}^{N}\|_2\right].
\end{align*}

\noindent{\bf Conclusions.}
Therefore, we can conclude that, up to constants and logarithmic factors, we have that with probability $1-\delta$
\[
|J(\pi^\star)-J(\pi^{(k)})|\leq O\left( NC _0 \left[C_1+ \sqrt{\sum_{t=1}^N\frac{ d_t}{(1-\eta)p_{\rm min}B} \log\frac{4kN}{\delta}}\right]\right)
\]
provided $B\geq \frac{2}{p_{\rm min} \eta^2} \log\frac{4kN}{\delta}$
where $\eta\in (0,1)$, $C_0\coloneqq \sqrt{\sum_{t=1}^N c_\infty(t)^2\kappa_t^2}$ and $C_1\coloneqq \sqrt{\sum_{t=1}^N \beta_t^2}$. 
\end{proof}
\newpage
\subsection{Comparison with Information Directed Sampling}\label{app:theoretical:ids}
In pure exploration IDS \citep{russo2018learning}  the main objective is to maximize the \emph{information gain}. For example, consider the BAI problem: we set $\alpha_t(a)=\mathbb{P}(\hat H=a|{\cal D}_t)$  to be the posterior distribution of the optimal arm. Then, the information gain is defined through the following quantity
\[
g_t(a)=\mathbb{E}[H(\alpha_t)-H(\alpha_{t+1})|{\cal D}_t, a_t=a],
\]
which measures the expected reduction in entropy of the posterior distribution of the best arm due to selecting arm $a$ at time $t$.

For the BAI problem, the authors in \citep{russo2018learning}  propose a \emph{myopic} sampling policy $a_t\in \argmax_a g_t(a)$, which only considers the information gain from the next sample. The reason for using a greedy policy stems from the fact that such a strategy is competitive with the optimal policy in problems where the information gain satisfies a property named \emph{adaptive submodularity} \citep{golovin2011adaptive}, a generalization of submodular set functions to adaptive policies. For example, in the noiseless Optimal Decision Tree problem, it is known \citep{zheng2005efficient} that a greedy strategy based on the information gain is equivalent to a nearly-optimal \citep{dasgupta2004analysis,golovin2010near,golovin2011adaptive} strategy named \emph{generalized binary search} (GBS) \citep{nowak2008generalized,bellala2010extensions}
, which maximizes the expected reduction of the \emph{version space} (the space of hypotheses consistent with the data observed so far). However, for the noisy case both strategies perform poorly \citep{golovin2010near}.

The myopic pure exploration IDS strategy $a_t \in \arg\max_a g_t(a)$ can perform poorly in environments where the sampling decisions influence the observation distributions, or where an action taken at time $t$ can greatly affect the complexity of the problem at a later stage (hence, IDS can perform poorly on credit assignments problems).

\paragraph{First example.}
As a first example, consider a bandit problem with $K$ arms, where the reward for each arm $a_i$ is distributed according to ${\cal N}(\mu_i, 1)$, with priors $\mu_1 = \delta_0$ and $\mu_i \sim {\cal U}([0,1])$ independently for each $i \in \{2,\dots,K\}$. Thus, almost surely, the optimal arm $a^\star$ lies within $\{2,\dots,K\}$, and the goal is to estimate $a^\star$

We introduce the following twist: if arm $a_1$ is sampled exactly twice, its reward distribution changes permanently to a Dirac delta distribution $\delta_{\phi(a^\star)}$, where $\phi$ is a known invertible mapping. Consequently, sampling arm $a_1$ twice fully reveals the identity of $a^\star$. However, if arm $a_1$ has not yet been sampled, the expected immediate information gain at any step $t$ is zero, i.e., $g_t(a_1)=0$, since arm $a_1$ is already known to be suboptimal. In contrast, the immediate information gain for any other arm remains strictly positive. Therefore, under this setting and for nontrivial values of $(\sigma,K)$, the myopic IDS strategy cannot achieve the optimal constant sample complexity, and instead scales linearly in $K$.

\paragraph{Second example.}
Another example is a bandit environment containing a chain of two magic actions $\{1,m\}$, where the index of the first magic action ($1$) is known. Action $1$ reveals the index $m$, and pulling arm $m$ subsequently identifies the best arm with certainty. In this scenario, IDS is myopic and typically neglects arm $1$ because of its inability to plan more than 1-step ahead in the future. However, depending on the total number of arms and reward variances, IDS may still select arm $1$ if doing so significantly reduces the set of candidate best arms faster than pulling other arms (e.g., if the variance is significantly large). The following theorem illustrates the sub-optimality of IDS.

\begin{theorem}\label{thm:ids}
Consider a bandit environment with a chain of 2 magic actions. The reward of the regular arms is ${\cal N}(\mu_a,1)$ with $\mu_a\sim {\cal U}([0,1]), a\neq 1,m$.
    For $K\geq 7$ there exists $\delta_0\in (0,1/2)$ such that for any $\delta\leq \delta_0$, we have that IDS is not sample optimal in the fixed confidence setting.
\end{theorem}

\begin{proof}[Proof of \cref{thm:ids}]
Let $Y_{t,a}$ be the random reward observed upon selecting arm $a_t=a$ .
We use that $g_t(a)=I_t(A^\star; Y_{t,a})={\rm KL}\left(\frac{\mathbb{P}(A^\star, Y_{t,a}|{\cal D}_t)}{\mathbb{P}(A^\star|{\cal D}_t)\mathbb{P}(Y_{t,a}|{\cal D}_t)}\right)$, with ${\cal D}_1$ containing an empty observation. The proof relies on showing that action $a_1$ is not chosen during the first two rounds for large values of $K$.

In the proofs, for brevity, we write $\mathbb{P}_t(\cdot)=\mathbb{P}(\cdot|{\cal D}_t)$.
Observe the following lemmas.
\begin{lemma}
Let $Y_{t,a}$ be the random reward observed upon selecting arm $a_t=a$ and let $S_{t,a}=\mathbf{1}\{a_t=a \hbox{ is magic}\}$. Under the assumption that  the agent knows with absolute certainty that $a$ is magic after observing $Y_{t,a}$, we have that $I_t(A^\star;Y_{t,a})=I_t(A^\star;Y_{t,a}, S_{t,a})$. 
\end{lemma}
\begin{proof}
    Note that
    \[
    I_t(A^\star;Y_{t,a}, S_{t,a})=H_t(Y_{t,a},S_{t,a})-H_t(Y_{t,a},S_t|A^\star).
    \]
   Note that by assumption we have that $H_t(S_{t,a}|Y_{t,a})=0$. Then, the first term can also be rewritten as
    \[
    H_t(Y_{t,a},S_t)=H_t(S_{t,a}|Y_{t,a}) + H_t(Y_{t,a})= H_t(Y_t,a).
    \]
    Similarly, we also have  $H_t(Y_{t,a},S_{t,a}|A^\star)=H_t(S_{t,a}|Y_{t,a},A^\star)+H_t(Y_{t,a}|A^\star)=H_t(Y_{t,a}|A^\star)$. Henceforth
    \[
    I_t(A^\star;Y_{t,a}, S_{t,a})=H_t(Y_{t,a})-H_t(Y_{t,a}|A^\star)=I_t(A^\star;Y_{t,a}).
    \]
\end{proof}
Using the decomposition from the previous lemma we can rewrite the mutual information between $A^\star$ and $Y_{t,a}$ as
\[
I_t(A^\star; Y_{t,a})=I_t(A^\star;Y_{t,a},S_{t,a})=I_t(A^\star;Y_{t,a}|S_{t,a}) + I_t(A^\star; S_{t,a}).
\]
\begin{lemma}
 Let ${\cal E}_t=\{(a_1,\dots,a_{t-1}) \hbox{ are not   magic actions}\}$, with ${\cal E}_1=\emptyset$.    Under $a_t=1$ we have that $I_t(A^\star;Y_{t,1}|{\cal E}_t,a_t=1)=\log\left(\frac{K-|{\cal A}_t|-1}{K-|{\cal A}_t|-2}\right)$ where ${\cal A}_t=\{a| \exists i<t: a_t=a\}$ is the unique number of actions chosen in $t\in\{1,\dots,t-1\}$.
\end{lemma}
\begin{proof}
    We use that $\mathbb{P}_t(S_{t,1}=1|a_t=1)=1$. Hence, for arm $1$ we have
    \begin{align*}
    I_t(A^\star;S_{t,1}|{\cal E}_t, a_t=1)&= H_t(S_{t,1}|{\cal E}_t, a_t=1)-H_t(S_{t,1}|A^\star, {\cal E}_t, a_t=1),\\
    &= 0 - 0 =0.
    \end{align*}

    Then, we have
    \begin{align*}
    I_t(A^\star;Y_{t,1}|S_{t,1},{\cal E}_t,a_t=1)&=I_t(A^\star;Y_{t,1}|S_{t,1}=1,{\cal E}_t,a_t=1),\\
    &= {\rm KL}\left( \mathbb{P}_t(A^\star,Y_{t,1}|a_t=1,{\cal E}_t) || \mathbb{P}_t(A^\star|a_t=1,{\cal E}_t)\mathbb{P}_t(Y_{t,1}|a_t=1,{\cal E}_t)\right),\\
    &= {\rm KL}\left( \mathbb{P}_t(Y_{t,1}|A^\star,a_t=1,{\cal E}_t) || \mathbb{P}_t(Y_{t,1}|a_t=1,{\cal E}_t)\right),\\
    &= \log\left(\frac{1/(K-|{\cal A}_t|-2)}{1/(K-|{\cal A}_t|-1)}\right),
    \end{align*}
    where we used that under ${\cal E}_t$ exactly ${\cal A}_t$ regular arms have been pulled and recognised as regular; the still‑unrevealed set of candidates for the second magic arm has therefore  size $K-|{\cal A}_t|-1$ (since arm $1$ is known to be magic).  
    Thus the result follows from applying the previous lemma.
\end{proof}

\begin{lemma}
    For any un-pulled arm $a\neq 1$ at time $t$ we have  that $I_t(A^\star;Y_{t,a}|{\cal E}_t,a_t= a)\geq \frac{1}{K-|{\cal A}_t|-1}\log(K-|{\cal A}_t|-2).$
\end{lemma}
\begin{proof}To compute the mutual information we use that
$I_t(A^\star;Y_{t,a}|{\cal E}_t)=I_t(A^\star;Y_{t,a},S_{t,a}|{\cal E}_t)=I_t(A^\star;Y_{t,a}|S_{t,a},{\cal E}_t)+I_t(A^\star;S_{t,a}|{\cal E}_t)\geq I_t(A^\star;Y_{t,a}|S_{t,a},{\cal E}_t).$. We start by computing the first term of this expression, and finding a non-trivial lower bound.

Note that for $a\neq 1$ we have
\begin{align*}
I_t(A^\star; Y_{t,a}|S_{t,a},{\cal E}_t,a_t=a)&=\mathbb{P}_t(S_{t,a}=0|{\cal E}_t,a_t=a)I_t(A^\star; Y_{t,a}|S_{t,a}=0,{\cal E}_t,a_t=a) \\
&\qquad\qquad+ \mathbb{P}_t(S_{t,a}=1|{\cal E}_t,a_t=a)I_t(A^\star; Y_{t,a}|S_{t,a}=1,{\cal E}_t,a_t=a),\\
&\geq \frac{1}{K-|{\cal A}_t|-1}I_t(A^\star; Y_{t,a}|S_{t,a}=1,{\cal E}_t,a_t=a),
\end{align*}
where we used that under ${\cal E}_t$, we have a uniform prior over the remaining $K-|{\cal A}_t|-1$ un-pulled arms, and the agent knows that arm $1$ is magic.

If $a\neq 1$ and $S_{t,a}=1$, then $a$ is the second magic arm. Therefore we have $\mathbb{P}_t( Y_{t,a}|A^\star,S_{t,a}=1,{\cal E}_t,a_t=a)=1$. Hence $I_t(A^\star; Y_{t,a}|S_{t,a}=1,{\cal E}_t)=\log\left(K-|{\cal A}_t|-2\right)$ since $Y_{t,a}$ can only take values uniformly over $K-|{\cal A}_t|-2$ arms under the event $\{S_{t,a}=1,{\cal E}_t,a_t=a\}$.

\end{proof}
\begin{lemma}
    Assume $a_1=j$ is a regular arm, pulled at the first timestep. Then $I_2(A^\star; Y_{2,j}|a_1=j)\leq \frac{1}{2}\ln(1+\frac{1}{12\sigma^2}).$
\end{lemma}
\begin{proof}
First, note that 
\[
I_2(A^\star;Y_{2,j}\mid a_1=j)
\leq 
I_2(\mu_j;Y_{2,j}\mid a_1=j) = H_2(Y_{2,j}|a_1=j)- \frac{1}{2}\ln(2\pi e\sigma^2)
\]
Then, since ${\rm Var}_2(Y_{2,j}|a_1=j)={\rm Var}_2(\mu_j|a_1=j)+\sigma^2\leq 1/12+\sigma^2$. Therefore $H_2(Y_{2,j}|a_1=j)\leq \frac{1}{2}\ln(2\pi e(1/12+\sigma^2))$. Hence $I_2(A^\star;Y_{2,j}|a_1=j) \leq \frac{1}{2}\ln(1+\frac{1}{12\sigma^2}).$
\end{proof}

Hence, one can verify that  for $K\geq 6$   the first magic arm  will never be chosen at the first timestep. Similarly, at the second timestep the first magic arm will not be chosen if $K\geq 7$.

Consider the fixed-confidence setting with some confidence level $\delta< 1/2$. Let ${\cal A}_1=\{\hbox{second magic arm sampled at } t=1\}$ and ${\cal A}_2=\{\hbox{second magic arm sampled at } t=2\}$.  Then, the sample complexity of IDS satisfies $\mathbb{E}[\tau_{IDS}|{\cal A}_1^c,{\cal A}_2^c]\geq 3$ for $\delta$ sufficiently small (since the sample complexity scales as $\log(1/(2.4\delta))$).

We also have that at the first timestep the decision is uniform over $\{2,\dots, K\}$. Lastly, if the first sampled arm is not magic, then it's a regular arm, and by the previous lemmas the information gain of such arm will be smaller than the information gain of another un-pulled arm.  In fact the inequality
\[
\frac{\log(x-3)}{x-2}   >\frac{1}{2}\ln(1+\frac{1}{12})
\]
it satisfied over $x\in \{5,\dots,121\}$.
Since it is sub-optimal to sample again the same regular arm, since the information gain on all the other arms remains the same, we have that the decision at the second timestep is again uniform over the remaining unchosen arms. Therefore
\begin{align*}
    \mathbb{E}[\tau_{IDS}] &= \mathbb{E}[\tau_{IDS}|{\cal A}_1]\mathbb{P}({\cal A}_1)+\mathbb{E}[\tau_{IDS}|{\cal A}_1^c]\mathbb{P}({\cal A}_1^c),\\
    &= \frac{1}{K-1} + \frac{K-2}{K-1}\mathbb{E}[\tau_{IDS}|{\cal A}_1^c],\\
    &=\frac{1}{K-1} + \frac{K-2}{K-1}\left(\mathbb{E}[\tau_{IDS}|{\cal A}_1^c, {\cal A}_2]\frac{1}{K-2} + \mathbb{E}[\tau_{IDS}|{\cal A}_1^c, {\cal A}_2^c]\frac{K-3}{K-2}\right),\\
    &\geq \frac{1}{K-1} + \frac{K-2}{K-1}\left(2\frac{1}{K-2} + 3\frac{K-3}{K-2}\right),\\
    &= \frac{3}{K-1} + 3\frac{K-3}{K-1},
\end{align*}
which is larger than $2$ for $K>4$. Since there is a policy with sample complexity $2$, we have that IDS cannot be sample optimal for $K\in \{7,\dots,121\}$.

Similarly, for large values of $K>121$, resampling the same regular arm at the second timestep leads IDS to a  sample complexity larger than $2$. And therefore cannot be sample optimal.

\end{proof}

\subsection{Sample Complexity Bounds for MAB Problems with Fixed Minimum Gap}\label{app:subsec:sample_complexity_minimum_gap}
We now derive a sample complexity lower bound for a MAB problem where the minimum gap is known and the rewards are normally distributed.

Consider a MAB problem wit $K$ arms $\{1,\dots, K\}$. To each arm  $a$ is associated a reward distribution $\nu_a ={\cal N}(\mu_a,\sigma^2)$ that is simply a Gaussian distribution. Let $a^\star(\mu)=\argmax_a \mu_a$, and define the gap in arm $a$ to be $\Delta_a(\mu)=\mu_{a^\star(\mu)}-\mu_a$. In the following, without loss of generality, we assume that $a^\star(\mu)=1$.

We define the minimum gap to be $\Delta_{\rm min}(\mu)=\min_{a\neq a^\star(\mu)}\Delta_a(\mu)$. Assume now to know that $\Delta_{\rm min}\geq \Delta_0 >0$.

Then, for any $\delta$-correct algorithm, guaranteeing that at some stopping time $\tau$ the estimated optimal arm $\hat a_\tau$ is $\delta$-correct, i.e., $\mathbb{P}_\mu(\hat a_\tau\neq a^\star(\mu))\leq \delta$, we have the following result.
\begin{theorem}
Consider a model $\mu$ satisfying $\Delta_{\rm min}\geq \Delta_0 >0$. Then,
    for any $\delta$-probably correct method ${\rm Alg}$, with $\delta\in (0,1/2)$, we have that the optimal sample complexity is bounded as
    \[
      \frac{1}{\max\left(\Delta_0^2, \frac{1}{\sum_{a\neq 1} 1/\Delta_a^2}\right)} \leq \inf_{\tau: {\rm Alg} \hbox{ is } \delta\hbox{-correct}}\frac{\mathbb{E}_\mu[\tau]}{2\sigma^2 {\rm kl}(1-\delta,\delta)} \leq 2\sum_a \frac{1}{(\Delta_a+\Delta_0)^2},
    \]
    with $\Delta_1=0$ and  ${\rm kl}(x,y)=x\log(x/y)+(1-x)\log((1-x)/(1-y))$. In particular, the solution $\omega_a \propto 1/(\Delta_a+\Delta_0)^2$ (up to a normalization constant) achieves the upper bound.
\end{theorem}
\begin{proof}
    
\noindent\textbf{Step 1: Log-likelihood ratio.}
The initial part of the proof is rather standard, and follows the same argument used in  the Best Arm Identification and Best Policy Identification literature \citep{garivier2016optimal,russo2025multi}.

Define the set of models
\[
{\cal S} = \left\{\mu' \in \mathbb{R}^K:  \Delta_{\rm min}(\mu') \geq \Delta_0 \right\},
\]
and the set of alternative models
\[
{\rm Alt}(\mu) = \left\{\mu' \in {\cal S}: \argmax_a \mu_a' \neq 1 \right\}.
\]

Take the expected log-likelihood ratio between $\mu$ and $\mu'\in {\rm Alt}(\mu)$ of the data observed up to  $\tau$ $\Lambda_\tau = \log \frac{{\rm d}\mathbb{P}_\mu(A_1, R_1,\dots, A_\tau, R_\tau)}{{\rm d}\mathbb{P}_{\mu'}(A_1, R_1,\dots, A_\tau, R_\tau)}$, where $A_t$ is the action taken in round $t$, and $R_t$ is the  reward observed upon selecting $A_t$. Then, we can write
\[
\Lambda_t = \sum_a \sum_{n=1}^t \mathbf{1}_{\{A_n=a\}} \log \frac{f_a(R_n)}{f_a'(R_n)}
\]
where $f_a,f_a'$, are, respectively, the reward density for action $a$ in the two models $\mu,\mu'$ with respect to the Lebesgue measure. Letting $N_a(t)$ denote the number of times action $a$ has been selected up to round $t$, by an application of Wald's lemma the expected log-likelihood ratio can be shown to be
\[
\mathbb{E}_\mu[\Lambda_\tau] = \sum_a \mathbb{E}_\mu[N_a(\tau)]{\rm KL}(\mu_a,\mu_a')
\]
where ${\rm KL}(\mu_a,\mu_a')$ is the KL divergence between two Gaussian distributions ${\cal N}(\mu_a,\sigma)$ and ${\cal N}(\mu_a',\sigma)$ (note that we have $\sigma_1$ instead of $\sigma$ for $a=1$).

We also know from the information processing inequality \citep{kaufmann2016complexity} that $\mathbb{E}_\mu[\Lambda_\tau]\geq \sup_{{\cal E}\in {\cal M}_\tau} {\rm kl}(\mathbb{P}_\mu({\cal E}), \mathbb{P}_{\mu'}({\cal E}))$, where ${\cal M}_t = \sigma(A_1,R_1,\dots, A_t, R_t)$. We use the fact that the algorithm is $\delta$-correct: by choosing ${\cal E}=\{\hat a_\tau = a^\star\}$ we obtain that $\mathbb{E}_\mu[\Lambda_\tau]\geq {\rm kl}(1-\delta,\delta)$, since $\mathbb{P}_\mu({\cal E})\geq 1-\delta$ and $\mathbb{P}_{\mu'}({\cal E})=1-\mathbb{P}_{\mu'}(\hat a_\tau \neq a^\star)\leq 1-\mathbb{P}_{\mu'}(\hat a_\tau =\argmax_a \mu_a')\leq \delta$ (we also used the monotonicity properties of the Bernoulli KL divergence). Hence
\[
\sum_a \mathbb{E}_\mu[N_a(\tau)]{\rm KL}(\mu_a,\mu_a')\geq {\rm kl}(1-\delta,\delta).
\]

Letting $\omega_a = \mathbb{E}_\mu[N_a(\tau)]/\mathbb{E}_\mu[\tau]$, we have that
\[
\mathbb{E}_\mu[\tau]\sum_a \omega_a{\rm KL}(\mu_a,\mu_a')\geq {\rm kl}(1-\delta,\delta).
\]
Lastly, optimizing over $\mu'\in {\rm Alt}(\mu)$ and $\omega\in \Delta(K)$ yields the bound:
\[
\mathbb{E}_\mu[\tau]\geq T^\star(\mu){\rm kl}(1-\delta,\delta),
\]
where $T^\star(\mu)$ is defined as
\[
(T^\star(\mu))^{-1}=\sup_{\omega\in \Delta(K)}\inf_{\mu'\in {\rm Alt}(\mu)}\sum_a \omega_a{\rm KL}(\mu_a,\mu_a').
\]

\noindent\textbf{Step 2: Optimization over the set of alternative models.}
We now face the problem of optimizing over the set of alternative models.

Defining ${\rm Alt}_a=\left\{\mu' \in \mathbb{R}^K: \mu_a'-\mu_b'\geq \Delta_0 \; \forall b\neq a\right\}$, the set of alternative models can be decomposed as
\begin{align*}
{\rm Alt}(\mu) &= \left\{\mu' \in \mathbb{R}^K: \argmax_a \mu_a' \neq 1,\;  \Delta_{\rm min}(\mu')\geq \Delta_0\right\},\\
&= \cup_{a\neq 1}{\rm Alt}_a.
\end{align*}

Hence, the optimization problem over the alternative models becomes
\[
\inf_{\mu'\in {\rm Alt}(\mu)}\sum_a \omega_a{\rm KL}(\mu_a,\mu_a') =  \min_{\bar a\neq 1}\inf_{\mu'\in {\rm Alt}_{\bar a}}\sum_a \omega_a \frac{(\mu_a-\mu_a')^2}{2\sigma^2}.
\]
The inner infimum over $\mu'$ can then be written as
\begin{equation}
\begin{aligned}
P_{\bar a}^\star(\omega) \coloneqq \inf_{\mu' \in \mathbb{R}^K} \quad & \sum_a \omega_a \frac{(\mu_a-\mu_a')^2}{2\sigma^2}.\\
\textrm{s.t.} \quad &  \mu_{\bar a}'-\mu_b'\geq \Delta_0 \quad \forall b\neq \bar a.
\end{aligned}
\end{equation}
While the problem is clearly convex, it does not yield an immediate closed form solution.

To that aim, we try to derive a lower bound and an upper bound of the value of this minimization problem.

\noindent \textbf{Step 3: Upper bound on $P_{\bar a}^\star$.} Note that an upper bound on $\min_{\bar a\neq 1}P_{\bar a}^\star(\omega)$ can be found by finding a feasible solution $\mu'$. Consider then the solution $\mu_1'=\mu_1-\Delta$,  $\mu_{\bar a}'=\mu_1$ and $\mu_b'=\mu_b$ for all other arms. Clearly We have that $\mu_{\bar a}'-\mu_b'\geq \Delta_0$ for all $b\neq \bar a$. Hence, we obtain
\[
\min_{\bar a\neq 1} P_{\bar a}^\star(\omega) \leq \omega_1 \frac{\Delta_0^2}{2\sigma^2} + \min_{\bar a\neq 1} \omega_{\bar a} \frac{\Delta_{\bar a}^2}{2\sigma^2}.
\]

At this point, one can easily note that if $\frac{\Delta_0^2}{2\sigma^2} \geq  \frac{1}{2\sigma^2\sum_{a\neq 1} \frac{1}{\Delta_a^2}}$, then $\sup_{\omega \in \Delta(K)}\min_{\bar a\neq 1} P_{\bar a}^\star(\omega)  \leq \frac{\Delta_0^2}{2\sigma^2}$. This corresponds to the case where all the mass is given to $\omega_1=1$. Otherwise, the solution is to set $\omega_1=0$ and $\omega_a = \frac{1/\Delta_a^2}{\sum_{b}1/\Delta_b^2}$ for $a\neq 1$.

Hence, we conclude that
\[
(T^\star(\mu))^{-1}=\sup_{\omega \in \Delta(K)}\min_{\bar a\neq 1} P_{\bar a}^\star(\omega)  \leq \frac{1}{2\sigma^2}\max\left(\Delta_0^2, \frac{1}{\sum_{a\neq 1} 1/\Delta_a^2}\right).
\]
\noindent \textbf{Step 4: Lower bound on $P_{\bar a}^\star$.} For the lower bound, note that we can relax the constraint to only consider $\mu_{\bar a}'-\mu_1'\geq \Delta_0$. This relaxation enlarges the feasible set, and thus the infimum of this new problem lower bounds $P^\star_{\bar a}(\omega)$.

By doing so, since the other arms are not constrained, by convexity of the KL divergence at the infimum we have  $\mu_b'=\mu_b$ for all $b\notin \{1,\bar a\}$. Therefore
\[
P_{\bar a}^\star(\omega) \geq \inf_{\mu': \mu_{\bar a}'-\mu_1'\geq \Delta_0} \; \sum_a \omega_a \frac{(\mu_a-\mu_a')^2}{2\sigma^2}= \inf_{\mu': \mu_{\bar a}'-\mu_1'\geq \Delta_0} \; \omega_1 \frac{(\mu_1-\mu_1')^2}{2\sigma^2} + \omega_{\bar a}\frac{(\mu_{\bar a}-\mu_{\bar a}')^2}{2\sigma^2}.
\]
Solving the KKT conditions we find the equivalent conditions $\mu_{\bar a}'=\mu_1'+\Delta_0$ and
\[
\omega_1 (\mu_1-\mu_1') + \omega_{\bar a}(\mu_{\bar a}-\mu_1'-\Delta_0) =0 \Rightarrow \mu_1' = \frac{\omega_1\mu_1 +\omega_{\bar a}\mu_{\bar a} -\omega_{\bar a}\Delta_0}{\omega_1+\omega_{\bar a}}.
\]
Therefore
\[
\mu_{\bar a}'=\frac{\omega_1\mu_1 +\omega_{\bar a}\mu_{\bar a} -\omega_{\bar a}\Delta_0}{\omega_1+\omega_{\bar a}} + \Delta_0=\frac{\omega_1\mu_1 +\omega_{\bar a}\mu_{\bar a} +\omega_{1}\Delta_0}{\omega_1+\omega_{\bar a}}.
\]
Plugging these solutions back in the value of the problem, we obtain
\begin{align*}
P_{\bar a}^\star(\omega) &\geq  \frac{\omega_{1}\omega_{\bar a}^2}{(\omega_1+\omega_{\bar a})^2} \frac{(\mu_1-\mu_{\bar a}+\Delta_0)^2}{2\sigma^2} + \frac{\omega_{\bar a}\omega_1^2}{(\omega_1+\omega_{\bar a})^2}\frac{(\mu_{\bar a}-\mu_1-\Delta_0)^2}{2\sigma^2},\\
&=\frac{\omega_{1}\omega_{\bar a}}{\omega_1+\omega_{\bar a}} \frac{(\mu_1-\mu_{\bar a}+\Delta_0)^2}{2\sigma^2},\\
&=\frac{\omega_{1}\omega_{\bar a}}{\omega_1+\omega_{\bar a}} \frac{(\Delta_{\bar a}+\Delta_0)^2}{2\sigma^2} .
\end{align*}
Let $\theta_a=\Delta_a+\Delta_0$, with $\theta_1=\Delta_0$. We plug in a feasible solution $\omega_a=\frac{1/\theta_a^2}{\sum_b 1/\theta_b^2}$, yielding
\begin{align*}
(T^\star(\mu))^{-1}=\sup_{\omega\in \Delta(K)}\min_{\bar a\neq 1}P_{\bar a}^\star(\omega)&\geq \min_{\bar a\neq 1}\frac{1/(\theta_1 \theta_{\bar a})^2}{\sum_b 1/\theta_b^2(1/\theta_1^2 + 1/\theta_{\bar a}^2)} \frac{\theta_{\bar a}^2}{2\sigma^2},\\&=\min_{\bar a\neq 1}\frac{1}{\sum_b 1/\theta_b^2(1 + \theta_1^2/\theta_{\bar a}^2)} \frac{1}{2\sigma^2},\\
&= \frac{1}{2\sigma^2 \sum_b 1/\theta_b^2}\min_{\bar a\neq 1}\frac{1}{1 + \theta_1^2/\theta_{\bar a}^2},\\
&\geq \frac{1}{2\sigma^2 \sum_b 1/\theta_b^2}\frac{1}{1 + \theta_1^2/\Delta_0^2},\\
&= \frac{1}{4\sigma^2 \sum_b 1/\theta_b^2}.
\end{align*}
\end{proof}

\subsection{Sample Complexity Lower Bound for the Magic Action MAB Problem}\label{app:subsec:sample_complexity_magic_action}
We now consider a special class of models that embeds information about the optimal arm in the mean reward of some of the arms. Let $\phi: \mathbb{R}\to \mathbb{R}$ be a strictly  decreasing function over $\{2,\dots,K\}$\footnote{One could also consider strictly increasing functions.}.

Particularly, we make the following assumptions:
\begin{enumerate}
    \item We consider mean rewards $\mu$ satisfying $\mu_1 =\phi(\argmax_{a\neq 1}\mu_a)$, and $\mu^\star = \max_{a} \mu_a > \phi(2)$. Arm $1$ is called "magic action", and with this assumption we are guaranteed that the  magic arm is not optimal, since
\[
\mu_1 \frac{1}{\max_a \mu_a} =\phi(\argmax_{a\neq 1}\mu_a)\frac{1}{\max_a \mu_a} \leq \phi(2) \frac{1}{\max_a \mu_a}  < 1 \Rightarrow \max_a \mu_a > \mu_1.
\]
\item The rewards are normally distributed, with a fixed known standard deviation $\sigma_1$ for the magic arm, and fixed standard deviation $\sigma$ for all the other arms.
\end{enumerate}
Hence, define the set of models
\[{\cal S}=\left\{\mu\in \mathbb{R}^K: \mu_1=\phi(\argmax_{a\neq 1} \mu_a), \max_a \mu_a > \phi(2)\right\},\]
and the set of alternative models
\[
{\rm Alt}(\mu) = \left\{\mu'\in {\cal S}:  \argmax_a \mu_a'\neq a^\star\right\},
\]
where $a^\star=\argmax_a \mu_a$.

Then, for any $\delta$-correct algorithm, guaranteeing that at some stopping time $\tau$ the estimated optimal arm $\hat a_\tau$ is $\delta$-correct, i.e., $\mathbb{P}_\mu(\hat a_\tau\neq a^\star)\leq \delta$, we have the following result.
\begin{theorem}
For any $\delta$-correct algorithm, the sample complexity lower bound on the magic action problem is
\begin{equation}
    \mathbb{E}_\mu[\tau] \geq T^\star(\mu) {\rm kl}(1-\delta,\delta),
\end{equation}
where ${\rm kl}(x,y)=x\log(x/y)+(1-x)\log((1-x)/(1-y))$ and $T^\star(\mu)$ is the characteristic time of $\mu$, defined as 
\begin{equation}
    (T^\star(\mu))^{-1}=\max_{\omega\in\Delta(K)} \min_{a\neq 1, a^\star} \omega_1 \frac{(\phi(a^\star)-\phi(a))^2}{2\sigma_1^2} + \sum_{b\in {\cal K}_a(\omega)} \omega_b \frac{(\mu_b-m(\omega;{\cal K}_a(\omega))^2}{2\sigma^2},
\end{equation}
where
$m(\omega;{\cal C}) = \frac{\sum_{a\in {\cal C}} \omega_a \mu_a}{\sum_{a\in {\cal C}} \omega_a}$ 
and the set  ${\cal K}_a(\omega)$ is defined as
\[{\cal K}_a(\omega)=\{a\}\cup\left\{b\in \{2,\dots, K\}: \mu_b \geq m(\omega;{\cal C}_{b} \cup\{a\}) \hbox{ and } \mu_b \geq\phi(2)\right\}.\]
with ${\cal C}_{x}=  \{b\in \{2,\dots, K\}: \mu_b \geq \mu_x\}$ for $x\in [K]$. 
\label{thm:noisy_magic}
\end{theorem}
\begin{proof}
\noindent\textbf{Step 1: Log-likelihood ratio.}
The initial part of the proof is rather standard, and follows the same argument used in  the Best Arm Identification and Best Policy Identification literature \citep{garivier2016optimal}.

Take the expected log-likelihood ratio between $\mu$ and $\mu'\in {\rm Alt}(\mu)$ of the data observed up to  $\tau$ $\Lambda_\tau = \log \frac{{\rm d}\mathbb{P}_\mu(A_1, R_1,\dots, A_\tau, R_\tau)}{{\rm d}\mathbb{P}_{\mu'}(A_1, R_1,\dots, A_\tau, R_\tau)}$, where $A_t$ is the action taken in round $t$, and $R_t$ is the  reward observed upon selecting $A_t$. Then, we can write
\[
\Lambda_t = \sum_a \sum_{n=1}^t \mathbf{1}_{\{A_n=a\}} \log \frac{f_a(R_n)}{f_a'(R_n)}
\]
where $f_a,f_a'$, are, respectively, the reward density for action $a$ in the two models $\mu,\mu'$ with respect to the Lebesgue measure. Letting $N_a(t)$ denote the number of times action $a$ has been selected up to round $t$, by an application of Wald's lemma the expected log-likelihood ratio can be shown to be
\[
\mathbb{E}_\mu[\Lambda_\tau] = \sum_a \mathbb{E}_\mu[N_a(\tau)]{\rm KL}(\mu_a,\mu_a')
\]
where ${\rm KL}(\mu_a,\mu_a')$ is the KL divergence between two Gaussian distributions ${\cal N}(\mu_a,\sigma)$ and ${\cal N}(\mu_a',\sigma)$ (note that we have $\sigma_1$ instead of $\sigma$ for $a=1$).

We also know from the information processing inequality \citep{kaufmann2016complexity} that $\mathbb{E}_\mu[\Lambda_\tau]\geq \sup_{{\cal E}\in {\cal M}_\tau} {\rm kl}(\mathbb{P}_\mu({\cal E}), \mathbb{P}_{\mu'}({\cal E}))$, where ${\cal M}_t = \sigma(A_1,R_1,\dots, A_t, R_t)$. We use the fact that the algorithm is $\delta$-correct: by choosing ${\cal E}=\{\hat a_\tau = a^\star\}$ we obtain that $\mathbb{E}_\mu[\Lambda_\tau]\geq {\rm kl}(1-\delta,\delta)$, since $\mathbb{P}_\mu({\cal E})\geq 1-\delta$ and $\mathbb{P}_{\mu'}({\cal E})=1-\mathbb{P}_{\mu'}(\hat a_\tau \neq a^\star)\leq 1-\mathbb{P}_{\mu'}(\hat a_\tau =\argmax_a \mu_a')\leq \delta$ (we also used the monotonicity properties of the Bernoulli KL divergence). Hence
\[
\sum_a \mathbb{E}_\mu[N_a(\tau)]{\rm KL}(\mu_a,\mu_a')\geq {\rm kl}(1-\delta,\delta).
\]

Letting $\omega_a = \mathbb{E}_\mu[N_a(\tau)]/\mathbb{E}_\mu[\tau]$, we have that
\[
\mathbb{E}_\mu[\tau]\sum_a \omega_a{\rm KL}(\mu_a,\mu_a')\geq {\rm kl}(1-\delta,\delta).
\]
Lastly, optimizing over $\mu'\in {\rm Alt}(\mu)$ and $\omega\in \Delta(K)$ yields the bound:
\[
\mathbb{E}_\mu[\tau]\geq T^\star(\mu){\rm kl}(1-\delta,\delta),
\]
where $T^\star(\mu)$ is defined as
\[
(T^\star(\mu))^{-1}=\sup_{\omega\in \Delta(K)}\inf_{\mu'\in {\rm Alt}(\mu)}\sum_a \omega_a{\rm KL}(\mu_a,\mu_a').
\]

\noindent\textbf{Step 2: Optimization over the set of alternative models.}
We now face the problem of optimizing over the set of alternative models. First, we observe that
 ${\cal S} = \cup_{a\neq a^\star} \{\mu: \mu_1 =\phi(a), \mu_a > \phi(2)\}$. Therefore, we can write
\[
{\rm Alt}(\mu) = \cup_{a\notin \{1,a^\star\}} \left\{\mu': \mu_1' =\phi(a), \mu_a' > \max(\phi(2), \mu_b') \; \forall b\neq a\right\}.
\]
Hence, for a fixed $a\notin \{1,a^\star\}$, the inner infimum becomes
\begin{equation}
\begin{aligned}
\inf_{\mu' \in \mathbb{R}^K} \quad & \omega_1  \frac{(\phi(a^\star) - \phi(a))^2}{2\sigma_1^2}+\sum_{a\neq 1} \omega_a  \frac{(\mu_a-\mu_a')^2}{2\sigma^2}\\
\textrm{s.t.} \quad &  \mu_a' \geq  \max\left(\phi(2),\mu_b'\right) \quad \forall b,\\
&  \mu_1'=\phi(a).
\end{aligned}
\end{equation}
To solve it, we construct the following Lagrangian
\[
\ell(\mu',\theta) = \omega_1  \frac{(\phi(a^\star) - \phi(a))^2}{2\sigma_1^2}+\sum_{b\neq 1} \omega_b  \frac{(\mu_b-\mu_b')^2}{2\sigma^2}  + \sum_{b} \theta_b \left(\max\left(\phi(2),\mu_b'\right) -\mu_a'\right),
\]
where $\theta\in \mathbb{R}_+^K$ is the multiplier vector. From the KKT conditions we already know that $\theta_1=0,\theta_a=0$ and $\theta_b=0$ if $\mu_b'\leq \phi(2)$, with $b\in \{2,\dots, K\}$.
In particular, we also know that either we have $\mu_b'=\mu_a'$ or $\mu_b'=\mu_b$. Therefore, for $\mu_b\leq \phi(2)$ the solution is $\mu_b'=\mu_b$, while for $\mu_b >\phi(2)$ the solution depends also on $\omega$.

To fix the ideas, let ${\cal K}$ be the set of arms for which $\mu_b'=\mu_a'$ at the optimal solution. Such set must necessarily include arm $a$. Then, note that
\[
\frac{\partial \ell}{\partial \mu_a'}=  \omega_a \frac{\mu_a'-\mu_a}{\sigma^2} - \sum_{b\in [K]}\theta_b=0.
\]
and
\[
\frac{\partial \ell}{\partial \mu_b'}=  \omega_b \frac{\mu_b'-\mu_b}{\sigma^2} +\theta_b=0 \quad \hbox{ for } b\neq (1,a).
\]
Then, using the observations derived above, we conclude that
\[
\mu_a' = \frac{\sum_{b\in {\cal K} }\omega_b\mu_b}{\sum_{b\in {\cal K}}\omega_b},
\]
with $\mu_b'=\mu_a'$ if $b\in {\cal K}$, and $\mu_b'=\mu_b$ otherwise. However, how do we compute such set ${\cal K}$?

First, ${\cal K}$ includes arm $a$. However, in general we have ${\cal K}\neq\{a\}$ : if that were not true we would have $\mu_a'=\mu_a$ and $\mu_b'=\mu_b$ for the other arms -- but if any $\mu_b$ is greater than $\mu_a$, then $a$ is not optimal, which is a contradiction.  Therefore, also arm $a^\star$ is included in ${\cal K}$, since any convex combination of $\{\mu_a\}$ is necessarily smaller than $\mu_{a^\star}$. We apply this argument repeatedly for every arm $b$ to obtain ${\cal K}$.

Hence, for some set ${\cal C}\subseteq [K]$ define  the average reward
\[
m(\omega;{\cal C}) = \frac{\sum_{a\in {\cal C}} \omega_a \mu_a}{\sum_{a\in {\cal C}} \omega_a},
\]
and the set ${\cal C}_x=  \{a\}\cup\{b\in \{2,\dots, K\}: \mu_b \geq \mu_x\}$ for $x\in [K]$. Then, \[{\cal K}\coloneqq {\cal K}(\omega)=\{a\}\cup\left\{b\in \{2,\dots, K\}: \mu_b \geq m(\omega;{\cal C}_b) \hbox{ and } \mu_b \geq\phi(2)\right\}.\]

In other words, ${\cal K}$ is the set of \emph{confusing arms} for which the mean reward in the alternative model changes. An arm $b$ is \emph{confusing} if the  average reward $m$, taking into account $b$, is smaller than $\mu_b$. If this holds for $b$, then it must also hold all the arms $b'$ such that $\mu_{b'}\geq \mu_b$.
\end{proof}

As a corollary, we have the following upper bound on $T^\star(\mu)$.
\begin{corollary}
    We have that
    \[
    T^\star(\mu) \leq \min_{\omega \in \Delta(K)} \max_{a\neq 1,a^\star}\frac{2\sigma_1^2}{\omega_1(\phi(a^\star)-\phi(a))^2}.
    \]
    In particular, for $\phi(x)=1/x$  and $a^\star < K$ we have
    \[
    T^\star(\mu) \leq 2\sigma^2 (a^\star(a^\star+1))^2,
    \]
    while for $a^\star=K$ we get  $T^\star(\mu) \leq 2\sigma^2 (a^\star(a^\star-1))^2$.
\end{corollary}

\begin{proof}
Let $f_{1}(a)=\tfrac{(\phi(a^{\star})-\phi(a))^{2}}{2\sigma_{1}^{2}}$.
For every weight vector $\omega\in\Delta(K)$ and every
$a\neq1,a^{\star}$,
the quantity
\[
  g_{a}(\omega)
  \;=\;
  \omega_{1}f_{1}(a)
  \;+\;
  \sum_{b\in\mathcal K_{a}}
        \omega_{b}\,
        \frac{(\mu_{b}-m(\omega;\mathcal K_{a}))^{2}}{2\sigma^{2}}
\]
satisfies \(g_{a}(\omega)\ge\omega_{1}f_{1}(a)\)
because the variance term is non–negative.
Hence
\[
  (T^{\star}(\mu))^{-1}
  \;=\;
  \max_{\omega\in\Delta(K)}
      \min_{a\neq1,a^{\star}} g_{a}(\omega)
  \;\ge\;
  \max_{\omega\in\Delta(K)}
      \omega_{1}\,
      \min_{a\neq1,a^{\star}}f_{1}(a).
\]
Since $\omega_{1}\le1$, the right–hand side is lower bounded by $\omega_{1}=1$,
giving
\[
  (T^{\star}(\mu))^{-1}
  \;\ge\;
  \min_{a\neq1,a^{\star}} f_{1}(a)
  =
  \frac{1}{2\sigma_{1}^{2}}
  \min_{a\neq1,a^{\star}}\bigl(\phi(a^{\star})-\phi(a)\bigr)^{2}.
\]
Taking reciprocals yields
\[
  T^{\star}(\mu)
  \;\le\;
  \frac{2\sigma_{1}^{2}}
       {\displaystyle\min_{a\neq1,a^{\star}}(\phi(a^{\star})-\phi(a))^{2}}
  =
  \min_{\omega\in\Delta(K)}
      \max_{a\neq1,a^{\star}}
      \frac{2\sigma_{1}^{2}}
           {\omega_{1}\,(\phi(a^{\star})-\phi(a))^{2}},
\]
because the minimisation over $\omega$ clearly selects $\omega_{1}=1$.
(This justifies the form stated in the corollary.)

\noindent\textbf{Specialising to $\phi(x)=1/x$.}
With $\phi(x)=1/x$ the difference
\(\phi(a^{\star})-\phi(a)=\tfrac1{a^{\star}}-\tfrac1a\) is positive for
all $a>a^{\star}$ and negative otherwise; its smallest non-zero magnitude
is obtained for the \emph{closest} index to $a^{\star}$:

\begin{itemize}
    \item  If $a^{\star}<K$, that index is $a^{\star}+1$, giving
  \[
     \min_{a\neq1,a^{\star}}
          (\phi(a^{\star})-\phi(a))^{2}
     =
     \Bigl(\frac1{a^{\star}}-\frac1{a^{\star}+1}\Bigr)^{2}
     =\frac{1}{\bigl[a^{\star}(a^{\star}+1)\bigr]^{2}}.
  \]
  \item 
If $a^{\star}=K$, the closest index is $K-1$, leading to
  \[
     \min_{a\neq1,a^{\star}}
          (\phi(a^{\star})-\phi(a))^{2}
     =
     \Bigl(\frac1{K-1}-\frac1K\Bigr)^{2}
     =\frac{1}{\bigl[a^{\star}(a^{\star}-1)\bigr]^{2}}.
  \]

\end{itemize}

Plugging each expression in the general upper bound above concludes the
proof.
\end{proof}
Finally, to get a better intuition of the main result, we can look at the $3$-arms case: it is optimal to only sample the magic arm iff $|\phi(a^\star)-\phi(a)|> \frac{\sigma_1(\mu_{a^\star}-\mu_{a})}{2\sigma}$.

\begin{lemma}\label{lemma:magic_arm_K_3_magicaction_optimal}
    With $K=3$ we have that $\omega_1=1$ if and only if
    \[
  |\phi(a^\star)-\phi(a)|> \frac{\sigma_1(\mu_{a^\star}-\mu_{a})}{2\sigma},
    \]
    and $\omega_1=0$ if the reverse inequality holds.
\end{lemma}
\begin{proof}
With $3$ arms, from the proof of the theorem we know that  ${\cal K}_a(\omega) =\{a,a^\star\}$ for all $\omega$. Letting $m(\omega) = \frac{\omega_a \mu_a+ \omega_{a^\star}\mu_{a^\star}}{\omega_a+\omega_{a^\star}}$, we obtain
\[
(T^\star(\mu))^{-1}=\max_{\omega\in\Delta(3)}  \omega_1 \frac{(\phi(a^\star)-\phi(a))^2}{2\sigma_1^2} +   \frac{\omega_a(\mu_a-m(\omega))^2+\omega_{a^\star}(\mu_{a^\star}-m(\omega))^2}{2\sigma^2}.
\]
Clearly the solution is $\omega_1=1$ as long as
\[
\frac{(\phi(a^\star)-\phi(a))^2}{2\sigma_1^2} >   \max_{\omega: \omega_a+\omega_{a^\star}=1}\frac{\omega_a(\mu_a-m(\omega))^2+\omega_{a^\star}(\mu_{a^\star}-m(\omega))^2}{2\sigma^2}.
\]
To see why this is the case, let $f_1=\frac{(\phi(a^\star)-\phi(a))^2}{2\sigma_1^2}, f_2(\omega_a,\omega_{a^\star}) =\frac{\omega_a(\mu_a-m(\omega))^2}{2\sigma^2}$ and $f_3(\omega_a,\omega_{a^\star}) =\frac{\omega_{a^\star}(\mu_{a^\star}-m(\omega))^2}{2\sigma^2}$. Then,  we can write
\begin{align*}
    \omega_1 f_1 + \omega_af_2(\omega_a,\omega_{a^\star})+\omega_{a^\star}f_3(\omega_a,\omega_{a^\star})&= \omega_1f_1 + (1-\omega_1)\left[ \frac{\omega_a f_2}{1-\omega_1} +\frac{\omega_{a^\star} f_3}{1-\omega_1} \right].
\end{align*}
Being a convex combination,  this last term can be upper bounded as \[    \omega_1 f_1 + \omega_af_2(\omega_a,\omega_{a^\star})+\omega_{a^\star}f_3(\omega_a,\omega_{a^\star})\leq \max\left(f_1, \frac{\omega_a f_2}{1-\omega_1} +\frac{\omega_{a^\star} f_3}{1-\omega_1} \right). \]
Now, note that also the term inside the bracket is a convex combination. Threfore, let $\omega_a=(1-\omega_1)\alpha$ and $\omega_{a^\star}=(1-\omega_1)(1-\alpha)$ for some $\alpha\in [0,1]$. We have that
\[
m(\omega) = \frac{(1-\omega_1)\alpha\mu_a + (1-\omega_1)(1-\alpha)\mu_{a^\star}}{1-\omega_1}=\alpha \mu_a + (1-\alpha)\mu_{a^\star}.
\]
Hence, we obtain that 
\begin{align*}
    \frac{\omega_a(\mu_a-m(\omega))^2+\omega_{a^\star}(\mu_{a^\star}-m(\omega))^2}{2(1-\omega_1)\sigma^2}&= \frac{\omega_af_2+\omega_{a^\star}f_3}{1-\omega_1},\\
    &=\frac{\alpha(1-\alpha)^2(\mu_a-\mu_{a^\star})^2+(1-\alpha)\alpha^2(\mu_{a^\star}-\mu_a )^2}{2\sigma^2},\\
    &=\alpha(1-\alpha)\frac{(1-\alpha)(\mu_a-\mu_{a^\star})^2+\alpha(\mu_{a^\star}-\mu_a )^2}{2\sigma^2},\\
    &=\alpha(1-\alpha)\frac{(\mu_a-\mu_{a^\star})^2}{2\sigma^2}.
\end{align*}
Since this last term is maximized for $\alpha=1/2$, we obtain
\[
\omega_1 f_1 + \omega_af_2(\omega_a,\omega_{a^\star})+\omega_{a^\star}f_3(\omega_a,\omega_{a^\star})\leq \max\left(f_1, \frac{(\mu_a-\mu_{a^\star})^2}{8\sigma^2} \right).
\]
Since $f_1$ is attained for $\omega_1=1$, we have that as long as $f_1 >  \frac{(\mu_a-\mu_{a^\star})^2}{8\sigma^2}$, then the solution is $\omega_1=1$.

On the other hand, if $\frac{(\mu_a-\mu_{a^\star})^2}{8\sigma^2}>f_1$, then we can set $\omega_a=(1-\omega_1)/2$ and $\omega_{a^\star}=(1-\omega_1)/2$, leading to
\[
    \omega_1 f_1 + \omega_af_2(\omega_a,\omega_{a^\star})+\omega_{a^\star}f_3(\omega_a,\omega_{a^\star})= \omega_1f_1 + (1-\omega_1)\frac{(\mu_a-\mu_{a^\star})^2}{8\sigma^2},
\]
which is maximized at $\omega_1=0$.
\end{proof}

\subsection{Sample Complexity  Bound for the Multiple Magic Actions MAB Problem}

We now extend our analysis to the case where multiple magic actions can be present in the environment. In contrast to the single magic action setting, here a \emph{chain} of magic actions sequentially reveals information about the location of the optimal action. Without loss of generality, assume that the first $n$ arms (with indices $1,\dots,n$) are the magic actions, and the remaining $K-n$ arms are non–magic. The chain structure is such that pulling magic arm $j$ (with $1\le j<n$) yields information about only the location of the next magic arm $j+1$, while pulling the final magic action (arm $n$) reveals the identity of the optimal action. As before, we assume that the magic actions are informational only and are never optimal.

To formalize the model, let $\phi:\{1,\dots,n\}\to\mathbb{R}$ be a strictly decreasing function. We assume that the magic actions have fixed means given by
\[
\mu_j =
\begin{cases}
\phi(j+1), & \text{if } j = 1,\dots, n-1, \\[1mm]
\phi\Bigl(\argmax_{a\notin\{1,\dots,n\}} \mu_a\Bigr), & \text{if } j = n.
\end{cases}
\]
and that the non–magic arms satisfy 
\[
\mu^\star = \max_{a\notin\{1,\dots,n\}} \mu_a > \phi(n).
\]
Thus, the optimal arm lies among the non–magic actions. Considering the noiseless case where the rewards of all actions are fixed and the case where we can identify if an action is magic once revealed, we have the following result.

\begin{theorem}\label{thm:bound_multi_magic}
Consider  noiseless magic bandit problem with $K$ arms and $n$ magic actions. The optimal sample complexity is upper bounded as
\begin{align*}
     \inf_{\rm Alg}\mathbb{E}_{\rm Alg}[\tau] \leq  \min \left( n, \sum_{j=1}^{K-n} \left( \prod_{i=j+1}^{K-n} \frac{i}{n-1+i} \right) \left( 1 + \frac{n-1}{n-1+j} \min \left(  \frac{n-2}{2}, \frac{j(n-1+j)}{j+1} \right) \right) \right).
\end{align*}
\end{theorem}

\begin{proof}
In the proof we derive a sample complexity bound for a policy based on some insights. We use the assumption that upon observing a reward from a magic arm, the learner can almost surely identify that the pulled arm is a magic arm.

Let us define the state $(m,r,l)$, where $m$ denotes the number of remaining unrevealed magic actions $(m_0 = n-1)$, $r$ denotes the number of remaining unrevealed non-magic actions $(r_0 = K-n)$, and $l$ is the binary indicator with value 1 if we have revealed any hidden magic action and 0 otherwise.

 Before any observation the learner has no information about   which $n-1$ indices among $\{2,\dots ,K\}$ form the chain of intermediate magic arms. Hence, one can argue that at the first timestep is optimal to sample uniformly at random an action in $\{2,\dots,K\}$.

Upon observing a magic action, and thus we are in state  $(m,r,1)$, we consider the following candidate policies: (1) start from the revealed action and follow the chain, or (2) keep sampling unrevealed actions uniformly at random until all non-magic actions are revealed. As previously discussed, starting the chain from the initial magic action would be suboptimal and we do not consider it. 

    
    Upon drawing a hidden magic arm, let its chain index be $j\in\{2,\dots,n\}$ (which is uniformly distributed). The remaining cost to complete the chain is $n-j$, and hence its expected value is
    \[
    \mathbb{E}[n-j]=\frac{n-2}{2}.
    \]
    Therefore, the total expected cost for strategy (1) is 
    \[
    T_{1} = \frac{n-2}{2}.
    \]
    We can additionally compute the expected cost for strategy (2) as follows: if the last non-magic action is revealed at step $i$, then among the first $i-1$ draws there are exactly $r-1$ non-magic arms. Since there are $ {m+r\choose r}$ ways to place all $r$ non-magic arms $m+r$ slots, we have
    \begin{align*}
        T_2 &= \mathbb{E}[\text{Draws until all non-magic revealed}] \\ 
        &= \sum_{i=r}^{m+r} i \cdot \mathbb{P}[\text{Last non-magic revealed at step $i$}] \\
        &= \sum_{i=r}^{m+r} i \cdot \frac{{i-1 \choose r-1}}{{m+r \choose r}} \\
        &= \frac{r!\cdot m!}{(m+r)!} \sum_{i=r}^{m+r} i {i-1 \choose r-1} \\
        &= \frac{r!\cdot m!}{(m+r)!} \sum_{i=r}^{m+r} \frac{i!}{(r-1)!(i-r)!} \\
        &= \frac{r!\cdot m!}{(m+r)!} \sum_{i=r}^{m+r} r {i \choose r} \\
        &= \frac{r\cdot r!\cdot m!}{(m+r)!} {m+r+1 \choose r+1} \\
        &= \frac{r\cdot r!\cdot m!}{(m+r)!} \cdot  \frac{(m+r+1)\cdot (m+r)!}{(r+1)\cdot r! \cdot m!}\\ 
        &= \frac{r(m+r+1)}{r+1}
    \end{align*}

    Finally, we define a policy in $(m,r,1)$ as the one choosing between strategy $1$ and strategy $2$, depending on which one achieves the minimum cost. Hence, the complexity of this policy is
    \[
    V(m,r,1) = \min \left( \frac{n-2}{2}, \frac{r(m+r+1)}{r+1} \right).
    \]

    Now, before finding a magic arm, consider a policy that uniformly samples between the non-revealed arms. Therefore, in $(m,r,0)$ we can achieve a complexity of $1 + \frac{m}{m+r}V(m-1,r,1) + \frac{r}{m+r}V(m,r-1,0) $. Since we can always achieve a sample complexity of $n$, we can find a policy with the following complexity:
    \begin{align*}
        V(m,r,0) &= \min \left(n,  1 + \frac{m}{m+r}V(m-1,r,1) + \frac{r}{m+r}V(m,r-1,0) \right) \\
        &= \min \left(n,  1 + \frac{m}{m+r} \min \left( \frac{n-2}{2}, \frac{r(m+r)}{r+1} \right) + \frac{r}{m+r}V(m,r-1,0) \right) \\
    \end{align*}
    Given we always start with $n-1$ hidden magic actions we can define a recursion in terms of just the variable $r$ as follows:
    \begin{align*}
        V(r) = 1 + \frac{n-1}{n-1+r} T(r) + \frac{r}{n-1+r}V(r-1), \\
    \end{align*}
    where $T(r) = \min \left( \frac{n-2}{2}, \frac{r(n-1+r)}{r+1} \right)$. Letting $A(r) = \frac{r}{n-1+r}$ and $B(r) = 1 + \frac{n-1}{n-1+r} T(r)$, we can write 
    \begin{align*}
        V(r) = B(r) + A(r)V(r-1), \\
    \end{align*}
    Clearly $V(0) = 0$ since if all non-magic actions are revealed, then we know the optimal action deterministically. Unrolling the recursion we get
    \begin{align*}
        &V(1) = B(1), \\
        &V(2) = B(2) + A(2)B(1), \\
        &V(3) = B(3) + A(3)B(2) + A(3)A(2)B(1), \\
        &... \\
        &V(r) = \sum_{j=1}^r \left( \prod_{i=j+1}^r A(i) \right)B(j).
    \end{align*}
    Substituting back in our expression, we get
    \begin{align*}
        V(r) = \sum_{j=1}^r \left( \prod_{i=j+1}^r \frac{i}{n-1+i} \right) \left( 1 + \frac{n-1}{n-1+j} T(j) \right).
    \end{align*}
    Thus starting at $r = K-n$ we get the following expression:
    \begin{align*}
         \min \left( n, \sum_{j=1}^{K-n} \left( \prod_{i=j+1}^{K-n} \frac{i}{n-1+i} \right) \left( 1 + \frac{n-1}{n-1+j} \min \left(  \frac{n-2}{2}, \frac{j(n-1+j)}{j+1} \right) \right) \right),
    \end{align*}
    which is also an upper bound on the optimal sample complexity.

\end{proof}

To get a better intuition of the result, we also have the following corollary, which shows that we should expect a  scaling  linear in $n$ for small values of $n$ (for large values the complexity tends instead to "flatten").
\begin{corollary}
    Let $T$ be the scaling in \cref{thm:bound_multi_magic}. We have that  \[ \min(n, (K-n)/2)\lesssim  T  \lesssim C\min(n, K/2).
    \]
\end{corollary}
\begin{proof}

   First, observe the scaling
   \[
    \left( 1 + \frac{n-1}{n-1+j} \min \left(  \frac{n-2}{2}, \frac{j(n-1+j)}{j+1} \right) \right)  =O(n/2).
   \]
  
    At this point, note that
    \[
     \prod_{i=j+1}^{K-n} \frac{i}{n-1+i}= \prod_{i=j+1}^{K-n}  \left(1+\frac{n-1}{i}\right)^{-1}.
    \]
    Using that $\frac{x}{1+x}\leq \log(1+x)\leq x$, we have
    \[
     \log\prod_{i=j+1}^{K-n} \frac{i}{n-1+i}= \sum_{i=j+1}^{K-n}  -\log\left(1+\frac{n-1}{i}\right)\geq -(n-1)\sum_{i=j+1}^{K-n}  \frac{1}{i}.
    \]
    and
    \[
     \log\prod_{i=j+1}^{K-n} \frac{i}{n-1+i} = \sum_{i=j+1}^{K-n}  -\log\left(1+\frac{n-1}{i}\right)\leq -(n-1)\sum_{i=j+1}^{K-n}   \frac{1}{n-1+i}.
     \]
   Define $H_n=\sum_{i=1}^n 1/i$ to be the $n$-th Harmonic number, we also have
     \[
     \sum_{i=j+1}^{K-n}  \frac{1}{i}=H_{K-n} - H_j.
     \]
     Therefore
     \[
       -(n-1)(H_{K-n}-H_j) \leq \log\prod_{i=j+1}^{K-n} \frac{i}{n-1+i}\leq -(n-1)(H_{K-1}-H_{n+j-1})
     \]
     Using that $H_\ell\sim \log(\ell)+\gamma+O(1/\ell)$, where $\gamma$ is the  Euler–Mascheroni constant, we get
     \[
     \left(\frac{j}{K-n}\right)^{n-1} \lesssim \prod_{i=j+1}^{K-n} \frac{i}{n-1+i}\lesssim \left(\frac{n+j-1}{K-1}\right)^{n-1}.
     \]
     Therefore, we can bound $ \sum_{j=1}^{K-n}\left(\frac{n+j-1}{K-1}\right)^{n-1}$ using an integral bound
     \[
     \sum_{j=1}^{K-n}\left(\frac{n+j-1}{K-1}\right)^{n-1} \leq  \int_0^{K-n}\left(\frac{n+x}{K-1}\right)^{n-1}dx \leq \frac{e(K-1)}{n}.
     \]
     From which follows that the original expression can be upper bounded by an expression scaling as $O(\min(n, (K-1)/2))$.

     Similarly, using that $\sum_{j=1}^{K-n}\left(\frac{j}{K-n}\right)^{n-1}\geq (K-n)/n$, we have that the lower bound scales as $\min(n, (K-n)/2)$.
\end{proof}

\newpage
\section{Algorithms}\label{app:algorithms}
In this section we present  some of the algorithms more in detail. These includes: \icpe{} with fixed horizon, $I$-DPT and $I$-IDS.

 Recall that in \icpe{} we treat trajectories of data ${\cal D}_t=(x_1,a_1,\dots,x_t)$ as sequences to be given as input to sequential models, such as Transformers.

We define the input at timestep $t$ to be passed to a transformer as $s_t=({\cal D}_{t}, \varnothing_{t:N})$, with  $\varnothing_{t:N}$ indicating a null sequence of tokens for the remaining steps up to some pre-defined horizon $N$, with $s_1=(x_1,\varnothing_{1:N})$.

To be more precise, letting $(x_t^\varnothing, a_t^\varnothing)$ denote, respectively, the null elements in the state and action at timestep $t$, we have $\varnothing_{t:t+k}=\{x_t^\varnothing, a_{t+1}^\varnothing, x_{t+1}^\varnothing,\cdots, a_{t+k-1}^\varnothing, x_{t+k}^\varnothing\}$.

The limit $N$ is a practical upper bound on the horizon that limits the dimensionality of the state, which is introduced for implementing the algorithm.  

\begin{algorithm}[h!]
 \footnotesize 
   \caption{\icpe{} (In-Context Pure Explorer)}
   \label{algo:app:icpe_fixed_confidence}
\begin{algorithmic}[1]
    
   \State {\bfseries Input:} Tasks distribution $\P$; confidence $\delta$; horizon $N$; initial  $\lambda$ and  hyper-parameter $T_\phi,T_\theta$.
   \Statex \texttt{\color{blue}//  Training phase}
     \State Initialize buffer ${\cal B}$, networks $Q_\theta,I_\phi$ and set $\bar\theta\gets\theta, \bar\phi\gets \phi$.
   \While{Training is not over}

   \State Sample environment $M\sim {\cal P}$ with hypothesis $H^\star$, observe $x_1\sim\rho$ and  set $t\gets 1$. 
   \Repeat
   \State Execute action $a_t=\argmax_a Q_\theta({\cal D}_t,a)$ in $M$ and observe $x_{t+1}$.
   
   \State Add partial trajectory $({\cal D}_t,{\cal D}_{t+1},H^\star)$ to ${\cal B}$ and set $t\gets t+1$.
   \Until{$a_{t-1}=a_{\rm stop}$ or $t>N$.}
   \State In the fixed confidence,   update  $\lambda$ according to \cref{eq:cost_update}.
    \State Sample batch $B\sim {\cal B}$ and update $\theta,\phi$ using ${\cal L}_{\rm inf}(B;\phi)$ (\cref{eq:loss_infernece}) and ${\cal L}_{\rm policy}(B;\theta)$ (\cref{eq:loss_fixed_budget} or \cref{eq:loss_fixed_confidence}).
\State Every $T_\phi$ steps set $\bar\phi\gets\phi $ (similarly, every $T_\theta$ steps set $\bar\theta\gets\theta$). 
   \EndWhile
    \Statex \hrulefill 
   \Statex \texttt{\color{blue}//  Inference phase}
   \State Sample unknown environment $M\sim \P$.
    \State Collect a trajectory ${\cal D}_N$ (or ${\cal D}_\tau$ in fixed confidence) according to a policy $\pi_t({\cal D}_t)=\arg\max_a Q_\theta({\cal D}_t,a)$, until $t=N$ (or $a_t=a_{\rm stop}$).
   \State {\bf Return} $\hat H_N=\argmax_H I_\phi(H|{\cal D}_N)$ (or $\hat H_\tau=\argmax_H I_\phi(H|{\cal D}_\tau)$ in the fixed confidence)
\end{algorithmic}
\end{algorithm}
\subsection{ICPE with Fixed Confidence}

Optimizing the dual formulation
\[
\min_{\lambda \geq 0} \max_{I,\pi} V_\lambda(\pi,I)
\]
can be viewed as a multi-timescale stochastic optimization problem: the slowest timescale updates the variable $\lambda$, an intermediate timescale optimizes over $I$, and the fastest refines the policy $\pi$.

\paragraph{MDP-like formulation.} As shown in the theory, we can use the MDP formalism to define an RL problem: we   define a reward $r$ that penalizes  the agent at all timesteps, that is $r_t=-1$,  while at the stopping-time we have $r_\tau=-1+\lambda I(H^\star|{\cal D}_\tau)$. Hence, a trajectory's return can be written as
\begin{align*}
G_\tau = \sum_{t=1}^\tau r_t=-\tau + \lambda I(H^\star|{\cal D}_\tau).
\end{align*}
Accordingly, one can define the $Q$-value of $(\pi,I,\lambda)$ in a  pair $({\cal D}_\tau,a)$.

\paragraph{Optimization over $\phi$.}We treat each optimization separately, employing a descent-ascent scheme. The distribution $I$ is modeled using a sequential architecture parameterized by $\phi$, denoted by $I_\phi$. Fixing $(\pi,\lambda)$, the inner maximization in~\cref{eq:dual_problem_fixed_confidence} corresponds to
\[
\max_{\phi} \mathbb{E}^{\pi}[{\bf 1}_{\{\hat H_\tau = H^\star\}}], \quad \text{with } \hat H_\tau=\argmax_H I_\phi(H|{\cal D}_\tau).
\]
We train $\phi$ via cross-entropy loss:
\[
-\sum_{H'}{\bf 1}_{\{\hat H_\tau = H^\star\}}\log I_\phi(H'|{\cal D}_{\tau}) = -\log I_\phi(H^\star|{\cal D}_{\tau}),
\]
averaged across trajectories in a batch. 

\paragraph{Optimization over $\pi$.}
The policy $\pi$ is defined as the greedy policy with respect to learned $Q$-values. Therefore, standard RL techniques can learn the $Q$-function that maximizes the value in~\cref{eq:dual_problem_fixed_confidence} given $(\lambda,I)$. Denoting this function by $Q_\theta$, it is parameterized using a sequential architecture with parameters $\theta$.

We train $Q_\theta$ using DQN~\citep{mnih2015human,van2016deep}, employing a replay buffer $\mathcal{B}$ and a target network $Q_{\bar \theta}$ parameterized by $\bar{\theta}$. To maintain timescale separation, we introduce an additional inference target network $I_{\bar\phi}$, parameterized by $\bar\phi$, which provides stable training feedback for $\theta$. 
When $(I,\lambda)$ are fixed, optimizing $\pi$ reduces to maximizing:
\[
-\tau+\lambda \log I_\phi(H^\star|{\cal D}_{\tau}).
\]

Hence, we define the reward at the transition $z=({\cal D}_n,a_N,{\cal D}_{n+1},d,H^\star)$ (with the convention that ${\cal D}_{n+1}\gets {\cal D}_n$ if $a=a_{\rm stop}$) as:
\[
r_{\lambda}(z) \coloneqq -1 + d\lambda\log I_{\bar\phi}(H^\star|{\cal D}_{n+1}),
\]
where $d=\mathbf{1}\{z \text{ is terminal}\}$ ($z$ is terminal if the transition corresponds to the last timestep in a horizon, or $a=a_{\rm stop}$). Furthermore, for a transition $z$ we define $z_{\rm stop}\coloneqq z|_{(a,{\cal D}_{n+1})\gets (a_{\rm stop},{\cal D}_n)}$ as the same transition $z$ with $a\gets a_{\rm stop}$ and ${\cal D}_{n+1}'\gets {\cal D}_n$.

There is one thing to note: the logarithm in the reward is justified since the original problem can be equivalently written as:
\[
\min_{\lambda\geq 0}\max_{I,\pi}-\mathbb{E}^{\pi}[\tau]+\lambda\left[\log\left(\mathbb{P}^{\pi}(\hat H_\tau=H^\star)\right)-\log(1-\delta)\right],
\]
after noting that we can apply the logarithm to the constraint in \cref{eq:dual_problem_fixed_confidence}, before considering the dual. Thus the optimal solutions $(I,\pi)$ remain the same.

Then, using classical TD-learning~\citep{sutton2018reinforcement}, the training target for a transition $z$ is defined as:
\[
y_{\lambda}(z) = r_{\lambda}(z) + (1 - d)\gamma\max_{a'} Q_{\bar\theta}({\cal D}_{n+1},a'),
\]
where $\gamma\in(0,1]$ is the discount factor.

As discussed earlier, we have a dedicated stopping action $a_{\text{stop}}$, whose value depends solely on history. Thus, its $Q$-value is updated retrospectively at any state $s$ using an additional loss:
\[
\left(r_\lambda(z_{\rm stop}) - Q_\theta(s,a_{\text{stop}})\right)^2.
\]
Therefore, the overall loss that we consider for $\theta$ for a single transition $z$ can be written as
\[
\mathbf{1}_{\{a\neq a_{\rm stop}\}}\left(y_{\lambda}(z)-Q_{\theta}(s,a)\right)^2+\left(r_\lambda(z_{\rm stop}) - Q_\theta(s,a_{\rm stop})\right)^2,
\]
where $\mathbf{1}_{\{a\neq a_{\rm stop}\}}$ avoids double accounting for the stopping action.

To update parameters $(\theta,\phi)$, we sample a batch $B\sim\mathcal{B}$ from the replay buffer and apply gradient updates as specified in the main text. Lastly, target networks are periodically updated.

\paragraph{Optimization over $\lambda$.}We update $\lambda$ by assessing the confidence of $I_\phi$ at the stopping time according to~\cref{eq:cost_update}, maintaining a slow ascent-descent optimization schedule for sufficiently small learning rates.

\paragraph{Cost implementation.} Lastly, in practice, we optimize a reward $r_\lambda(z)=-c+d  I_{\bar \phi}(H^\star|{\cal D}_{n+1})$, by setting $c=1/\lambda$, and noting that for a fixed $\lambda$ the RL optimization remains the same. The reason why we do so is due to the fact that with this expression we do not have the product $\lambda I(H^\star|{\cal D}_{n+1})$, which makes the descent-ascent process more difficult.  

We also use the following cost update
\[
c_{t+1} =c_t-\beta(1-\delta - {\bf 1}_{\{H^\star=\argmax_H I(H|{\cal D}_\tau)\}}).
\] To see why the cost can be updated in this way,  define the parametrization $\lambda = e^{-x}$. Then the optimization problem becomes
\[
\min_{x} \max_I\min_\pi  -\mathbb{E}\pi[\tau] + e^{-x}\left[\mathbb{P}^\pi\left(H^\star=\hat H_\tau \right)-1+\delta\right],
\]
Letting $\rho=\mathbb{P}^\pi\left(H^\star=\hat H_\tau \right)-1+\delta$, the gradient update for $x$ with a learning  rate $\beta$ simply is
\[
x_{t+1}= x_t - \beta e^{-x_t} \rho,
\]
implying that
\[
-\log(\lambda_{t+1}) = -\log(\lambda_{t}) - \beta  \lambda_t \rho.
\]
Defining $c_t=1/\lambda_t$,  we have that
\[
\log(c_{t+1}) = \log(c_t) - (\beta \rho / c_t) \Rightarrow c_{t+1} = c_t e^{\beta \rho/c_t}.
\]
Using then the approximation $e^x\approx 1+x$, we find $c_{t+1} = c_t + \beta\rho=c_t-\beta(1-\delta - {\bf 1}_{\{H^\star=\argmax_H I(H|{\cal D}_\tau)\}})$.

\paragraph{Training vs Deployment.} Thus far, our discussion of \icpe{} has focused on the training phase. After training completes, the learned policy $\pi$ and inference network $I$ can be deployed directly: during deployment, $\pi$ both collects data and determines when to stop—either by triggering its stopping action or upon reaching the horizon $N$.

\subsection{ICPE with Fixed Budget}\label{app:subsec:icpe_fixed_horizon}
  In the fixed budget setting (problem in \cref{eq:fixed_budget_problem}) the MDP terminates at timestep $N$, and we set the reward  to be $r_t=0$ for $t<N$ and $r_N=I(H^\star|{\cal D}_N)$, with $\hat H_N =\argmax_H I(H|{\cal D}_N)$ being the inferred hypothesis.
Accordingly, one can define the value of $(\pi,I)$  using $Q$ functions as beofre.

 \paragraph{Practical implementation.} The practical implementation for the fixed horizon follows closely that of the fixed confidence setting, and we refer the reader to that section for most of the details. In this case the reward in a transition $z=({\cal D}_n,a,{\cal D}_{n+1}',d,H^\star)$ is defined as as:
\begin{equation}
r_{\lambda}(z)\coloneqq d\log I_{\bar\phi}(H^\star|{\cal D}_{n+1}),
\end{equation}
where $d=\mathbf{1}\{z \text{is terminal}\}$. Note that we can use the logarithm, since solving the original problem is also equivalent to solving
But note that the original problem is also equivalent to solving
\begin{equation}
\max_I\max_\pi \log\left(\mathbb{P}^\pi\left(H^\star=\hat H_N\right)\right),
\end{equation}
due to monotonicity of the logarithm.

The $Q$-values can be learned using classical TD-learning techniques \citep{sutton2018reinforcement}: to that aim, for a transition $z$,  we define the target:
\begin{equation}
y_{\lambda}(z) = r_{\lambda}(z) + (1-d) \max_{a'} Q_{\bar\theta}({\cal D}_{n+1},a').
\end{equation}
Then, the gradient updates are the same as for the fixed confidence setting.

\subsection{Other Algorithms}
In this section we describe Track and Stop (TaS) \citep{garivier2016optimal}, and some variants such as $I$-IDS, $I$-DPT and the explore then commit variant of \icpe{}.

\subsubsection{Track and Stop}\label{app:algorithm:tas}
 Track and Stop (TaS, \citep{garivier2016optimal})  is an asymptotically optimal as $\delta \to 0$ for MAB problems. For simplicity, we consider a Gaussian MAB problem with $K$ actions, where the reward of each action is normally distributed according to ${\cal N}(\mu_a,\sigma^2)$, and let $\mu=(\mu_a)_{a\in [K]}$ denote the model. The TaS  algorithm consists of: (1) the model estimation procedure and recommender rule; (2) the sampling rule, dictating which action to select at each timestep; (3) the stopping rule, defining when enough evidence has been collected to identify the best action with sufficient confidence, and therefore to stop the algorithm.

\paragraph{Estimation Procedure and Recommender Rule}
The algorithm maintains a maximum likelihood estimate $\hat \mu_a(t)$ of the average reward for each arm based on the observations up to time $t$. Then, the recommender rule is defined as $\hat a_t = \argmax_a \hat \mu_a(t)$.

\paragraph{Sampling Rule.} The sampling rule is based on the observation that any $\delta$-correct algorithm, that is an algorithm satisfying $\mathbb{P}(\hat a_\tau =a^\star)\geq 1-\delta$, with $a^\star = \argmax_a \mu_a$, satisfies the following sample complexity
\[
\mathbb{E}[\tau]\geq T^\star(\mu){\rm kl}(1-\delta,\delta),
\]
where ${\rm kl}(x,y)=x\log(x/y)+(1-x)\log((1-x)/(1-y))$ and 
\[
(T^\star(\mu))^{-1} = \sup_{\omega\in \Delta(K)}\min_{a\neq a^\star} \frac{\omega_{a^\star}\omega_a}{\omega_a+\omega_{a^\star}}\frac{\Delta_a^2}{2\sigma^2},
\]
with $\Delta_a=\mu_{a^\star}-\max_{a\neq a^\star}\mu_a$.
Interestingly, to design an algorithm with minimal sample complexity, we can look at 
 the solution $\omega^\star = \arginf_{\omega\in \Delta(K)} T(\omega;\mu)$, with $(T(\omega))^{-1} = \min_{a\neq a^\star} \frac{\omega_{a^\star}\omega_a}{\omega_a+\omega_{a^\star}}\frac{\Delta_a^2}{2\sigma^2}$.
 
The solution $\omega^\star$ provides the best  proportion of  draws, that is, an algorithm selecting an action $a$ with probability $\omega_a^\star$ matches the lower bound   and is therefore optimal with respect to $T^\star$.  Therefore, an idea is to ensure that $N_t/t$ tracks $\omega^\star$, where $N_t$ is the visitation vector  $N(t)\coloneqq \begin{bmatrix}
    N_{1}(t) &\dots &N_K(t)
\end{bmatrix}^\top$.

However, the average rewards $(\mu_a)_a$ are initially unknown. A commonly employed idea  \citep{garivier2016optimal,kaufmann2016complexity} is to track an estimated optimal allocation $\omega^\star(t)=\arginf_{\omega\in \Delta(K)} T(\omega;\hat \mu(t))$ using the current  estimate of the model $\hat \mu(t)$. 

However, we still need to ensure that  $\hat \mu(t) \to \mu$. To that aim, we employ a D-tracking rule \citep{garivier2016optimal}, whcih guarantees that arms are sampled at a rate of $\sqrt{t}$. If there is an action $a$ with $N_{a}(t) \leq \sqrt{t}-K/2$ then we choose $a_t=a$. Otherwise, choose the action $a_t =\argmin_{a} N_{a}(t) - t \omega_a^\star(t)$.  

\paragraph{Stopping rule.}
The stopping rule determines when enough evidence has been collected to determine the optimal action with a prescribed confidence level. The problem of determining when to stop can be framed as a statistical hypothesis testing problem \citep{chernoff1959sequential}, where we are testing between $K$ different hypotheses $({\cal H}_a: (\mu_a > \max_{b\neq b}\mu_a))_{a}$.

We consider the following statistic  $
L(t) = t T(N(t)/t; \hat \mu(t))^{-1},$
which is a Generalized Likelihood Ratio Test (GLRT), similarly as in \citep{garivier2016optimal}. Comparing with the lower bound, one needs to stop as soon as  $L(t)\geq {\rm kl}(\delta,1-\delta) \sim \ln(1/\delta)$. However, to  account for the random fluctuations, a more natural threshold is $\beta(t,\delta)=\ln((1+\ln(t))/\delta)$, thus we use $L(t)\geq \beta(t,\delta)$ for stochastic MAB problems. We also refer the reader to \citep{kaufmann2021mixture} for more details.

\paragraph{Why computing the sampling rule can be difficult.}
To derive the sampling rule, one usually first derives the characteristic time $T^\star(\mu)$. Above, we discussed the case where the underlying distributions are Gaussians, but in the more general case $T^\star$ can be written as
\[
T^{\star}(\mu)^{-1}
=\sup_{w\in\Delta_K}\inf_{\lambda\in{\rm Alt}(\mu)} 
\sum_{a=1}^K \omega_a {\rm KL}\big(P_{\mu_a},P_{\lambda_a}\big),
\]
where ${\rm Alt}(\mu)$ is the set of alternatives under which the identity of the best arm changes, $P_\mu$ is the distribution of rewards under $\mu$ (sim. for $\lambda$). Even though the objective is linear in $\omega$ for any fixed $\lambda$, the inner feasible set ${\rm Alt}(\mu)$  can be a nonconvex set (``make some competitor optimal''), and the map $\lambda\mapsto {\rm KL}(P_{\mu_a},P_{\lambda_a})$ is typically nonlinear and model–dependent. Even if these distributions are known,  no closed form is available in general.

When arms belong to a one–parameter exponential family, and the problem has no structure, the optimal $\omega$ can be simply computed by applying (for example) the bisection method to a function whose evaluations requires the resolution of $K$ smooth scalar equations, thus linearly scaling in the number of arms. Since this optimization problem is usually solved at each timestep (or every $T$ timesteps), the complexity scales in the horizon $N$ and the number of arms $K$ as $NK$.

For general distributions, the situation worsens, and may be intractable without additional modeling assumptions.  Lastly, note that  the supremum over $\omega$ is a convex program in principle. First–order methods such as Frank–Wolfe can be applied to find an approximate solution. However,  any tractable implementation presumes structural knowledge (e.g., an exponential–family model, smoothness) to guarantee a number of necessary properties.

\subsubsection{$I$-IDS}
\label{app:iids}
\begin{algorithm}[h!]
\caption{$I$-IDS}
\label{alg:iids}
\begin{algorithmic}[1]
\State \textbf{Input:} Pre-trained inference network $I_\phi$; prior means and variances $\mu_a, \sigma_a^2$ for all $a \in \mathcal{A}$; target error threshold $\delta$
\State \textbf{Initialize:} $f_a(x) = \mathcal{N}(x \mid \mu_a, \sigma_a^2)$ for each $a$
\For{$t = 1, 2, \ldots$}
  \If{$\max_a I_\phi(a \mid \mathcal{D}_{t-1}) \geq 1 - \delta$}
    \State \textbf{return} $\arg\max_a I_\phi(a \mid \mathcal{D}_{t-1})$
  \EndIf
  \For{each arm $a \in \mathcal{A}$}
    \State Approximate information gain:
    \[
    g_t(a) = H\left(I_\phi(\cdot \mid \mathcal{D}_{t-1})\right) - \mathbb{E}_{r \sim p(r \mid a, \mathcal{D}_{t-1})} \left[ H\left(I_\phi(\cdot \mid \mathcal{D}_{t-1}, a, r)\right) \right]
    \]
  \EndFor
  \State Select action $a_t = \arg\max_a g_t(a)$
  \State Observe reward $r_t$
  \State Update dataset $\mathcal{D}_t = \mathcal{D}_{t-1} \cup \{(a_t, r_t)\}$
\EndFor
\end{algorithmic}
\end{algorithm}

We implement a variant of Information Directed Sampling (IDS)~\citep{russo2018learning}, where we use the inference network $I_\phi$ learned during ICPE training as a posterior over optimal arms. This approach enables IDS to exploit latent structure in the environment without explicitly modeling it via a probabilistic model; instead, the learned $I$-network implicitly captures such structure.

By using the same inference network in both ICPE and $I$-IDS, we directly compare the exploration policy learned by ICPE to the IDS heuristic applied on top of the same posterior distribution. While computing the expected information gain in IDS requires intractable integrals, we approximate them using a Monte Carlo grid of 30 candidate reward values per action. The full pseudocode for $I$-IDS is given in Algorithm~\ref{alg:iids}.

\subsubsection{In-Context Explore-then-Commit}
\label{app:icetc}

We implement an ICPE variant for regret minimization via an \textit{explore-then-commit} framework. This method reuses the exploration policy and inference network learned during fixed-confidence training. The agent interacts with the environment using the learned exploration policy until it selects the stopping action. At that point, it commits to the arm predicted to be optimal by the $I$-network and repeatedly pulls that arm for the remainder of the episode. The full pseudo-code is provided in Algorithm~\ref{alg:icetc}.

\begin{algorithm}[h!]
   \caption{In-Context Explore-then-Commit}
   \label{alg:icetc}
\begin{algorithmic}[1]
\State \textbf{Input:} Environment $M \sim \mathcal{P}(\mathcal{M})$; pre-trained critic network $Q_\theta$; pre-trained inference network $I_\phi$
\State Initialize \textit{stopped} $\leftarrow$ False
\State Observe initial state $s_1 \sim \rho$
\For{$t = 1$ to $N$}
    \If{\textit{stopped} = False \textbf{and} $a_{\text{stop}} \neq \arg\max_a Q_\theta(s_t, a)$}
        \State Execute $a_t = \arg\max_a Q_\theta(s_t, a)$ and observe $s_{t+1}$
    \ElsIf{\textit{stopped} = False \textbf{and} $a_{\text{stop}} = \arg\max_a Q_\theta(s_t, a)$}
        \State Set \textit{stopped} $\leftarrow$ True
        \State Execute $a_t = \arg\max_a I_\phi(s_t)$ and observe $s_{t+1}$
    \Else
        \State Execute $a_t = \arg\max_a I_\phi(s_t)$ and observe $s_{t+1}$
    \EndIf
\EndFor
\end{algorithmic}
\end{algorithm}

\subsubsection{$I$-DPT}\label{app:algorithm:idpt}

We implement a variant of DPT \citep{lee2023supervised} using the inference network. The idea is to act greedily with respect to the posterior distribution $I$  at inference time.

First, we train $I$ using \icpe{}. Then, at deployment we act with respect to $I$: in round $t$ we selection action $a_t = \argmax_H I(H|D_t)$. Upon observing $x_{t+1}$, we update $D_{t+1}$  and stop as soon as $\argmax_H I(H|D_t)\geq 1-\delta$.

\subsection{Transformer Architecture}\label{app:algorithm:transformer_architecture}

Here we briefly describe the architecture of the inference network $I$ and of the  network $Q$.

Both networks are implemented using a Transformer architecture. For the inference network, it is designed to predict a hypothesis $H$ given a sequence of observations. Let the input tensor be denoted by $X \in \mathbb{R}^{B \times H \times m}$, where:

\begin{itemize}
\item $B$ is the batch size,
\item $H$ is the sequence length (horizon), and
\item $m=(d+|{\cal A}|),$ where $d$ is the dimensionality of each observation $x_t$.
\end{itemize}

The inference network operates as follows:
\begin{enumerate}
\item \textbf{Embedding Layer}: Each observation vector $m_t=(x_t,a_t)$ is first embedded into a higher-dimensional space of size $d_e$ using a linear transformation followed by a GELU activation:
$         h_t = \text{GELU}(W_{\text{embed}} m_t + b_{\text{embed}}), \quad h_t \in \mathbb{R}^{d_e}
    $.
    
\item \textbf{Transformer Layers}: The embedded sequence $h \in \mathbb{R}^{B \times H \times d_e}$ is then passed through multiple Transformer layers (specifically, a GPT-2 model configuration). The Transformer computes self-attention over the embedded input to model dependencies among observations:
\[
    h' = \text{Transformer}(h), \quad h' \in \mathbb{R}^{B \times H \times d_e}.
\]

\item \textbf{Output Layer}: The final hidden state corresponding to the last element of the sequence ($h'_{:, -1, :}$) is fed into a linear output layer that projects it to logits representing the hypotheses:
\[
    o = W_{\text{out}} h'_{:, -1, :} + b_{\text{out}}, \quad o \in \mathbb{R}^{B \times |{\cal H}|}.
\]

\item \textbf{Probability Estimation}: The output logits are transformed into log-probabilities via a log-softmax operation along the last dimension
\[
    \log p(H|X) = \text{log\_softmax}(o).
\]

\end{enumerate}
For $Q$, we use the same architecture, but do not take a log-softmax at the final step.

\begin{figure}

\begin{center}
\begin{tikzpicture}[scale=0.9, every node/.style={transform shape}]

\node (input) [rectangle, draw, rounded corners, align=center, minimum width=3cm, minimum height=1cm] {Input $X \in \mathbb{R}^{B \times H \times d}$};

\node (embed) [rectangle, draw, rounded corners, align=center, minimum width=3cm, minimum height=1cm, below=1cm of input] {Embedding Layer Linear + GELU};

\node (transformer) [rectangle, draw, rounded corners, align=center, minimum width=4cm, minimum height=1.5cm, below=1cm of embed] {Transformer\\(GPT-2)};

\node (output) [rectangle, draw, rounded corners, align=center, minimum width=3cm, minimum height=1cm, below=1cm of transformer] {Output LayerLinear};

\node (softmax) [rectangle, draw, rounded corners, align=center, minimum width=3cm, minimum height=1cm, below=1cm of output] {Log-softmax\\$\log p(H|X)$};

\draw[->, thick] (input) -- (embed);
\draw[->, thick] (embed) -- (transformer);
\draw[->, thick] (transformer) -- node[right] {Last hidden state} (output);
\draw[->, thick] (output) -- (softmax);

\end{tikzpicture}
\end{center}
\caption{Model architecture for the inference network $I$ (similarly for $Q$).}
\label{fig:model_architecture}
\end{figure}
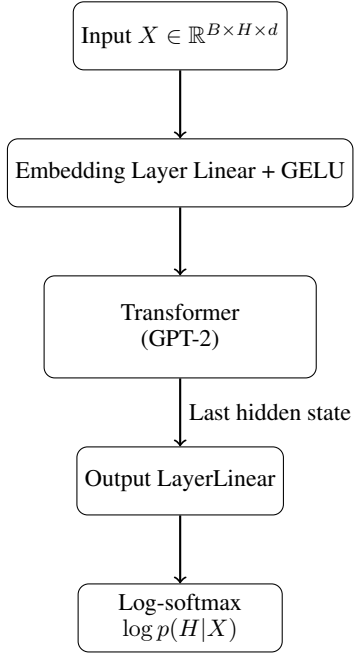

\newpage
\section{Experiments}
\label{app:experiments}

This section provides additional experimental results, along with detailed training and evaluation protocols to ensure clarity and reproducibility. All experiments were conducted using four NVIDIA A100 GPUs.

For more informations about the hyper-parameters, we also refer the reader to the {\tt README.md} file in the code, as well as the training configurations in the {\tt configs/experiments} folder.

\paragraph{Libraries used in the experiments.} We set up our experiments using Python 3.10.12 \citep{van1995python} (For more information, please refer to the following link \url{http://www.python.org}), and made use of the following libraries:  NumPy \citep{harris2020array}, SciPy \citep{2020SciPy-NMeth}, PyTorch  \citep{NEURIPS2019_9015},  Seaborn \citep{michael_waskom_2017_883859}, Pandas \citep{mckinney2010data}, Matplotlib \citep{hunter2007matplotlib}, CVXPY \citep{diamond2016cvxpy}, Wandb \citep{wandb}, Gurobi \citep{gurobi}.
Changes, and new code, are published under the MIT license. To run the code, please, read the attached README file for instructions.

\paragraph{Hierarchical bootstrapping.} For each experiment, we compute confidence intervals using hierarchical bootstrapping. Our data is organized at three levels: seed, environment, and trajectory. For each training seed we evaluate multiple environments, and for each environment we roll out multiple trajectories. Hierarchical bootstrapping allows us to correctly account for this nested structure when estimating uncertainty. This approach captures variability across seeds, environments, and trajectories, yielding a more reliable characterization of performance compared to classical bootstrapping.

To fix the ideas, consider the following random-effects model
$$
Y_{a,b,c} = \mu + \alpha_a + \beta_{a,b,} + \gamma_{a,b,c},\quad
\alpha_a\sim{\cal N}(0,\sigma^2_{\text{seed}}),\;
\beta_{a,b}\sim{\cal N}(0,\sigma^2_{\text{env}}),\;
\gamma_{a,b,c}\sim{\cal N}(0,\sigma^2_{\text{traj}}),
$$
and let $\bar Y=\frac{1}{mKN}\sum_{a=1}^m\sum_{b=1}^k\sum_{c=1}^N (Y_{a,b,c})$ with $m$ seeds, $K$ environments per seed and $N$ trajectories per environment.  

The variance of $\bar Y$ then is
$\text{Var}(\bar Y) = \frac{\sigma^2_{\text{seed}}}{m}+ \frac{\sigma^2_{\text{env}}}{mK} + \frac{\sigma^2_{\text{traj}}}{mKN}$, which hierarchical bootstrapping accounts for.
Instead, a naive bootstrap over the $mKN$ trajectories targets 
$\text{Var}_{\rm naive}(\bar Y)
\approx \frac{\text{ Var}(Y_{a,b,c})}{mKN}=\frac{\sigma^2_{\text{seed}}+\sigma^2_{\text{env}}+\sigma^2_{\text{traj}}}{mKN}$, 
which effectively reduces the contribution at the seed-level by a factor $KN$.

\subsection{Bandit Problems}
\label{app:bandits}

Here, we provide the implementation and evaluation details for the bandit experiments reported in Section~\ref{subsec:bandits}, covering deterministic, stochastic, and structured settings. Note that for this setting the observations are simply the observed rewards, i.e., $x_t=r_t$.

\paragraph{Model Architecture and Optimization.}
For all bandit tasks, \icpe{} uses a Transformer with 3 layers, 2 attention heads, hidden dimension 256, GELU activations, and dropout of 0.1 applied to attention, embeddings, and residuals (see also \cref{app:algorithm:transformer_architecture} for a description of the architecture). Layer normalization uses $\epsilon = 10^{-5}$. Inputs are one-hot action-reward pairs with positional encodings. Models are trained using Adam with learning rates between $1 \times 10^{-4}$ and $1 \times 10^{-6}$, and batch sizes from 128 to 1024 depending on task complexity.

\subsubsection{Deterministic Bandits with Fixed Budget}

Each environment consists of $K$ arms, where $K \in \{4, 6, 8, \dots, 20\}$. Mean rewards for each arm are sampled uniformly from $[0,1]$, and rewards are deterministic (i.e., zero variance). Agents interact with the environment for exactly $K$ steps and are then required to predict the optimal arm. Success is measured by the probability of correctly identifying the best arm. We also compute the average number of unique arms selected during training episodes as a proxy for exploration diversity.

\icpe{} is compared against three baselines in the deterministic setting: \textit{Uniform Sampling}, which selects arms uniformly at random; \textit{DQN}, a deep Q-network trained directly on environmental rewards ~\citep{mnih2013playing}; and \textit{I-DPT}, which performs posterior sampling using \icpe{}’s $I$-network \citep{lee2023supervised}. All methods were evaluated over five seeds, with 900 environments per seed. 95\% confidence intervals were computed with hierarchical bootstrapping.

\begin{figure}[h!]
    \centering    \includegraphics[width=\linewidth]{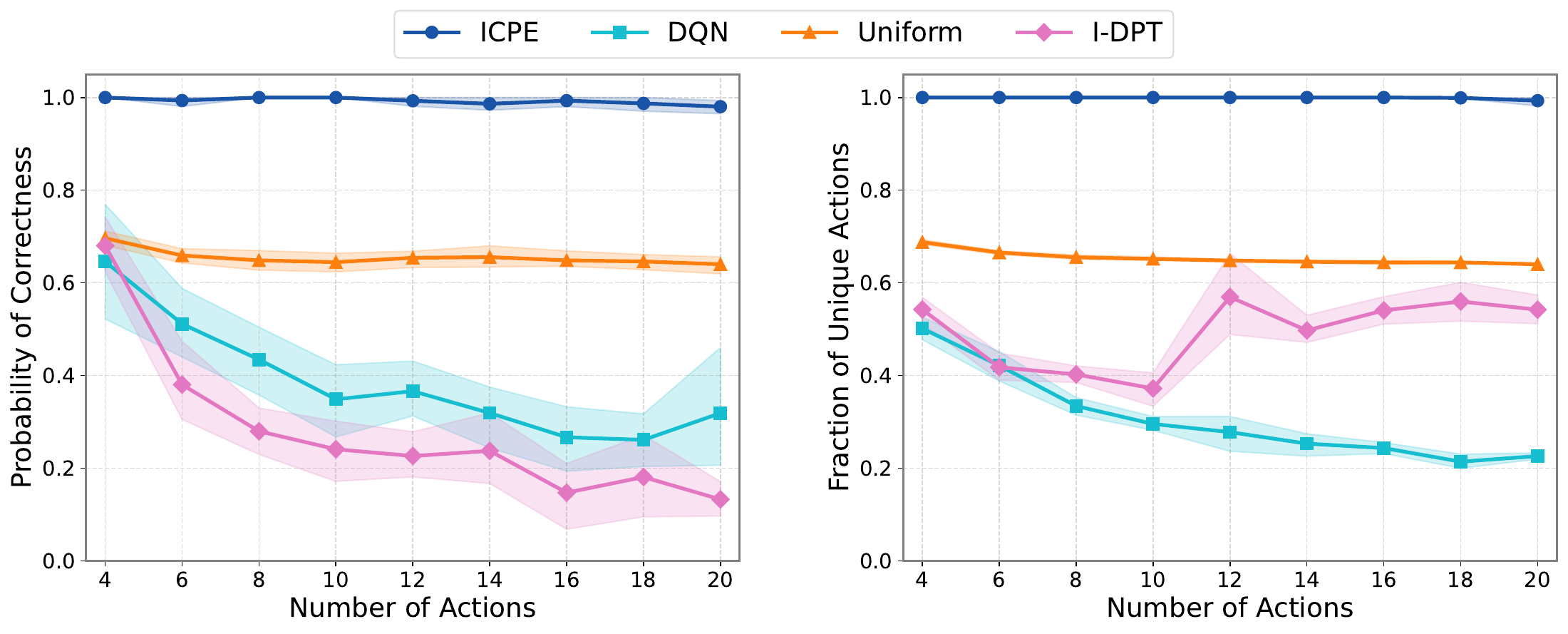}
    \caption{Deterministic bandits: (left) probability of correctly identifying the best action vs.  $K$; (right) average fraction of unique actions selected during exploration vs.   $K$.}
\end{figure}

\subsubsection{Stochastic Bandits Problems}\label{app:experimental_results:stochastic_bandit}

In the stochastic Gaussian bandit setting, we evaluate \icpe{} on best-arm identification tasks with $K \in \{4, 6, 8, \dots, 14\}$. Arm means are sampled uniformly from $[0, 0.4K]$, with a guaranteed minimum gap of $1/K$ between the top two arms. All arms have a fixed reward standard deviation of 0.5. The target confidence level is set to $\delta = 0.1$.

We compare \icpe{} against several widely used baselines: \textit{Top-Two Probability Sampling (TTPS)}~\citep{jourdan2022top}, \textit{Track-and-Stop (TaS)}~\citep{garivier2016optimal}, \textit{Uniform Sampling}, and \textit{I-DPT}. For \textit{I-DPT}, stopping occurs when the predicted optimal arm has estimated confidence at least $1-\delta$. For \textit{TTPS} and \textit{TaS}, we apply the classical stopping rule based on the characteristic time $T^\star(N_t/t;\hat{\mu}_t)$ (explained in \cref{app:algorithm:tas}):
\[
t \cdot T^\star(N_t/t;\hat{\mu}_t) \geq \log\left( \frac{1+\log t}{\delta}\right).
\]
Each method is evaluated over three seeds, with 300 environments, and 15 trajectories per environment. 95\% confidence intervals were computed with hierarchical bootstrapping.

\begin{figure}[h!]
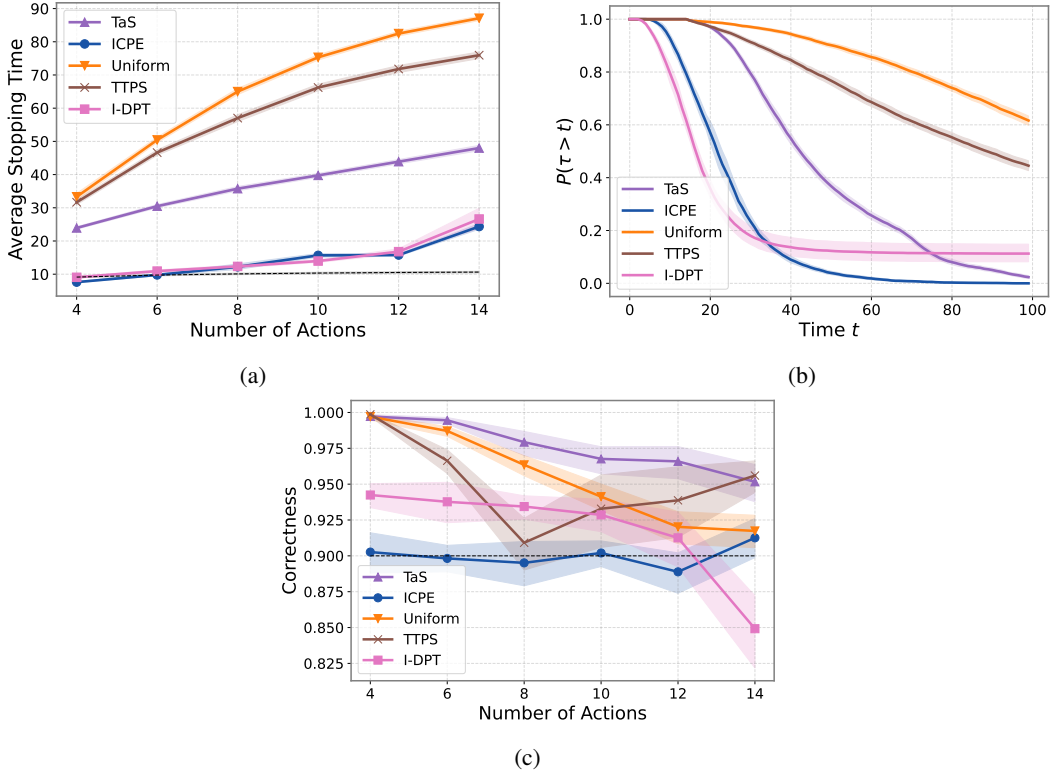

    \centering
    \begin{subfigure}[t]{0.48\textwidth}
        \centering
    \includegraphics[width=\linewidth]{figures/stochastic_mab_fixed_confidence/stochastic_mab_fixed_confidence_stopping_time.pdf}
    
    \caption{}
    
    \end{subfigure}
    \hfill
    \begin{subfigure}[t]{0.48\textwidth}
        \centering\includegraphics[width=\linewidth]{figures/stochastic_mab_fixed_confidence/stochastic_mab_fixed_confidence_stopping_time_prob.pdf}
        \caption{}
        
    \end{subfigure}
    \hfill
    \begin{subfigure}[t]{0.48\textwidth}
        \centering
    \includegraphics[width=\linewidth]{figures/stochastic_mab_fixed_confidence/stochastic_mab_fixed_confidence_correctness.pdf}
    \caption{}
    \end{subfigure}
    \caption{Results for stochastic MABs with fixed confidence $\delta=0.1$ and $N=100$: (a) average stopping time $\tau$; (b) survival function of $\tau$; (c) probability of correctness $\mathbb{P}^\pi(H^\star=\hat H_\tau)$.}
\end{figure}

\paragraph{Does \icpe{} learn randomized policies?} An intriguing question is whether \icpe{} is capable of learning randomized policies. Intuitively, one might expect randomized methods, such as actor-critic algorithms, to perform better. However, we observe that this is not the case for \icpe{}. Crucially, the inherent randomness of the environment, when passed as input to the transformer architecture, already serves as a source of stochasticity. Thus, although \icpe{} employs a deterministic mapping (via DQN) from observed trajectories, these trajectories themselves constitute random variables, rendering the policy's output effectively stochastic. To illustrate this, we examine an \icpe{} policy trained with fixed confidence ($\delta=0.1$) in a setting with $K=14$ actions (see the two rightmost plots in \cref{fig:icpe_stochastic_mab_probabilities}). By analyzing 100 trajectories from this environment and computing an averaged policy, we clearly observe how trajectory randomness influences the policy’s outputs. Specifically, exploration intensity peaks around the middle of the horizon and diminishes as the confidence level increases.

\paragraph{Does \icpe{} resembles Track and Stop?} In \cref{fig:icpe_stochastic_mab_probabilities} (left figure) we compare  an \icpe{} policy trained in the fixed confidence setting ($\delta=0.1$) with an almost optimal version of TaS, that can be easily computed without solving any optimization problem. Let $\hat \Delta_t(a)=\hat\mu_{\hat a_t}(t)-\max_{a\neq \hat a_t}\hat\mu_a$, where $\hat\mu_a(t)$ is the empirical reward of arm $a$  in round $t$ and $\hat a_t=\argmax_a \hat \mu_a(t)$ is the estimated optimal arm. Then, the approximate TaS policy is defined as
\[
\pi_t(a) =\frac{1/\hat\Delta_a(t)}{\sum_b 1/\hat\Delta_b(t)},
\]
with $\hat \Delta_{\hat a_t}(t)=\min_{a\neq \hat a_t}\hat \Delta_a(t)$. 
In the figure we sampled $100$ trajectories from a single environment with $K=14$, and computed an average \icpe{} policy. Then, we compared this policy to the approximate TaS policy, and computed the total variation. We can see that the two policies are not always similar. We believe this is due to the fact that \icpe{} is exploiting prior information on the environment, including the minimum gap assumption, and the fact that the average rewards are bounded in $[0,0.4K]$.

\begin{figure}
    \centering
    \includegraphics[width=\linewidth]{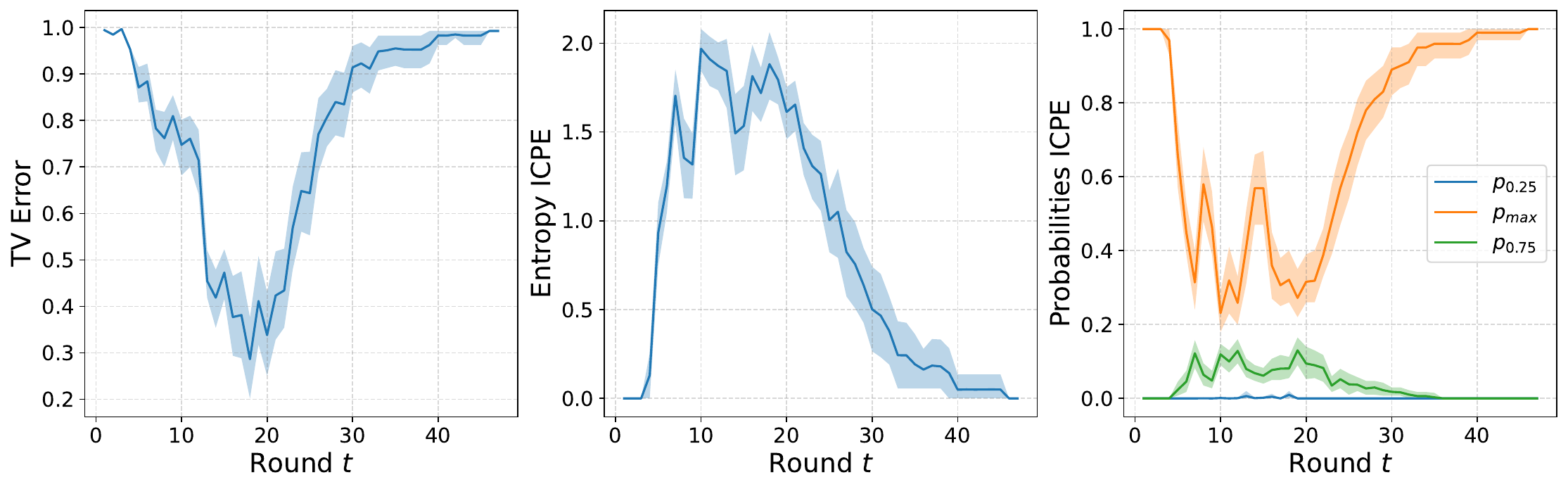}
    \caption{Statistics of \icpe{} with fixed confidence on $100$ trajectories from a single environment, with $K=14$. From left to right: Total variation error between the average \icpe{} policy and the approximate Track and Stop policy;  entropy of the average policy of \icpe{}; probabilities of the  average \icpe{} policy, with $p_{max}$ representing the maximum probability and $p_\alpha$ the $\alpha$-quantile.}
    \label{fig:icpe_stochastic_mab_probabilities}
\end{figure}

\begin{figure}[h!]
\centering
\includegraphics[width=0.48\linewidth]{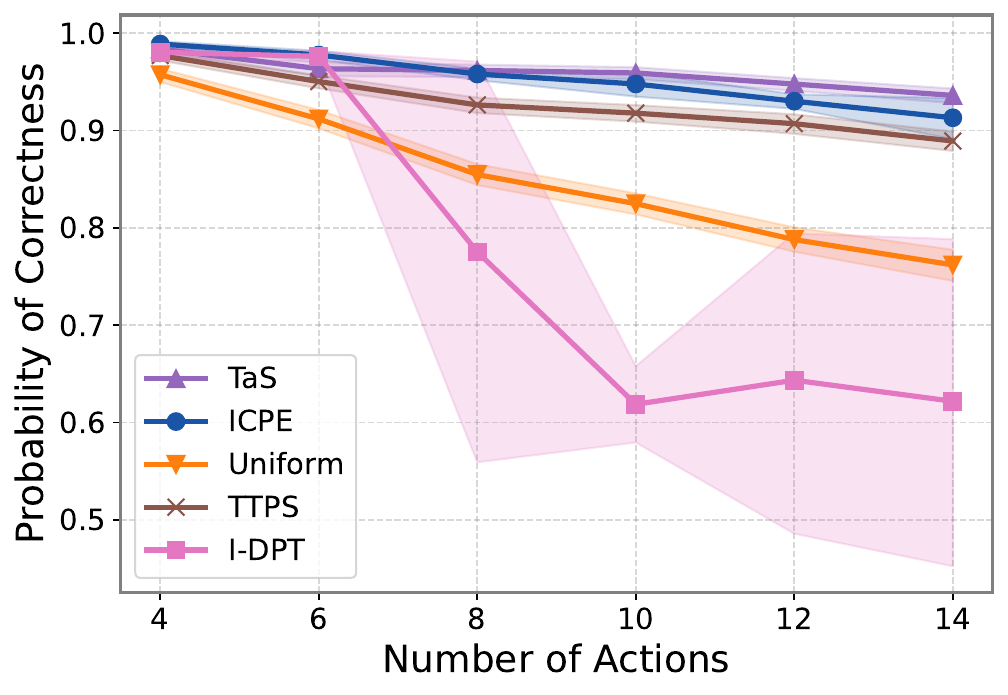}
    \caption{Correctness $\mathbb{P}^\pi(H^\star=\hat H_N)$ for stochastic MABs with fixed budget $N=30$.}
\end{figure}

\paragraph{Is \icpe{} robust to distribution shift?}

\begin{table}[t]
  \centering
  \begin{tabular}{c|cc|cc}
    \toprule
    \multicolumn{1}{c|}{\begin{tabular}[b]{@{}c@{}}
        \textit{KL Divergence}\\
        \textit{From Uniform}
    \end{tabular}} &
    \multicolumn{2}{c|}{\textit{Number of Actions = 4}} &
    \multicolumn{2}{c}{\textit{Number of Actions = 6}} \\
    & Correctness & Avg.\ Stop Time & Correctness & Avg.\ Stop Time \\
    \midrule
    0.00 & 0.91 $\pm$ 0.01 &  7.76 $\pm$ 0.20 & 0.91 $\pm$ 0.01 &  9.78 $\pm$ 0.22 \\
    0.25 & 0.91 $\pm$ 0.01 &  7.59 $\pm$ 0.19 & 0.91 $\pm$ 0.01 &  9.97 $\pm$ 0.26 \\
    0.50 & 0.90 $\pm$ 0.01 &  7.65 $\pm$ 0.21 & 0.91 $\pm$ 0.01 &  9.79 $\pm$ 0.26 \\
    1.00 & 0.90 $\pm$ 0.01 &  7.68 $\pm$ 0.20 & 0.90 $\pm$ 0.01 &  9.89 $\pm$ 0.28 \\
    2.00 & 0.89 $\pm$ 0.01 &  7.63 $\pm$ 0.21 & 0.90 $\pm$ 0.01 &  9.86 $\pm$ 0.28 \\
    4.00 & 0.89 $\pm$ 0.01 &  7.73 $\pm$ 0.22 & 0.90 $\pm$ 0.01 & 10.07 $\pm$ 0.28 \\
    \bottomrule
  \end{tabular}

  \vspace{0.6em}

  \begin{tabular}{c|cc|cc}
    \toprule
    \multicolumn{1}{c|}{\begin{tabular}[b]{@{}c@{}}
        \textit{KL Divergence}\\
        \textit{From Uniform}
    \end{tabular}} &
    \multicolumn{2}{c|}{\textit{Number of Actions = 8}} &
    \multicolumn{2}{c}{\textit{Number of Actions = 10}} \\
    & Correctness & Avg.\ Stop Time & Correctness & Avg.\ Stop Time \\
    \midrule
    0.00 & 0.90 $\pm$ 0.01 & 11.37 $\pm$ 0.22 & 0.91 $\pm$ 0.01 & 15.41 $\pm$ 0.37 \\
    0.25 & 0.89 $\pm$ 0.01 & 11.45 $\pm$ 0.26 & 0.92 $\pm$ 0.01 & 15.13 $\pm$ 0.37 \\
    0.50 & 0.89 $\pm$ 0.01 & 11.54 $\pm$ 0.26 & 0.91 $\pm$ 0.01 & 15.35 $\pm$ 0.40 \\
    1.00 & 0.90 $\pm$ 0.01 & 11.33 $\pm$ 0.24 & 0.90 $\pm$ 0.01 & 15.33 $\pm$ 0.42 \\
    2.00 & 0.89 $\pm$ 0.01 & 11.41 $\pm$ 0.28 & 0.91 $\pm$ 0.01 & 15.54 $\pm$ 0.41 \\
    4.00 & 0.88 $\pm$ 0.01 & 11.47 $\pm$ 0.28 & 0.91 $\pm$ 0.01 & 15.22 $\pm$ 0.40 \\
    \bottomrule
  \end{tabular}

  \vspace{0.6em}

  \begin{tabular}{c|cc|cc}
    \toprule
    \multicolumn{1}{c|}{\begin{tabular}[b]{@{}c@{}}
        \textit{KL Divergence}\\
        \textit{From Uniform}
    \end{tabular}} &
    \multicolumn{2}{c|}{\textit{Number of Actions = 12}} &
    \multicolumn{2}{c}{\textit{Number of Actions = 14}} \\
    & Correctness & Avg.\ Stop Time & Correctness & Avg.\ Stop Time \\
    \midrule
    0.00 & 0.91 $\pm$ 0.01 & 18.86 $\pm$ 0.51 & 0.91 $\pm$ 0.01 & 22.23 $\pm$ 0.72 \\
    0.25 & 0.91 $\pm$ 0.02 & 18.28 $\pm$ 0.52 & 0.92 $\pm$ 0.01 & 22.63 $\pm$ 0.71 \\
    0.50 & 0.91 $\pm$ 0.01 & 18.55 $\pm$ 0.53 & 0.91 $\pm$ 0.01 & 22.18 $\pm$ 0.75 \\
    1.00 & 0.90 $\pm$ 0.01 & 18.78 $\pm$ 0.52 & 0.91 $\pm$ 0.01 & 22.36 $\pm$ 0.72 \\
    2.00 & 0.91 $\pm$ 0.01 & 19.00 $\pm$ 0.60 & 0.91 $\pm$ 0.01 & 22.57 $\pm$ 0.75 \\
    4.00 & 0.91 $\pm$ 0.01 & 18.52 $\pm$ 0.54 & 0.91 $\pm$ 0.01 & 22.97 $\pm$ 0.75 \\
    \bottomrule
  \end{tabular}
  \caption{ICPE performance on stochastic fixed-confidence $(\delta=0.1, N=100)$ bandits when test environments drift from the uniform training distribution by a prescribed KL distance (symmetric Dirichlet). Reported values are $95\%$ confidence intervals computed via hierarchical bootstrapping.}
  \label{tab:icpe_dist_shift}
\end{table}

As an in-context learning method, ICPE is designed to be meta-trained on the same family of tasks on which it will be deployed. That said, understanding robustness to changes in the environment distribution is important for assessing practicality. Therefore, we trained ICPE in the stochastic fixed-confidence bandit setting described above, where environments are sampled from a uniform distribution over Gaussian bandits with a minimum gap. At test time, we then evaluated the same trained model on bandit instances drawn from \emph{shifted} environment distributions. We constructed these shifts by sampling reward means from a symmetric Dirichlet distribution with parameter $\alpha$ chosen so that
\[
\mathrm{KL}(\mathrm{Dir}(\alpha,\dots,\alpha)\,\|\,\mathrm{Dir}(1,\dots,1)) = \text{target KL},
\]
thereby controlling the divergence from the uniform training distribution. Intuitively, varying the target KL controls how concentrated generated samples are with respect to the simplex. ICPE’s correctness and average stopping time across a range of KL values and number of actions is reported in \cref{tab:icpe_dist_shift}. Across all experiments, we observe that both correctness and stopping time remain remarkably stable, with only minor fluctuations within the reported confidence intervals. This suggests that ICPE is not excessively sensitive to moderate shifts in the environment distribution around the training family.

\subsubsection{Bandit Problems with Hidden Information}\label{app:subsec:bandit_hiden_info}

\paragraph{Magic Action Environments}

We evaluate \icpe{} in bandit environments where certain actions reveal information about the identity of the optimal arm, testing its ability to uncover and exploit latent structure under the fixed-confidence setting.

Each environment contains $K = 5$ arms. Non-magic arms have mean rewards sampled uniformly from $[1,5]$, while the mean reward of the designated \textit{magic action} (always arm 1) is defined as $\mu_m = \phi(\arg\max_{a \ne a_m} \mu_a)$ with $\phi(i) = i/K$. The magic action is not the optimal arm, but it encodes information about which of the other arms is. To control the informativeness of this signal, we vary the standard deviation of the magic arm $\sigma_m \in \{0.0, 0.1, \dots, 1.0\}$, while fixing the standard deviation of all other arms to $\sigma = 1 - \sigma_m$.

\icpe{} is trained under the fixed-confidence setting with a target confidence level of $0.9$. For each $\sigma_m$, we compare \icpe{}’s sample complexity to two baselines: (1) the average theoretical lower bound computed for the problem computed via averaging the result of Theorem~\ref{thm:noisy_magic} over numerous environmental mean rewards, and (2) \textit{I-IDS}, a pure-exploration information-directed sampling algorithm that uses \icpe{}’s $I$-network for posterior estimation.
All methods are over 500 environments, with 10 trajectories per environment. 95\% confidence intervals are computed using hierarchical bootstrapping with two levels.


Beyond the exploration efficiency analysis shown in Figure~\ref{fig:noisy_magic}, we also assess the correctness of each method’s final prediction and its usage of the magic action. As shown in Figure~\ref{fig:cor_noisy}, both \icpe{} and \textit{I-IDS} consistently achieve the target accuracy of 0.9, validating their reliability under the fixed-confidence formulation.

Figure~\ref{fig:magicpct} plots the proportion of total actions that were allocated to the magic arm across different values of $\sigma_m$. While both methods adapt their reliance on the magic action as its informativeness degrades, \textit{I-IDS} tends to abandon it earlier. This behavior suggests that \icpe{} is better able to retain and exploit structured latent information beyond what is captured by simple heuristics for expected information gain.

\begin{figure}[ht]
    \centering
    \begin{subfigure}[b]{0.45\textwidth}
        \includegraphics[width=\linewidth]{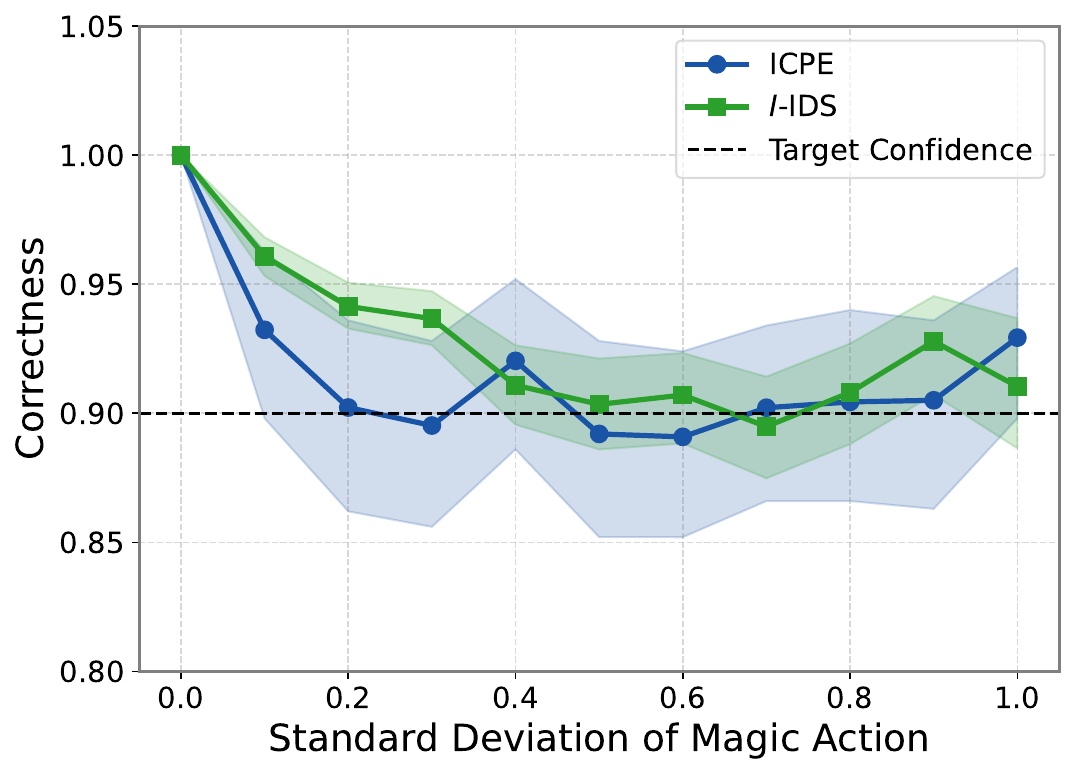}
        \caption{}
        \label{fig:cor_noisy}
    \end{subfigure}
    \hfill
    \begin{subfigure}[b]{0.45\textwidth}
        \includegraphics[width=\linewidth]{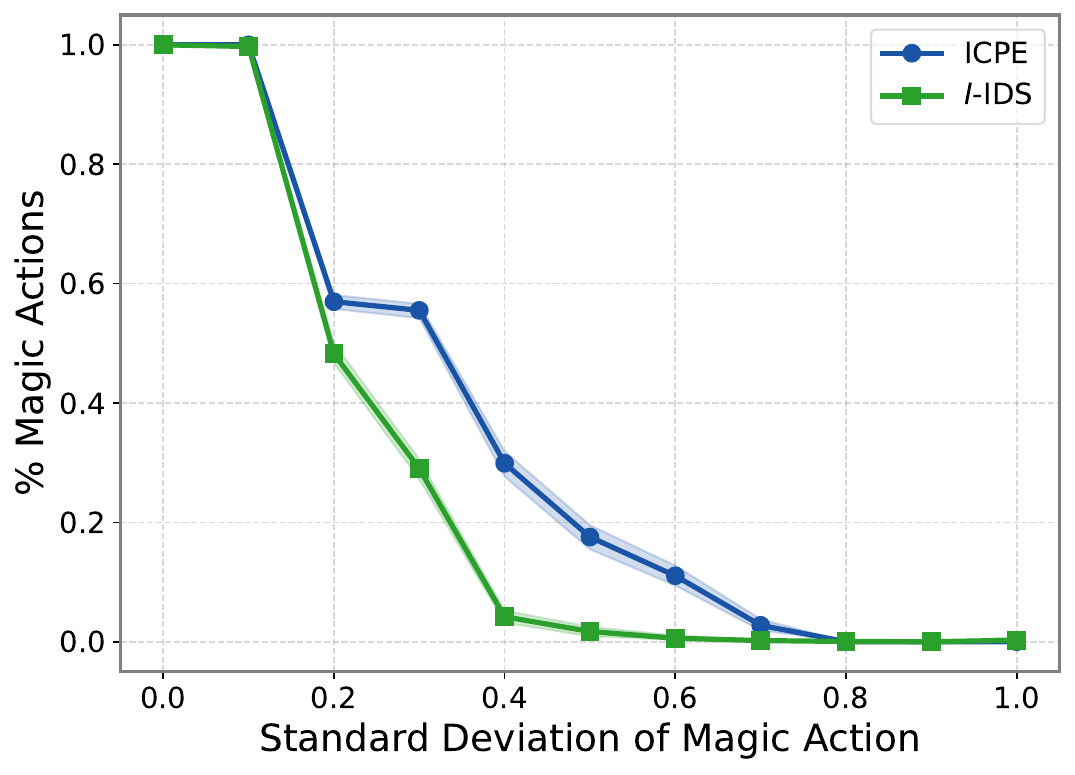}
        \caption{}
        \label{fig:magicpct}
    \end{subfigure}
    \caption{(a) Final prediction accuracy across varying levels of noise in the magic action ($\sigma_m$). Both \icpe{} and \textit{I-IDS} consistently achieve the target confidence threshold of 0.9. (b) Percentage of actions allocated to the magic arm as a function of $\sigma_m$. \icpe{} continues to exploit the magic action longer than \textit{I-IDS}, suggesting more robust use of latent structure.}
    \label{fig:app_noisy}
\end{figure}

We also assess \icpe{}’s performance in a regret minimization setting. We define an \textit{In-Context Explore-then-Commit} variant of \icpe{}, which explores until the $I$-network reaches confidence $1 - \delta$, then repeatedly selects the estimated optimal action. We compare this policy’s cumulative regret to that of three standard algorithms: \textit{UCB}, \textit{Thompson Sampling}, and \textit{IDS}, each initialized with Gaussian priors. For this evaluation, we fix $\sigma_m = 0.1$, $\sigma = 0.9$, and $\delta=0.01$.

Implementation details for $I$-IDS and In-Context Explore-then-Commit are provided in Sections \ref{app:iids} and \ref{app:icetc} respectively.

\paragraph{Magic Chain Environments}

To assess \icpe{}’s ability to perform multi-step reasoning over latent structure, we evaluate it in environments where identifying the optimal arm requires sequentially uncovering a chain of informative actions. In these \textit{magic chain} environments, each magic action reveals partial information about the next, culminating in identification of the best arm.

We use $K = 10$ arms and vary the number of magic actions $n \in \{1, 2, \dots, 9\}$. Mean rewards for magic actions are defined recursively as:
\[
\mu_{i_j} =
\begin{cases}
\phi(i_{j+1}), & \text{if } j = 1,\dots, n-1, \\[1mm]
\phi\Bigl(\arg\max_{a\notin\{i_1,\dots,i_n\}} \mu_a\Bigr), & \text{if } j = n,
\end{cases}
\]
where $\phi(i) = i/K$, and the remaining arms have mean rewards sampled uniformly from $[1, 2]$. All rewards are deterministic (zero variance).

\icpe{} is trained under the fixed-confidence setting with $\delta = 0.99$. For each $n$, five models are trained across five seeds. We compare \icpe{}’s average stopping time to the theoretical optimum (computed via Theorem~\ref{thm:bound_multi_magic}) and to the \textit{I-IDS} baseline equipped with access to the $I$-network. Each model is evaluated over 1000 test environments per seed. 95\% confidence intervals are computed using hierarchical
bootstrapping.

In interpreting the results from Figure~\ref{fig:mult_magic}, we observe that for environments with one or two magic actions, \icpe{} reliably learns the optimal policy of following the magic chain to completion. In these cases, the agent is able to identify the optimal arm without ever directly sampling it, by exploiting the structured dependencies embedded in the reward signals of the magic actions. Figure~\ref{fig:magic2} illustrates a representative trajectory from the two-magic-arm setting, where the first magic action reveals the location of the second, which in turn identifies the optimal arm. The episode terminates without requiring the agent to explicitly sample the best arm itself.

\begin{figure}[ht]
    \centering
    \includegraphics[width=0.7\linewidth]{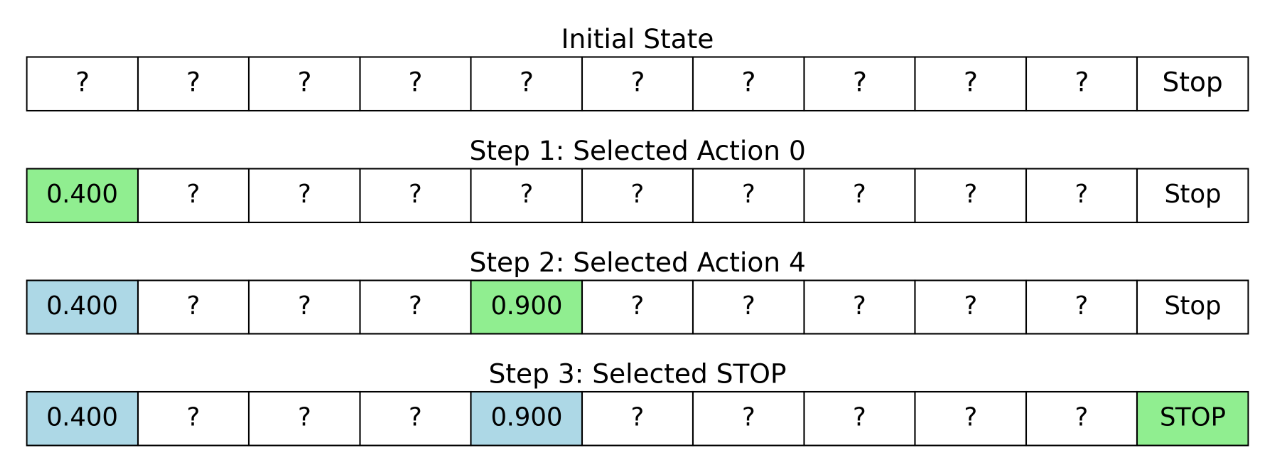}
    \caption{Example trajectory in the 2-magic-arm environment. \icpe{} follows the magic chain: the first magic action reveals the second, which identifies the optimal arm.}
    \label{fig:magic2}
\end{figure}

For environments with more than two magic actions, however, \icpe{} learns a different strategy. As the length of the magic chain increases, the expected sample complexity of following the chain from the start becomes suboptimal. Instead, \icpe{} learns to randomly sample actions until it encounters one of the magic arms and then proceeds to follow the chain from that point onward. This behavior represents an efficient, learned compromise between exploration cost and informativeness. Figure~\ref{fig:magic6} shows an example trajectory from the six-magic-arm setting, where the agent initiates random sampling until it lands on a magic action, then successfully follows the remaining chain to identify the optimal arm.

\begin{figure}[ht]
    \centering
    \includegraphics[width=0.7\linewidth]{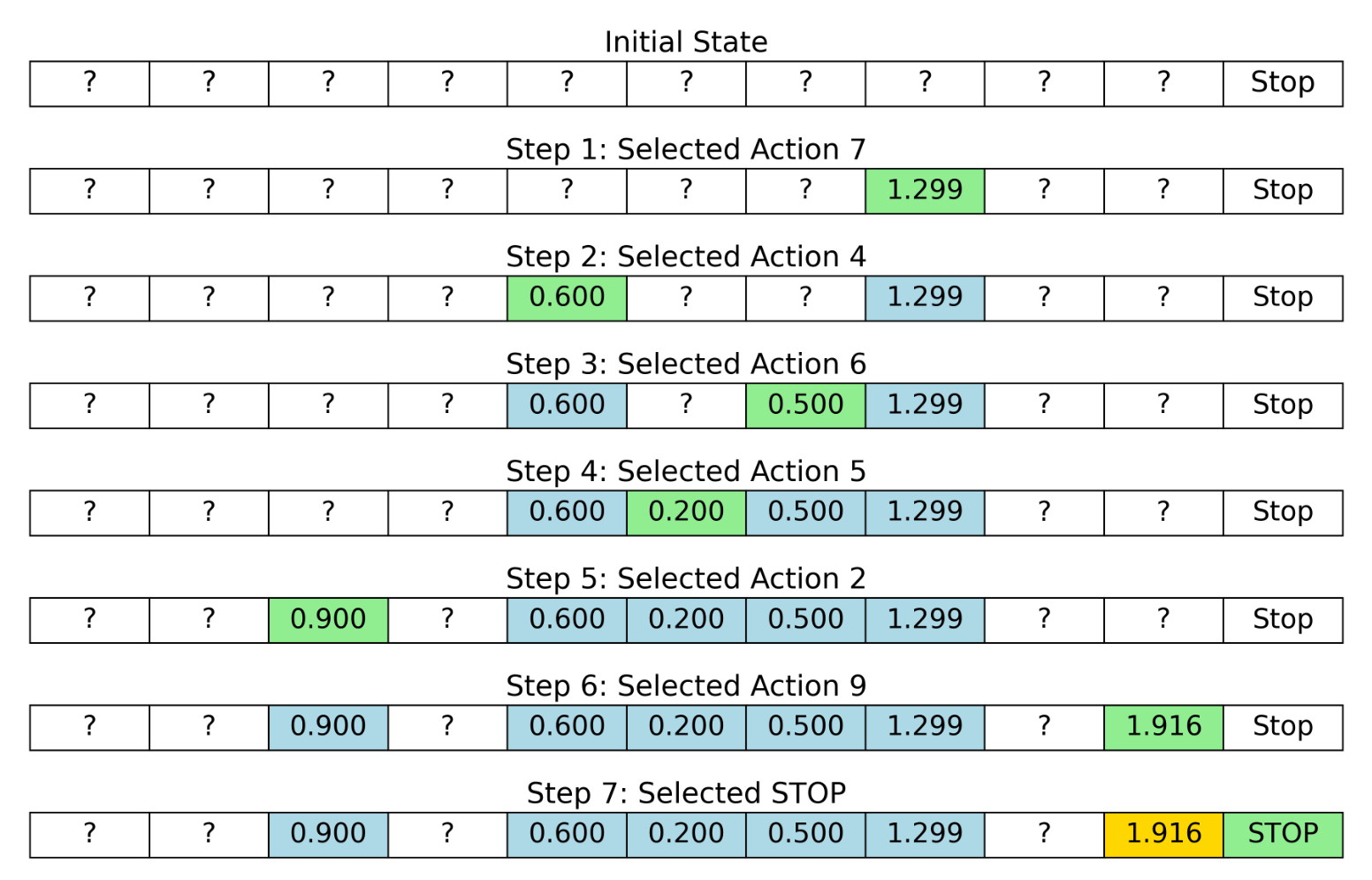}
    \caption{Example trajectory in the 6-magic-arm environment. Rather than starting from the first magic action, \icpe{} samples randomly until finding a magic action and then follows the chain to the optimal arm.}
    \label{fig:magic6}
\end{figure}

\subsection{Semi-Synthetic Pixel Sampling}
\label{app:pixel-sampling}

To evaluate \icpe{} in a setting that more closely resembles real-world decision-making tasks, we designed a semi-synthetic environment based on the MNIST dataset \citep{lecun1998gradient}, where the agent must adaptively reveal image regions to classify a digit while minimizing the number of pixels observed. This experiment serves as a proof-of-concept for using \icpe{} in perceptual tasks where observations are costly and information must be acquired efficiently.

\paragraph{Environment Details.}
Each MNIST image is partitioned into $36$ non-overlapping $5 \times 5$ pixel regions, defining an action space of size $K = 36$. At each timestep, the agent selects a region to reveal, progressively unmasking the image. The agent begins with a fully masked image and has a fixed budget of $N = 12$ steps to acquire information and make a prediction.

To prevent overfitting and encourage generalizable policies, we apply strong augmentations at each episode: random rotations ($\pm 30^\circ$), translations (up to 2 pixels), Gaussian noise ($\mathcal{N}(0, 0.1)$), elastic deformations, and random contrast adjustments. These augmentations ensure that agents cannot memorize specific pixel layouts and must instead learn adaptive exploration strategies.

\paragraph{Model Architecture and Optimization.}
Due to the visual nature of the task, we use a convolutional encoder rather than a transformer. The \icpe{} critic network combines a CNN image encoder with a separate action-count encoder. The CNN consists of 3 convolutional blocks with 16 base channels, followed by max pooling and global average pooling. The action counts (which track how often each region has been sampled) are passed through a linear embedding layer with 32 output units, followed by ReLU activation and LayerNorm. The image and action embeddings are concatenated and processed through two residual MLP layers, producing $Q$-values over actions. The $I$-network shares the same architecture but outputs logits over 10 digit classes. 

All models are implemented in PyTorch and trained with Adam using a learning rate of $1 \times 10^{-4}$. Training is performed over 500{,}000 episodes using 40 parallel environment instances. We use a batch size of 128, a replay buffer of size 100{,}000, and a discount factor $\gamma = 0.999$. The $Q$-network is updated using Polyak averaging with coefficient 0.01, and the $I$-network is updated every two steps using 30 bootstrap batches. To populate the buffer initially, we perform 10 batches of bootstrap updates before standard training begins. Gradients are clipped to a maximum norm of 2.

\paragraph{Pretraining the Inference Network.}
To provide stable reward signals and ensure consistency with baselines, we pretrain a separate CNN classifier to predict digit labels from fully revealed images. This classifier consists of three convolutional layers with max pooling, followed by two linear layers and a softmax head. The classifier is trained on the same augmented data used during \icpe{} training and is frozen during exploration learning. Its softmax confidence for the correct digit is used as the reward signal. This setup simulates realistic scenarios in which high-quality predictive models already exist for fully observed data (e.g., in clinical diagnosis).

\paragraph{Evaluation.}
We compare \icpe{} to two baselines: \textit{Uniform Sampling}, which selects image regions uniformly at random at each timestep, and \textit{Deep CMAB}~\citep{collier2018deep}, a contextual bandit algorithm that uses a Bayesian neural network to model $p(r \mid x, a)$ and performs posterior sampling via dropout.

The Deep CMAB model uses a convolutional encoder to extract image features, which are concatenated with a learned embedding of the action count vector. The combined representation is passed through a multilayer perceptron with dropout applied to each hidden layer. At each decision point, the agent samples a dropout rate from a uniform distribution over $(0,1)$ and uses the resulting forward pass as a sample from the posterior over rewards (Thompson sampling). The reward signal for each action is computed using the pretrained MNIST classifier: specifically, the increase in softmax probability for the correct digit class after a new region is revealed.

We train Deep CMAB for 100{,}000 episodes using Adam optimization. During training, the agent interacts with multiple MNIST instances in parallel, and updates its model based on the marginal improvement in confidence after each action. The model learns to maximize this incremental reward signal by associating particular visual contexts with the most informative actions.

For each trained model, we sample 1000 test environments and report on (1) the average final classification accuracy by the pretrained classifier at the end of trajectory, and (2) the average number of regions used before prediction. Confidence intervals are computed via bootstrapping.

\paragraph{Adaptive Sampling Analysis.}
To assess whether agents learn digit-specific exploration strategies, we analyze the distribution of selected image regions across digit classes. For each agent, we compute pairwise chi-squared tests between all digit pairs, testing whether the distributions of selected regions are statistically distinguishable.

To ensure sufficient support for the test, we only compare digit pairs that each have at least five trajectories and remove actions that appear in fewer than five total samples across the two classes. For each qualifying digit pair, we construct a $2 \times \tilde{K}$ contingency table, where $\tilde{K}$ is the number of region indices that are meaningfully used by either digit. The rows correspond to digit classes, and each column counts how many samples from each class selected the corresponding region at least once. 

We apply the chi-squared test of independence to each contingency table. A low p-value indicates that the region selection distributions for the two digits are significantly different, suggesting digit-specific adaptation. By comparing the number and strength of significant differences across agents, we evaluate the extent to which each method tailors its exploration policy to the structure of the input class.

We visualize the resulting pairwise p-values in Figure~\ref{fig:chisq} using a heatmap. Each cell shows the chi-squared test p-value between a pair of digits. Lower values (blue cells) indicate greater divergence in sampling behavior, and thus more adaptive and digit-specific strategies.

\begin{figure}[ht]
    \centering
    \includegraphics[width=\linewidth]{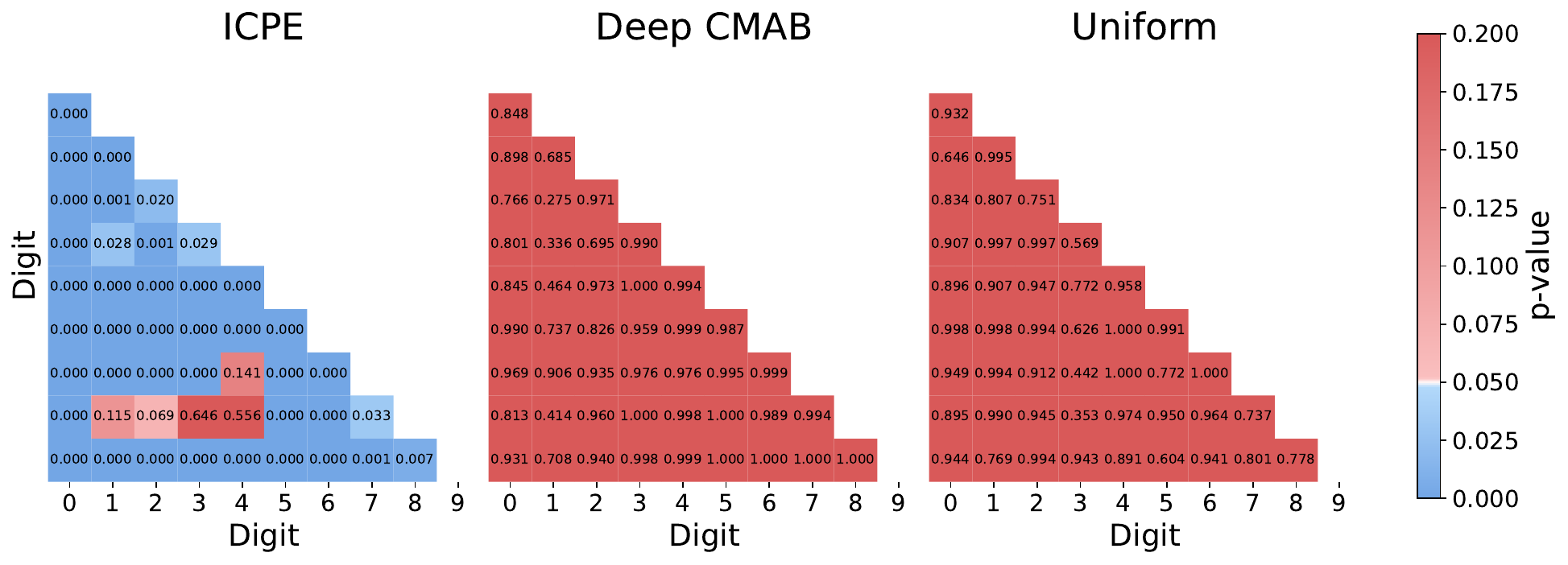}
    \caption{Pairwise chi-squared test p-values for region selection distributions across digit classes. Lower values indicate more statistically distinct exploration behaviors.}
    \label{fig:chisq}
\end{figure}

For further intuition into the sampling process, Figure~\ref{fig:sampling_example} shows a representative example of the \icpe{} pipeline progressively revealing image regions and correctly classifying the digit ‘2’. This highlights the interplay between exploration and inference as the agent strategically uncovers informative regions to guide its decision.

To illustrate the impact of input corruption, Figure~\ref{fig:inc_sampling_example} presents an example where \icpe{} fails to correctly classify the digit. Although the agent successfully reveals the central body of the digit, the applied augmentations distort the image to the extent that the digit becomes visually ambiguous. In this case, the agent incorrectly predicts an ‘8’ when the true label is a ‘9’, underscoring the challenge introduced by realistic image corruptions in this setting.

\begin{figure}[ht]
    \centering
    \includegraphics[width=0.7\linewidth]{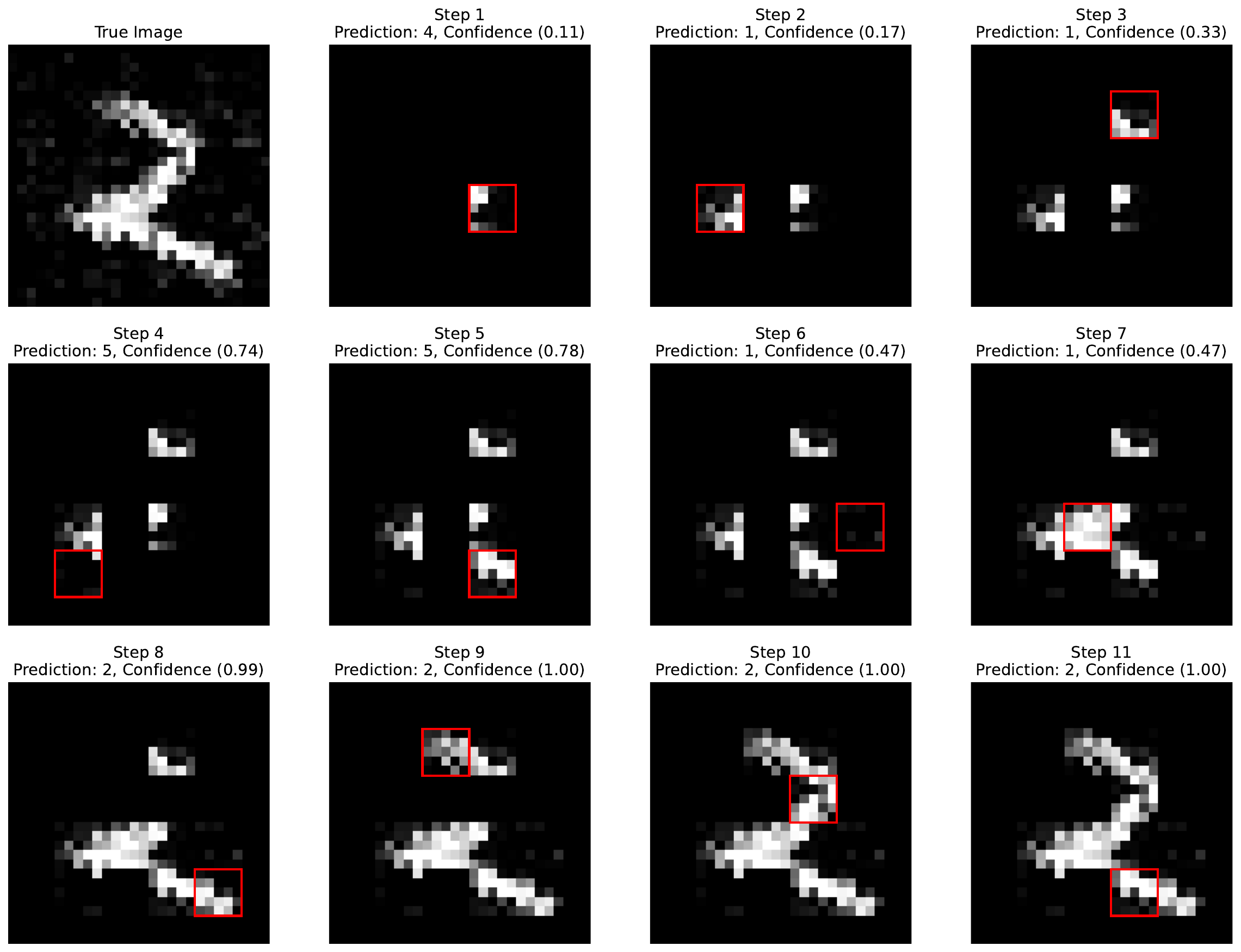}
    \caption{Illustrative example of the \icpe{} agent revealing regions of an MNIST digit and correctly classifying it as a ‘2’. The sequence shows the intermediate revealed image and predicted label at each timestep.}
    \label{fig:sampling_example}
\end{figure}

\begin{figure}[ht]
    \centering
    \includegraphics[width=0.7\linewidth]{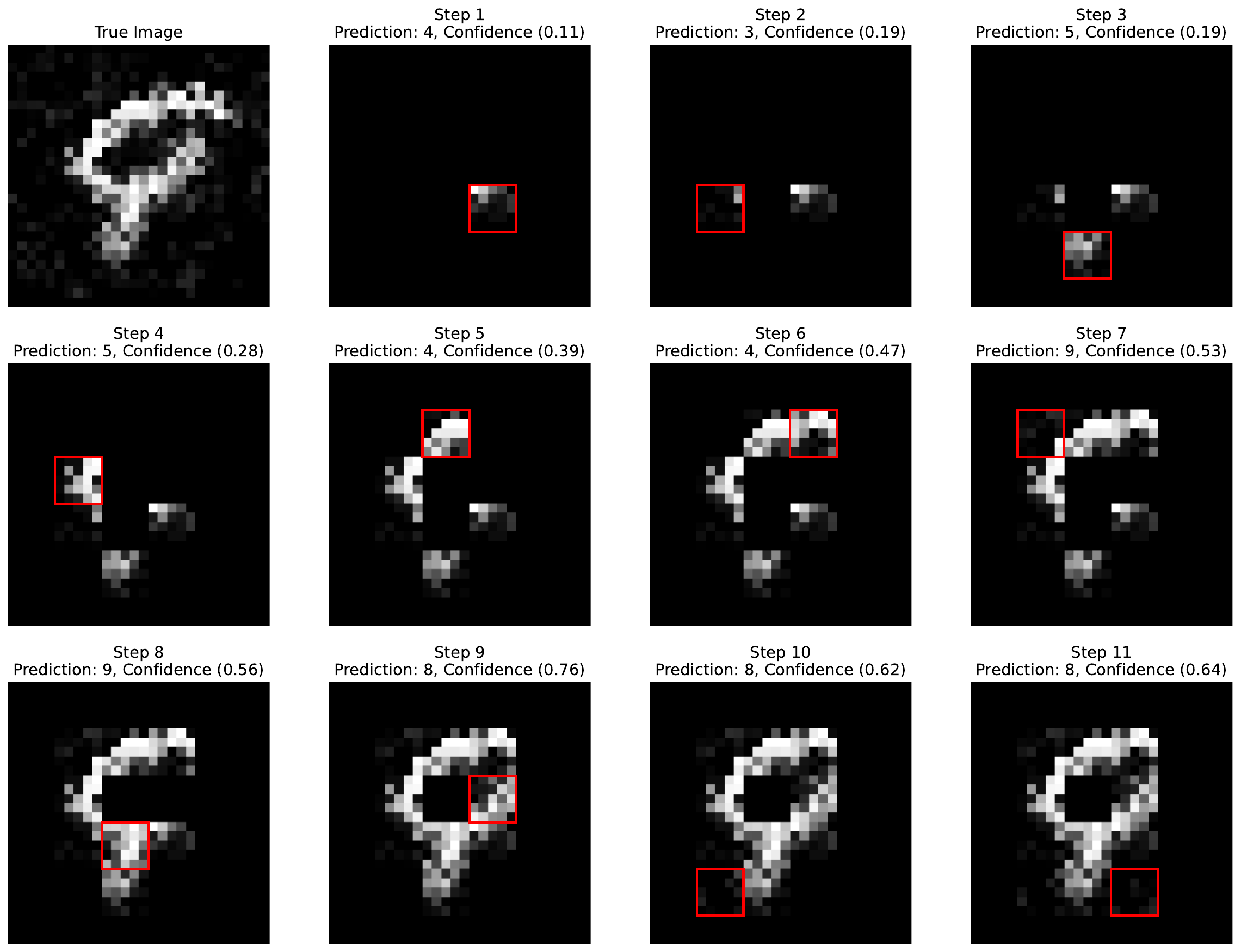}
    \caption{Example of an incorrect classification due to aggressive data augmentations. Although the agent reveals the central region of the digit, the distortions cause it to misclassify a ‘9’ as an ‘8’.}
    \label{fig:inc_sampling_example}
\end{figure}

\subsection{MDP Problems: Magic Room}\label{app:subsec:magic_room}

The Magic Room is a sequential decision-making environment structured as a $K\times K$ grid-shaped room containing four doors, each positioned at the midpoint of one of the four walls (top, bottom, left, right). At the beginning of each episode, exactly one of these doors is randomly chosen to be the correct door ($H^\star$), unknown to the agent. 

The agent's goal is to identify and pass through the correct door. Each episode lasts for a maximum of $N = K^2$ time steps, during which the agent navigates the grid, observes clues, and attempts to determine the correct door. Two binary clues, each randomly assigned a location within the sub-grid $[1,1] \times [K-1,K-1]$, are placed in the room at the start of each episode. Each clue has a binary value, randomly set to either $-1$ or $1$. Collecting both clues provides sufficient information to unambiguously determine the correct door, given that the agent has learned the mapping from clue configurations to door identity.

At each time step $t$, the agent observes the state vector:
\[
    x_t = (z_t, y_t, c_{1,t}, c_{2,t}, r_t),
\]
where:
\begin{itemize}
    \item $(z_t,y_t)$ are the agent's current coordinates on the grid.
    \item $c_{i,t} \in \{-1, 0, 1\}$ indicates the status of clue $i$: it equals $0$ if clue $i$ has not yet been observed by the agent, or it equals either $-1$ or $1$ if the clue has been observed.
    \item $r_t \in \{0,1\}$ represents the reward received at time $t$. Specifically, upon passing through a door:
    \begin{itemize}
        \item If the chosen door is the correct one, the agent receives a reward of $1$ with probability $\frac{1}{4}$, and a reward of $0$ otherwise.
        \item If the chosen door is incorrect, the agent always receives a reward of $0$.
    \end{itemize}
\end{itemize}

An episode terminates when the agent chooses to pass through any of the four doors, irrespective of correctness, or when the horizon $N=K^2$ steps is reached. Upon termination, the agent is required to explicitly select which door it believes to be the correct one.

\begin{table}[h!]
  \centering
\resizebox{\columnwidth}{!}{%
  \begin{tabular}{lcc|cc}
    \toprule
    Method & \multicolumn{2}{c}{Average Correctness} & \multicolumn{2}{c}{Average Stopping Time} \\
    \cmidrule(lr){2-3} \cmidrule(lr){4-5}
           & \(K=6\) & \(K=8\) & \(K=6\) & \(K=8\) \\
    \midrule
    \icpe{}   & \(0.953\,(0.940,\,0.968)\) & \(0.948\,(0.941,\,0.954)\) & \(13.721\,(13.298,\,14.165)\) & \(27.704\,(27.296,\,28.086)\) \\
    \bottomrule
  \end{tabular}
  }
    \caption{Magic Room: correctness and stopping times (mean and 95\% CI) for \(K=6\) and \(K=8\).}
  \label{tab:magic_room}
\end{table}

This setup provides two distinct strategies for the agent:
\begin{enumerate}
    \item \textbf{Luck-based strategy}: The agent directly attempts to pass through a door, observing the reward to determine correctness. A positive reward conclusively indicates the correct door; a zero reward provides no additional information.
    \item \textbf{Inference-based strategy}: The agent efficiently navigates the room, locates both clues to deduce the identity of the correct door, and subsequently exits through that door.
\end{enumerate}

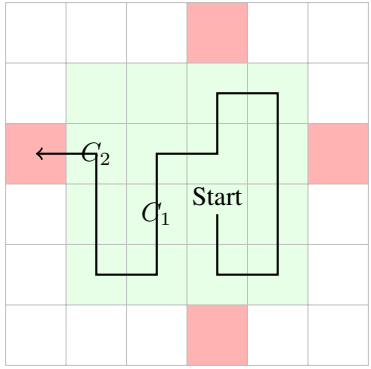
\begin{figure}[t]
    \centering

   \begin{tikzpicture}[scale=0.8]

     \fill[red!30] (0,3) rectangle +(1,1);  
  \fill[red!30] (5,3) rectangle +(1,1);  
    \fill[red!30] (3,0) rectangle +(1,1);  
        \fill[red!30] (3,5) rectangle +(1,1);  

              \fill[green!10] (1,1) rectangle +(4,4);  
  \draw[step=1cm,gray!50,very thin] (0,0) grid (6,6);
  
  \coordinate (P0)  at (3.5,2.5);
  \coordinate (P1)  at (3.5,1.5);
  \coordinate (P2)  at (4.5,1.5);
  \coordinate (P3)  at (4.5,2.5);
  \coordinate (P4)  at (4.5,3.5);
  \coordinate (P5)  at (4.5,4.5);
  \coordinate (P6)  at (3.5,4.5);
  \coordinate (P7)  at (3.5,3.5);
  \coordinate (P8)  at (2.5,3.5);
  \coordinate (P9)  at (2.5,2.5);
  \coordinate (P10) at (2.5,1.5);
  \coordinate (P11) at (1.5,1.5);
  \coordinate (P12) at (1.5,2.5);
  \coordinate (P13) at (1.5,3.5);
\coordinate (P14) at (0.5,3.5);

  \draw[->, thick]
    (P0) -- (P1) -- (P2) -- (P3) -- (P4) -- (P5)
         -- (P6) -- (P7) -- (P8) -- (P9) -- (P10)
         -- (P11)-- (P12)-- (P13) -- (P14);

  \node at (P9)  {$C_1$};
  \node at (P13) {$C_2$};
  
  \node[above] at (P0)  {Start};
\end{tikzpicture}
\caption{Magic room: example of trajectory of the \\icpe{}{} agent.}
\label{fig:magic_room:example_trajectory}
\end{figure}

Thus, optimal behavior requires an effective exploration of the room to finish as quickly as possible.

\begin{figure}[h!]
    \centering
    \includegraphics[width=0.48\linewidth]{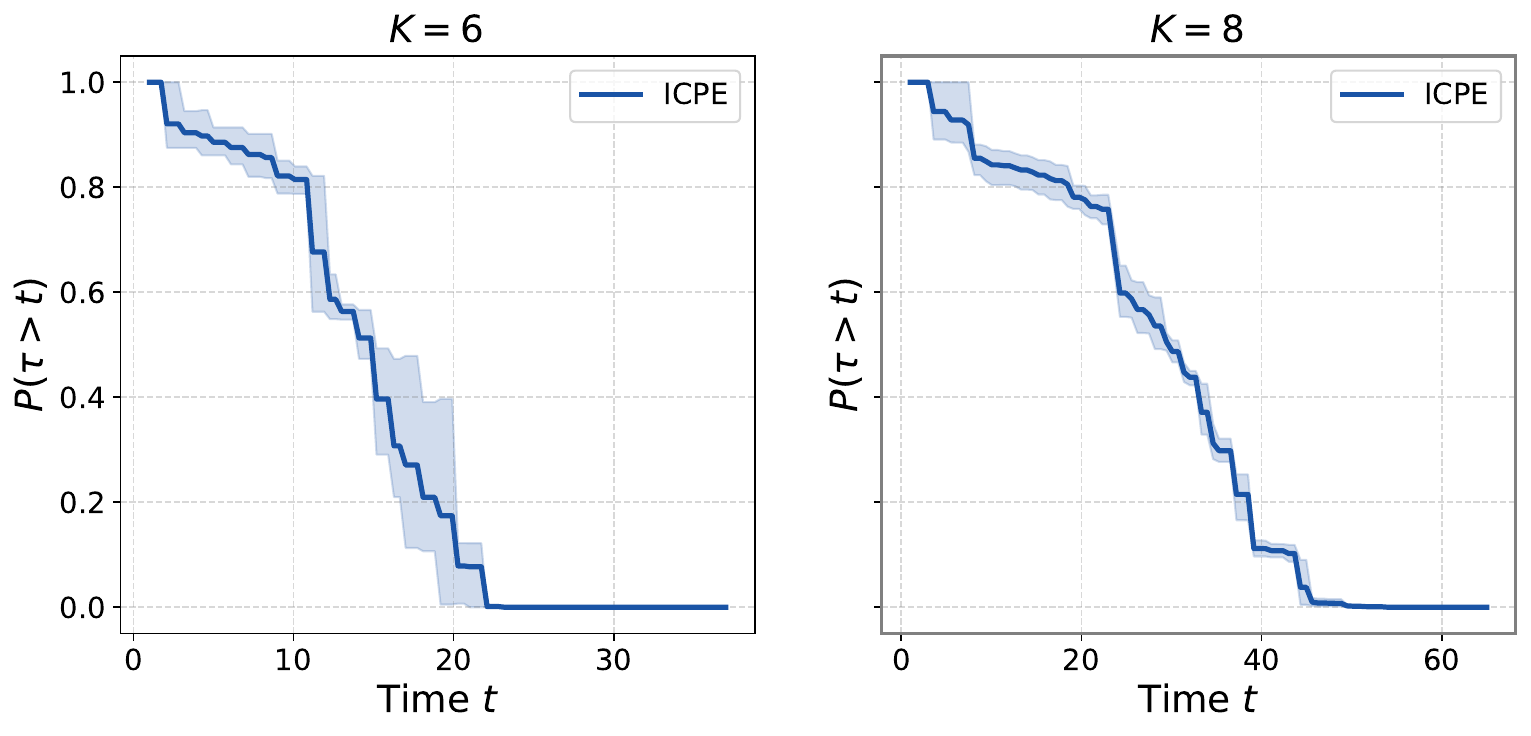}
    \hfill
    \includegraphics[width=0.48\linewidth]{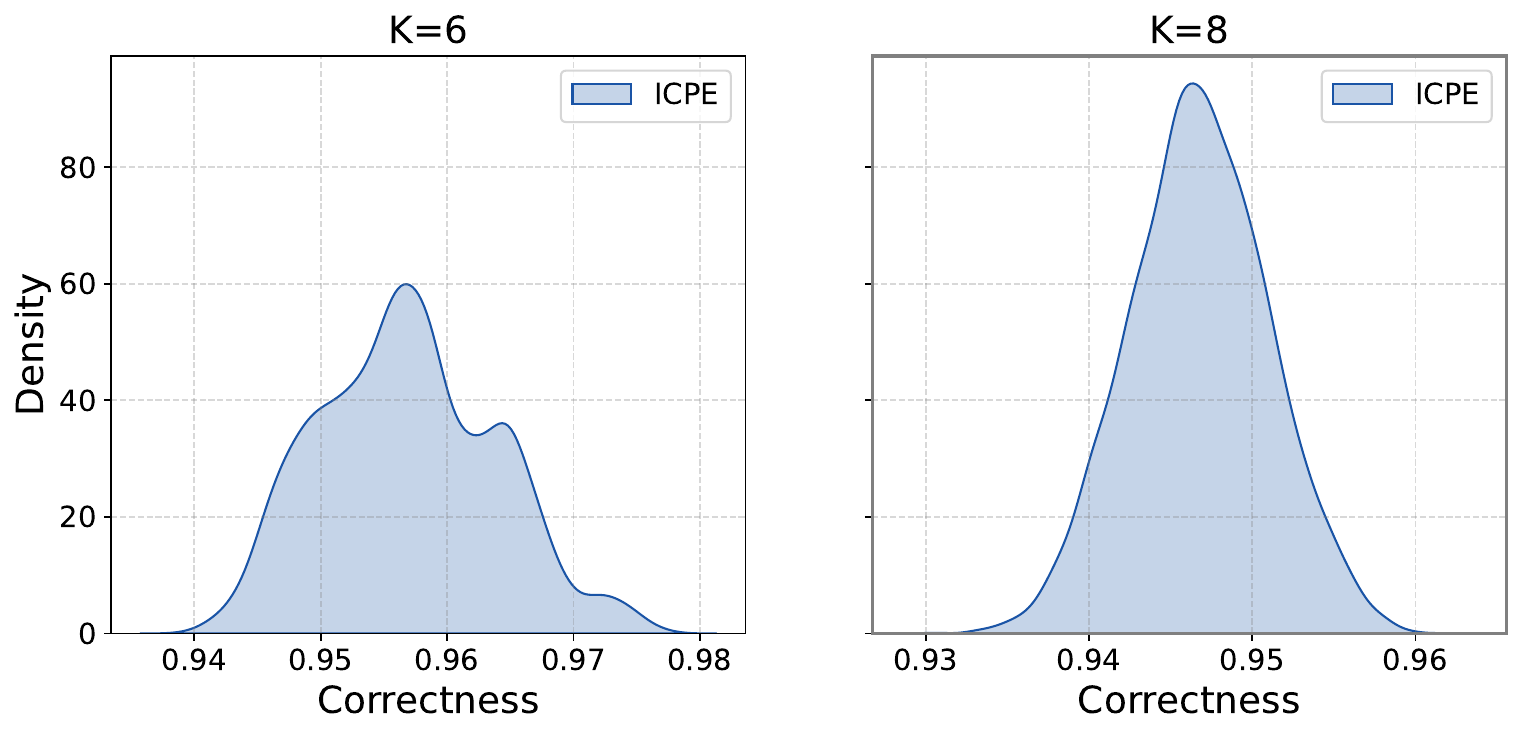}
    \caption{Magic room environment. Left: survival function $\mathbb{P}(\tau >t)$ for $K=6$ and $K=8$. Right: density of the correctness for $K=6$ and $K=8.$}
    \label{fig:magic_room:survival_stopping_time_correctness}
\end{figure}
We trained \icpe{} on $3$ seeds, using the fixed confidence setting (disabling the stopping action) using $\delta=0.05$ and evaluated the policies on $4500$ episodes for $K=6$ and $K=8$. In \cref{tab:magic_room} are shown the statistics of the average correctness and of the stopping time.

In \cref{fig:magic_room:example_trajectory} we can see a sample trajectory taken by \icpe{}. Starting from the middle of the room, \icpe{} follows a path that allows to find the clues $C_1,C_2$ in the green area. As soon as the second clue is found, it goes through the closest door.

In \cref{fig:magic_room:survival_stopping_time_correctness}, we present the survival functions of the stopping time $\tau$ for environments with grid sizes $K=6,8$, alongside the corresponding correctness densities. Lastly, \cref{fig:magic_room:clues} illustrates the relationship among agent correctness, the number of clues observed, and the stopping time. Specifically, smaller stopping times correlate with fewer observed clues, leading to  lower correctness. Conversely, when the agent observes both clues, it consistently selects the correct door, demonstrating that it has effectively learned the association between the clues and the correct hypothesis.

\begin{figure}
    \centering
    \includegraphics[width=0.7\linewidth]{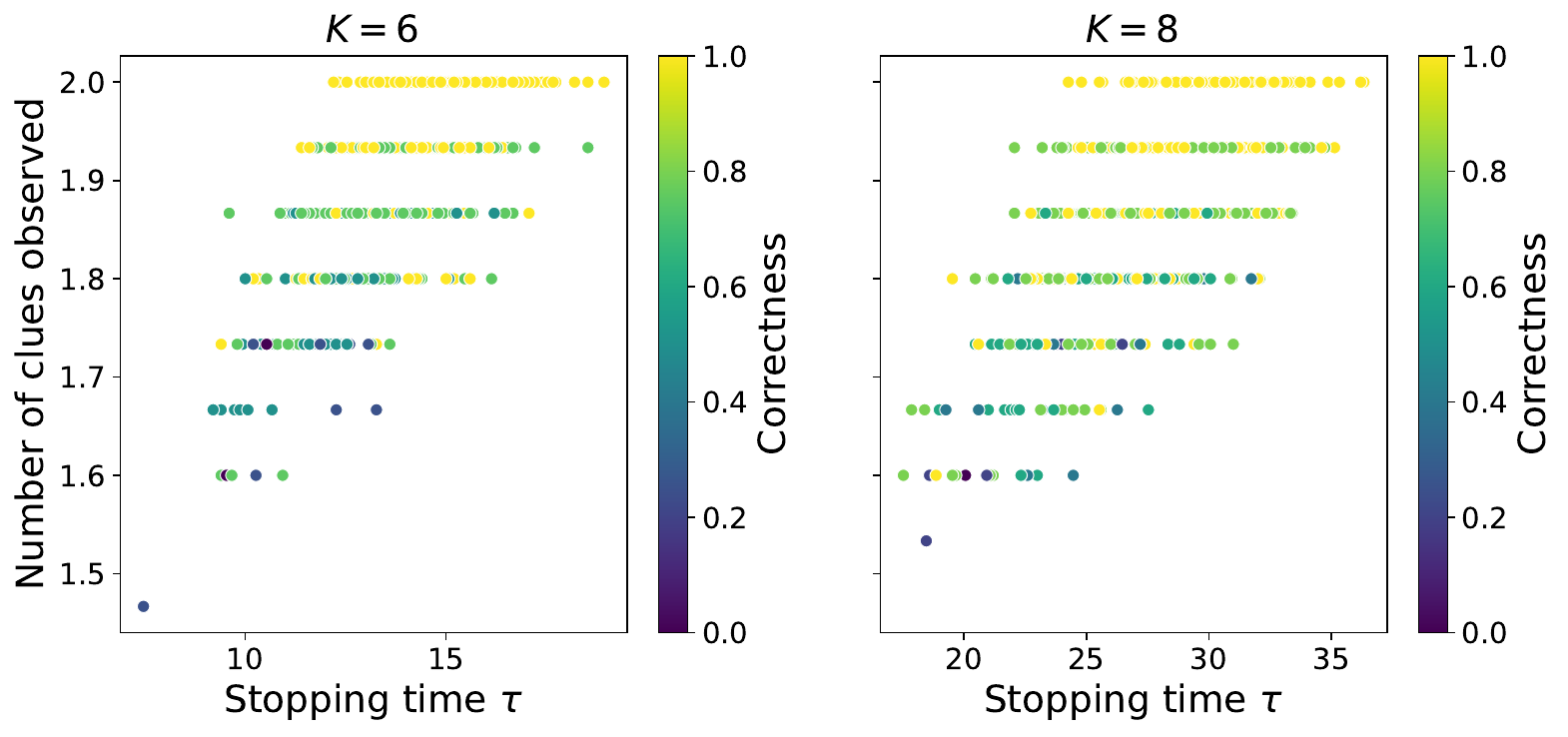}
    \caption{Magic room environment. Relationship among agent correctness, the number of clues observed, and the stopping time.}
    \label{fig:magic_room:clues}
\end{figure}

\subsection{Exploration on Feedback Graphs}\label{app:subsec:feedback_graphs}

In the standard bandits setting we studied in Section~\ref{subsec:bandits}, the learner observes the reward of the selected action, while in full-information settings, all rewards are revealed. Feedback graphs generalize this spectrum by specifying, via a directed graph $G$ which additional rewards are observed when a particular action is chosen. Each node corresponds to an action, and an edge from $u$ to $v$ means that playing $u$ may reveal feedback about $v$.

While feedback graphs have been widely studied for regret minimization~\citep{mannor2011bandits}, their use in pure exploration remains relatively underexplored~\citep{russo2025pure}. We study them here as a challenging and structured testbed for in-context exploration. Unlike unstructured bandits, these environments contain latent relational structure and stochastic feedback dependencies that must be inferred and exploited to explore efficiently.

Formally, we define a feedback graph as an adjacency matrix \( G \in [0,1]^{K \times K} \), where \( G_{u,v} \) denotes the probability that playing action \( u \) reveals the reward of action \( v \). The learner observes a feedback vector \( r \in \mathbb{R}^K \), where each coordinate is revealed independently with probability \( G_{u,v} \):
\[
r_v \sim
\begin{cases}
\mathcal{N}(\mu_v, \sigma^2), & \text{with probability } G_{u,v}, \\
\text{no observation}, & \text{otherwise}.
\end{cases}
\]

This setting allows us to test whether \icpe{} can learn to uncover and leverage latent graph structure to guide exploration. As in the bandits setting, we have a finite number of actions ${\cal A} = \{1, \dotsc, K\}$, corresponding to the actions (or vertices) in a feedback graph $G$. The learner's goal is to identify the best action, where $H^\star = \arg \max_a \mu_a$. At each time step $t$, the observation is the partially observed reward vector $x_t = r_t$. 

We evaluate performance on best-arm identification tasks across three representative feedback graph families:

\begin{itemize}
    \item \textbf{Loopy Star Graph} (Figure~\ref{fig:loopy_star_graph_app}): A star-shaped graph with self-loops, parameterized by $(p,q,r)$. The central node observes itself with probability $q$, one neighboring node with probability $p$, and all others with probability $r$. When $p$ is small, it may be suboptimal to pull the central node, requiring the agent to adapt its strategy accordingly.
    
    \item \textbf{Ring Graph} (Figure~\ref{fig:ring_graph_details}): A cyclic graph where each node observes its right neighbor with probability $p$ and its left neighbor with probability $1-p$. Effective exploration requires reasoning about which neighbors provide more informative feedback.
    
    \item \textbf{Loopless Clique Graph} (Figure~\ref{fig:looplessclique_graph_details}): A fully connected graph with no self-loops. Edge probabilities are defined as:
    \[
    G_{u,v} = 
    \begin{cases}
    0 & \text{if } u = v, \\
    \frac{p}{u} & \text{if } v \neq u \text{ and } v \text{ is odd}, \\
    1 - \frac{p}{u} & \text{otherwise}.
    \end{cases}
    \]
    Here, informativeness varies systematically with action index, requiring the learner to infer which actions are most useful.
\end{itemize}

These environments offer a diverse testbed for evaluating whether \icpe{} can uncover and exploit complex feedback structures without direct access to the underlying graph.

\begin{figure}[t]
    \centering
    \begin{minipage}[t]{0.3\textwidth}
        \centering
        \begin{tikzpicture}[
          scale=0.9, transform shape,
          vertex/.style = {circle, draw, minimum size=0.5cm},
          loop/.style = {looseness=3, in=60, out=120, min distance=7mm},
          every edge/.style = {draw, thick}
        ]
        \node[vertex] (A1) at (0,0) {$\mu_1$};
        \node[vertex] (B1) at (1.5,1.5) {$\mu_5$};
        \node[vertex] (C1) at (1.25,-1.5) {$\mu_4$};
        \node[vertex] (D1) at (-1.25,-1.5) {$\mu_3$};
        \node[vertex] (E1) at (-1.5,1.5) {$\mu_2$};
        \draw[->] (A1) -- node[above] {$p$} (B1);
        \draw[->] (A1) -- node[right] {$r$} (C1);
        \draw[->] (A1) -- node[left] {$r$} (D1);
        \draw[->] (A1) -- node[above] {$r$} (E1);
        \draw[->, loop] (A1) to node[above] {$q$} (A1);
        \draw[->, loop] (B1) to [out=240,in=300] node[below] {$1-p$} (B1);
        \draw[->, loop] (C1) to [out=240,in=300] node[below] {$(1-2p)^+$} (C1);
        \draw[->, loop] (D1) to [out=240,in=300] node[below] {$(1-2p)^+$} (D1);
        \draw[->, loop] (E1) to [out=240,in=300] node[below] {$(1-2p)^+$} (E1);
        \end{tikzpicture}
        \caption{Loopy star graph.}
        \label{fig:loopy_star_graph_app}
    \end{minipage}
    \hfill
    \begin{minipage}[t]{0.3\textwidth}
        \centering
        \begin{tikzpicture}[
          scale=0.9, transform shape,
          vertex/.style = {circle, draw, minimum size=0.5cm},
          every edge/.style = {draw, thick}
        ]
        \node[vertex] (A1) at (0,0) {$\mu_1$};
        \node[vertex] (B1) at (2.25,-2) {$\mu_2$};
        \node[vertex] (C1) at (0,-4) {$\mu_3$};
        \node[vertex] (D1) at (-2.25,-2) {$\mu_4$};
        \draw[->, bend right=20] (A1) to node[left] {$p$} (B1);
        \draw[->, bend right=20] (A1) to node[left] {$1-p$} (D1);
        \draw[->, bend right=20] (B1) to node[right] {$1-p$} (A1);
        \draw[->, bend right=20] (B1) to node[left] {$p$} (C1);
        \draw[->, bend right=20] (C1) to node[right] {$p$} (D1);
        \draw[->, bend right=20] (C1) to node[right] {$1-p$} (B1);
        \draw[->, bend right=20] (D1) to node[right] {$p$} (A1);
        \draw[->, bend right=20] (D1) to node[left] {$1-p$} (C1);
        \end{tikzpicture}
        \caption{Ring graph.}
        \label{fig:ring_graph_details}
    \end{minipage}
    \hfill
    \begin{minipage}[t]{0.3\textwidth}
        \centering
        \begin{tikzpicture}[
          scale=0.9, transform shape,
          vertex/.style = {circle, draw, minimum size=0.5cm},
          every edge/.style = {draw, thick}
        ]
        \node[vertex] (A1) at (0,0) {$\mu_1$};
        \node[vertex] (B1) at (2.25,-2) {$\mu_2$};
        \node[vertex] (C1) at (0,-4) {$\mu_3$};
        \node[vertex] (D1) at (-2.25,-2) {$\mu_4$};
        \draw[->, bend right=20] (A1) to (B1);
        \draw[->, bend right=20] (A1) to (D1);
        \draw[->, bend right=20] (B1) to (A1);
        \draw[->, bend right=20] (B1) to (C1);
        \draw[->, bend right=20] (C1) to (D1);
        \draw[->, bend right=20] (C1) to (B1);
        \draw[->, bend right=20] (D1) to (A1);
        \draw[->, bend right=20] (D1) to (C1);
        \draw[->, bend left=10] (D1) to (B1);
        \draw[->, bend left=10] (B1) to (D1);
        \draw[->, bend left=10] (A1) to (C1);
        \draw[->, bend left=10] (C1) to (A1);
        \end{tikzpicture}
        \caption{Loopless clique graph.}
        \label{fig:looplessclique_graph_details}
    \end{minipage}
\end{figure}

\paragraph{Fixed-Horizon.}
For each graph family, mean rewards were sampled uniformly from $[0,1]$ with fixed variance $0.2$, using hyperparameters: $(p, q, r) = (0.25, 0.3, 0.35)$ for the loopy star graph, $p = 0.3$ for the ring, and $p = 0.5$ for the loopless clique. We considered both small ($K=5$, $H=25$) and large ($K=10$, $H=50$) environments.

\icpe{} was compared to three baselines: Uniform Sampling, EXP3.G~\citep{rouyer2022near}, and Tas-FG~\citep{russo2025pure}. All methods performed maximum likelihood inference at the end of the trajectory. Table~\ref{tab:fg_fixed_horizon} reports the average probability of correctly identifying the best arm.

\begin{table}[ht]
\centering
\renewcommand{\arraystretch}{1.2}
\resizebox{\textwidth}{!}{
\begin{tabular}{|c||cc|cc|cc|}
\hline
\textbf{Algorithm} &
\multicolumn{2}{c|}{\textbf{Loopy Star}} &
\multicolumn{2}{c|}{\textbf{Loopless Clique}} &
\multicolumn{2}{c|}{\textbf{Ring}} \\
\cline{2-7}
& \textbf{Small} & \textbf{Large} & \textbf{Small} & \textbf{Large} & \textbf{Small} & \textbf{Large} \\
\hline \hline
\icpe{} & \textbf{0.88} $\pm$ \textbf{0.01} & 0.59 $\pm$ 0.02 & \textbf{0.95} $\pm$ \textbf{0.01} & 0.79 $\pm$ 0.04 & \textbf{0.79} $\pm$ \textbf{0.01} & 0.51 $\pm$ 0.03 \\
TasFG & 0.82 $\pm$ 0.01 & \textbf{0.73} $\pm$ \textbf{0.02} & 0.84 $\pm$ 0.01 & \textbf{0.83} $\pm$ \textbf{0.01} & 0.70 $\pm$ 0.02 & 0.56 $\pm$ 0.02 \\
EXP3.G & 0.66 $\pm$ 0.02 & 0.40 $\pm$ 0.01 & 0.84 $\pm$ 0.01 & 0.78 $\pm$ 0.02 & 0.77 $\pm$ 0.02 & 0.52 $\pm$ 0.02 \\
Uniform & 0.73 $\pm$ 0.02 & 0.60 $\pm$ 0.02 & 0.86 $\pm$ 0.01 & 0.79 $\pm$ 0.02 & 0.78 $\pm$ 0.02 & \textbf{0.62} $\pm$ \textbf{0.02} \\
\hline
\end{tabular}
}
\vspace{0.1in}
\caption{
Probability of correctly identifying the best arm. Small environments: $K=5$, $H=25$; Large: $K=10$, $H=50$. Results reported as mean ± 95\% CI.
}
\label{tab:fg_fixed_horizon}
\end{table}

\icpe{} outperforms all baselines in small environments across all graph families, highlighting its ability to learn efficient strategies from experience. Performance slightly degrades in larger environments, likely due to difficulty in credit assignment over long horizons. Still, \icpe{} remains competitive, validating its capacity to generalize across graph-structured settings.

\paragraph{Fixed-Confidence.}
We next tested \icpe{} in a fixed-confidence setting, using the same graph families but setting the optimal arm’s mean to $1$ and all others to $0.5$ to facilitate faster convergence. \icpe{} was trained for $K = 4, 6, \dotsc, 14$ with a target error rate of $\delta = 0.1$. We compared it to Uniform Sampling, EXP3.G, and Tas-FG using a shared stopping rule from~\citep{russo2025pure}.

\begin{figure}[ht]
    \centering
    \begin{subfigure}[t]{0.32\textwidth}
        \centering
        \includegraphics[width=\linewidth]{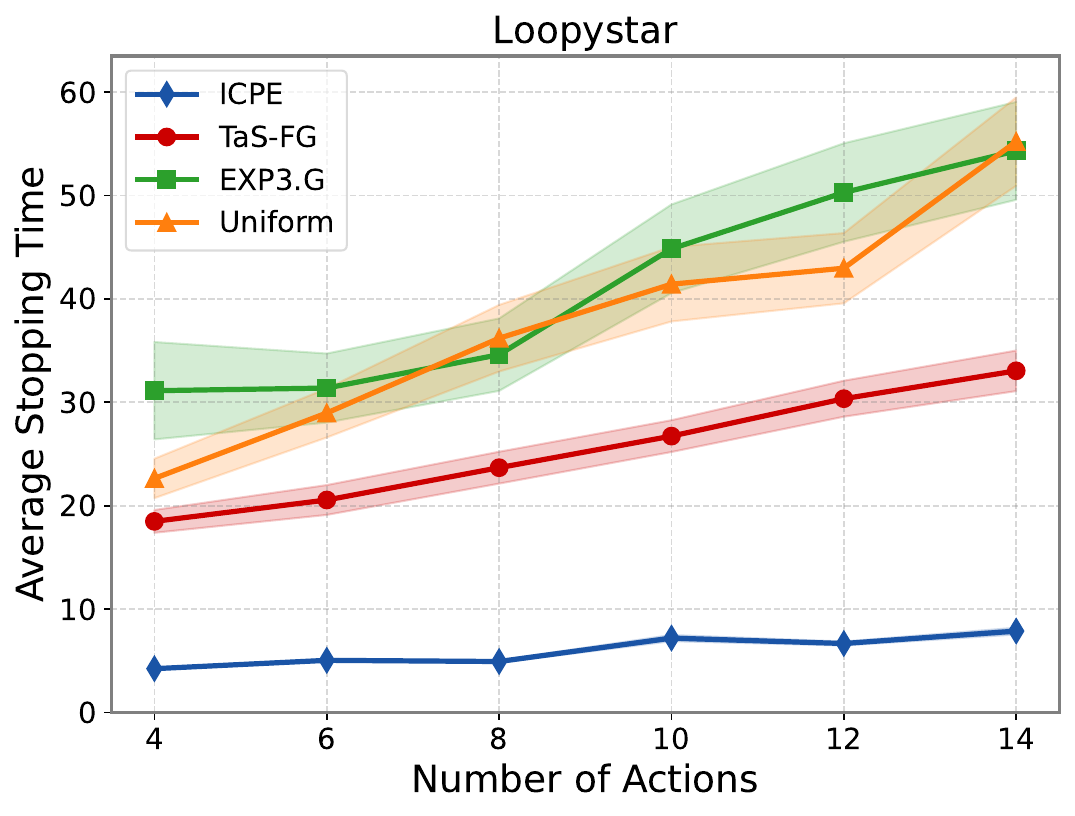}
        \caption{}
        \label{fig:subfig_a}
    \end{subfigure}
    \hfill
    \begin{subfigure}[t]{0.32\textwidth}
        \centering
        \includegraphics[width=\linewidth]{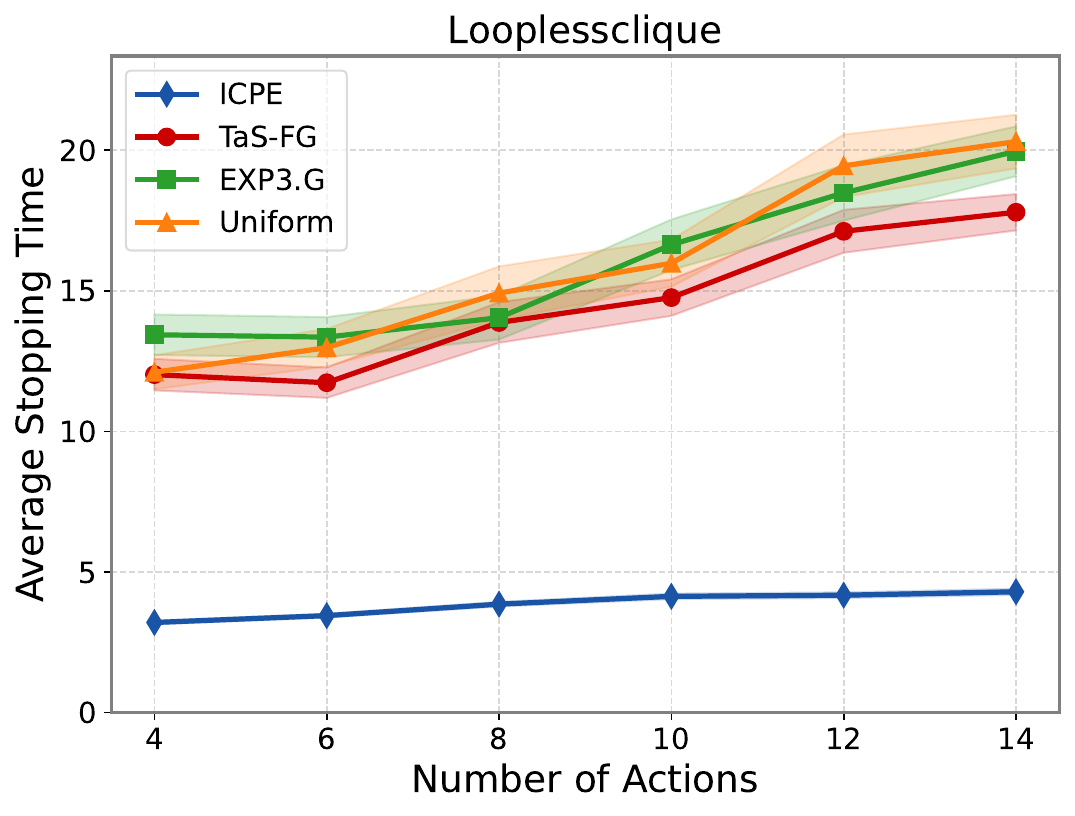}
        \caption{}
        \label{fig:subfig_b}
    \end{subfigure}
    \hfill
    \begin{subfigure}[t]{0.32\textwidth}
        \centering
        \includegraphics[width=\linewidth]{figures/benchmark_comparison_looplessclique.pdf}
        \caption{}
        \label{fig:subfig_c}
    \end{subfigure}
    \caption{Sample complexity comparison under the fixed-confidence setting for: (a) Loopy Star, (b) Loopless Clique, and (c) Ring graphs.}
    \label{fig:fg_fixed_confidence}
\end{figure}

As shown in Figure~\ref{fig:fg_fixed_confidence}, \icpe{} consistently achieves significantly lower sample complexity than all baselines. This suggests that \icpe{} is able to meta-learn the underlying structure of the feedback graphs and leverage this knowledge to explore more efficiently than \emph{uninformed} strategies. These results align with expectations: when environments share latent structure, learning to explore from experience offers a substantial advantage over fixed heuristics that cannot adapt across tasks.

\subsection{Meta-Learning Binary Search}\label{app:subsec:binary_search}

To test \icpe{}’s ability to recover classical exploration algorithms, we evaluate whether it can autonomously meta-learn binary search.

We define an action space of ${\cal A} = \{1, \dotsc, K\}$, where $K$ is the upper bound on the possible location of the hidden target $H^\star \sim \mathcal{A}$. Pulling an arm above or below $H^\star$ yields a observation $x_t = -1$ or $x_t = +1$, respectively—providing directional feedback.

We train \icpe{} under the fixed-confidence setting for $K = 2^3, \dotsc, 2^8$, using $150{,}000$ in-context episodes and a target error rate of $\delta = 0.01$. Evaluation was conducted on $100$ held-out tasks per setting. We report the minimum accuracy, mean stopping time, and worst-case stopping time, and compare against the theoretical binary search bound $O\left(\log_2K\right)$.

\begin{table}[ht]
\centering
\renewcommand{\arraystretch}{1.2}
\begin{adjustbox}{width=\textwidth}
\begin{tabular}{r|c|c|c|c}
\toprule
\textbf{Number of Actions ($K$)} & \textbf{Minimum Accuracy} & \textbf{Mean Stopping Time} & \textbf{Max Stopping Time} & $\log_2 K$\\
\midrule
8   & 1.00 & $2.13 \pm 0.12$ & 3 & 3 \\
16  & 1.00 & $2.93 \pm 0.12$ & 4 & 4 \\
32  & 1.00 & $3.71 \pm 0.15$ & 5 & 5 \\
64  & 1.00 & $4.50 \pm 0.21$ & 6 & 6 \\
128 & 1.00 & $5.49 \pm 0.23$ & 7 & 7 \\
256 & 1.00 & $6.61 \pm 0.26$ & 8 & 8 \\
\bottomrule
\end{tabular}
\end{adjustbox}
\caption{\icpe{} performance on the binary search task as the number of actions $K$ increases.}
\label{tab:binary_search}
\end{table}

As shown in Table~\ref{tab:binary_search}, \icpe{} consistently achieves perfect accuracy with worst-case stopping times that match the optimal $\log_2(K)$ rate, demonstrating that it has successfully rediscovered binary search purely from experience. While simple, this task illustrates \icpe{}’s broader potential to learn efficient search strategies in domains where no hand-designed algorithm is available.

\end{document}